\documentclass[
10pt, 
english, 
nohyperref, 
headsepline, 
]{MastersDoctoralThesis} 

\usepackage{pdfpages} 

\usepackage{mathpazo} 
\usepackage[backend=bibtex,style=authoryear,natbib=true]{biblatex} 
\usepackage[autostyle=true]{csquotes} 
\usepackage[a-1b]{pdfx}   

\usepackage{amsmath,amsthm,amssymb, enumerate}
\usepackage{booktabs}       
\usepackage{amsfonts}       
\usepackage{nicefrac}       
\usepackage{amsfonts}  
\usepackage{graphicx}
\usepackage{algorithm}
\usepackage{algpseudocode}
\usepackage{slashbox,placeins}
\usepackage{graphicx, float}
\usepackage{subcaption}
\usepackage{url}
\usepackage{epsfig}
\usepackage{color}
\usepackage{refcount}
\usepackage{verbatim}
\usepackage{enumerate}
\usepackage{multirow}
\usepackage{framed}
\usepackage{spverbatim}
\usepackage[cal=rsfso,calscaled=.96]{mathalfa}
\usepackage{authblk}
\usepackage[toc,page]{appendix}

\usepackage{graphicx}
\usepackage{mathtools}
\newcommand*{\bigO}{\scalebox{2}{\ensuremath{\mathcal{O}}}}

\usepackage{color}
\definecolor{clemson-orange}{RGB}{234,106,32}
\definecolor{chicago-maroon}{RGB}{128,0,0}
\definecolor{cincinnati-red}{RGB}{190,0,0}
\definecolor{soft-cyan}{RGB}{68,85,90}
\definecolor{firebrick}{RGB}{178,34,34}
\definecolor{crimson}{RGB}{220,20,60}
\definecolor{jaam}{rgb}{0.45,0.0,0.45}

\usepackage{hyperref}
\hypersetup{
  colorlinks   = true,    
  urlcolor     = blue,    
  linkcolor    = violet,    
  citecolor    = purple      
}

\usepackage{parskip} 
\usepackage{thmtools}
\declaretheoremstyle[spaceabove=10pt,spacebelow=10pt]{mystyle}
\theoremstyle{mystyle}
\newtheorem{theorem}{Theorem}[section]
\newtheorem{lemma}[theorem]{Lemma}
\newtheorem{corollary}[theorem]{Corollary}
\newtheorem{prop}[theorem]{Proposition}
\newtheorem{claim}{Claim}
\newtheorem{definition}{Definition}
\newtheorem*{remark}{Remark}

\newtheorem{example}[theorem]{Example}

\newif\ifsolutions \solutionstrue

\newcommand*{\cventry}[7][.25em]{}

\AtBeginEnvironment{subappendices}{%
\counterwithin{figure}{section}
\counterwithin{table}{section}
}

\DeclareMathOperator*{\Tr}{Tr}

\DeclareMathOperator*{\argmin}{argmin}

\DeclareMathOperator*{\verts}{vert}
\DeclareMathOperator*{\relu}{{\textrm{ReLU}}}

\newcommand{\pwl}{\textrm{PWL}}
\newcommand{\poly}{\textrm{poly}}
\newcommand{\comp}{\textrm{comp}}

\newcommand{\abs}[2][]{#1\lvert#2 #1\rvert}
\newcommand{\norm}[2][]{#1\lVert #2 #1\rVert}

\def\ve#1{\mathchoice{\mbox{\boldmath$\displaystyle\bf#1$}}
{\mbox{\boldmath$\textstyle\bf#1$}}
{\mbox{\boldmath$\scriptstyle\bf#1$}}
{\mbox{\boldmath$\scriptscriptstyle\bf#1$}}}

\newcommand{\x}{{\ve x}}
\newcommand{\y}{{\ve y}}
\newcommand{\z}{{\ve z}}
\renewcommand{\v}{{\ve v}}
\newcommand{\g}{{\ve g}}
\newcommand{\e}{{\ve e}}
\renewcommand{\u}{{\ve u}}
\renewcommand{\a}{{\ve a}}
\renewcommand{\c}{{\ve c}}

\newcommand{\0}{{\ve 0}}
\newcommand{\m}{{\ve m}}

\newcommand{\w}{{\ve w}}
\renewcommand{\b}{{\ve b}}

\newcommand{\rr}{{\ve r}}

\newcommand{\A}{\textrm{A}}
\newcommand{\B}{\textrm{B}}
\newcommand{\W}{\textrm{W}}

\newcommand{\s}{\textrm{s}}
\newcommand{\E}{\mathbb{E}}

\newcommand{\F}{\mathcal{F}}
\newcommand{\bb}{\mathbb}
\newcommand{\R}{\bb R}
\newcommand{\Z}{{\bb Z}}
\newcommand{\N}{{\bb N}}

\renewcommand{\P}{\mathcal{P}}

\AtBeginDocument{\AtBeginShipoutNext{\AtBeginShipoutDiscard}}

\begin{document}
\frontmatter 



\pagestyle{headings}


\renewcommand{\thechapter}{\arabic{chapter}}
\renewcommand{\mdtChapapp}{}{}
\renewcommand{\chapteralign}{\raggedleft}
\renewcommand{\chapterfont}{\color{black}\huge}
\renewcommand{\chapterinbetweenskip}{}{\vspace*{0pt}}
\renewcommand{\chapterprefixfont}{\fontsize{56pt}{0pt}\selectfont\color{mdtRed}}

\begin{titlepage}

\begin{center}


{\huge \bfseries A STUDY \\ OF \\ THE MATHEMATICS OF DEEP LEARNING   \par}\vspace{0.4cm} 
\HRule \\[4cm] 
 
 
\begin{minipage}[t]{0.4\textwidth}
\begin{flushleft} \large
\emph{Author: }\\
\href{https://sites.google.com/view/anirbit/home}{Anirbit Mukherjee} 

\end{flushleft}
\end{minipage}
\begin{minipage}[t]{0.4\textwidth}
\begin{flushright} \large
\emph{Supervisor:} \\
\href{http://www.ams.jhu.edu/~abasu9/}{Amitabh Basu} 
\end{flushright}
\end{minipage}\\[3cm]
 
\vfill

\large \textit{A dissertation submitted to Johns Hopkins University in conformity with the requirements for the degree of Doctor of Philosophy.}\\[0.3cm] 
\textit{Baltimore, Maryland}\\[0.5cm]
July, $2020$ \\[1cm] 
 
\vfill

{ \copyright $\text{ }2020$ Anirbit Mukherjee}\\[1cm] 
All Rights Reserved

\vfill
\end{center}
\end{titlepage}


\vspace*{0.30\textheight}

\begin{center}
\itshape {\bf \color{jaam} \fontsize{25}{20} This thesis is dedicated to my mother Dr. Suranjana Sur (``Mame").} 
~\\ ~\\ ~\\
{\color{violet} Long before I had started learning arithmetic in school, my mother got me a geometry box and taught me how to construct various angles using a compass. Years before I had started formally studying any of the natural sciences, she created artificial clouds inside the home and taught me about condensation and designed experiments with plants to teach me how they do respiration. This thesis began with the scientific knowledge that my mother imparted regularly to me from right since I was a kid.}
\end{center}





\chapter{Thesis Abstract}

"Deep Learning"/"Deep Neural Nets" is a technological marvel that is now increasingly deployed at the cutting-edge of artificial intelligence tasks. This ongoing revolution can be said to have been ignited by the iconic 2012 paper from the University of Toronto titled ``ImageNet Classification with Deep Convolutional Neural Networks'' by Alex Krizhevsky, Ilya Sutskever and Geoffrey E. Hinton. This paper showed that deep nets can be used to classify images into meaningful categories with almost human-like accuracies! As of $2020$ this approach continues to produce unprecedented performance for an ever widening variety of novel purposes ranging from playing chess to self-driving cars to  experimental astrophysics and high-energy physics. But this new found astonishing success of deep neural nets in the last few years has been hinged on an enormous amount of heuristics and it has turned out to be extremely challenging to be mathematically rigorously explainable. In this thesis we take several steps towards building strong theoretical foundations for these new paradigms of deep-learning. 

Our proofs here can be broadly grouped into three categories,

\begin{itemize}
    \item {\bf Understanding Neural Function Spaces} We show new circuit complexity theorems for deep neural functions over real and Boolean inputs and prove classification theorems about these function spaces which in turn lead to exact algorithms for empirical risk minimization for depth $2$ $\relu$ nets. 
    
    We also motivate a measure of complexity of neural functions and leverage techniques from polytope geometry to constructively establish the existence of high-complexity neural functions. 
    
    \item {\bf Understanding Deep Learning Algorithms} 
    We give fast iterative stochastic algorithms which can learn near optimal approximations of the true parameters of a $\relu$ gate in the realizable setting. (There are improved versions of this result available in our papers \cite{mukherjee2020guarantees,mukherjee2020study} which are not included in the thesis.) 
    
    We also establish the first ever (a) mathematical control on the behaviour of noisy gradient descent on a $\relu$ gate and (b) proofs of convergence of stochastic and deterministic versions of the widely used adaptive gradient deep-learning algorithms, RMSProp and ADAM. This study also includes a first-of-its-kind detailed empirical study of the hyper-parameter values and neural net architectures when these modern algorithms have a significant advantage over classical acceleration based methods.
    
    \item {\bf Understanding The Risk Of (Stochastic) Neural Nets} We push forward the emergent technology of PAC-Bayesian bounds for  the risk of stochastic neural nets to get bounds which are not only empirically smaller than contemporary theories but also demonstrate smaller rates of growth w.r.t increase in width and depth of the net in experimental tests. These critically depend on our novel theorems proving noise resilience of nets. 
    
    This work also includes an experimental investigation of the geometric properties of the path in weight space that is traced out by the net during the training. This leads us to uncover certain seemingly uniform and surprising geometric properties of this process which can potentially be leveraged into better bounds in future.
\end{itemize}

\newpage 
\thispagestyle{plain} 
\vspace*{0.40\textheight}

\begin{center}
\enquote{\itshape \color{violet} \fontsize{20}{15}Study hard what interests you the most in the most undisciplined, irreverent, and original manner possible.}
\end{center}

\bigbreak

\begin{center}
{\bf - Richard Feynman}
\end{center}

\tableofcontents 

\clearpage

\thispagestyle{plain} 
\vspace*{0.35\textheight}

\begin{center}
{\bf \color{violet} \fontsize{20}{15}Special Thanks!}
\end{center}

\bigbreak

\begin{center}
{\rm A lot of debt is owed to the following people who have at various points given me critical help with setting up the LaTeX style files and subsequent editing, my sister Shubhalaxmi Mukherjee (IISER, Pune), Zachary Lubberts (J.H.U) and Monosij Mondal (UPenn). }
\end{center}

\listoftables 

\listoffigures 


\chapter{\vspace{10pt} An Informal Introduction to Deep Learning}\label{chapinf}


We who speak Bengali owe an infinite debt to the legendary Satyajit Ray and it goes well beyond him having given us the timeless movies that he created. Growing up in a typical Bengali household full of books (on almost every conceivable subject!), maybe for many of us our first idea of ``artifical intelligence" can be traced to the friendly humanoids like Robu and Bidhushekhar which were created by Professor Shonku in the famous series of stories penned by Satyajit Ray. In retrospect it was indeed lucky that Professor Shonku happened in our lives much before we encountered the more ominous view of robots as was made famous by Isaac Asimov's ``Three Laws of Robotics". Ofcourse even today in $2020$ we are still nowhere close to what was imagined in Satyajit Ray's fiction but something has dramatically changed in the last $5{-}6$ years. In this chapter we will try to get a feel of this ongoing revolution while keeping the technical aspects low enough to be accessible within the scope of high-school science. We should note at the very outset that opinions remains widely divided about how we should perceive this recent upsurge and much of what we present here is obviously heavily coloured by the technical parts of the thesis that will follow this chapter. 

Maybe many of the readers have probably heard of the recent spectacular successes of ``machines" called AlphaZero in being able to play games like chess at unprecedented levels of proficiency. These successes have revealed structures and possible strategies about the game of chess which had never been seen before! But these forms of artificial intelligence (unfortunately!) do not look like Professor Shonku's robots. Turns out that anthropomorphism isnt of any particular advantage if we limit our notions of intelligence to such abilities as required to play difficult strategy games like chess or poker or being able to create new paintings which mimic the style of Vincent van Gogh or being able to fluently translate between multiple languages. In the last couple of years suddenly these have become possible to do in an automated way because of our new found ability to computationally leverage the power of what are called ``Deep Neural Networks" or DNNs or ``neural nets" or ``neural networks" or sometimes just ``nets". The myriad of ways in which we can ``train" a DNN to perform human-like tasks are collectively called ``Deep Learning" 

It can be somewhat tedious to install on one's home computer the (freely available) softwares like TensorFlow or PyTorch and get a hands-on feel for the advanced applications of neural nets that were mentioned above. The developers of these softwares continue to make progress to make the installation processes increasingly easy so that more people can put this evolving technology to use. Such efforts have led to the creation of platforms like ``Google Colab" where one can write codes to run small neural nets without having to install the full softwares. For immediate motivation let's see this incredibly beautiful (and mind-bogglingly surprising!) demonstration that is easily available on this website, \url{https://thispersondoesnotexist.com/}. Every time we refresh this page we will be shown a seemingly human photograph (which sometimes might have minor defects), just that this photograph is completely artificially generated by a neural network! In a sense this person is purely the net's imagination and he/she does not actually exist! So how did the net manage to ``draw" such realistic human faces? This mechanism is still highly ill-understood and our best efforts at making sense of this involves the branch of mathematics called ``Optimal Transport". This is the same field of research for which Cedric Villani got the Fields Medal in $2010$. This esoteric mathematical idea of optimal transport has mysterious ramifications in the world of neural nets and we have possibly only barely scratched the surface of this interface. 

Though applications like the one described above about artificial generation of human-like faces are the cutting-edge of applied research in neural nets, these are not the commonly used tests for theory. There are more standardized artificial intelligence tasks on which we have decades of benchmarks of performance and new techniques are often compared on those. One such task is of classifying images into meaningful categories when the neural net (or in general any candidate ``machine") is input a high-dimensional vector representing the image. For comparison recall that its at about $9$ months of age that a human baby first starts being able to match daily life objects to their photographs. But the nets are no match for babies! Babies can recognize a banana the next time even after having seen just a single banana once. Unfortunately our best nets still need to see a lot of bananas before learning to categorize it correctly when shown a new one! This human-machine gap is deeply mysterious and an emerging direction of research. 

There are two common datasets of images which are used for this test namely the ``CIFAR"  (Canadian Institute For Advanced Research) database and the MNIST (``Modified National Institute of Standards and Technology") database.  CIFAR dataset was created in $2009$ by Alex Krizhevsky, Vinod Nair, and Geoffrey Hinton. It contains millions of low resolution images grouped into thousands of categories like birds, aeroplanes, cars etc. The task of the trained machine is to correctly predict the category when a randomly picked image from this set is input to the machine. In the figure below we have shown a sample of the MNIST dataset which contains images of hand-written digits from $0$ to $9$ and the task of the trained machine is to recognize the number correctly when shown a randomly picked handwritten digit from the set. This was introduced and explored in the seminal paper from $1998$ called, ``Gradient-Based Learning Applied to Document Recognition" written by some of the biggest stalwarts in the field, Y. LeCun, L. Bottou, Y. Bengio and P. Haffner. And even today we continue to use MNIST as a baseline for testing theory about classification tasks.

\begin{figure}[h] 
\centering
\includegraphics[scale=0.30]{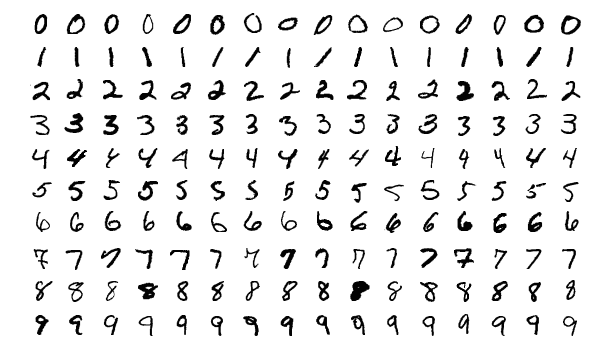}
\caption*{A small part of the famous MNIST database}
\end{figure}

Its worth pointing out that CIFAR is widely considered to be a much more difficult test than MNIST. There are fundamental questions about being able to mathematically justify this difference in difficulty and theory of this kind is still not fully developed. This brings us to a deep mystery that we hardly understand as to when is an artificial intelligence task easy and when is it difficult! Research is only beginning in this direction. 

Now that we have seen some cutting-edge applications and methods of testing artificial intelligence let us focus on understanding the specific implementation of DNNs that we are interested in. 

 DNNs have existed in some form or the other since the $1958$ work by a psychologist at the Cornell University named Frank Rosenblatt, who then called his idea ``Perceptron". Many might say that our current way of thinking about neural nets comes from the famous $1986$ paper titled {\tt ``Learning representations by back-propagating errors''} by David Rumelhart, Ronald Williams and Geoffrey Hinton. Geoffrey E. Hinton is often credited to be the pioneer of deep learning and interestingly Hinton too did his undergraduate studies in psychology! Its worth noting that despite the ideas having been there since decades, till very recently we could never actually get nets to do anything surprising in practice. This so-called ``A.I winter" was finally ended in part by the recent dramatic developments in computer hardware. The ongoing artificial intelligence revolution can be said to have been ignited by the iconic $2012$ paper from University of Toronto titled {\tt ``ImageNet Classification with Deep Convolutional Neural Networks''} by Alex Krizhevsky, Ilya Sutskever and Geoffrey E. Hinton. This showed that deep nets can be used to classify images into meaningful categories with almost human-like accuracies! Now lets try to understand what is the precise mathematical description of a DNN! 
 
 DNNs are what could be called ``mathematical circuits". These can be thought of as a certain peculiar class of functions which are defined via diagrams which look like circuits. In school we get very familiar with the physics of electrical circuits - which carry electrical potential. Mathematical circuits are similar but instead their imagined wires carry algebraic instructions to multiply or add numbers to the input to the wire. We know that electrical circuits can be augmented to do more useful things by embedding inside them ``non-linear" components like resistors, capacitors and inductors. There are called non-linear because the voltage drop across them is not a linear function of the current passing through them. Similarly these DNNs have embedded inside them gates each of which is designed to implement a certain ``activation function" which is typically a non-linear function mapping the real line to itself. The following diagram represents an example of such a single gate and is thus one of the most elementary possible examples of a neural net. 

\begin{figure}[htbp!]
\centering
\includegraphics[scale=0.20]{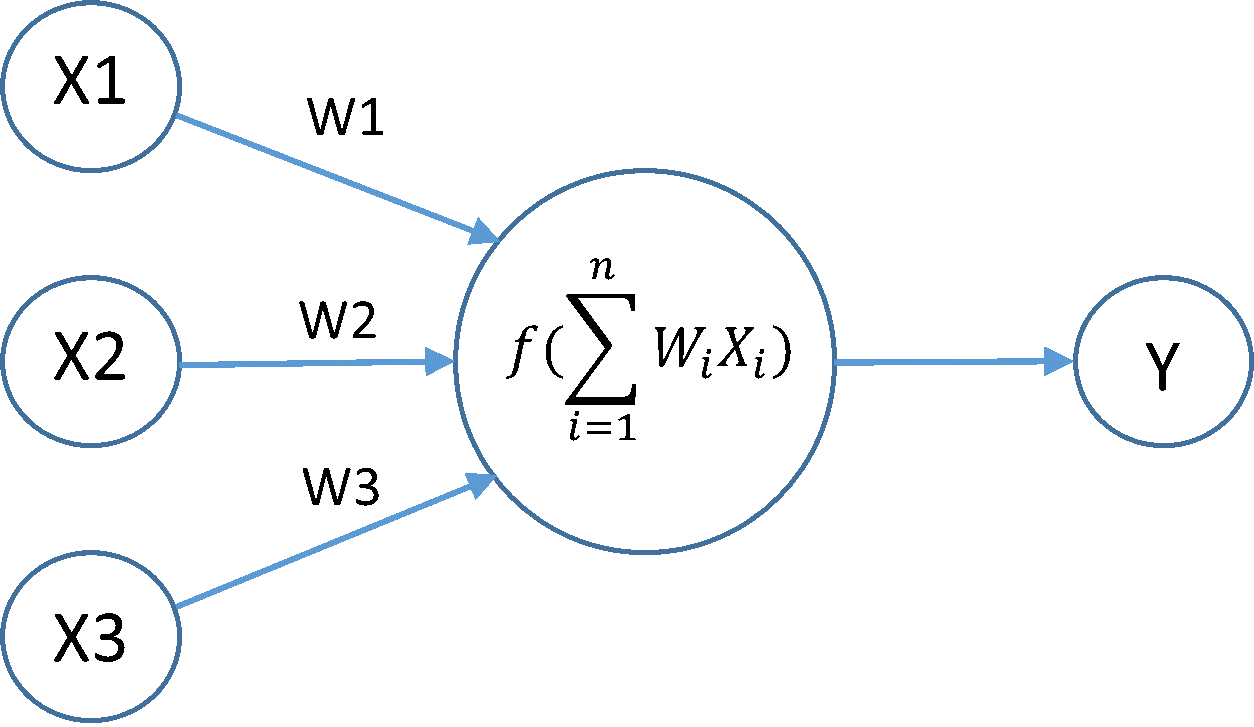}
\end{figure}

The above neural gate will be said to use $f : \R \rightarrow \R$ as the activation function, to create a map/function which takes as input any $3-$dimensional vector $(x_1,x_2,x_3)$ and gives as output the real number, $f(w_1x_1+w_2x_2+w_3x_3)$. We think of the activation function at the gate  $f$, to be getting as input the linear sum, $\sum_{i=1}^3 w_i x_i$. These three real parameters above, $w_1,w_2$ and $w_3$ are called the `weights'. Usually there are many wires coming out of the gate  (instead of the single $Y$ in the above) and in that case the gate is defined to pass on the same value to all of them.

As of today almost all implementations of DNNs use the ``Rectified Linear Unit (ReLU)''

\begin{align*}
\text{ReLU} : \R &\rightarrow \R \\
x &\mapsto \max \{0,x\}
\end{align*}

To develop more intuition lets use the building blocks above to construct a net with a few more gates  which actually computes a familiar useful function. 

\begin{figure}[h]
\centering
\includegraphics[scale=0.68]{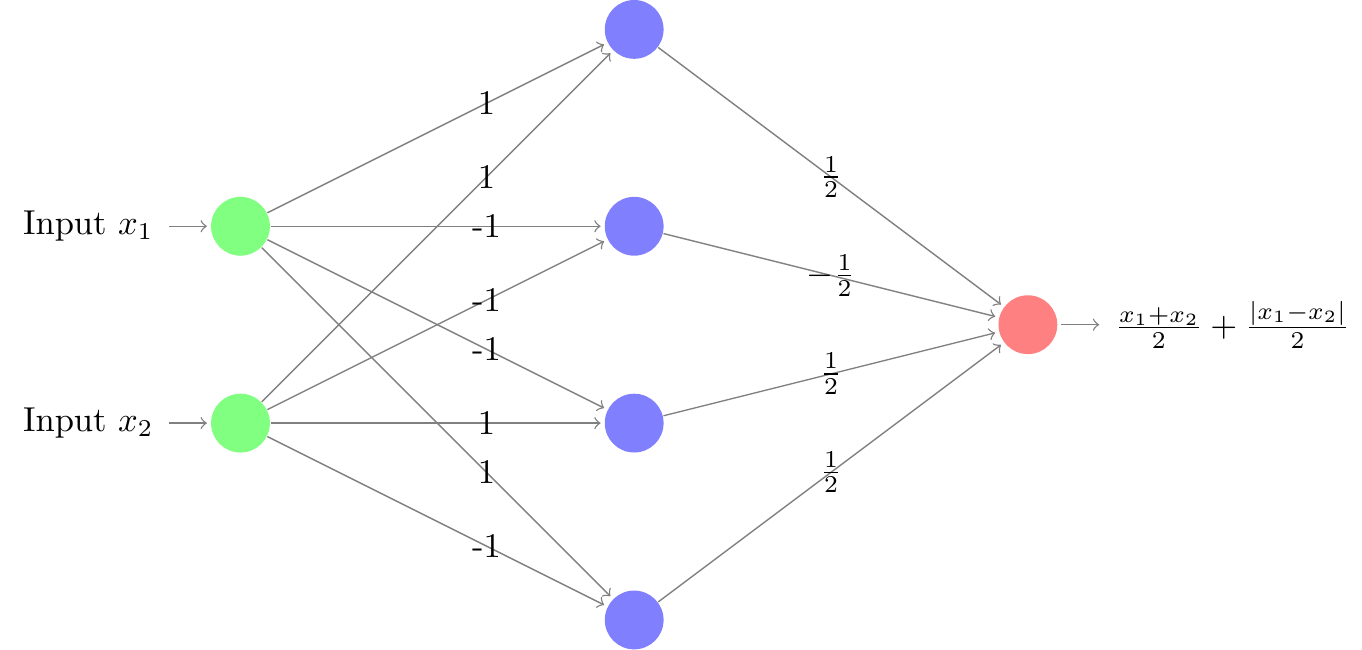}
\end{figure}

In the above we see an example of a ``1-DNN'' i.e. a DNN with one layer of gates indicated in blue. The above neural network would be said to be of size $4$ since it has $4$ activation gates. Lets assume that the activation function at these blue gates is the ReLU function defined above. Then we claim that the above circuit is computing the $\R^2 \rightarrow \R$ function given as $(x_1,x_2) \mapsto \max \{x_1,x_2\}$.  But how do we convince ourselves that this is indeed what is happening? We can start with realizing that the top most blue activation gate is getting as input the number $x_1 + x_2$ and this can be inferred from the weights on its two incoming edges. (recall how the single gate was defined to operate in the diagram on the previous page) Further one can read off from the weight on the outgoing edge of this top most ReLU gate that it is passing on to the red output gate the number $\frac{1}{2}\max \{0, x_1 + x_2 \}$. Thus if we carefully follow the computations happening on each edge and gate then we can conclude that the mathematical circuit/neural net above is indeed computing the maximum of its two inputs. Let's see a specific example for check : say $x_1 = 2$ and $x_2 = 5$. Then the blue ReLU gates are getting $7,-7,3$ and $-3$ as inputs respectively, from the top most gate to the bottom most. These gates are then passing on to the red output gate the numbers, $\frac{7}{2}, -\frac{1}{2} \max \{0,-7\} = 0, \frac{1}{2} \max \{0,3\} = \frac{3}{2}$ and $\frac{1}{2}\max \{0,-3\}=0$. Finally the red output gate is adding these up to give as output $\frac{7}{2} + \frac{3}{2} = 5$ which is indeed $\max\{2,5\}$. 

Later in the thesis in Chapter \ref{chapfunc} we will prove that to compute the maximum of $n$ numbers atmost $\log(n)$ layers of activation are sufficient. In that same chapter we will study another useful class of neural functions which have been very important in theory building. For some real numbers $w >0$ and $a >0$ consider the function, $f(x) = \max \Big \{ 0, \frac{1}{w} - \frac{1}{w^2}\times \vert x - a\vert \Big \}$. Its easy to see that this is zero everywhere except on the interval $[a-w,a+w]$ where it rises up as a triangle peaked at $x =a$. Now that we have seen the max function example above, one can try to solve the fun puzzle of writing down a neural net with a single layer of activations which can represent this ``triangle wave" function. With some more tricks we will see in  Chapter \ref{chapfunc} that one can try to create nets which will represent a wave form with multiple triangles. Interestingly its still unclear as to what is the appropriate analogue of these waves in high-dimensions! 

These nets which compute the maximum of their input numbers and the triangle waves above often form the building blocks of how we think about the more complicated functions that the nets can compute. In general questions about the representation power of a specific circuit/network design can be extremely difficult to answer and in some cases the answers have required the use of very sophisticated mathematics as we will see in Chapter \ref{chapfunc}. At the core of trying to explain cutting-edge mind-boggling experiments cited earlier like \url{https://thispersondoesnotexist.com/}, there lies in effect more advanced forms these kinds of questions about the function space of nets. 

To see more advanced ideas we need to think more generally in terms of diagrams or ``architectures" : an example of which is given below. (For readers familiar with graph theory one can imagine the underlying diagram to be that of a directed acyclic graph where all edges are pointing to the right.)  

\begin{figure}[h] 
\centering
\includegraphics[scale=0.25]{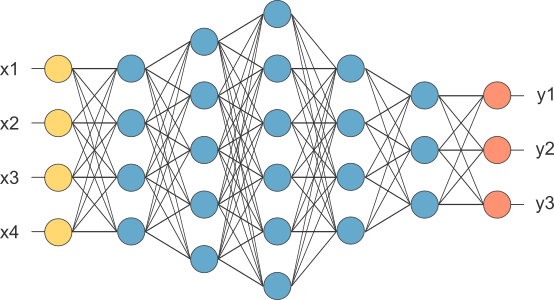}
\end{figure}

Unlike the previous two examples, in the above circuit no ``weights'' have been assigned to the edges of the above graph. So one should think of this diagram as representing the entire set of all the $\R^4 \rightarrow \R^3$ functions which can be computed by the above architecture for a *fixed* choice of ``activation functions'' (like, ReLU as defined above) at each of the {\color{blue}blue} nodes and for all possible values of weights/real-numbers that can be assigned to the edges. The $4$ yellow nodes are where a $4-$ dimensional input vector will go in as input and the $3$ orange nodes are where the $3-$dimensional output vector will come out. Unlike the previous diagrams where every edge carried a single real number, in general every edge can be assigned two real numbers/weights/parameters say $(a,b)$. It is to be understood as specifying that when that edge gets a real number say $x$ as input on its left end then it will give the number $ax +b$ as the output on its right end. Thus the total number of parameters specifying a neural function can be at most twice the number of edges. Some of the largest nets in operation today (called ``AmoebaNet-D") have $600$ million parameters. For comparison recall that the human brain has about $60{-}80$ billion neurons and each of them have about $10^4$ synaptic connections to other neurons. This might motivate one to say that these nets which hope to model intelligence are still quite small in comparison to the human brain! 

The above diagram would be said to represent a class of neural functions with $5$ ``layers" of activations, the blue nodes. We recall that the function we had seen earlier, $f(x_1,x) = \max \{x_1,x_2 \}$ was such that it could be represented using just $1$ layer of activation. In this context it is worth pointing out that we still don't know if say the maximum of $5$ numbers can be computed using $2$ layers of gates or does it necessarily need $3$ such layers of gates!

Now that we have started to think in terms of architectures we end this chapter by pointing out that a very crucial unresolved issue here is to be able to understand, that of all the functions that can be well {\it approximated} by a chosen architecture, how many of them show rapid oscillations - like imagine the  triangle-wave that we saw earlier but with many (but finite) number of triangles. 

This was just the tip of the iceberg and there is a lot more to this story. Lets get a quick glimpse of some of that! In deep-learning we will often assume that the actual artificial intelligence task that one desires to accomplish can be reformulated as trying to minimize some real valued function which is often called the ``loss function".  Our ever increasing experience is that with enough ingenuity one can often write the correct loss function - which will capture the original question as a function which maps the space of functions of a neural architecture (and available data) to non-negative real numbers. And as one might expect we try to minimize the loss function over the space of weights of the net (and hence over the space of functions represented by the given architecture) by approximately moving along the local gradients of the loss function. Thus it is immensely critical that we choose the right architecture - and this is currently almost a form of art! Research is only beginning in this direction of finding systematic methods for making a good choice of architecture. Even after an architecture has been chosen we are faced with the massive question of actually doing this search through its space of functions/weights of the net to find the minimum of the loss.

In the description above we have so far hidden an immense complication which we now necessarily need to confront - that in actual practice information about this loss function is often only partially known! In general this is a very complicated question about searching for an optimal function in a function space while being guided only by a crude estimate of the true optimality criteria. This brings us to the vast field of ``stochastic optimization" - and we will see many provable avatars of this in this thesis. 

We hope that the appetite of the reader has been adequately whetted and at this point they might want to read some of the recently released books on deep-learning like these three freely available beautiful texts, which give a magnificent overview of this exciting new subject  \url{https://www.deeplearningbook.org/}, \url{d2l.ai} and \url{https://mjt.cs.illinois.edu/dlt/} .

\mainmatter 
\pagestyle{thesis} 


\chapter{\vspace{10pt} A Summary of the Results in This Thesis}\label{chapsum} 


Deep learning has brought about a paradigm shift in our quest for general artificial intelligence ~\citep{lecun2015deep}. Powered by concurrent technological advances neural nets have in recent times 
beaten all previous benchmarks in playing hard strategy games like chess and Go, \citep{silver2017mastering, silver2018general}  and have also radically pushed forward the technology towards self-driving cars \citep{fridman2017autonomous}. But on the other hand the methods employed to make deep learning practical remain highly mysterious and challenging to prove guarantees about. During my PhD. I have been extremely passionate about figuring out mathematically rigorous ways to understand deep-learning. We begin to give a summary of the results obtained by first setting up the mathematical notation needed to talk about nets.

\section{Defining Deep Neural Nets} 

The crucial component that goes into defining a neural net is the ``activation function", often denoted as $\sigma$. Historically the $\sigma$ that was in vogue at the beginning of the subject was the ``sigmoid function", $\R \ni x \mapsto \sigma(x) = \frac{1}{1+e^{-\lambda x}}$ for some $\lambda >0$.  But for almost all applications of neural nets today it seems that the most widely used activation function is the ``Rectified Linear Unit (ReLU)''
\begin{align*}
\text{ReLU} : \R &\rightarrow \R \\
x &\mapsto \max \{0,x\}
\end{align*}

In standard practice the notion of $\relu$ is overloaded to denote the following function operating entrywise, 
$\relu : \R^n \ni \x \mapsto  (\max\{0,x_1\}, \max\{0,x_2\}, \ldots, \max\{0,x_n\}) \in \R^n$

\begin{definition}\label{def:relu-dnn}[{\bf ReLU DNNs}] Given $k, w_0, w_1, w_2, \ldots, w_k, w_{k+1} \in \N$, one defines a {\it depth $k+1$} ``ReLU Deep Neural Net (DNN)" as the following function, 

\begin{equation}\label{eq:DNN-def}
\R ^{w_0} \ni \x \mapsto  f(\x) = \A_{k+1}\circ\relu \circ \A_k\circ \cdots \circ \A_2 \circ\relu \circ \A_1 \in \R^{\w_{k+1}}
\end{equation}

where $\A_i : \R^{w_{i-1}} \to \R^{w_i}$ for $i=1, \ldots, k+1$ is a set of $k+1$ affine transformations.  The positive integers $w_1,\ldots,w_k$ are said to specify the {\it widths} of the {\it hidden layers} or {\it layers of activation}. The number $\max\{w_1, \ldots, w_k\}$ is called the {\it width} of this ReLU DNN. The {\em size} of the ReLU DNN is defined to be the number of univariate activation gates used and that can be easily seen to be $w_1 + w_2 + \ldots + w_k$. 

Such a ReLU DNN is sometimes also called a $(k+1)$-layer ReLU DNN, and is said to have $k$ {\it hidden layers}. Number of layers or depth of the net can be seen to be measuring as the length of the shortest path from the input to the output of the directed acyclic graph that naturally represents such a neural network - of which we have already seen examples in the previous chapter on informal summary and we see another in Figure \ref{supp:auto} given below. 
\qed
\end{definition}



For any $(m,n)\in \N$, let ${\cal A}_m^n$ adenote the class of affine and affine transformations from $\R^m\to \R^n$, respectively. Thus we introduce a compact notation for the class of width specified DNNs as follows,  

\begin{definition}
We denote the class of $\R^{w_0}\to \R^{w_{k+1}}$ ReLU DNNs with $k$ hidden layers of widths $\{ w_i \}_{i=1}^{k}$ by ${\cal F}_{\{ w_i \}_{i=0}^{k+1}}$, i.e.

\begin{equation}
{\cal F}_{\{ w_i \}_{i=0}^{k+1}} := \{ \A_{k+1}\circ\relu \circ \A_k\circ \cdots \circ \A_2 \circ\relu \circ \A_1  \mid \A_i \in {\cal A}_{w_{i-1}}^{w_i} \forall i\in\{1,\ldots,k+1\} \} 
\end{equation}
\end{definition}

Corresponding to any affine transformation $\A_i$ above we will typically decompose its action via a linear transformation ({\it ``weight matrix"}) $\W_i$ and a vector $\b_i$ s.t $\x \mapsto \A_i \x = \W_i \x + \b_i$. This is particularly helpful in setting up the notation for a particular class of neural nets ``Autoencoders" that we shall often consider in this thesis and which we specify below.  


\subsection{Notation for a special class of Autoencoders}\label{def:autoencoder}

Let $\y \in \mathbb{R}^n$ be the input vector to the autoencoder, $\{\W_i\}_{i=1,..,\ell}$ denote the weight matrices of the net and $\{ \b_i \}_{i=1,..,2\ell}$ be the bias vectors. Then the output $\hat{\y} \in \mathbb{R}^n$ of the autoencoder (mapping $\R^n \rightarrow \R^n$) is defined as,

\[ \hat{\y} =  \W_1^\top \sigma(\dots \sigma(\W_{\ell-1}^\top\sigma(\W_\ell^\top \a + \b_{\ell+1}) + \b_{\ell+2}) \dots) + \b_{2\ell} \]

\[ where \]

\[ \a = \sigma(\W_\ell \sigma(\dots \sigma(\W_2 \sigma(\W_1\y + \b_1) + \b_2) \dots) + \b_\ell) \]

This defines an autoencoder with $2\ell - 1$ hidden layers using the $\ell$ weight matrices and the $2\ell$ bias vectors defined above. The particular symmetry that has been imposed among the layers leading up to the $\a$ and those that act on $\a$ is what leads to this arrangement being called {\it ``weight tied"}. Such autoencoders are a fairly standard setup that have been used in previous work \citep{arpit2015regularized, baldi2012autoencoders, kuchaiev2017training, vincent2010stacked}. 

A special case of the above that we shall focus on is when $\ell = 1$ i.e its a weight tied autoencoder of depth $2$ and $\b_2 =0$. We shall use $\W = \W_1$, $\b_1 = \epsilon \in \R^h$ where $h$ is the width of the net and the number of activation units used.  Denoting the output of the hidden layer of activations as $\rr \in \R^h$ we have for this case, 

\begin{align}\label{eqn:autoencoder}
\hat{\y} = \W^T \rr \text{ where } \rr = \relu \left( \W \y - \epsilon \right)
\end{align}

We shall define the columns of $\W^\top$ (rows of $\W$) as $\{\W_i\}_{i=1}^h$. A pictorial representation of such a depth $2$ autoencoder is as given in Figure \ref{supp:auto}.   

\begin{figure}[h]
\centering
\includegraphics[scale = 0.3]{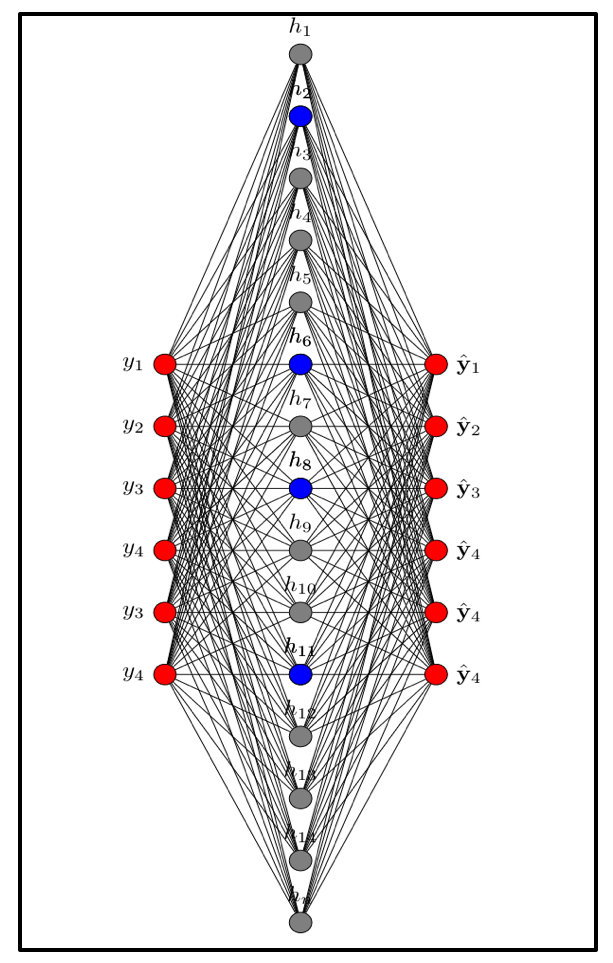}
\caption{The above is the circuit representation of a depth $2$, width $15$ autoencoder mapping, $\R^4 \ni y \mapsto \hat{y} \in \R^4$}\label{supp:auto}
\end{figure}



Deep learning, refers to a suite of computational techniques that have been developed recently for training DNNs. It started with the work of~\cite{hinton2006fast} (deep belief networks) and ~\cite{salakhutdinov2009deep} (deep Boltzmann machines) which gave empirical evidence that if deep architectures are initialized properly (for instance, using unsupervised pre-training), then we can find good solutions in a reasonable amount of runtime. This work was soon followed by a series of early successes of deep learning at significantly improving the state-of-the-art AI systems in speech recognition, image classification and natural language processing based on deep neural nets~\citep{hinton2012deep,dahl2013improving,krizhevsky2012imagenet,le2013building,sutskever2014sequence}. While there is less of evidence now that pre-training actually helps, several other solutions have since been put forth to address the issue of efficiently training DNNs. These include heuristics such as dropouts~\citep{srivastava2014dropout}, but also considering alternate deep architectures such as convolutional neural networks~\citep{sermanet2014overfeat}


One of the fascinating aspects of trying to build a theory for deep-learning is that if we view this project through the lens of optimization theory then its setup is essentially opposite to  how theory of optimization is studied in standard textbooks and courses. Typically one starts off with a well defined optimization problem (like Conic Programming) and then one studies the properties of its optima and the algorithmic aspects of solving it. But deep-learning has developed entirely on the sturdy shoulders of thousands of highly innovative experimenters who have caused this artificial intelligence revolution by developing a vast array of mysterious heuristics which work to get the neural net to perform tasks which would be "human like". To give an obvious example : there is no unambiguous way to quantify the fact that state-of-the-art GAN outputs look like realistic images but this is exactly the criteria we would want to use to judge whether a GAN has been trained well or not! As a subject deep-learning is predominantly defined by wildly successful algorithmic heuristics and often it's entirely unclear as to how the obviously wonderful performance of the trained net can be described as finding good solutions of some optimization problem! For some of the most exotic applications of neural nets, the debates continue to happen about what is the right optimization problem whose solution would correctly capture the success of the net.  

But if we can agree about the ``loss function (say $\ell$)" to be used then at least for the most ordinary use cases the challenge of modern deep learning can be abstracted out as a particularly hard case of usual ``learning theory". In such benign situations we can focus on wanting to solve the following function optimization/``risk minimization" question,
\begin{align}\label{Def:risk}
\min_{\mathbf{N} \in {\cal N}} \mathbb{E}_{\z \in {\cal D}} [\ell (\mathbf{N}, \z)]
\end{align}
where $\ell$ is some lower-bounded non-negative function, members of ${\cal N}$ are continuous piecewise linear functions representable by some chosen neural net architecture and  we only have sample access to the distribution ${\cal D}$. This reduces to the ``empirical risk minimization" question when this ${\cal D}$ is an uniform distribution on a finite set of points.
In the light of the previous discussion, the research results presented in this thesis can be seen to be focused on the following $3$ critical aspects, (a) understanding mathematical properties of the neural function spaces on which this risk minimization is being attempted (Section \ref{space}),  (b) proving guarantees about algorithms which can be used to approximately solve this question of neural risk minimization (Section \ref{alg}) and (c) understanding the structure of the nets which are solutions to such risk minimization questions (Section \ref{pac})

 
\section {Understanding the space of neural functions}\label{space}
 In Chapter \ref{chapfunc} we provide $3$ main kinds of insights about the nature of neural functions. Firstly, we extend the recently published results in \cite{telgarsky2016benefits} to show that for every $k \in \Z^+$ there exists a {\it continuum} of hard functions which require $O(k^3)$ size to represent at depths $1+k^2$ but will require $\Omega(k^k)$ (super-exponential in depth) size to approximate at depth $1+k$. We also show that a kind of polytopes, called ``zonotopes" have a natural relationship to neural nets i.e the ReLU nets can represent the gauge function of zonotopes and this in turn gives us an explicit construction of a {\it continuum} of ReLU functions with the largest number of affine pieces for large classes of architectures. 
 
Secondly, we were intrigued by the question of finding non-trivial upperbounds on the run-time of algorithms which can find the {\it exact global minima} of empirical risk. We show how a collection of convex programming subroutines can be used to get algorithms for exact empirical risk minimization in depth $2$ which run in $\poly(\text{data})$ time at a fixed depth. Such faster-than-brute-force exact optimization algorithms remain unknown for higher depths. (Recently there has been a very interesting complexity theoretic paper from Berkeley, \cite{manurangsi2018computational} which builds further on this algorithm of ours.) 

Lastly, we also investigate depth hierarchy theorems for $\relu$ nets (ending in a ``Linear Threshold Function" (LTF) gate which maps $\R \ni y \mapsto -1+2{\mathbf 1}_{y\geq 0} \in \R$) trying to compute Boolean functions. Many of the key results in this direction were achieved by extending to ReLU nets a method of random restrictions recently developed by Daniel Kane and Ryan Williams. This line of investigation has thrown up a lot of puzzling open questions about whether or not ReLU nets are more efficient at representing Boolean functions than usual Boolean circuits.

\section {Landscape of neural nets and deep-learning algorithms}\label{alg}

This theme is what can be said to be the mainstay of this thesis and it spans across $3$ chapters.  

In Chapter \ref{chaptrain_new} we show $2$ kinds of insights about training a $\relu$ gate. {\it Firstly} we give a very simple iterative stochastic algorithm to recover the underlying parameter $\w_*$ of the $\relu$ gate when realizable data allowed to be sampled online is of the form $(\x, \max \{ 0, \w_*^\top\x \} )$. Compared to all previous such attempts the distributional condition we use is very mild, which essentially just captures the intuition that enough of our samples have to be such that $\w_*^\top\x >0$.

{\it Secondly} we give an argument which establishes a first-of-its-kind mathematical control on the behaviour of gradient descent (with deliberate injection of noise) on the squared loss function of a single $\relu$ gate. It is to be noted that this argument doesn't need any distributional assumption beyond realizability of the labels and thus it makes us optimistic that this is a potentially interesting step towards explaining the success of this ubiquitously used heuristic. The key idea here is that of ``coupling" which shows that from the iterates of noise injected gradient descent on the squared loss of a $\relu$ gate one can create a discrete super-martingale.  




In Chapter \ref{chapaut} we focus on autoencoders and make progress about explaining their success. We were particularly inspired by the experimental works of Brendan Frey and Alireza Makhzani. We checked that actually an off-the-shelf {\rm RMSProp} algorithm very easily do reasonably good autoencoding on MNIST even at depth $2$. This piqued our interest to understand this better and we analyzed the landscape of the autoencoder under the usual sparse-coding generative model. Via a very elaborate analysis we are able to estimate the value of the gradient of the squared loss on depth $2$ autoencoders whose input/output dimension is the same as that of the observed vectors in sparse-coding and the width of the network is the same as the sparse-code dimension. 

This intricate analysis leads to the insight that the norm of this gradient decreases in a small neighbourhood of the original (unknown) dictionary as the sparse-code dimension increases. Such a proof of asymptotic criticality around the dictionary takes a step towards explaining why neural nets should be able to do dictionary learning. Works like \cite{nguyen2019dynamics} have recently built on top of our analysis framework to show trainability proofs for autoencoders.

In Chapter \ref{chapadam} we focus on understanding the specific  adaptive gradient algorithms, {\rm RMSProp} and {\rm ADAM}, which are implemented widely across almost all deep-learning tasks and are known to be the state-of-the-art in almost every application. We give the first ever proofs that (deterministic) RMSPRop and (deterministic) ADAM converge to criticality for smooth objectives without the assumption of convexity. We also motivate a class of first order moment constrained oracles in the presence of which we can show the first ever proof of convergence of stochastic RMSProp with no convexity assumptions and at the same speed as SGD on convex functions. 

We emphasize that this is particularly exciting in the context of recent results \cite{reddi2018convergence} which have shown that under the same setting of constant hyperparameter values ADAM used as an online optimizer cannot always get asymptotically zero average regret. We also shown extensive experiments on VGG-9 running on CIFAR-10 and across various sizes of autoencoders running on MNIST that ADAM's performance gets a consistent and curious boost (and thus it outperforms its competitors) when its $\beta_1$ (the parameter that controls the influence of the history of grdients on the current update) is pushed closer to $1$ than its usual settings. 




\section{Estimating the risk function of neural nets}\label{pac}

The long standing open-question in deep-learning is to be able to theoretically explain as to when neural nets which are massively over-parameterized happen to be high-quality solutions of the risk minimization problem defined in \ref{Def:risk} - even when they fit the training data arbitrarily accurately. In recent times it has been increasingly realized that good risk bounds possibly necessarily need to depend on the training algorithm as well as the training data. The currently available methods to bound the risk function have been beautifully reviewed in this paper \citep{audibert2007combining}. Here the authors have clubbed the techniques into primarly four categories, $(1)$ ``Supremum Bounds" (like generic chaining, Dudley integral, Rademacher complexity), $(2)$ ```Variance Localized Bounds", $(3)$ ``Data-Dependent Bounds" and $(4)$ ``Algorithm Dependent Complexity". The last category includes PAC-Bayes bounds which have risen to prominence in recent times and is the crux of our most recently completed work described in Chapter \ref{chapPAC} 

Rademacher complexity based bounds like \cite{golowich2018size} and \cite{bartlett2017spectrally} fail to give non-vacuous bounds when evaluated on the gigantic neural nets used in practice. In the PAC-Bayesian framework we slightly move away from trying to bound the risk and instead we try to bound an instance of ``stochastic risk" which can be thought of as allowing for the neural net's weights/parameters to be noisy. This can be argued to be the most  natural quantity to bound given that all successful neural training algorithms are stochastic and hence the trained net obtained from it is essentially a sample from a distribution on the neural function space induced by the training algorithm. By doing this shift in viewpoint, recent works like \cite{dziugaite2017computing} and \cite{zhou2018non} have shown for the very first time that PAC-Bayesian bounds can give non-trivial risk bounds for practical neural nets.  But the above bounds are ``computational" in the sense that obtaining them requires an algorithmic search over a certain parametric space of distributions. These experiments strongly motivate our current work seeking rigorous theoretical exploration of the power of PAC-Bayesian technology in explaining the learning ability of neural nets. 

Previous PAC-Bayes bounds have used data dependent priors on the geometric mean of the spectral norms of the layer matrices to try to track the distance in the parameter space of the trained net from a fixed point in that space. In our work the first key idea we initiate is to track the distance of the trained net from its initialization by looking at two independent quantities (a) a non-compact parameter : the change from initialization of the norm of the vector of weights of the net (i.e the sum of Frobenius norms of the layer matrices for a net without bias weights) and (b) a compact part : the angular deflection of this vector of weights from initialization to the end of training.  In this work we instantiate an elaborate mechanism of putting a two indexed grid of priors which can simultaneously be sensitive to both the above properties of the neural net training process. 

Our second key idea is to realize that in the PAC-Bayesian framework one can leverage more out of the angle parameter by also simultaneously training a cluster of nets which are initialized close to the original net. Because of this use of clusters, compared to previous bounds our dependency on the distance from initialization is not only more intricate but we are also able to get more sensitive to the average case behaviour. 

Compared to previous theories in this direction, \citep{neyshabur2017pac} we build into the formalism a larger number of data-dependent (and hence tunable) parameters.  As a consequence we get a risk bound on nets  which is empirically not only seen to be tighter than \cite{neyshabur2017pac} but also has better/lower ``rates" of dependency on the neural architectural parameters like depth and width.

We emphasize that the aforesaid ability to leverage the use of clusters of nets in tandem is critically hinged on us being able to prove methods of creating multi-parameter families of mixture of Gaussian distributions such that the given neural function remains stable when its weights are perturbed by noise sampled from these distributions. There are potentially far reaching implications of such theorems because of the intricate arguments made in recent times which motivate why finding provable compression algorithms for any class of nets is tied to being able to prove the existence of noise distributions to which this same class is resilient.

We go on to demonstrate two kinds of insights in our experiments. Over synthetic data and standard tests like CIFAR-10 we show that our bound performs consistently better than existing PAC-Bayesian bounds. Next we show in the experiments that the two parameters said above have a lot more structure than what theory is currently capable of leveraging. We observe in our experiments that the $2-$norm of the weight vector described above always undergoes a slight dilation during the training. We also demonstrate that the angular deflection is predominantly determined by the underlying data-set/data-distribution and is only very slightly affected by the architecture of the net. 

Current wisdom in the field suggests that observations like above about systematic behaviours of neural net training can potentially be leveraged into increasingly creative risk bounds for nets. Thus these experiments pave the way for our continuing exploration of even better bounds which can eventually lead to principled methods of choosing the right net to use for a given artificial intelligence task at hand.



\chapter{\vspace{10pt} Exploring the Space of Neural Functions}\label{chapfunc}

\section{Introduction}

Neural networks with a single hidden layer of finite size can approximate any continuous function on a compact subset of $\R^n$ arbitrary well. This universal approximation result was first given for sigmoidal activation function in ~\cite{cybenko1989approximation}, and later generalized by Hornik to an arbitrary bounded and non-constant activation function~\citep{hornik1991approximation} (and in turn it applied to $\relu$ nets as well). Furthermore, neural networks ending in a {\rm LTF} gate have finite VC dimension (depending polynomially on the number of edges in the network), 
and therefore, are PAC (Probably Approximately Correct) learnable using a sample of size that is polynomial in the size of the networks~\citep{anthony2009neural}. 
However, neural networks based methods were shown to be computationally hard to learn~\citep{anthony2009neural} and had mixed empirical success. Consequently, DNNs fell out of favor by the late 90s.
~\\ \\
In this chapter, we formally study deep neural networks with rectified linear units; we refer to these deep architectures as ReLU DNNs. Our work is inspired by these recent attempts to understand the reason behind the successes of deep learning, both in terms of the structure of the functions represented by DNNs, \citep{T1,T2,KW,OS}, as well as efforts which have tried to understand the non-convex nature of the training problem of DNNs better \citep{K,RH}.
Our investigation of the function space represented by ReLU DNNs also takes inspiration from the classical theory of circuit complexity; we refer the reader to \cite{AB,Shp,Jun,Sap,All} for various surveys of this deep and fascinating field. 
In particular, our gap results are inspired by results like the ones by Hastad~\cite{hastad1986almost}, Razborov~\cite{razborov1987lower} and Smolensky~\cite{smolensky1987algebraic} which show a strict separation of complexity classes. We make progress towards similar statements with deep neural nets with ReLU activation. 

\subsection{Notation and Definitions}

\begin{definition}\label{def:PWL}[Piecewise linear functions] We say a function $f\colon \R^n\to \R$ is {\em continuous piecewise linear (PWL)} if there exists a {\em finite} set of polyhedra whose union is $\R^n$, and $f$ is affine linear over each polyhedron (note that the definition automatically implies continuity of the function because the affine regions are closed and cover $\R^n$, and affine functions are continuous). 
The {\em number of pieces of $f$} is the number of maximal connected subsets of $\R^n$ over which $f$ is affine linear (which is finite).
\end{definition}

Many of our important statements will be phrased in terms of the following simplex.

\begin{definition} Let $M > 0$ be any positive real number and $p \geq 1$ be any natural number. Define the following set: 
$$\Delta^p_{M} := \{\x \in \R^p: 0 < \x_1 < \x_2 < \ldots < \x_p < M\}.$$
\end{definition}

\section {Exact characterization of function class represented by ReLU DNNs}\label{sec:exact-character}

One of the main advantages of DNNs is their representational ability. 
In this section, we give an exact characterization of the functions representable by ReLU DNNs. Moreover, we show how structural properties of ReLU DNNs, specifically their depth and width, affects their expressive power. It is clear from definition that any function from $\R^n \to \R$ represented by a ReLU DNN is a continuous piecewise linear (PWL) function. In what follows, we show that the converse is also true, that is any PWL function is representable by a ReLU DNN. In particular, the following theorem establishes a one-to-one correspondence between the class of ReLU DNNs and PWL functions. 

\begin{theorem}\label{cor:all-pwl-are-dnn} Every $\R^n \to \R$ ReLU DNN represents a piecewise linear function, and every piecewise linear function $\R^n \to \R$ can be represented by a ReLU DNN with at most $\lceil \log_2(n+1) \rceil + 1$ depth.
\end{theorem}

{\bf{Proof Sketch: }}
It is clear that any function represented by a ReLU DNN is a PWL function. To see the converse, we first note that any PWL function can be represented as a linear combination of piecewise linear convex functions. More formally, by Theorem~1 in~\citep{wang2005generalization},
for every piecewise linear function $f :\R^n \to \R$, there exists a finite set of affine linear functions $\ell_1, \ldots, \ell_k$ and subsets $S_1, \ldots, S_p \subseteq \{1, \ldots, k\}$ (not necessarily disjoint) where each $S_i$ is of cardinality at most $n+1$, such that \begin{equation}\label{eq:hinged-hyperplane}f = \sum_{j = 1}^p s_j \bigg(\max_{i \in S_j} \ell_i\bigg),\end{equation} where $s_j \in \{-1,+1\}$ for all $j=1, \ldots, p$. 
Since a function of the form $\max_{i \in S_j} \ell_i$ is a piecewise linear convex function with at most $n+1$ pieces (because $|S_j| \leq n+1$), Equation~(\ref{eq:hinged-hyperplane}) says that any continuous piecewise linear function (not necessarily convex) can be obtained as a linear combination of piecewise linear convex functions each of which has at most $n+1$ affine pieces.  Furthermore, Lemmas~\ref{lem:comp-DNN}, \ref{lem:add-DNN} and~\ref{lem:max-DNN} in the Appendix, show that composition, addition, and pointwise maximum of PWL functions are also representable by ReLU DNNs. In particular, in Lemma~\ref{lem:max-DNN} we note that $\max\{ x,y \}=\frac{x+y}{2}+\frac{|x-y|}{2}$ is implementable by a two layer ReLU network and use this construction in an inductive manner to show that maximum of $n+1$ numbers can be computed using a ReLU DNN with depth at most $\lceil \log_2(n+1) \rceil$.

~\\ 
While Theorem~\ref{cor:all-pwl-are-dnn} gives an upper bound on the depth of the networks needed to represent all continuous piecewise linear functions on $\R^n$, it does not give any tight bounds on the {\em size} of the networks that are needed to represent a given piecewise linear function. For $n =1$, we give tight bounds on size as follows:

\begin{theorem}\label{thm:1-dim-2-layer} Given any piecewise linear function $\R \to \R$ with $p$ pieces there exists a 2-layer DNN with at most $p$ nodes that can represent $f$. Moreover, any 2-layer DNN that represents $f$ has size at least $p-1$.
\end{theorem}

Finally, the main result of this section follows from Theorem~\ref{cor:all-pwl-are-dnn}, and well-known facts that the piecewise linear functions are dense in the family of compactly supported continuous functions and the family of compactly supported continuous functions are dense in $L^q(\R^n)$~\citep{royden}). 
Recall that $L^q(\R^n)$ is the space of Lebesgue integrable functions $f$ such that $\int |f|^q d\mu < \infty$, where $\mu$ is the Lebesgue measure on $\R^n$ (see Royden~\cite{royden}).

\begin{theorem}\label{thm:Lpapprox}
Every function in $L^q(\R^n),\ (1 \leq q \leq \infty)$ can be arbitrarily well-approximated in the $L^q$ norm (which for a function $f$ is given by $\vert \vert f \vert \vert _q = (\int \vert f \vert^q)^{1/q}$) by a ReLU DNN function with at most $\lceil \log_2(n+1)\rceil$ hidden layers. Moreover, for $n=1$, any such $L^q$ function can be arbitrarily well-approximated by a 2-layer DNN, with tight bounds on the size of such a DNN in terms of the approximation.
\end{theorem}

\newpage 
Proofs of Theorems~\ref{thm:1-dim-2-layer} and~\ref{thm:Lpapprox} are provided in Appendix~\ref{sec:thm-1-rep}. We would like to remark that a weaker version of Theorem~\ref{cor:all-pwl-are-dnn} was observed in~\cite[Proposition 4.1]{goodfellow2013maxout} (with no bound on the depth), along with a universal approximation theorem~\citep[Theorem 4.3]{goodfellow2013maxout} similar to Theorem~\ref{thm:Lpapprox}. The authors of~\cite{goodfellow2013maxout} also used a previous result of Wang~\citep{wang2004general} for obtaining their result. 
In a subsequent work Boris Hanin \citep{hanin2017universal} has, among other things, found a width and depth upper bound  for ReLU net representation of positive PWL functions on $[0,1]^n$. The width upperbound is n+3 for general positive PWL functions and $n+1$ for convex positive PWL functions. For convex positive PWL functions his depth upper bound is sharp if we disallow dead ReLUs.

\section{Benefits of Depth}
Success of deep learning has been largely attributed to the depth of the networks, i.e. number of successive affine transformations followed by nonlinearities, which is shown to be extracting hierarchical features from the data. In contrast, traditional machine learning frameworks including support vector machines, generalized linear models, and kernel machines can be seen as instances of shallow networks, where a linear transformation acts on a single layer of nonlinear feature extraction. In this section, we explore the importance of depth in ReLU DNNs. In particular, in Section~\ref{sec:gap-results}, we provide a smoothly parametrized family of $\R\to\R$ ``hard'' functions representable by ReLU DNNs, which requires exponentially larger size for a shallower network to represent. Furthermore, in Section~\ref{sec:zon}, we construct a continuum of $\R^n\to\R$ ``hard'' functions representable by ReLU DNNs, which to the best of our knowledge is the first explicit construction of ReLU DNN functions whose number of affine pieces grows exponentially with input dimension. 


\subsection{Circuit lower bounds for $\R \rightarrow \R$ ReLU DNNs}\label{sec:gap-results}
In this section, we are only concerned about $\R\to\R$ ReLU DNNs, i.e. both input and output dimensions are equal to one. The following theorem shows the depth-size trade-off in this setting.

\begin{theorem}\label{thm:high-to-low} For every pair of natural numbers $k \geq 1$, $w \geq 2$, there exists a family of ``hard" functions representable by a $\R \to \R$ $(k+1)$-layer ReLU DNN of width $w$ such that if it is also representable by a $(k'+1)$-layer ReLU DNN for any $k' \leq k$, then this $(k'+1)$-layer ReLU DNN has size at least $\frac12k'w^{\frac{k}{k'}}-1$.
\end{theorem}

\noindent In fact our family of hard functions described above has a very intricate structure as stated below.

\begin{theorem}\label{thm:RELU_DNN_hard}
For every $k \geq 1$, $w \geq 2$, every member of the family of hard functions in Theorem~\ref{thm:high-to-low} has $w^k$ pieces and this family can be parametrized by \begin{equation}\label{eq:torus-simplex}\bigcup_{M > 0}\underbrace{(\Delta^{w-1}_M \times \Delta^{w-1}_M\times \ldots \times \Delta^{w-1}_M)}_{k \textrm{ times}},\end{equation}
i.e., for every point in the set above, there exists a distinct function with the stated properties. 
\end{theorem}

\noindent The following is an immediate corollary of Theorem~\ref{thm:high-to-low} by choosing the parameters carefully.

\begin{corollary}\label{cor:special-case-1} For every $k \in \N$ and $\epsilon > 0$, there is a family of functions defined on the real line such that every function $f$ from this family can be represented by a $(k^{1+\epsilon}) + 1$-layer DNN with size $k^{2+\epsilon}$ and if $f$ is represented by a $k+1$-layer DNN, then this DNN must have size at least $\frac12k\cdot k^{k^\epsilon} - 1$. Moreover, this family can be parametrized as,  $\cup_{M>0}\Delta^{k^{2+\epsilon}-1}_M$.
\end{corollary}

\noindent A particularly illuminative special case is obtained by setting $\epsilon = 1$ in Corollary~\ref{cor:special-case-1}: 
\begin{corollary}\label{cor:special-case-2} For every natural number $k \in \N$, there is a family of functions parameterized by the set $\cup_{M>0}\Delta^{k^3-1}_M$ such that any $f$ from this family can be represented by a $k^2 + 1$-layer DNN with $k^3$ nodes, and every $k+1$-layer DNN that represents $f$ needs at least $\frac12k^{k + 1} - 1$ nodes.
\end{corollary}

Towards proving the above two theorems we first need the following definition and lemma, 

\begin{definition}\label{def:hardfunc} 
For $p \in \N$ and $\a \in \Delta^p_M$, we define a function $h_{\a}:\R \to \R$ which is piecewise linear over the segments $(-\infty, 0], [0, \a_1], [\a_1, \a_2], \ldots, [\a_p,M], [M, +\infty)$ defined as follows: $h_{\a}(x) = 0$ for all $x \leq 0$, $h_{\a}(\a_i) = M(i \mod 2)$, and $h_{\a}(M) = M-h_{\a}(\a_p)$ and for $x \geq M$, $h_\a(x)$ is a linear continuation of the piece over the interval $[\a_p,M]$. Note that the function has $p+2$ pieces, with the leftmost piece having slope $0$. Furthermore, for $\a^1, \ldots, \a^k \in \Delta^p_M$, we denote the composition of the functions $h_{\a^1}, h_{\a^2}, \ldots, h_{\a^k}$ by $$ H_{\a^1, \ldots, \a^k} := h_{\a^k} \circ h_{\a^{k-1}}\circ \ldots \circ h_{\a^1}.$$
\end{definition}

\begin{lemma}\label{lem:construct-hard} For any $M >0$, $p \in \N$, $k \in \N$ and $\a^1, \ldots, \a^k \in \Delta^p_M$, if we compose the functions $h_{\a^1}, h_{\a^2}, \ldots, h_{\a^k}$ the resulting function is a piecewise linear function with at most $(p+1)^k + 2$ pieces, i.e., $$ H_{\a^1, \ldots, \a^k} := h_{\a^k} \circ h_{\a^{k-1}}\circ \ldots \circ h_{\a^1}$$ is piecewise linear with at most $(p+1)^k + 2$ pieces, with $(p+1)^k$ of these pieces in the range $[0,M]$ (see Figure~\ref{fig:sawtoothcomposition}). Moreover, in each piece in the range $[0,M]$, the function is affine with minimum value $0$ and maximum value $M$.
\end{lemma}

\begin{proof} 
Simple induction on $k$.
\end{proof}

\begin{proof}[Proof of Theorem~\ref{thm:RELU_DNN_hard}] Given $k \geq 1$ and $w \geq 2$, choose any point $$ (\a^1, \ldots, \a^k) \in \bigcup_{M > 0}\underbrace{(\Delta^{w-1}_M \times \Delta^{w-1}_M\times \ldots \times \Delta^{w-1}_M)}_{k \textrm{ times}}.$$ 

By Definition~\ref{def:hardfunc}, each $h_{\a^i}$, $i=1, \ldots, k$ is a piecewise linear function with $w+1$ pieces and the leftmost piece having slope $0$. Thus, by Corollary~\ref{cor:tighter-bound-2-layer}, each $h_{\a^i}$, $i=1, \ldots, k$ can be represented by a 2-layer ReLU DNN with size $w$. Using Lemma~\ref{lem:comp-DNN}, $H_{\a^1, \ldots, \a^k}$ can be represented by a $k+1$ layer DNN with size $wk$; in fact, each hidden layer has exactly $w$ nodes.
\end{proof}

\begin{proof}[Proof of Theorem~\ref{thm:high-to-low}] Follows from Theorem~\ref{thm:RELU_DNN_hard} and Lemma~\ref{lem:size-bounds}.
\end{proof}

\begin{figure}[h]
\includegraphics[width=1\textwidth]{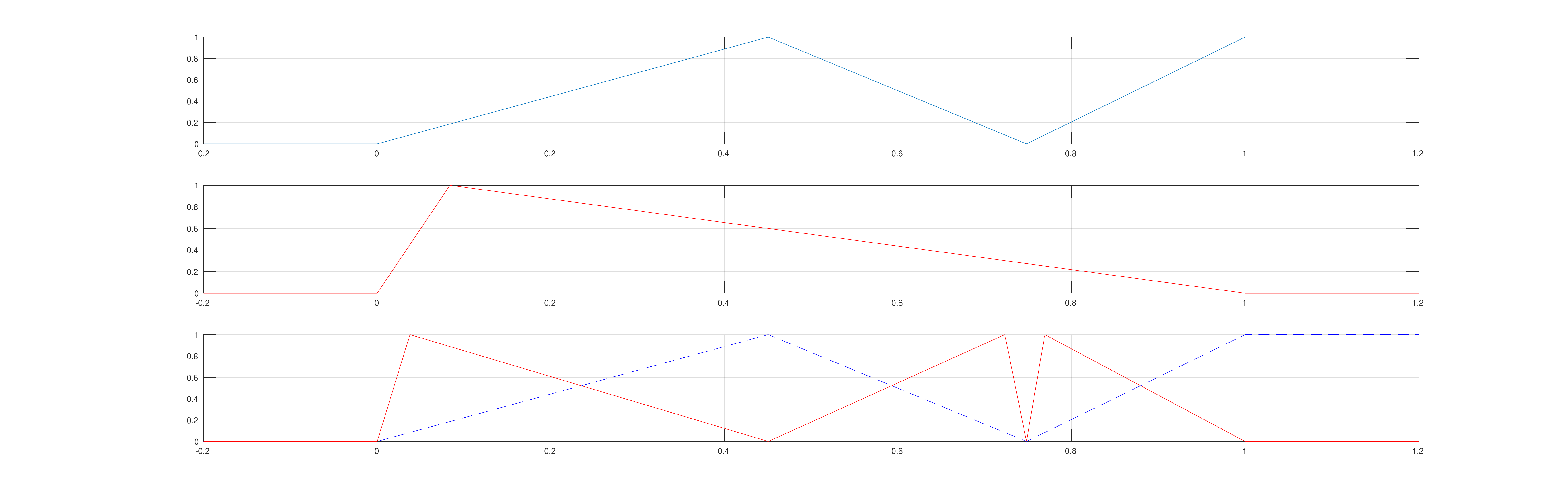}
\caption{Top: $h_{\a^1}$ with $\a^1\in \Delta_1^2$ with 3 pieces in the range $[0,1]$. Middle: $h_{\a^2}$ with $\a^2 \in \Delta_1^1$ with 2 pieces in the range $[0,1]$. Bottom: $H_{\a^1,\a^2}=h_{\a^2} \circ h_{\a^1}$ with $2\cdot 3 = 6$ pieces in the range $[0,1]$. The dotted line in the bottom panel corresponds to the function in the top panel. It shows that for every piece of the dotted graph, there is a full copy of the graph in the middle panel.}\label{fig:sawtoothcomposition}
\end{figure}

We can also get hardness of approximation versions of Theorem~\ref{thm:high-to-low} and Corollaries~\ref{cor:special-case-1} and~\ref{cor:special-case-2}, with the same gaps (upto constant terms), using the following theorem.

\begin{theorem}\label{thm:approximation} For every $k \geq 1$, $w \geq 2$, there exists a function $f_{k,w}$ that can be represented by a $(k+1)$-layer ReLU DNN with $w$ nodes in each layer, such that for all $\delta > 0$ and $k' \leq k$ the following holds:
$$\inf_{g \in \mathcal{G}_{k',\delta}}\int_{x=0}^{1}{\vert f_{k,w}(x) - g(x) \vert dx} > \delta,$$ where $\mathcal{G}_{k',\delta}$ is the family of functions representable by ReLU DNNs with depth at most $k'+1$, and size at most $k'\frac{w^{k/k'}(1 -4 \delta)^{1/k'}}{2^{1+1/k'}}$.
\end{theorem}

The depth-size trade-off results in Theorems~\ref{thm:high-to-low}, and~\ref{thm:approximation} extend and improve Telgarsky's theorems from~\citep{T1,T2} in the following three ways: 

\begin{itemize} 
\item[(i)] If we use our Theorem \ref{thm:approximation} to the pair of neural nets considered by Telgarsky in Theorem $1.1$ in \cite{T2} which are at depths $k^3$ (of size also scaling as $k^3$) and $k$ then for this purpose of approximation in the $\ell_1-$norm we would get a size lower bound for the shallower net which scales as $\Omega(2^{k^2})$ which is exponentially (in depth) larger than the lower bound of $\Omega(2^k)$ that Telgarsky can get for this scenario. 

\item[(ii)] Telgarsky's family of hard functions is parameterized by a single natural number $k$. In contrast, we show that for every {\em pair} of natural numbers $w$ and $k$, and a point from the set in equation~\ref{eq:torus-simplex}, there exists a ``hard'' function which to be represented by a depth $k'$ network would need a size of at least $w^{\frac{k}{k'}}k'$. With the extra flexibility of choosing the parameter $w$, for the purpose of showing gaps in representation ability of deep nets we can shows size lower bounds which are \emph {super}-exponential in depth as explained in Corollaries~\ref{cor:special-case-1} and~\ref{cor:special-case-2}.

\item [(iii)]  A characteristic feature of the ``hard'' functions in Boolean circuit complexity is that they are usually a countable family of functions and not a ``smooth'' family of hard functions. In fact, in the last section of~\cite{T1}, Telgarsky states this as a ``weakness'' of the state-of-the-art results on ``hard'' functions for both Boolean circuit complexity and neural nets research.  In contrast, we provide a smoothly parameterized family of ``hard'' functions in Section~\ref{sec:gap-results} (parametrized by the set in equation~\ref{eq:torus-simplex}). Such a continuum of hard functions wasn't demonstrated before this work. 
\end{itemize}

We point out that Telgarsky's results in~\citep{T2} apply to deep neural nets with a host of different activation functions, whereas, our results are specifically for neural nets with rectified linear units. In this sense, Telgarsky's results from~\citep{T2} are more general than our results in this paper, but with weaker gap guarantees. Eldan-Shamir~\citep{OS,eldan2016power} show that there exists an $\R^n\to \R$ function that can be represented by a 3-layer DNN, that takes exponential in $n$ number of nodes to be approximated to within some constant by a 2-layer DNN. While their results are not immediately comparable with Telgarsky's or our results, it is an interesting open question to extend their results to a constant depth hierarchy statement analogous to the recent result of Rossman et al~\citep{rossman2015average}. We also note that in last few years, there has been much effort in the community to show size lowerbounds on ReLU DNNs trying to approximate various classes of functions which are themselves not necessarily exactly representable by ReLU DNNs~\citep{yarotsky2016error,liang2016deep,safran2017depth}.

\begin{proof}[Proof of Theorem~\ref{thm:approximation}] 
Given $k \geq 1$ and $w\geq 2$ define $q:=w^k$ and $s_q := \underbrace{h_\a\circ h_\a \circ \ldots \circ h_\a}_{k \textrm{ times}}$ where $\a = (\frac1w, \frac2w, \ldots, \frac{w-1}{w}) \in \Delta_1^{q-1}$. 
Thus, $s_q$ is representable by a ReLU DNN of width $w+1$ and depth $k+1$ by Lemma~\ref{lem:comp-DNN}. In what follows, we want to give a lower bound on the $\ell^1$ distance of $s_q$ from any continuous $p$-piecewise linear comparator $g_p: \R \to \R$. The function $s_q$ contains $\lfloor \frac{q}{2} \rfloor$ triangles of width $\frac{2}{q}$ and unit height. A $p$-piecewise linear function has $p-1$ breakpoints in the interval $[0,1]$. So that in at least $\lfloor \frac{w^k}{2}\rfloor  - (p-1)$ triangles, $g_p$ has to be affine. In the following we demonstrate that inside any triangle of $s_q$, any affine function will incur an $\ell^1$ error of at least $\frac{1}{2w^k}$.
\begin{align*}
\int_{x=\frac{2i}{w^k}}^{\frac{2i+2}{w^k}}&{|s_q(x) - g_p(x)| dx} = \int_{x=0}^{\frac{2}{w^k}}{\left|s_q(x) - (y_1 + (x-0)\cdot \frac{y_2-y_1}{\frac{2}{w^k}-0}) \right| dx}\\
 &= \int_{x=0}^{\frac{1}{w^k}}{\left| xw^k - y_1 - \frac{w^k x}{2} (y_2-y_1) \right| dx}+\int_{x=\frac{1}{w^k}}^{\frac{2}{w^k}}{\left| 2-xw^k - y_1 - \frac{w^k x}{2} (y_2-y_1) \right| dx}\\
 &= \frac{1}{w^k}\int_{z=0}^{1}{\left| z - y_1 - \frac{z}{2} (y_2-y_1) \right| dz}+ \frac{1}{w^k}\int_{z=1}^{2}{\left| 2-z - y_1 - \frac{z}{2} (y_2-y_1) \right| dz}\\
 &= \frac{1}{w^k}\left(-3 + y_1 + \frac{2 y_1^2}{2 + y_1 - y_2} + y_2 + \frac{2 (-2 + y_1)^2}{2 - y_1 + y_2}\right)
\end{align*}
The above integral attains its minimum of $\frac{1}{2w^k}$ at $y_1=y_2=\frac{1}{2}$.
 Putting together, $$\| s_{w^k} - g_p \|_1 \geq \left(\lfloor \frac{w^k}{2} \rfloor - (p-1)\right)\cdot \frac{1}{2w^k} \geq \frac{w^k - 1 -2(p-1)}{4w^k} = \frac{1}{4} - \frac{2p - 1}{4w^k}$$ Thus, for any $\delta > 0$, $$p \leq \frac{w^k -4w^k \delta +1}{2} \implies 2p-1 \leq (\frac{1}{4} - \delta)4w^k \implies \frac{1}{4} - \frac{2p-1}{4w^k} \geq \delta \implies \| s_{w^k} - g_p \|_1 \geq \delta.$$ The result now follows from Lemma~\ref{lem:size-bounds}.
\end{proof}

\subsection{A continuum of hard functions for $\R^n \to \R$ for $n\geq 2$} \label{sec:zon}

One measure of complexity of a family of $\R^n\to \R$ ``hard'' functions represented by ReLU DNNs is the asymptotics of the number of pieces as a function of dimension $n$, depth $k+1$ and size $s$ of the ReLU DNNs.  More precisely, suppose one has a family $\mathcal{H}$ of functions such that for every $n,k,w \in \N$ the family contains at least one $\R^n\to \R$ function representable by a ReLU DNN with depth at most $k+1$ and maximum width at most $w$. The following definition formalizes a notion of complexity for such a $\mathcal{H}$. 

\begin{definition}[$\comp_{\F}(n,k,w)$]
The measure $\comp_{\F}(n,k,w)$ is defined as the maximum number of pieces (see Definition~\ref{def:PWL}) of a $\R^n\to \R$ function from $\F$ that can be represented by a ReLU DNN with depth at most $k+1$ and maximum width at most $w$.
\end{definition}

Similar measures have been studied in previous works~\citep{Mon,Pas,raghu2016expressive}. {The best known families $\F$ are the ones from Theorem~4 of~\citep{Mon} and a mild generalization of Theorem~$1.1$ of~\cite{T2} to $k$ layers of ReLU activations with width $w$; these constructions achieve 
$\bigg(\lfloor(\frac{w}{n})\rfloor\bigg)^{(k-1)n}(\sum_{j=0}^{n}  {w \choose j})$}and $\comp_{\F}(n,k,s) = O(w^k)$, respectively. At the end of this section we would explain the precise sense in which we improve on these numbers. An analysis of this complexity measure is done using integer programming techniques in~\cite{serra2017bounding}.

\begin{definition} Let $\b^1, \ldots, \b^m \in \R^n$. The zonotope formed by $\b^1, \ldots, \b^m \in \R^n$ is defined as $$Z(\b^1, \ldots, \b^m):= \{\lambda_1\b^1 + \ldots + \lambda_m\b^m: -1 \leq \lambda_i \leq 1, \;\; i =1, \ldots, m\}.$$
The set of vertices of $Z(\b^1, \ldots, \b^m)$ will be denoted by $\verts(Z(\b^1, \ldots, \b^m))$. The {\em support function} $\gamma_{Z(\b^1, \ldots, \b^m)}:\R^n \to \R$ associated with the zonotope $Z(\b^1, \ldots, \b^m)$ is defined as $$\gamma_{Z(\b^1, \ldots, \b^m)}(\rr) = \max_{\x \in Z(\b^1, \ldots, \b^m)} \langle \rr, \x \rangle.$$
\end{definition}
\vspace{-3mm}
The following results are well-known in the theory of zonotopes~\citep{ziegler1995lectures}.

\begin{theorem}\label{thm:zonotope-struct} The following are all true.
\begin{enumerate}
\item { $|\verts(Z(\b^1, \ldots, \b^m))| \leq \sum_{i=0}^{n-1}{m-1 \choose i}$.} The set of $(\b^1, \ldots, \b^m) \in \R^n \times \ldots \times \R^n$ such that this {\em does not} hold at equality is a 0 measure set. 
\item $\gamma_{Z(\b^1, \ldots, \b^m)}(\rr) = \max_{\x \in Z(\b^1, \ldots, \b^m)} \langle \rr, \x \rangle = \max_{\x \in \verts(Z(\b^1, \ldots, \b^m))} \langle \rr, \x \rangle,$ and $\gamma_{Z(\b^1, \ldots, \b^m)}$ is therefore a piecewise linear function with $|\verts(Z(\b^1, \ldots, \b^m))|$ pieces.
\item $\gamma_{Z(\b^1, \ldots, \b^m)}(\rr) = |\langle \rr, \b^1\rangle| + \ldots + |\langle \rr, \b^m\rangle|$.
\end{enumerate}
\end{theorem}

\begin{definition}[extremal zonotope set]\label{extremal} The set $S(n,m)$ will denote the set of $(\b^1, \ldots, \b^m) \in \R^n \times \ldots \times \R^n$ such that 
{ $|\verts(Z(\b^1, \ldots, \b^m))| = \sum_{i=0}^{n-1}{m-1 \choose i}$.} $S(n,m)$ is the so-called ``extremal zonotope set'', which is a subset of $\R^{nm}$, whose complement has zero Lebesgue measure in $\R^{nm}$.
\end{definition}

\begin{lemma}\label{lem:zonotope-DNN} Given any $\b^1, \ldots, \b^m \in \R^n$, there exists a 2-layer ReLU DNN with size $2m$ which represents the function $\gamma_{Z(\b^1, \ldots, \b^m)}(\rr)$.
\end{lemma}

\begin{proof}[Proof of Lemma~\ref{lem:zonotope-DNN}] By Theorem~\ref{thm:zonotope-struct}(part 3.), $\gamma_{Z(\b^1, \ldots, \b^m)}(\rr) = |\langle \rr, \b^1\rangle| + \ldots + |\langle \rr, \b^m\rangle|$. It suffices to observe $$|\langle \rr, \b^1\rangle| + \ldots + |\langle \rr, \b^m\rangle| = \max\{\langle \rr, \b^1\rangle, -\langle \rr, \b^1\rangle\} + \ldots + \max\{\langle \rr, \b^m\rangle, -\langle \rr, \b^m\rangle\}.$$
\end{proof}

\bigskip
\begin{figure}[t!]
\begin{subfigure}{.33\textwidth}
  \centering
  \includegraphics[width=.8\linewidth]{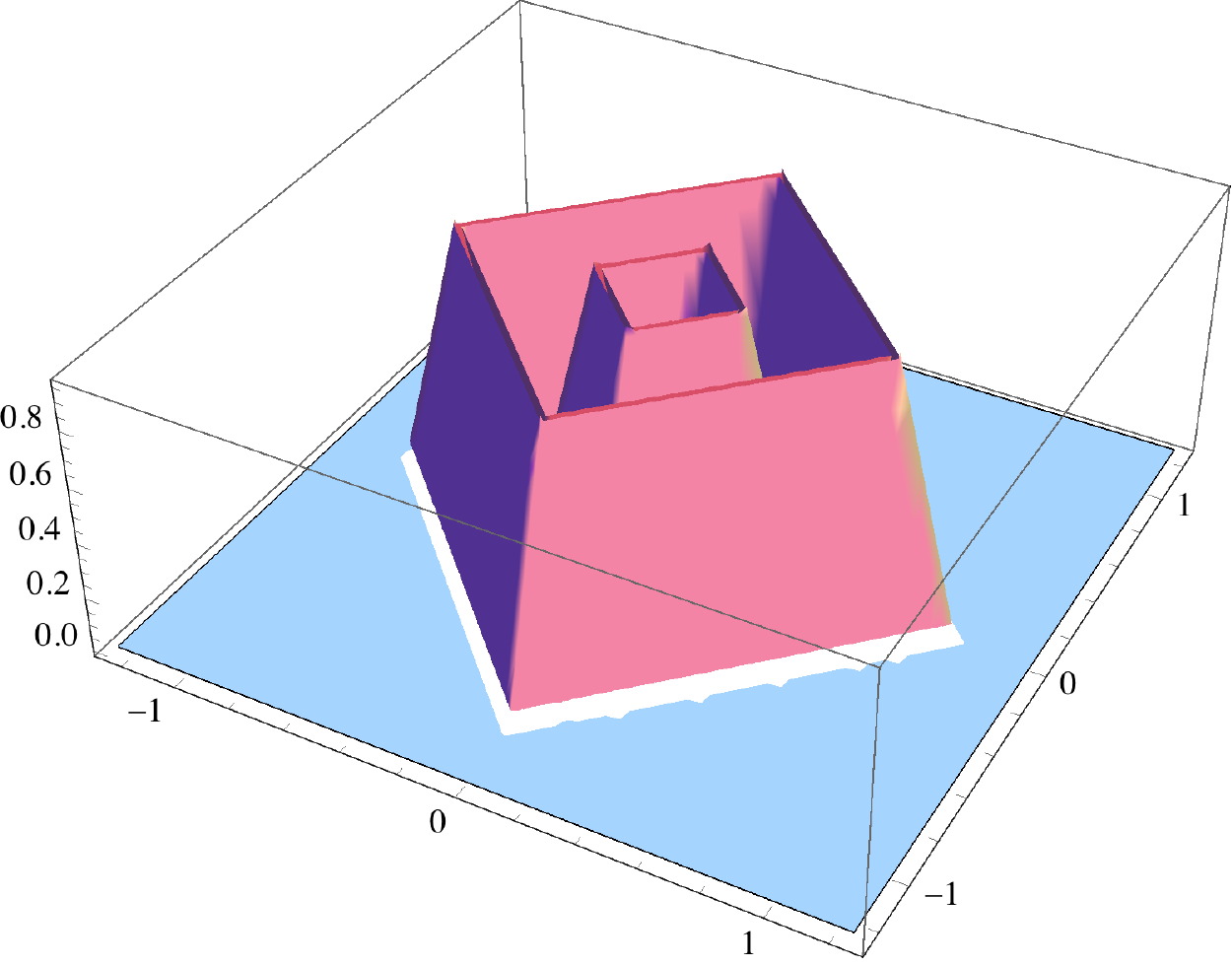}
  \caption{$H_{\frac{1}{2},\frac{1}{2}} \circ N_{\ell_1}$}
  \label{fig:sfig1}
\end{subfigure}%
\begin{subfigure}{.33\textwidth}
  \centering
  \includegraphics[width=.8\linewidth]{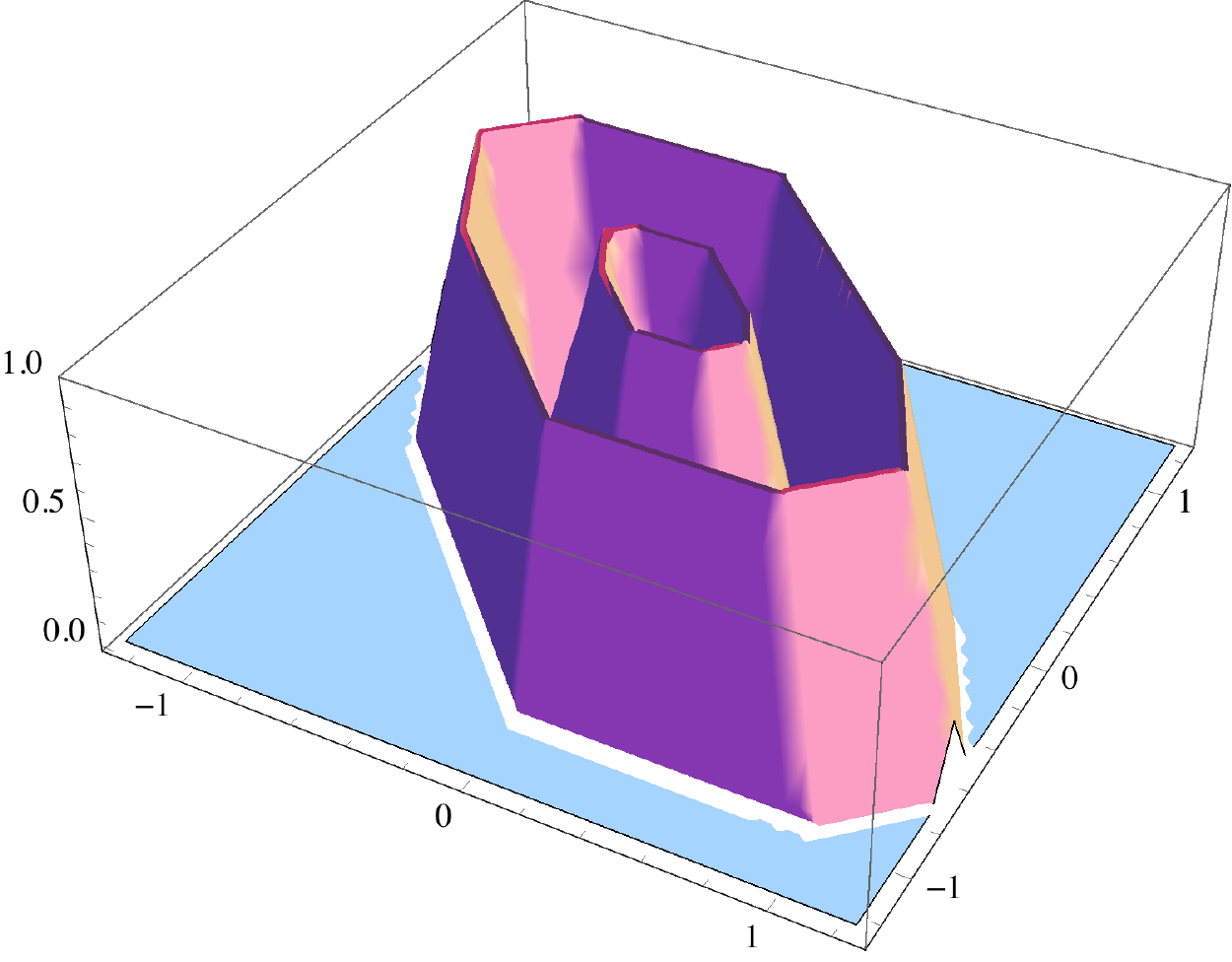}
  \caption{$H_{\frac{1}{2},\frac{1}{2}} \circ \gamma_{Z(\b^1,\b^2,\b^3,\b^4)}$}
  \label{fig:sfig2}
\end{subfigure}
\begin{subfigure}{.33\textwidth}
  \centering
  \includegraphics[width=.8\linewidth]{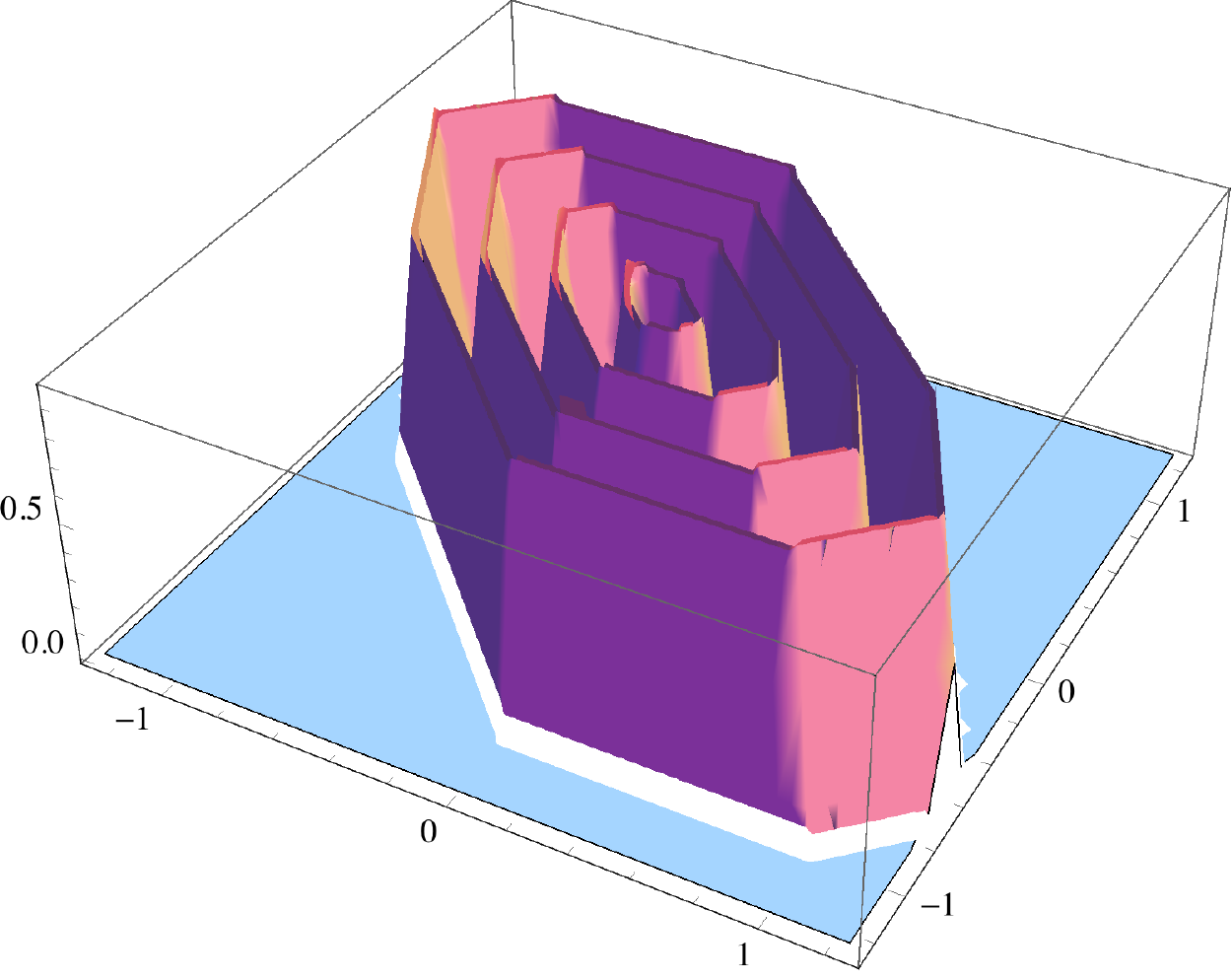}
  \caption{$H_{\frac{1}{2},\frac{1}{2},\frac{1}{2}} \circ \gamma_{Z(\b^1,\b^2,\b^3,\b^4)}$}
  \label{fig:sfig3}
\end{subfigure}
\caption{We fix the $\a$ vectors for a two hidden layer $\R \rightarrow \R$ hard function as $\a ^1 = \a ^2 = (\frac{1}{2}) \in \Delta ^1_1$ Left: A specific hard function induced by $\ell_1$ norm: $\operatorname{ZONOTOPE}_{2,2,2}^{2}[\a ^1, \a ^2 , \b ^1, \b ^2]$ where $\b ^1 = (0,1)$ and $\b ^2 = (1,0)$. Note that in this case the function can be seen as a composition of $H_{\a ^1, \a ^2}$ with $\ell_1$-norm $N_{\ell_1}(x):=\|x\|_1 = \gamma_{Z\left((0,1),(1,0)\right)}$. Middle: A typical hard function $\operatorname{ZONOTOPE}_{2,2,4}^{2}[\a ^1 , \a ^2,\c ^1,\c ^2,\c ^3,\c ^4]$ with generators $\c ^1 = (\frac{1}{4},\frac{1}{2}), \c ^2=(-\frac{1}{2},0), \c ^3=(0,-\frac{1}{4})$ and $\c ^4=(-\frac{1}{4},-\frac{1}{4})$. Note how increasing the number of zonotope generators makes the function more complex. Right: A \textit{harder} function from $\operatorname{ZONOTOPE}_{3,2,4}^{2}$ family with the same set of generators $\c_1,\c_2,\c_3,c_4$ but one more hidden layer $(k=3)$. Note how increasing the depth make the function more complex. (For illustrative purposes we plot only the part of the function which lies above zero.)}
\label{fig:fig}
\end{figure}


\begin{prop}\label{prop:zonotope-main} Given any tuple $(\b^1, \ldots, \b^m) \in S(n,m)$ and any point $$(\a^1, \ldots, \a^k) \in \bigcup_{M > 0}\underbrace{(\Delta^{w-1}_M \times \Delta^{w-1}_M\times \ldots \times \Delta^{w-1}_M)}_{k \textrm{ times}}, \vspace{-1mm}$$ the function $\operatorname{ZONOTOPE}_{k,w,m}^{n}[\a^1, \ldots, \a^k, \b^1, \ldots, \b^m] := H_{\a^1, \ldots, \a^k}\circ\gamma_{Z(\b^1, \ldots, \b^m)}$ has $(m-1)^{n-1}w^k$ pieces and it can be represented by a $k+2$ layer ReLU DNN with size $2m + wk$.
\end{prop}

\begin{proof} [Proof of Proposition~\ref{prop:zonotope-main}] The fact that $\operatorname{ZONOTOPE}_{k,w,m}^{n}[\a^1, \ldots, \a^k, \b^1, \ldots, \b^m]$ can be represented by a $k+2$ layer ReLU DNN with size $2m + wk$ follows from Lemmas~\ref{lem:zonotope-DNN} and~\ref{lem:comp-DNN}. The number of pieces follows from the fact that $\gamma_{Z(\b^1, \ldots, \b^m)}$ has $\sum_{i=0}^{n-1}{m-1 \choose i}$ distinct linear pieces by parts 1. and 2. of Theorem~\ref{thm:zonotope-struct}, and $H_{\a^1, \ldots, \a^k}$ has $w^k$ pieces by Lemma~\ref{lem:construct-hard}.
\end{proof}

Finally, we are ready to state the main result of this section.

\begin{theorem}\label{thm:high-to-low-n} For every tuple of natural numbers $n,k,m \geq 1$ and $w \geq 2$, there exists a family of $\R^n\to \R$ functions, which we call $\operatorname{ZONOTOPE}^n_{k,w,m}$ with the following properties:
\begin{enumerate}
\item[(i)] Every $f \in \operatorname{ZONOTOPE}^n_{k,w,m}$ is representable by a ReLU DNN of depth $k+2$ and size $2m + wk$, and has  {$\left ( \sum_{i=0}^{n-1}{m-1 \choose i} \right ) w^k$ pieces}.
\item[(ii)] Consider any $f \in \operatorname{ZONOTOPE}^n_{k,w,m}$. If $f$ is represented by a $(k'+1)$-layer DNN for any $k' \leq k$, then this $(k'+1)$-layer DNN has size at least $\max\left\{\frac12(k'w^{\frac{k}{k'n}})\cdot(m-1)^{(1-\frac{1}{n})\frac{1}{k'}} - 1\;\;, \;\;\frac{w^{\frac{k}{k'}}}{n^{1/k'}}k'\right\} \vspace{-2mm}$. 
\item[(iii)] The family $\operatorname{ZONOTOPE}^n_{k,w,m}$ is in one-to-one correspondence with $$S(n,m)\times \bigcup_{M > 0}\underbrace{(\Delta^{w-1}_M \times \Delta^{w-1}_M\times \ldots \times \Delta^{w-1}_M)}_{k \textrm{ times}}. \vspace{-3mm}$$
\end{enumerate}
\end{theorem}

\begin{proof}[Proof of Theorem~\ref{thm:high-to-low-n}] Follows from Proposition~\ref{prop:zonotope-main} (and invoking Lemma \ref{lem:size-bounds} to get the size lowerbounds).
\end{proof}

\paragraph{Comparison to the results in~\citep{Mon}}~\\

\emph{Firstly} we note that the construction in~\cite{Mon} requires all the hidden layers to have width at least as big as the input dimensionality $n$. In contrast, we do not impose such restrictions and the network size in our construction is independent of the input dimensionality. Thus our result probes networks with bottleneck architectures whose complexity cant be seen from their result. 

\emph{Secondly}, in terms of our complexity measure, there seem to be regimes where our bound does better. One such regime, for example, is when $n \leq w < 2n$ and $k \in \Omega(\frac{n}{\log(n)})$, by setting in our construction $m < n$.

\emph{Thirdly}, it is not clear to us whether the construction in~\cite{Mon} gives a smoothly parameterized family of functions other than by introducing small perturbations of the construction in their paper. In contrast, we have a smoothly parameterized family which is in one-to-one correspondence with a well-understood manifold like the higher-dimensional torus.

\section {Training 2-layer $\mathbb{R}^n \rightarrow \mathbb{R}$ ReLU DNNs to global optimality}\label{sec:training}

In this section we consider the following empirical risk minimization problem. 
Given $D$ data points $(x_i, y_i) \in \R^n \times \R$, $i =1, \ldots, D$, find the function $f$ represented by 2-layer $\R^n\to \R$ ReLU DNNs of width $w$, that minimizes the following optimization problem,

\vspace{-2mm}
\begin{equation}\label{eq:train}
\min_{f\in {\cal F}_{\{ n,w,1 \}}}{\frac1D \sum_{i=1}^D\ell(f(x_i), y_i)} \quad \equiv \min_{T_1 \in {\cal A}_n^w, \ T_2 \in {\cal L}_w^1}{\frac1D \sum_{i=1}^D \ell\big(\;T_2(\sigma(T_1(x_i))), y_i\;\big)}
\end{equation}

where $\ell:\R \times \R \to \R$ is a convex {\em loss function} (common loss functions are the squared loss, $\ell(y,y') = (y-y')^2$, and the hinge loss function given by $\ell(y,y')=\max\{0,1-yy'\}$). Our main result of this section gives an algorithm to solve the above empirical risk minimization problem to global optimality.

\begin{theorem}\label{thm:optimal-training}
There exists an algorithm to find a global optimum of Problem~\ref{eq:train} in time \\ $O(2^w(D)^{nw}\poly(D,n,w))$. Note that the running time $O(2^w(D)^{nw}\poly(D,n,w))$ is polynomial in the data size $D$ for fixed $n,w$. 
\end{theorem}

\noindent{\bf Proof Sketch:} 
Before giving the full proof of Theorem~\ref{thm:optimal-training} below here we first provide a sketch of it. When the empirical risk minimization problem is viewed as an optimization problem in the space of weights of the ReLU DNN, it is a nonconvex, quadratic problem. However, one can instead search over the space of functions representable by 2-layer DNNs by writing them in the form similar to~\eqref{eq:hinged-hyperplane}. This breaks the problem into two parts: a combinatorial search and then a convex problem that is essentially linear regression with linear inequality constraints. This enables us to guarantee global optimality.

 \begin{algorithm}
   \caption{Empirical Risk Minimization\label{alg:erm}}
    \begin{algorithmic}[1]
{\small \tt 
      \Function{ERM}{${\cal D}$}\Comment{\textrm{Where ${\cal D}=\{ (x_i,y_i) \}_{i=1}^D \subset \R^n\times \R$}}

        \State ${\cal S}=\{ +1, -1 \}^w$\Comment{\textrm{All possible instantiations of top layer weights}}
        \State ${\cal P}^i = \{ (P_+^i, P_-^i ) \}, \ i=1,\ldots,w$\Comment{\textrm{All possible partitions of data into two parts}}
        \State ${\cal P} = {\cal P}^1\times {\cal P}^2 \times \cdots \times {\cal P}^w$
        \State $\text{count} = 1$\Comment{\textrm{Counter}}
        \For{$s \in {\cal S}$}
            \For{$\{(P_+^i, P_-^i )\}_{i=1}^w \in {\cal P}$}
               \State $\text{loss(count)} = \left\{\min_{\tilde{a},\tilde{b}}{\sum_{j=1}^D \sum_{i: j \in P^i_+} \ell(y_j,s_i(\tilde a^i\cdot x_j + \tilde b_i))}{\begin{array}{c}
\tilde a^i\cdot x_j + \tilde b_i \leq 0\quad \forall j \in P^i_-\\
\tilde a^i\cdot x_j + \tilde b_i \geq 0\quad \forall j \in P^i_+
\end{array}}\right.$
\State $\text{count} ++$
        	\EndFor
            \State $\text{OPT}=\argmin{\text{loss(count)}}$
        \EndFor
        \State {return $\{ \tilde\a \}, \{ \tilde\b \}, s$ corresponding to OPT's iterate}
       \EndFunction
}
\end{algorithmic}
\end{algorithm}

Let $T_1(x) = Ax + b$ and $T_2(y) = a'\cdot y$ for $A \in \R^{w\times n}$ and $b,a' \in \R^w$. If we denote the $i$-th row of the matrix $A$ by $a^i$, and write $b_i, a'_i$ to denote the $i$-th coordinates of the vectors $b,a'$ respectively, due to homogeneity of ReLU gates, the network output can be represented as $$f(x) = \sum_{i=1}^w a'_i\max\{0,a^i\cdot x + b_i\}=\sum_{i=1}^w s_i\max\{0,\tilde a^i\cdot x + \tilde b_i\}.$$ where $\tilde a^i \in \R^n$, $\tilde b_i \in \R$ and $\s_i \in \{-1,+1\}$ for all $i=1, \ldots, w$. 

For any hidden node $i \in \{ 1 \ldots, w\}$, the pair $(\tilde a^i, \tilde b_i)$ induces a partition ${\cal P}^i:=(P^i_+, P^i_-)$ on the dataset, given by $P^i_- = \{j : \tilde{a}^i \cdot x_j + \tilde{b_i} \leq 0 \}$ and $P^i_+ = \{1,\ldots,D  \} \backslash P^i_-$. Algorithm~\ref{alg:erm} proceeds by generating all combinations of the partitions ${\cal P}^i$ as well as the top layer weights $s \in \{+1,-1\}^w$, and minimizing the loss $\sum_{j=1}^D \sum_{i: j \in P^i_+} \ell(s_i(\tilde a^i\cdot x_j + \tilde b_i), y_j)$ subject to the constraints $\tilde a^i\cdot x_j + \tilde b_i \leq 0\quad \forall j \in P^i_-$ and $\tilde a^i\cdot x_j + \tilde b_i \geq 0\quad \forall j \in P^i_+$ which are imposed for all $i=1, \ldots, w$,which is a convex program. 




\begin{proof}[Proof of Theorem~\ref{thm:optimal-training}] 
Let $\ell:\R\to \R$ be any convex loss function, and let $(x_1, y_1), \ldots, (x_D, y_D) \in \R^n\times \R$ be the given $D$ data points. As stated in~\eqref{eq:train}, the problem requires us to find an affine transformation $T_1:\R^n \to \R^w$ and a linear transformation $T_2:\R^w \to \R$, so as to minimize the empirical loss as stated in~\eqref{eq:train}. Note that $T_1$ is given by a matrix $A \in \R^{w\times n}$ and a vector $b\in \R^w$ so that $T(x) = Ax + b$ for all $x\in \R^n$. Similarly, $T_2$ can be represented by a vector $a' \in \R^w$ such that $T_2(y) = a'\cdot y$ for all $y\in \R^w$. If we denote the $i$-th row of the matrix $A$ by $a^i$, and write $b_i, a'_i$ to denote the $i$-th coordinates of the vectors $b,a'$ respectively, we can write the function represented by this network as $$f(x) = \sum_{i=1}^w a'_i\max\{0,a^i\cdot x + b_i\} = \sum_{i=1}^w \operatorname{sgn}(a'_i)\max\{0,(|a'_i|a^i)\cdot x + |a'_i|b_i\}.$$ 

In other words, the family of functions over which we are searching is of the form \begin{equation}\label{eq:form-train}f(x) =\sum_{i=1}^w s_i\max\{0,\tilde a^i\cdot x + \tilde b_i\}\end{equation}
where $\tilde a^i \in \R^n$, $b_i \in \R$ and $\s_i \in \{-1,+1\}$ for all $i=1, \ldots, w$.

We now make the following observation. For a given data point $(x_j, y_j)$ if $\tilde a^i\cdot x_j + \tilde b_i \leq 0$, then the $i$-th term of~\eqref{eq:form-train} does not contribute to the loss function for this data point $(x_j,y_j)$. Thus, for every data point $(x_j, y_j)$, there exists a set $S_j \subseteq \{1, \ldots, w\}$ such that $f(x_j) = \sum_{i \in S_j}s_i (\tilde a^i\cdot x_j + \tilde b_i)$. In particular, if we are given the set $S_j$ for $(x_j,y_j)$, then the expression on the right hand side of~\eqref{eq:form-train} reduces to a linear function of $\tilde a^i, \tilde b_i$. For any fixed $i \in \{1, \ldots, w\}$, these sets $S_j$ induce a partition of the data set into two parts. In particular, we define $P^i_+ := \{j : i \in S_j\}$ and $P^i_- := \{1, \ldots, D\} \setminus P^i_+$. Observe now that this partition is also induced by the hyperplane given by $\tilde a^i, \tilde b_i$: $P^i_+ = \{j : \tilde a^i\cdot x_j + \tilde b_i > 0 \}$ and $P^i_+ = \{j : \tilde a^i\cdot x_j + \tilde b_i \leq 0 \}$. Our strategy will be to {\em guess} the partitions $P^i_+, P^i_-$ for each $i=1, \ldots, w$, and then do linear regression with the constraint that regression's decision variables $\tilde a^i, \tilde b_i$ induce the guessed partition.

More formally, the algorithm does the following. For each $i=1, \ldots, w$, the algorithm guesses a partition of the data set $(x_j,y_j)$, $j=1, \ldots, D$ by a hyperplane. Let us label the partitions as follows $(P^i_+, P^i_-)$, $i=1, \ldots, w$. So, for each $i=1, \ldots, w$, $P^i_+ \cup P^i_- = \{1, \ldots, D\}$, $P^i_+$ and $P^i_-$ are disjoint, and there exists a vector $c \in \R^n$ and a real number $\delta$ such that $P^i_- = \{j : c\cdot x_j + \delta \leq 0 \}$ and $P^i_+ = \{j : c\cdot x_j + \delta > 0 \}$. Further, for each $i=1,\ldots, w$ the algorithm selects a vector $s$ in $\{+1,-1\}^w$. 

For a fixed selection of partitions $(P^i_+, P^i_-)$, $i=1, \ldots, w$ and a vector $s$ in $\{+1,-1\}^w$, the algorithm solves the following convex optimization problem with decision variables $\tilde a^i \in \R^n$, $\tilde b_i \in\R$ for $i=1, \ldots, w$ (thus, we have a total of $(n+1)\cdot w$ decision variables). The feasible region of the optimization is given by the constraints
\begin{equation}\label{eq:constraints-1}
\begin{array}{c}
\tilde a^i\cdot x_j + \tilde b_i \leq 0\quad \forall j \in P^i_-\\
\tilde a^i\cdot x_j + \tilde b_i \geq 0\quad \forall j \in P^i_+
\end{array}
\end{equation}
which are imposed for all $i=1, \ldots, w$. Thus, we have a total of $D\cdot w$ constraints. Subject to these constraints we minimize the objective $\sum_{j=1}^D \sum_{i: j \in P^i_+} \ell(s_i(\tilde a^i\cdot x_j + \tilde b_i), y_j).$
Assuming the loss function $\ell$ is a convex function in the first argument, the above objective is a convex function. Thus, we have to minize a convex objective subject to the linear inequality constraints from~\eqref{eq:constraints-1}. 

We finally have to count how many possible partitions $(P^i_+,P^i_-)$ and vectors $s$ the algorithm has to search through. It is well-known~\citep{matousek2002lectures} that the total number of possible hyperplane partitions of a set of size $D$ in $\R^n$ is at most $2{D\choose n} \leq D^n$ whenever $n \geq 2$. Thus with a guess for each $i=1, \ldots, w$, we have a total of at most $D^{nw}$ partitions. There are $2^w$ vectors $s$ in $\{-1, +1\}^w$. This gives us a total of $2^wD^{nw}$ guesses for the partitions $(P^i_+,P^i_-)$ and vectors $s$. For each such guess, we have a convex optimization problem with $(n+1)\cdot w$ decision variables and $D\cdot w$ constraints, which can be solved in time $\poly(D,n,w)$. Putting everything together, we have the running time claimed in the statement.

The above argument holds only for $n\geq 2$, since we used the inequality $2{D\choose n} \leq D^n$ which only holds for $n\geq 2$. For $n=1$, a similar algorithm can be designed, but one which uses the characterization achieved in Theorem~\ref{thm:1-dim-2-layer}.

Let $\ell:\R\to \R$ be any convex loss function, and let $(x_1, y_1), \ldots, (x_D, y_D) \in \R^2$ be the given $D$ data points. Using Theorem~\ref{thm:1-dim-2-layer}, to solve problem~\eqref{eq:train} it suffices to find a $\R \to \R$ piecewise linear function $f$ with $w$ pieces that minimizes the total loss. In other words, the optimization problem~\eqref{eq:train} is equivalent to the problem \begin{equation}\label{eq:transformed}\min\left\{\sum_{i=1}^D \ell(f(x_i), y_i): f \textrm{ is piecewise linear with }w\text{ pieces}\right\}.\end{equation}

We now use the observation that fitting piecewise linear functions to minimize loss is just a step away from linear regression, which is a special case where the function is contrained to have exactly one affine linear piece. Our algorithm will first guess the optimal partition of the data points such that all points in the same class of the partition correspond to the same affine piece of $f$, and then do linear regression in each class of the partition. Alternatively, one can think of this as guessing the interval $(x_i, x_{i+1})$ of data points where the $w-1$ breakpoints of the piecewise linear function will lie, and then doing linear regression between the breakpoints.

More formally, we parametrize piecewise linear functions with $w$ pieces by the $w$ slope-intercept values $(a_1, b_1), \ldots, (a_2, b_2), \ldots, (a_w,b_w)$ of the $w$ different pieces. This means that between breakpoints $j$ and $j+1$, $1 \leq j \leq w-2$, the function is given by $f(x) = a_{j+1}x + b_{j+1},$ and the first and last pieces are $a_1x+b_1$ and $a_wx +b_w$, respectively. 

Define $\mathcal{I}$ to be the set of all $(w-1)$-tuples $(i_1, \ldots, i_{w-1})$ of natural numbers such that $1 \leq i_1 \leq \ldots \leq i_{w-1} \leq D$. Given a fixed tuple $I =(i_1, \ldots, i_{w-1}) \in \mathcal{I}$, we wish to search through all piecewise linear functions whose breakpoints, in order, appear in the intervals $(x_{i_1}, x_{i_1+1}), (x_{i_2}, x_{i_2+1})$, $\ldots, (x_{i_{w-1}}, x_{i_{w-1}+1})$.  Define also $\mathcal{S} = \{-1,1\}^{w-1}$. Any $S\in \mathcal{S}$ will have the following interpretation: if $S_j = 1$ then $a_j \leq a_{j+1}$, and if $S_j = -1$ then $a_j \geq a_{j+1}$. Now for every $I \in \mathcal{I}$ and $S \in \mathcal{S}$, requiring a piecewise linear function that respects the conditions imposed by $I$ and $S$ is easily seen to be equivalent to imposing the following linear inequalities on the parameters $(a_1, b_1), \ldots, (a_2, b_2), \ldots, (a_w,b_w)$:

\begin{equation}\label{eq:constraints}
\begin{array}{r}
S_j(b_{j+1}-b_j - (a_{j} - a_{j+1})x_{i_j}) \geq 0\\
S_j(b_{j+1}-b_j - (a_{j} - a_{j+1})x_{i_{j}+1}) \leq 0\\ 
S_j(a_{j+1} - a_j) \geq 0
\end{array}
\end{equation}
Let the set of piecewise linear functions whose breakpoints satisfy the above be denoted by $\pwl^1_{I,S}$ for $I \in \mathcal{I}, S \in \mathcal{S}$.

Given a particular $I \in \mathcal{I}$, we define $$\begin{array}{lc} D_1 := \{x_i: i \leq i_1\}, & \\ D_j := \{x_i: i_{j-1}< i \leq i_1\} & j=2, \ldots, w-1, \\ D_w := \{x_i: i > i_{w-1}\}& \end{array}.$$
Observe that \begin{equation}\label{eq:fixed-I}\min\{\sum_{i=1}^D \ell(f(x_i) - y_i): f\in \pwl^1_{I,S}\} = \min\{\sum_{j=1}^{w}\bigg(\sum_{i \in D_j} \ell(a_j\cdot x_i + b_j - y_i)\bigg): (a_j,b_j) \textrm{ satisfy~\eqref{eq:constraints}}\}\end{equation}

The right hand side of the above equation is the problem of minimizing a convex objective subject to linear constraints. Now, to solve~\eqref{eq:transformed}, we need to simply solve the problem~\eqref{eq:fixed-I} for all $I \in \mathcal{I}, S \in \mathcal{S}$ and pick the minimum. Since $|\mathcal{I}| = {D\choose w} = O(D^w)$ and $|\mathcal{S}| = 2^{w-1}$ we need to solve $O(2^w\cdot D^w)$ convex optimization problems, each taking time $O(\poly(D))$. Therefore, the total running time is $O((2D)^w\poly(D))$.
\end{proof}

\subsection{Discussion on the complexity of solving ERM on deep-nets}
The running time of the algorithm (Algorithm~\ref{alg:erm}) that we gave above to find the exact global minima of a two layer ReLU-DNN is exponential in the input dimension $n$ and the number of hidden nodes $w$. The exponential dependence on $n$ can not be removed unless $P=NP$; see ~\cite{shalev2014understanding,blum1992training,dasgupta1995complexity,dey2018approximation}. However, we are not aware of any complexity results which would rule out the possibility of an algorithm which trains to global optimality in time that is polynomial in the data size and/or the number of hidden nodes, assuming that the input dimension is a fixed constant.
Resolving this dependence on network size would be another step towards clarifying the theoretical complexity of training ReLU DNNs and is a good open question for future research, in our opinion. Thus our training result of solving the ERM on depth $2$ nets in  time polynomial in the number of data points is a step towards resolving this gap in the complexity literature. 

A related result for {\em improperly} learning ReLUs has been recently obtained in ~\cite{goel2016reliably}. In contrast, our algorithm returns a ReLU DNN from the class being learned. Another difference is that their result considers the notion of {\em reliable learning} as opposed to the empirical risk minimization objective considered in~\eqref{eq:train} for which we give a quick definition below, 

\begin{definition}
Suppose distribution ${\cal D}$ is supported on $X \times [0,1]$. For $[0,1] \subseteq Y'$ let $h : X \rightarrow Y'$ be some function and let $\ell : Y' \times [0,1] \rightarrow \R^+$ be a loss function. The we define two notions of expected loss, 

\begin{align*}
L_{=0}(h, {\cal D}) &= \mathbb{P}_{(\x,y) \sim {\cal D}} [h(\x) \neq 0 \text{ and } y =0 ]\\
L_{>0}(h, {\cal D}) &= \mathbb{E}_{(\x,y) \sim {\cal D}} [\ell (h(\x),y). {\mathfrak 1}_{y >0} ]
\end{align*}
We say that a concept class ${\cal C} \subseteq [0,1]^X$ is ``reliably agnostically learnable with respect to a loss function, $\ell : Y' \times [0,1] \rightarrow \mathbb{R}^+$'' (where $[0,1] \subseteq Y'$) if for every $\epsilon, \delta >0$ there exists a learning algorithm which satisfies the following :
~\\ \\
That $\forall$ distributions ${\cal D}$ over $X \times [0,1]$ given access to examples drawn from ${\cal D}$, the algorithm outputs a hypothesis $h : X \rightarrow Y$ such that, 

\[ L_{=0}(h, {\cal D}) \leq \epsilon \text{ and } L_{>0}(h, {\cal D}) \leq \epsilon + \min_{c \in {\cal C}'({\cal D})} L_{>0}(c) \]
where, 
\[ {\cal C}'({\cal D}) = \{ c \in {\cal C} \vert L_{=0}(c,{\cal D}) = 0 \}\]
Further if if $X \subseteq \mathbb{R}^n$ and $s$ is a parameter that captures the representation complexity (i.e description length) of concepts in $c \in {\cal C}$ then we say that ${\cal C}$ is ``efficiently reliably agnostically learnable to error $\epsilon$" if the running time of the above algorithm that is supposed to exist is $poly(n,s,\frac 1 \delta)$. 
\end{definition}

Asking or $L_{=0}(h, {\cal D})$ to be low captures mathematically the idea of trying to minimize the rate of ``false positives''.

Perhaps a big breakthrough would be to get optimal training algorithms for DNNs with two or more hidden layers and this seems like a substantially harder nut to crack. We end this discussion by pointing out some recent progress towards that which has been made in \cite{boob2018complexity}.  


\section{Understanding neural functions over Boolean inputs}

The classic paper ~\cite{maass1997bounds}, established complexity results for the entire class of functions represented by circuits where the gates can come from a very general family while the inputs are restricted to discrete domains. This is complemented by papers that study a very specific family of gates such as the sigmoid gate or the LTF gate ($\R \ni y \mapsto \mathbf{1}_{y \geq 0}$) ~\citep{impagliazzo1997size}, ~\citep{siu1994rational,sherstov2007powering,krause1994computational}, ~\citep{buhrman2007computation,sherstov2009separating,razborov2010sign, bun2016improved}. Many associated results can also be found in these reviews like ~\cite{lee2009lower} and  \cite{razborov1992small}. Recent circuit complexity results in~\cite{kane2016super},  ~\cite{tamaki2016satisfiability}, ~\cite{chen2016average}, ~\cite{kabanets2017polynomial} stand out as significant improvements over known lower (and upper) bounds on circuit complexity with threshold gates. The results of ~\cite{maass1997bounds} also show that very general families of neural networks can be converted into circuits with only LTF gates with at most a constant factor blow up in depth and polynomial blow up in size of the circuits.
~\\ \\
Some of the  prior results which apply to general gates, such as the ones in~\cite{maass1997bounds}, also apply to ReLU gates, because those results apply to gates that compute a piecewise polynomial function (ReLU is a piecewise linear function with only two pieces). However, as witnessed by results on LTF gates, one can usually make much stronger claims about specific classes of gates. The main focus of this work is to study circuits computing Boolean functions mapping $\{-1,1\}^m \rightarrow \{-1,1\}$ which use ReLU gates in their intermediate layers, and have an LTF gate at the output node (to ensure that the output is in $\{-1,1\}$). We remark that using an LTF gate at the output node while allowing more general analog gates in the intermediate nodes is a standard practice when studying the Boolean complexity of analog gates (see, for example,~\cite{maass1997bounds}).
~\\ \\
Other than~\cite{williams2018limits}, we are not aware of an analysis of lower bounds for ReLU circuits when applied to only Boolean inputs. In contrast, there has been recent work on the analysis of such circuits when viewed as a function from $\R^n$ to $\R$ (i.e., allowing real inputs and output). From ~\cite{eldan2016power} and  ~\cite{daniely2017depth} (with restrictions on the domain and the weights) we know of (super-)exponential lowerbounds on the size of Sum-of-ReLU circuits for certain easy Sum-of-ReLU-of-ReLU functions . 
Depth v/s size tradeoffs 
for such circuits have recently also been studied in ~\cite{telgarsky2016benefits,hanin2017universal,liang2016deep, yarotsky2016error,safran2016depth} and in this chapter so far. But to the best of our knowledge no lowerbounds scaling exponentially with the dimension are known for analog deep neural networks of depths more than $2$.
~\\ \\
In what follows, the {\em depth} of a circuit will be the length of the longest path from the output node to an input variable, and the {\em size} of a circuit will be the total number of gates in the circuit. We will also use the notation {\em Sum-of-ReLU} to refer to circuits whose inputs feed into a single layer of ReLU gates, whose outputs are combined into a weighted sum to give the final output. Similarly, Sum-of-ReLU-of-ReLU denotes the circuit with depth 3, where the output node is a simple weighted sum, and the intermediate gates are all ReLU gates in the two ``hidden" layers. We analogously define Sum-of-LTF, LTF-of-LTF, LTF-of-ReLU, LTF-of-LTF-of-LTF, LTF-of-ReLU-of-ReLU and so on. We will also use the notation LTF-of-(ReLU)$^k$ for a circuit of the form LTF-of-ReLU-of-RELU-$\ldots$-ReLU with $k\geq 1$ levels of ReLU gates.

\subsection{Statement and discussion of our results over Boolean inputs} 

\paragraph{Boolean v/s real inputs.} We begin our study with the following observation which shows that ReLU circuits have markedly different behaviour when the inputs are restricted to be Boolean, as opposed to arbitrary real inputs. Since AND and OR gates can both be implemented by ReLU gates, it follows that {\em any} Boolean function can be implemented by a ReLU-of-ReLU circuit. In fact, it is not hard to show something slightly stronger:

\begin{lemma}
Any function $f : \{-1,1\}^n \rightarrow \mathbb{R}$ can be implemented by a Sum-of-ReLU circuit using at most $\min\{2^n,\sum_{\hat{f}(S) \neq 0} \vert S \vert\}$ number of ReLU gates, where $\hat f(S)$ denotes the Fourier coefficient of $f$ for the set $S \subseteq \{1, \ldots, n\}$.
\end{lemma}

The Lemma follows by observing that the indicator functions of each vertex of the Boolean hypercube $\{-1,1\}^n$ can be implemented by a single ReLU gate, and the parity function on $k$ variables can be implemented by $k$ ReLU gates (see Appendix~\ref{sec:Parity-ReLU}). Thus, if one does not restrict the size of the circuit, then Sum-of-ReLU circuits can represent any pseudo-Boolean function. In contrast, we will now show that if one allows real inputs, then there exist functions with just 2 inputs (i.e., $n=2$) which cannot be represented by any Sum-of-ReLU circuit, no matter how large.

\begin{prop}\label{thm:max-0-x-y} The function $\max\{0,x_1,x_2\}$ cannot be computed by any Sum-of-ReLU circuit, no matter how many ReLU gates are used. It can be computed by a Sum-of-ReLU-of-ReLU circuit.
\end{prop}

The first part of the above proposition (the impossibility result) is proved in Appendix~\ref{sec:proof-max-0-x-y}. The second part follows from Lemma \ref{cor:all-pwl-are-dnn}, which stated that any $\R^n \to \R$ function that can be implemented by a circuit of ReLU gates, can always be implemented with at most $\lceil \log(n+1) \rceil$ layers of ReLU gates (with a weighted Sum to give the final output). 

\paragraph{Restricting to Boolean inputs.} From this point on, we will focus entirely on the situation where the inputs to the circuits are restricted to $\{-1,1\}$. One motivation behind our results is the desire to understand the strength of the ReLU gates vis-a-vis LTF gates. It is not hard to see that any circuit with LTF gates can be simulated by a circuit with ReLU gates with at most a constant blow-up in size (because a single LTF gate can be simulated by 2 ReLU gates when the inputs are a discrete set -- see Appendix~\ref{sec:LTF-ReLU}). The question is whether ReLU gates can do significantly better than LTF gates in terms of depth and/or size.
~\\ \\
A quick observation is that Sum-of-ReLU circuits can be linearly (in the dimension $n$) smaller than Sum-of-LTF circuits. More precisely, 

\begin{prop} The function $f:\{-1,1\}^n \to \R$ given by $f(x) =\sum_{i=1}^n 2^i \big(\frac{1+x_i}{2}\big)$ can be implemented by a Sum-of-ReLU circuit with 2 ReLU gates, and any Sum-of-LTF that implements $f$ needs $\Omega(n)$ gates.
\end{prop}

The above result follows from the following two facts: 1) any linear function is implementable by 2 ReLU gates, and 2) any Sum-of-LTF circuit with $w$ LTF gates gives a piecewise constant function that takes at most $2^w$ different values. Since $f$ takes $2^n$ different values (it evaluates every vertex of the Boolean hypercube to the corresponding natural number expressed in binary), we need $w \geq n$ gates.

In the context of these preliminary results, we now state our main contributions. 
For the next result we recall the definition of the Andreev function ~\citep{andreev1987one} which has previously many times been used to prove computational lower bounds ~\citep{paterson1993shrinkage,impagliazzo1988decision,impagliazzo2012pseudorandomness}.

\begin{definition}[{\bf Andreev's function}]\label{Andreev}
The Andreev's function is the following mapping, 

\begin{align*}
A_n : \{0,1\}^{\lfloor \frac {n}{2} \rfloor} \times \{0,1\}^{ \lfloor \log (\frac{n}{2} ) \rfloor \times \lfloor \frac{n}{2\lfloor \log (\frac {n}{2}) \rfloor} \rfloor} &\longrightarrow \{0,1\} \\
(\x, [a_{ij}]) &\longmapsto x_{\text{bin} \left( z([a_{ij}])  \right )}
\end{align*}
where $z([a_{ij}]) = \{ (\sum_{j=1}^{\lfloor \frac{n}{2 \lfloor \log (\frac {n}{2}) \rfloor} \rfloor} a_{ij}) \mod 2 \}_{i=1,2,..,\lfloor \log (\frac {n}{2} ) \rfloor }$ is the binary string constructed by noting down the odd/even parity of each of the row sums in the matrix $[a_{ij}]$ and ``bin" is the function that gives the decimal number that can be represented by its input bit string.
\end{definition}

 ~\cite{kane2016super} have recently established the first super linear lower bounds for approximating the Andreev function using LTF-of-LTF circuits. In the following theorem we show that their techniques can be adapted to also establish an almost linear lower bound on the size of LTF-of-ReLU circuits approximating this Andreev function with no restriction on the weights $\w, b$ for each gate.

\begin{theorem}\label{thm:andreev-LTF-ReLU}
For any $\delta \in (0,\frac{1}{2})$, there exists $N(\delta) \in \N$ such that for all $n \geq N(\delta)$ and $\epsilon > \sqrt{\frac{2\log^{\frac 2 {2-\delta}}(n)}{n}}$, any LFT-of-ReLU circuit on $n$ bits that matches the Andreev function on $n-$bits for at least $1/2 + \epsilon$ fraction of the inputs, has size $\Omega(\epsilon^{2(1-\delta)}n^{1-\delta})$. 
\end{theorem}

It is well known that proving lower bounds without restrictions on the weights is much more challenging even in the context of LTF circuits. In fact, the recent results in ~\cite{kane2016super} are the first superlinear lower bounds for LTF circuits with no restrictions on the weights. With restrictions on some or all the weights, e.g., assuming $\poly(n)$ bounds on the weights (typically termed the ``small weight asssumption") in certain layers, exponential lower bounds have been established for LTF circuits ~\citep{hajnal1987threshold,impagliazzo1997size,sherstov2009separating,sherstov2011unbounded}. Our next results are of this flavor: under certain kinds of weight restrictions, we prove exponential size lower bounds on the size of LTF-of-(ReLU)$^{d-1}$ circuits. {\em We emphasize that our weight restrictions are assumed only on the bottom layer (closest to the input). The other layers can have gates with unbounded weights.} Nevertheless, our weight restrictions are somewhat unconventional. 

\begin{definition}\label{def:ordering-cone} ({\bf The polyhedral cones $P_{m,\sigma}$}) Let $m \in \N$ and $\sigma$ be any permutation of $\{1, \ldots, 2^m\}$. Let us also consider an arbitrary sequencing $\{\x^1, \ldots, \x^{2^m}\}$ of the vertices of the hypercube $\{-1,1\}^m$. Define the following polyhedral cone, $$P_{m,\sigma} := \{\a \in \R^m: \langle \a, \x^{\sigma(1)} \rangle \leq \langle \a, \x^{\sigma(2)} \rangle \leq \ldots \langle \a, \x^{\sigma(2^m)} \rangle\}.$$
In words, $P_{m,\sigma}$ is the set of all linear objectives that order the vertices of the $m$-dimensional hypercube in the order specified by $\sigma$. \qed \end{definition}

\begin{definition}\label{def:weight-restriction}({\bf Our weight restriction condition})  Below, we shall be considering circuits on $2m$ inputs which come partitioned into two blocks $(\x, \y)$ so that $\x, \y \in \{-1,1\}^m$. The weight restriction we impose is that there exist permutations $\sigma_1$ and $\sigma_2$ of $\{1, \ldots, 2^m\}$ such that for \emph{each ReLU gate in the bottom layer} mapping as, $(\x,\y) \mapsto \max \{0,b+\langle \w_1,\x \rangle + \langle \w_2,\y \rangle \}$, for some bias value of $b$ and weight vectors $\w_1$ and $\w_2$, satisfy the following two conditions, {\bf (1)} $\w_i \in P_{m,\sigma_i}$ for $i=1,2$ (see Definition~\ref{def:ordering-cone})and {\bf (2)} all weights are integers with magnitude bounded by some $W>0$. 
~\\ \\
We emphasize the existence of a single $\sigma$ defining a single polyhedral cone $P_{m,\sigma}$ which contains all the weight vectors corresponding to $\x$ in the bottom most layer of the net (similarly for all the weights corresponding to $\y$). But the two cones, one for $\x$ and one for $\y$, are allowed to be different. \qed
\end{definition}
\begin{remark}One can see that $\R^m$ is a disjoint union of the different face-sharing polyhedral cones $P_{m,\sigma}$ obtained for different $\sigma \in S_{2m}$. Thus part (1) of the above weight restriction is equivalent to asking all the weight vectors in the bottom layer of the net corresponding to $\x$ part of the input to lie in any one of these special cones (and similarly for the $\y$ part of the input).
\end{remark}

Let OMB is the ODD-MAX-BIT function which is a $\pm 1$ threshold gate which evaluates to $-1$ on say a $n-$bit input $\x$ if $\sum_{i=1}^n (-1)^{i+1}2^i(1+x_i) \geq \frac{1}{2}$. We will prove our lower bounds against the function proposed by Arkadev Chattopadhyay and Nikhil Mande in ~\cite{chattopadhyay2017weights},
\begin{align}\label{AN}
g : {\rm OMB}_n^0 \circ {\rm OR}_{n^{\frac 1 3} + \log n } \circ {\rm XOR}_2 : \{-1,1\}^{2(n^{\frac 4 3} +  n \log n)} \rightarrow \{-1,1\}
\end{align}
which we will refer to as the Chattopadhyay-Mande function in the remainder of the paper. Here we use the notation from ~\cite{chattopadhyay2017weights} whereby if $p_m$ and $q_n$ are two Boolean functions taking $m$ and $n$ bits respectively for input, we denote a composition of them as, $p_m \circ q_n : \{-1,1\}^{mn} \rightarrow \{-1,1\}$. Here its understood that the input implicitly comes grouped into $m$ blocks of size $n$ on each of which $q$ acts and $p_m$ acts on the $m-$tuple of outputs of these $q_n$ functions. 
~\\ \\
We show the following exponential lowerbound against this Chattopadhyay-Mande function.
\begin{theorem}\label{deepReLU}
Let $m, d, W\in \N$. Any depth $d$ LTF-of-(ReLU)$^{d-1}$ circuits on $2m$ bits such that the weights in the bottom layer are restricted as per Definition~\ref{def:weight-restriction} that implements the Chattopadhyay-Mande function on $2m$ bits will require a circuit size of,

\[ \Omega \left ( (d-1) \left [ \frac{2^{m^{\frac 1 8}}}{mW} \right ]^{\frac 1 {(d-1)}} \right ).\]

Consequently, one obtains the same size lower bounds for circuits with only LTF gates of depth $d$.
\end{theorem}

\remark Note that this is an exponential in dimension size lowerbound for even super-polynomially growing bottom layer weights (and additional constraints as per Definition~\ref{def:weight-restriction}) and upto depths scaling as $d = O(m^\xi)$ for any $\xi < \frac{1}{8}$.

We note that the Chattopadhyay-Mande function can be represented by an $O(m)$ size LTF-of-LTF circuit with no restrictions on weights (see Theorem~\ref{CMTheorem} below). In light of this fact, Theorem~\ref{deepReLU} is somewhat surprising as it shows that for the purpose of representing Boolean functions a deep ReLU circuit (ending in a LTF) gate can get exponentially weakened when just its bottom layer weights are restricted as per Definition~\ref{def:weight-restriction}, even if the integers are allowed to be super-polynomially large. Moreover, the lower bounds also hold of LTF circuits of arbitrary depth $d$, under the same weight restrictions on the bottom layer. We are unaware of any exponential lower bounds on LTF circuits of arbitrary depth under any kind of weight restrictions.



We will use the method of sign-rank to obtain the exponential lowerbounds in Theorems~\ref{deepReLU}.
The {\emph sign-rank} of a real matrix $A$ with all non-zero entries is the least rank of a matrix $B$ of the same dimension with all non-zero entries such that for each entry $(i,j)$, $sign(B_{ij}) = sign(A_{ij})$. For a Boolean function $f$ mapping, $f : \{-1,1\}^m \times \{-1,1\}^m \rightarrow \{-1,1\}$ one defines the ``sign-rank of f" as the sign-rank of the $2^m \times 2^m$ dimensional matrix $[f(\x,\y)]_{\x ,\y \in \{-1,1\}^m}$. This notion of a sign-rank has been used to great effect in diverse fields from communication complexity to circuit complexity to learning theory. Explicit matrices with a high sign-rank were not known till the breakthrough work by Forster, \cite{forster2002linear}. Forster et. al. showed elegant use of this complexity measure to show exponential lowerbounds against LTF-of-MAJ circuits in \cite{forster2001relations}. Lot of the previous literature about sign-rank has been reviewed in the book ~\cite{lokam2009complexity}. Most recently ~\cite{chattopadhyay2017weights} have proven a strict containment of LTF-of-MAJ in LTF-of-LTF.  The following theorem statement is a combination of their Theorem $5.2$ and intermediate steps in their Corollary $1.2$, 

\begin{theorem}[{\bf Chattopadhyay-Mande ($2017$)}]\label{CMTheorem}
~\\
The Chattopadhyay-Mande function $g$ in equation \ref{AN} can be represented by a linear sized LTF-of-LTF circuit and $\text{sign-rank}(g) \geq 2^{\frac{n^{\frac 1 3}}{81} - 3}$
\end{theorem}

In Appendix \ref{app:signranksmall} we will prove our Theorem \ref{deepReLU} by showing a small upper bound on the sign-rank of LTF-of-(ReLU)$^{d-1}$ circuits which have their bottom most layer's weight restricted as given in Definition \ref{def:weight-restriction}.
 
\section{Lower bounds for LTF-of-ReLU against the Andreev function (Proof of Theorem~\ref{thm:andreev-LTF-ReLU})}

We will use the classic ``method of random restrictions" ~\citep{subbotovskaya1961realizations,stad1998shrinkage,hastad1986almost,yao1985separating,rossman2008constant} to show a lowerbound for weight unrestricted LTF-of-ReLU circuits for representing the Andreev function. The basic philosophy of this method is to take any arbitrary LTF-of-ReLU circuit which supposedly matches the Andreev function on a large fraction of the inputs and to randomly fix the values on some of its input coordinates and also do the same fixing on the same coordinates of the input to the Andreev function. Then we show that upon doing this restriction the Andreev function collapses to an arbitrary Boolean function on the remaining inputs (what it collapses to depends on what values were fixed on its inputs that got restricted).  But on the other hand we show that the LTF-of-ReLU collapses to a circuit which is of such a small size that with high-probability it cannot possibly approximate a randomly chosen Boolean function on the remaining inputs. This contradiction leads to a lowerbound. 

There are two important concepts towards implementing the above idea. First one is about being able to precisely define as to when can a ReLU gate upon a partial restriction of its inputs be considered to be removable from the circuit. Once this notion is clarified it will automatically turn out that doing random restrictions on ReLU is the same as doing random restriction on a LTF gate as was recently done in ~\cite{kane2016super}. And secondly it needs to be true that at any fixed size, LTF-of-ReLU circuits cannot represent too many of all the Boolean functions possible at the same input dimension.  For this very specific case of LTF-of-ReLU circuits where ReLU gates necessarily have a fan-out of $1$, Theorem 2.1 in ~\cite{maass1997bounds} applies and we have from there that LTF-of-ReLU circuits over $n-$bits with $w$ ReLU gates can represent at most $N = 2^{O((wn + w + w+1 +1)^2\log (wn + w + w+1 +1))} = 2^{O((wn+2w+2)^2\log (wn+2w+2))}$ number of Boolean functions. We note that slightly departing from the usual convention with neural networks here in this work by Wolfgaang Mass he allows for direct wires from the input nodes to the output LTF gate. This flexibility ties in nicely with how we want to define a ReLU gate to be becoming removable under the random restrictions that we use.

\paragraph{Random Boolean functions vs any circuit class}

In everything that follows all samplings being done (denoted as $\sim$) are to be understood as sampling from an uniform distribution unless otherwise specified. Firstly we note this well-known lemma, 

\begin{claim}
Let $f : \{-1,1\}^n \rightarrow \{-1,1\}$ be any given Boolean function. Then the following is true, 
\[ \mathbb{P}_{g \sim \{ \{-1,1\}^n \rightarrow \{-1,1\} \}} \left [ \mathbb{P}_{\x \sim \{-1,1\}^n} [ f(\x) = g(\x) ] \geq \frac {1}{2} + \epsilon \right ] \leq e^{-2^{n+1}\epsilon^2} \]
\end{claim}

From the above it follows that if $N$ is the total number of functions in any circuit class (whose members be called $C$) then we have by union bound,

\begin{align}\label{randomapprox}
\mathbb{P}_{g \sim \{ \{-1,1\}^n \rightarrow \{-1,1\} \}} \left [ \exists C \text{ s.t } \mathbb{P}_{\x \sim \{-1,1\}^n} [ C(\x) = g(\x) ] \geq \frac {1}{2} + \epsilon \right ] \leq N e^{-2^{n+1}\epsilon^2}
\end{align}

Equipped with these basics we are now ready to begin the proof of the lowerbound against weight unrestricted LTF-of-ReLU circuits, 

\begin{proof}[Proof of Theorem~\ref{thm:andreev-LTF-ReLU}]
\begin{definition}
Let $D$ denote arbitrary LTF-of-ReLU circuits over $\lfloor \log (\frac{n}{2} ) \rfloor$ bits.
\end{definition}
 
 For some $\frac{\epsilon}{3} \leq \frac{1}{2}$ and a size function denoted as $s(n,\epsilon)$ we use equation \ref{randomapprox} , the definition of $D$ above and the upperbound given earlier for the number of LTF-of-ReLU functions at a fixed circuit size (now used for circuits on $\lfloor \log (\frac{n}{2} ) \rfloor$  bits) to get, 
\begin{align*}
\mathbb{P}_{\substack{f \sim \{0,1\}^{\lfloor \log (\frac{n}{2} ) \rfloor} \rightarrow \{0,1\} }} & \Bigg [\forall D \text{ s.t } \vert D \vert \leq s(n,\epsilon)\text{ }\vert \mathbb{P}_{ \y \sim \{0,1\}^{\lfloor \log (\frac{n}{2} ) \rfloor}} [ f(\y) = D(\y) ]  \leq \Big (\frac {1}{2} + \frac {\epsilon}{3} \Big ) \Bigg ]\\
&\geq 1 - 2^{O(s^2\log^2(\frac{n}{2})\log(\log(\frac{n}{2})s))}e^{-\left (\frac{\epsilon^2}{9} \right)2^{1+\lfloor \log \left ( \frac{n}{2} \right ) \rfloor}}\\
&\geq 1 - 2^{O(s^2k^2\log(ks))}e^{-\left (\frac{2\epsilon^2}{9} \right)2^{k}} \geq 1 - e^{O(s^2k^2\log(ks))-\left (\frac{2\epsilon^2}{9} \right)2^{k}}
\end{align*}

whereby in the last inequality above we have assumed that $n=2^{k+1}$. This assumption is legitimate because we want to estimate certain large $n$ asymptotics. Now for some $\theta >0$ if for large $n$ we choose, $\epsilon > \sqrt{\frac{2\log^{2+\theta} (\frac{n}{2})}{n}}$ and $s = s(n,\epsilon) \leq O(\frac{\epsilon^{\frac{2}{2+\theta}}n^{\frac{1}{2+\theta}}}{2^{\frac{1}{2+\theta}}\log (\frac{n}{2})})$ then we have,


\begin{align}\label{smaass}
\nonumber \mathbb{P}_{\substack{f \sim \{0,1\}^{\lfloor \log (\frac{n}{2} ) \rfloor} \rightarrow \{0,1\} }} & \Bigg [\forall D \text{ s.t } \vert D \vert \leq s(n,\epsilon)\text{ }\vert \mathbb{P}_{ \y \sim \{0,1\}^{\lfloor \log (\frac{n}{2} ) \rfloor}} [ f(\y) = D(\y) ]  \leq \Big (\frac {1}{2} + \frac {\epsilon}{3} \Big ) \Bigg ]\\
&\geq  1 - \frac{\epsilon}{3} 
\end{align}

\begin{definition}[$\mathbf{ F^*}$]\label{def:F*}
Let $F^*$ be the subset of all these $f$ above for which the above event is true i.e\\
\[ F^* :=  \bigg \{ f : \{0,1\}^{\lfloor \log (\frac{n}{2} ) \rfloor} \rightarrow \{0,1\}  \mid \forall D \text{ s.t } \vert D \vert \leq s(n,\epsilon)\text{ }\vert \mathbb{P}_{ \y \sim \{0,1\}^{\lfloor \log (\frac{n}{2} ) \rfloor}} [ f(\y) = D(\y) ]  \bigg \}  \]
\end{definition}

Now we recall the definition of the Andreev function in equation \ref{Andreev} for the following definition and the claim, 

\begin{definition}[]
Let $\rho$ be a choice of a ``restriction" whereby one is fixing all the input bits of $A_n$ except $1$ bit in each row of the matrix $a$. So the restricted function (call it $A_n \vert _{\rho}$) computes a function of the form,

\[ A_n \vert_{\rho} : \{0,1\}^{ \lfloor \log (\frac{n}{2} ) \rfloor } \rightarrow \{0,1\}\]
\end{definition}

Note that we shall henceforth be implicitly fixing a bijection mapping, \\ $\{ 0,1\}^n \rightarrow \{0,1\}^{\lfloor \frac {n}{2} \rfloor} \times \{0,1\}^{ \lfloor \log (\frac{n}{2} ) \rfloor \times \lfloor \frac{n}{2\lfloor \log (\frac {n}{2}) \rfloor} \rfloor}$ and hence for any function $C : \{0,1\}^n \rightarrow \{0,1\}$, it would be meaningful to talk of $C \vert_\rho$. From the definitions of $A_n$ and $\rho$ above, the following is immediate, 

\begin{claim}
The truth table of $A_n \vert _\rho$ is the $\x$ string in the input to $A_n$ that gets fixed by $\rho$. Thus we observe that if $\rho$ is chosen uniformly at random then $A_n \vert _\rho$ is a $\lfloor \log (\frac{n}{2} ) \rfloor$ bit Boolean function chosen uniformly at random. 
\end{claim}
 
Let $f^*$ be any arbitrary member of $F^*$. Let $\x^* \in \{0,1\}^{\lfloor \frac {n}{2} \rfloor}$ be the truth-table of $f^*$. Let $\rho(\x^*)$ be restrictions on the input of $A_n$ which fix the $\x$ part of its input to $\x^*$. So when we are sampling restrictions uniformly at random from the restrictions of the type $\rho(\x^*)$ these different instances differ in which bit of each row of the matrix $a$ (of the input to $A_n$) they left unfixed and to what values did they fix the other entries of $a$. Let $C$ be a $n$ bit LTF-of-ReLU Boolean circuit of size say $w(n,\epsilon)$. Thus under a restriction of the type $\rho(\x^*)$ both $C$ and $A_n$ are $\lfloor \log (\frac{n}{2} ) \rfloor$ bit Boolean functions.

Now we note that a ReLU gate over $n$ bits upon a random restriction becomes redundant (and hence removable) iff its linear argument either reduces to a non-positive definite function or a positive definite function. In the former case the gate is computing the constant function zero and in the later case it is computing a linear function which can be simply implemented by introducing wires connecting the inputs directly to the output LTF gate. Thus in both the cases the resultant function no more needs the ReLU gate for it to be computed.  (We note that such direct wires from the input to the output gate were allowed in how the counting was done of the total number of LTF-of-ReLU Boolean functions at a fixed circuit size.) Combining both the cases we note that the conditions for collapse (in this sense) of a ReLU gate is identical to that of the conditions of collapse for a LTF gate for which ~\cite{kane2016super} in their Lemma $1.1$ had proven the following, 

\begin{lemma}[{\bf Lemma $1.1$ of ~\cite{kane2016super}}]
Let $f: \{0,1\}^n \rightarrow \{0,1\}$ be a linear
threshold function. Let ${\cal P}$ be a partition of $[n]$ into parts of equal size, and let ${\cal R}_{\cal P}$ be the distribution on restrictions $\rho : [n] \rightarrow \{0,1,*\}$ that randomly fixes all but one element of each part of ${\cal P}$. Then we have, 
\[ \mathbb{P}_{\rho \sim {\cal R}_{\cal P}} \left [ f \text{ is not forced to a constant by }\rho \right ] = O \left ( \frac{\vert {\cal P} \vert }{\sqrt{n}}  \right ) \]
\end{lemma}

In our context the above implies, 

\[ \mathbb{P}_{\rho(\x^*)} [  \text{ReLU}\vert_{\rho(\x^*)} \text {is removable }] \geq \eta \]
where $\eta = 1 - O(\frac{\log n}{\sqrt{n}})$

~\\ \newline
The above definition of $\eta$ implies, 
\begin{align}\label{useofeta}
\nonumber &\mathbb{P}_{\rho(\x^*)} [ \text { A $n-$bit  \text{ReLU} is \emph {not} forced to a constant }] \leq 1 - \eta\\
\nonumber \implies &\mathbb{E}_{\rho(\x^*)} [ \text { Number of  \text{ReLU}s of C \emph {not} forced to a constant }] \leq w(n,\epsilon)(1 - \eta)\\
\nonumber \implies &\mathbb{P}_{\rho(\x^*)} [\text { Number of  \text{ReLU}s of C \emph {not} forced to a constant } > s(n,\epsilon) ]\\
\nonumber &\leq \frac {\mathbb{E}_{\rho(\x^*)} [ \text { Number of  \text{ReLU}s of C \emph {not} forced to a constant }]}{s(n,\epsilon)}\\
\nonumber \implies &\mathbb{P}_{\rho(\x^*)} [\text { Number of  \text{ReLU}s of C \emph {not} forced to a constant } \geq s(n,\epsilon) ] \leq \frac {w(n,\epsilon)(1-\eta)}{s(n,\epsilon)}\\
\nonumber \implies &\mathbb{P}_{\rho(\x^*)} [\text { Size of } C \vert _{\rho(\x^*)} \leq s(n,\epsilon) ] \geq 1 - \frac {w(n,\epsilon)(1-\eta)}{s(n,\epsilon)}\\
\end{align}

Now we compare with the definitions of $\epsilon$ and $f^*$ to observe that (a) with probability at least $1 - \frac {w(n,\epsilon)(1-\eta)}{s(n,\epsilon)}$, $C \vert _{\rho(\x^*)}$ is a circuit of the type called ``$D$" in the event in equation \ref{smaass} and (b) by definition of the Andreev function it follows that $A_n\vert_{\rho(\x^*)}$ has its truth table given by $\x^*$ and hence it specifies the same function as $f^* \in F^*$. Hence $\forall \x^* \text{ and } \rho(\x^*)$ we can read off from  equation \ref{smaass}, 

\begin{align}\label{Anisrandom}
\mathbb{P}_{\y \sim \{0,1\}^{\lfloor \log (\frac{n}{2} ) \rfloor}} [C\vert _{\rho(\x^*)}(\y) = A_n \vert _{\rho(\x^*)}(\y) \vert  \text{ Size of } C\vert _{\rho(\x^*)} \leq s(n,\epsilon) ] \leq \frac {1}{2} + \frac {\epsilon}{3}
\end{align}


Recalling Definition \ref{def:F*}, the equation \ref{smaass} can be written as, 
\begin{align}\label{useofs}
\mathbb{P}_{f \sim \{0,1\}^{\lfloor \log (\frac{n}{2} ) \rfloor} \rightarrow \{0,1\} } [ f \in F^* ]  \geq 1- \frac {\epsilon}{3}
\end{align}

\begin{claim}{Circuits $C$ have low correlation with the Andreev function}
\begin{align*}
\mathbb{P}_{\z \sim \{0,1\}^n}[C(z) = A_n(z)] \leq \frac {\epsilon}{3} + \frac {w(n,\epsilon)(1-\eta)}{s(n,\epsilon)} + \frac {1}{2} + \frac {\epsilon}{3}\\
\end{align*}
\end{claim}
\begin{proof}
We think of sampling a $z \sim \{0,1\}^n$ as a two step process of first sampling a $\tilde{f}$, a $\lfloor \log (\frac n 2) \rfloor$ bit Boolean function and fixing the first $\lfloor \frac n 2 \rfloor$ bits of $z$ to be the truth-table of $\tilde{f}$ and then we randomly assign values to the remaining $\lfloor \frac n 2 \rfloor$ bits of $z$. Call these later $\lfloor \frac n 2 \rfloor$ bit string to be $\x_{other}$. 
\begin{align*}
\mathbb{P}_{\z \sim \{0,1\}^n}[C(z) = A_n(z)] 
&= \mathbb{P}_{\z \sim \{0,1\}^n}[(C(z) = A_n(z))\cap(\tilde{f} \in F^*)]+ \mathbb{P}_{\z \sim \{0,1\}^n}[(C(z) = A_n(z))\cap(\tilde{f} \notin F^*)]\\
&= \mathbb{P}_{\z \sim \{0,1\}^n}[(C(z) = A_n(z))\mid (\tilde{f} \in F^*)]\mathbb{P}_{\z \sim \{0,1\}^n} [\tilde{f} \in F^* ]\\
&+ \mathbb{P}_{\z \sim \{0,1\}^n}[(C(z) = A_n(z))\cap(\tilde{f} \notin F^*)]\\
&\leq \mathbb{P}_{\z \sim \{0,1\}^n}[(C(z) = A_n(z))\mid (\tilde{f} \in F^*)] + \mathbb{P}_{\z \sim \{0,1\}^n}[\tilde{f} \notin F^*]\\
&\leq \mathbb{P}_{\z \sim \{0,1\}^n}[(C(z) = A_n(z))\mid (\tilde{f} \in F^*)] + \frac \epsilon 3\\
\end{align*}
In the last line above we have invoked equation \ref{useofs}. Now we note that sampling the $n$ bit string $z$ such that $\tilde{f} \in F^*$ is the same as doing a random restriction of the type $\rho(\tilde{f})$ and then randomly picking a  $\lfloor \log (\frac n 2) \rfloor$ bit string say $\y$. So we can rewrite the last inequality as, 

\begin{align*}
\mathbb{P}_{\z \sim \{0,1\}^n}[C(z) = A_n(z)]&\leq \mathbb{P}_{(\rho(\tilde{f}),\y)}[C(\rho(\tilde{f}),\y) = A_n(\rho(\tilde{f}),\y)] + \frac \epsilon 3\\
&\leq \mathbb{E}_{(\rho(\tilde{f}),\y)}[\mathfrak{1}_{C(\rho(\tilde{f}),\y) = A_n(\rho(\tilde{f}),\y)} \mid (\tilde{f} \in F^*)] + \frac \epsilon 3\\
&\leq \mathbb{E}_{(\rho(\tilde{f}),\y)}[\mathfrak{1}_{C(\rho(\tilde{f}),\y) = A_n(\rho(\tilde{f}),\y)}\mathfrak{1}_{\text{Size of } C\vert_{\rho(\tilde{f})} < s(n,\epsilon)} \mid (\tilde{f} \in F^*)]\\
&+ \mathbb{E}_{(\rho(\tilde{f}),\y)}[\mathfrak{1}_{C(\rho(\tilde{f}),\y) = A_n(\rho(\tilde{f}),\y)}\mathfrak{1}_{\text{Size of } C\vert_{\rho(\tilde{f})} \geq s(n,\epsilon)}\mid (\tilde{f} \in F^*)] + \frac \epsilon 3\\
&\leq \mathbb{P}_{(\rho(\tilde{f}),\y)}[C(\rho(\tilde{f}),\y) = A_n(\rho(\tilde{f}),\y) \mid \left ( (\text{Size of } C\vert_{\rho(\tilde{f})} < s(n,\epsilon)) \cap  (\tilde{f} \in F^*) \right )]\\
&+\mathbb{P}_{(\rho(\tilde{f}),\y)}[\text{Size of } C\vert_{\rho(\tilde{f})} \geq s(n,\epsilon)\mid (\tilde{f} \in F^*)] + \frac \epsilon 3\\
&\leq \left ( \frac 1 2 + \frac \epsilon 3\right ) + \frac {w(n,\epsilon)(1-\eta)}{s(n,\epsilon)} + \frac \epsilon 3
\end{align*}
In the last step above we have used equations \ref{Anisrandom} and \ref{useofeta}. 
\end{proof}

So after putting back the values of $\eta$ and the largest scaling of $s(n,\epsilon)$ that we can have (from equation \ref{smaass}), the upperbound on the above probability becomes, 
\begin{align*}
\frac{1}{2} + \frac{2\epsilon}{3} + O \Bigg ( \frac{w(n,\epsilon)\log (n) }{\sqrt{n} (\frac{\epsilon^{\frac{2}{2+\theta}}n^{\frac{1}{2+\theta}}}{2^{\frac{1}{2+\theta}}\log (\frac{n}{2})})} \Bigg ) 
\end{align*}
~\\
Thus the probability is upperbounded by $\frac{1}{2} + \epsilon$ as long as $w(n,\epsilon) = O \Bigg ( \frac{\epsilon^{1+\frac{2}{2+\theta}} n^{\frac{1}{2} + \frac{1}{2+\theta}} \log \Big ( \frac{n}{2} \Big ) }{\log (n) }\Bigg )$

Stated as a lowerbound we have that if a LTF-of-ReLU has to match the $n-$bit Andreev function on more than $\frac{1}{2}+\epsilon$ fraction of the inputs for $\epsilon > \sqrt{\frac{2\log^{2+\theta} (\frac{n}{2})}{n}}$ for some $\theta >0$ (asymptotically this is like having a constant $\epsilon$) then the LTF-of-ReLU needs to be of size $\Omega(\epsilon^{\frac{4+\theta}{2+\theta}}n^{\frac{1}{2} + \frac{1}{2+\theta}})$. Now we define $\delta \in (0,\frac 1 2)$ such that $\delta = \frac {\theta}{2(2+\theta)}$ and that gives the form of the almost linear lowerbound as stated in the theorem. 
\end{proof}

\begin{subappendices}

\chapter*{\vspace{10pt} Appendix To Chapter \ref{chapfunc}}\label{app:func} 

\section{Expressing piecewise linear functions using ReLU DNNs}\label{sec:thm-1-rep}

\begin{proof}[Proof of Theorem \ref{thm:1-dim-2-layer}]

Any continuous piecewise linear function $\mathbb{R} \rightarrow \mathbb{R}$ which has $m$ pieces can be specified by three pieces of information, $(1)$ $s_L$ the slope of the left most piece, $(2)$ the coordinates of the non-differentiable points specified by a $(m-1)-$tuple $\{ (a_i,b_i) \}_{i=1}^{m-1}$ (indexed from left to right) and $(3)$ $s_R$ the slope of the rightmost piece. A tuple $(s_L, s_R, (a_1,b_1), \ldots, (a_{m-1},b_{m-1})$ uniquely specifies a $m$ piecewise linear function from $\mathbb{R} \rightarrow \mathbb{R}$ and vice versa. Given such a tuple, we construct a $2$-layer DNN which computes the same piecewise linear function.

One notes that for any $a, r \in \mathbb{R}$, the function  
\begin{eqnarray}\label{eq:right-flap}
f(x) = \begin{cases} 
      0 & x\leq a \\
      r(x-a) & x>a
   \end{cases}
\end{eqnarray}
is equal to $\operatorname{sgn}(r)\max\{\vert r \vert (x-a),0\}$, which can be implemented by a 2-layer ReLU DNN with size 1. Similarly, any function of the form, 
\begin{eqnarray}\label{eq:left-flap}
g(x) = \begin{cases} 
      t(x-a) & x \leq a \\
      0 & x>a
   \end{cases}
\end{eqnarray}
is equal to $-\operatorname{sgn}(t)\max\{ -\vert t \vert (x-a), 0\}$, which can be implemented by a 2-layer ReLU DNN with size 1. The parameters $r,t$ will be called the {\em slopes} of the function, and $a$ will be called the {\em breakpoint} of the function.

If we can write the given piecewise linear function as a sum of $m$ functions of the form~\eqref{eq:right-flap} and~\eqref{eq:left-flap}, then by Lemma~\ref{lem:add-DNN} we would be done.
It turns out that such a decomposition of any $p$ piece PWL function $h:\R\to \R$ as a sum of $p$ flaps can always be arranged where the breakpoints of the $p$ flaps all are all contained in the $p-1$ breakpoints of $h$. First, observe that adding a constant to a function does not change the complexity of the ReLU DNN expressing it, since this corresponds to a bias on the output node. Thus, we will assume that the value of $h$ at the last break point $a_{m-1}$ is $b_{m-1} = 0$.

We now use a single function $f$ of the form~\eqref{eq:right-flap} with slope $r$ and breakpoint $a = a_{m-1}$, and $m-1$ functions $g_1, \ldots, g_{m-1}$ of the form~\eqref{eq:left-flap} with slopes $t_1, \ldots, t_{m-1}$ and breakpoints  $a_1, \ldots, a_{m-1}$, respectively.

Thus, we wish to express $h = f + g_1 + \ldots + g_{m-1}$.
Such a decomposition of $h$ would be valid if we can find values for $r, t_1, \ldots, t_{m-1}$ such that $(1)$ the slope of the above sum is $= s_L$ for $x < a_1$, $(2)$ the slope of the above sum is $=s_R$ for $x>a_{m-1}$, and $(3)$ for each $i \in \{1,2,3,..,m-1\}$ we have $b_i = f(a_i) + g_1(a_i) + \ldots + g_{m-1}(a_i)$.

The above corresponds to asking for the existence of a solution to the following set of simultaneous linear equations in $r, t_1, \ldots, t_{m-1}$: 

$$s_R = r,\;\;s_L = t_1 + t_2 + \ldots + t_{m-1},\;\; b_i = \sum_{j=i+1}^{m-1}t_j (a_{j-1} - a_j) \textrm{ for all } i=1, \ldots, m-2$$

It is easy to verify that the above set of simultaneous linear equations has a unique solution. Indeed, $r$ must equal $s_R$, and then one can solve for $t_1, \ldots, t_{m-1}$ starting from the last equation $b_{m-2} = t_{m-1}(a_{m-2} - a_{m-1})$ and then back substitute to compute $t_{m-2}, t_{m-3}, \ldots, t_1$.

The lower bound of $p-1$ on the size for any $2$-layer ReLU DNN that expresses a $p$ piece function follows from Lemma~\ref{lem:size-bounds}.
\end{proof}

One can do better in terms of size when the rightmost piece of the given function is flat, i.e., $s_R = 0$.  In this case $r =0$, which means that $f= 0$; thus, the decomposition of $h$ above is of size $p-1$. A similar construction can be done when $s_L = 0$. This gives the following statement which will be useful for constructing our forthcoming hard functions. 

\begin{corollary}\label{cor:tighter-bound-2-layer} If the rightmost or leftmost piece of a $\R \to \R$ piecewise linear function has $0$ slope, then we can compute such a $p$ piece function using a $2$-layer DNN with size $p-1$. 
\end{corollary}

\begin{proof}[Proof of theorem~\ref{thm:Lpapprox}]
Since any piecewise linear function $\R^n \to \R$ is representable by a ReLU DNN by Corollary~\ref{cor:all-pwl-are-dnn}, the proof simply follows from the fact that the family of continuous piecewise linear functions is dense in any $L^p(\R^n)$ space, for $1\leq p \leq \infty$.
\end{proof}

Now we will collect some straightforward observations that will be used often in constructing complex neural functions starting from simple ones. The following operations preserve the property of being representable by a ReLU DNN.

\begin{lemma}\label{lem:comp-DNN} [Function Composition] If $f_1 : \R^{d} \to \R^{m}$ is represented by a $d,m$ ReLU DNN with depth $k_1 + 1$ and size $s_1$, and $f_2 : \R^{m} \to \R^{n}$ is represented by an $m,n$ ReLU DNN with depth $k_2+1$ and size $s_2$, then $f_2 \circ f_1$ can be represented by a $d,n$ ReLU DNN with depth $k_1 + k_2 + 1$ and size $s_1 + s_2$.
\end{lemma}

\begin{proof} 
Follows from~\eqref{eq:DNN-def} and  the fact that a composition of affine transformations is another affine transformation.
\end{proof}

\begin{lemma}\label{lem:add-DNN} [Function Addition] If $f_1 : \R^{n} \to \R^{m}$ is represented by a $n,m$ ReLU DNN with depth $k + 1$ and size $s_1$, and $f_2 : \R^{n} \to \R^{m}$ is represented by a $n,m$ ReLU DNN with depth $k+1$ and size $s_2$, then $f_1 + f_2$ can be represented by a $n,m$ ReLU DNN with depth $k + 1$ and size $s_1 + s_2$.
\end{lemma}

\begin{proof}
We simply put the two ReLU DNNs in parallel and combine the appropriate coordinates of the outputs.
\end{proof}

\begin{lemma}\label{lem:max-DNN} [Taking maximums/minimums] Let $f_1, \ldots, f_m : \R^{n} \to \R$ be functions that can each be represented by $\R^n \to \R$ ReLU DNNs with depths $k_i + 1$ and size $s_i$, $i=1, \ldots, m$.  Then the function $f:\R^n \to \R$ defined as $f(\x) := \max\{f_1(\x), \ldots, f_m(\x)\}$ can be represented by a ReLU DNN of depth at most $\max\{k_1, \ldots, k_m\} + \log(m) + 1$ and size at most $s_1 + \ldots s_m + 4(2m-1)$. Similarly, the function $g(\x) := \min\{f_1(\x), \ldots, f_m(\x)\}$ can be represented by a ReLU DNN of depth at most $\max\{k_1, \ldots, k_m\} + \lceil\log(m)\rceil + 1$ and size at most $s_1 + \ldots s_m + 4(2m-1)$.
\end{lemma}

\begin{proof} We prove this by induction on $m$. The base case $m=1$ is trivial. For $m\geq 2$, consider $g_1 := \max\{f_1, \ldots, f_{\lfloor \frac{m}{2}\rfloor}\}$ and $g_2 := \max\{f_{\lfloor \frac{m}{2} \rfloor+1}, \ldots, f_m\}$. By the induction hypothesis (since $\lfloor \frac{m}{2}\rfloor, \lceil \frac{m}{2}\rceil < m$ when $m \geq 2$),  $g_1$ and $g_2$ can be  represented by  ReLU DNNs of depths at most $\max\{k_1, \ldots, k_{\lfloor \frac{m}{2}\rfloor}\} + \lceil\log(\lfloor \frac{m}{2}\rfloor)\rceil + 1$ and $\max\{k_{\lfloor \frac{m}{2} \rfloor+1}, \ldots, k_m\} + \lceil\log(\lceil \frac{m}{2}\rceil)\rceil + 1$ respectively, and sizes at most $s_1 + \ldots s_{\lfloor \frac{m}{2}\rfloor} + 4(2\lfloor \frac{m}{2}\rfloor - 1)$ and  $s_{\lfloor \frac{m}{2} \rfloor+1} + \ldots + s_m + 4(2\lfloor \frac{m}{2}\rfloor-1)$, respectively.  Therefore, the function $G: \R^n\to \R^2$ given by  $G(\x) = (g_1(\x), g_2(\x))$ can be implemented by a ReLU DNN with depth at most $\max\{k_1, \ldots, k_m\} + \lceil\log(\lceil \frac{m}{2}\rceil)\rceil + 1$ and size at most $s_1 + \ldots + s_m + 4(2m-2)$.

We now show how to represent the function $T:\R^2\to \R$ defined as $T(x,y) = \max\{x, y\} = \frac{x+y}{2} + \frac{|x-y|}{2}$ by a 2-layer ReLU DNN with size 4 -- see Figure~\ref{fig:max-comp}. The result now follows from the fact that $f = T\circ G$ and Lemma~\ref{lem:comp-DNN}.
\end{proof}

\begin{figure}[h!]
\centering
\includegraphics[width=.7\textwidth]{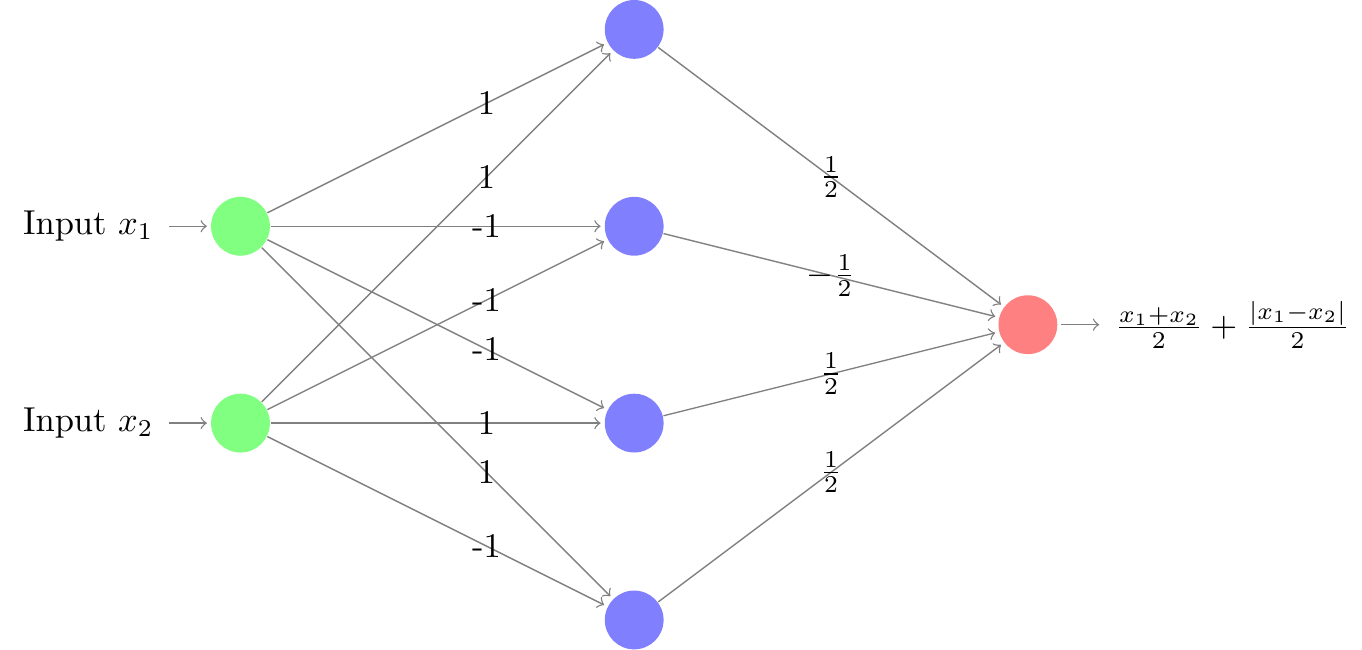}
\caption{A 2-layer ReLU DNN computing $\max \{ x_1,x_2 \} = \frac{x_1+x_2}{2}+\frac{|x_1-x_2|}{2}$}\label{fig:max-comp} 
\end{figure}


\begin{lemma}\label{lem:affine-DNN} Any affine transformation $T:\R^n \to \R^m$ is representable by a 2-layer ReLU DNN of size $2m$.
\end{lemma}

\begin{proof} Simply use the fact that $T = (I\circ\sigma\circ T) + (-I\circ\sigma\circ (-T))$, and the right hand side can be represented by a 2-layer ReLU DNN of size $2m$ using Lemma~\ref{lem:add-DNN}.
\end{proof}

\begin{lemma}\label{lem:size-bounds-0} Let $f:\R \to \R$ be a function represented by a $\R \to \R$ ReLU DNN with depth $k+1$ and widths $w_1, \ldots, w_k$ of the $k$ hidden layers. Then $f$ is a PWL function with at most  $2^{k-1}\cdot(w_1+1)\cdot w_2 \cdot \ldots \cdot w_k$ pieces.
\end{lemma}

\begin{proof}
\begin{figure}[h!]
\includegraphics[width=1\textwidth]{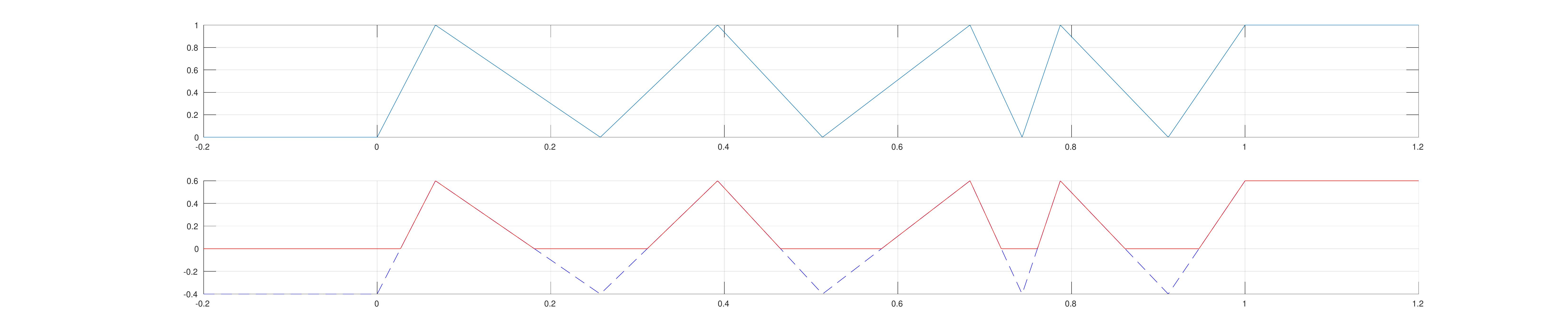}
\caption{The number of pieces increasing after activation. If the blue function is $f$, then the red function $g = \max\{0,f+b\}$ has at most twice the number of pieces as $f$ for any bias $b\in \R$.}\label{fig:maxpieces} 
\end{figure}
We prove this by induction on $k$. The base case is $k=1$, i.e, we have a 2-layer ReLU DNN. Since every activation node can produce at most one breakpoint in the piecewise linear function, we can get at most $w_1$ breakpoints, i.e., $w_1 + 1$ pieces.

Now for the induction step, assume that for some $k\geq 1$, any $\R \to \R$ ReLU DNN with depth $k+1$ and widths $w_1, \ldots, w_k$ of the $k$ hidden layers produces at most  $2^{k-1}\cdot(w_1+1)\cdot w_2 \cdot \ldots \cdot w_k$ pieces. 

Consider any $\R \to \R$ ReLU DNN with depth $k+2$ and widths $w_1, \ldots, w_{k+1}$ of the $k+1$ hidden layers. Observe that the input to any node in the last layer is the output of a $\R \to \R$ ReLU DNN with depth $k+1$ and widths $w_1, \ldots, w_k$. By the induction hypothesis, the input to this node in the last layer is a piecewise linear function $f$ with at most $2^{k-1}\cdot(w_1+1)\cdot w_2 \cdot \ldots \cdot w_k$ pieces. When we apply the activation, the new function $g(x) = \max\{0,f(x)\}$, which is the output of this node, may have at most twice the number of pieces as $f$, because each original piece may be intersected by the $x$-axis; see Figure~\ref{fig:maxpieces}. Thus, after going through the layer, we take an affine combination of $w_{k+1}$ functions, each with at most $2\cdot(2^{k-1}\cdot(w_1+1)\cdot w_2 \cdot \ldots \cdot w_k)$ pieces. In all, we can therefore get at most $2\cdot(2^{k-1}\cdot(w_1+1)\cdot w_2 \cdot \ldots \cdot w_k)\cdot w_{k+1}$ pieces, which is equal to $2^{k}\cdot(w_1+1)\cdot w_2 \cdot \ldots \cdot w_k\cdot w_{k+1},$ and the induction step is completed.
\end{proof}

Lemma~\ref{lem:size-bounds-0} has the following consequence about the depth and size tradeoffs for expressing functions with agiven number of pieces.

\begin{lemma}\label{lem:size-bounds} Let $f:\R\to \R$ be a piecewise linear function with $p$ pieces. If $f$ is represented by a ReLU DNN with depth $k+1$, then it must have size at least $\frac12kp^{1/k} - 1$. Conversely, any piecewise linear function $f$ that is represented by a ReLU DNN of depth $k+1$ and size at most $s$, can have at most $(\frac{2s}{k})^{k}$ pieces.
\end{lemma}

\begin{proof}
Let widths of the $k$ hidden layers be $w_1, \ldots, w_k$. By Lemma~\ref{lem:size-bounds-0}, we must have \begin{equation}\label{eq:piece-bound}2^{k-1}\cdot(w_1+1)\cdot w_2 \cdot \ldots \cdot w_k \geq p.\end{equation}
By the AM-GM inequality, minimizing the size $w_1 + w_2 + \ldots + w_k$ subject to \eqref{eq:piece-bound}, means setting $w_1 + 1 = w_2 = \ldots = w_k$. This implies that $w_1 + 1 = w_2 = \ldots = w_k \geq \frac12p^{1/k}$. The first statement follows. The second statement follows using the AM-GM inequality again, this time with a restriction on $w_1+ w_2 + \ldots + w_k$.
\end{proof}

\section{Proof of Proposition~\ref{thm:max-0-x-y}}\label{sec:proof-max-0-x-y}

We first observe that the set of points where $\max\{0,x_1,x_2\}$ is not differentiable is precisely the union of the three half-lines (or rays) $\{(x_1,x_2): x_1=x_2, x_1\geq 0\} \cup \{(0,x_2): x_2 \leq 0\} \cup \{(x_1,0): x_1 \leq 0\}$. 
On the other hand, consider any Sum-of-ReLU circuit, which can be expressed as a function of the form
$$f(x) = \sum_{i=1}^w c_i\max\{0,\langle a^i, x \rangle + b_i\},$$ where $w\in \N$ is the number of ReLU gates in the ciruit, and $a^i \in \R^2$, $b_i, c_i \in \R$ for all $i=1, \ldots, w$. This implies that $f(x)$ is piecewise linear and the set of points where $f(x)$ is not differentiable is {\em precisely} the union of the $w$ lines $\langle a^i, x \rangle + b_i = 0$, $i=1, \ldots, w$. 
Since a union of lines cannot equal the union of the three half-lines $\{(x_1,x_2): x_1=x_2, x_1\geq 0\} \cup \{(0,x_2): x_2 \leq 0\} \cup \{(x_1,0): x_1 \leq 0\}$, we obtain the consequence that $\max\{0,x_1,x_2\}$ cannot be represented by a Sum-of-ReLU circuit, no matter how many ReLU gates are used.

\section{Simulating an LTF gate by a ReLU gate}\label{sec:LTF-ReLU}

\begin{claim}\label{LLbyLR}
Any LTF gate $\{-1,1\}^n \rightarrow \{-1,1\}$ can be simulated by a Sum-of-ReLU circuit with at most $2$ ReLU gates. 
\end{claim}
\begin{proof}
Given a LTF gate $(2{\mathbf 1}_{\langle a, x\rangle + b \geq 0}-1)$ it separates the points in $\{-1,1\}^n$ into two subsets such that the plane $\langle a, x\rangle + b = 0$ is a separating hyperplane between the two sets. Let $-p<0$ be the value of the function $\langle a, x\rangle + b$ at  that hypercube vertex on the ``-1'' side which is closest to this separating plane. Now imagine a continuous piecewise linear function $f : \mathbb{R} \rightarrow \mathbb{R}$ such that $f(x) =-1$ for $x \leq -p$, $f(x) = 1$ for $x \geq 0$ and for $x \in (-p,0)$ $f$ is the straight line function connecting $(-p,-1)$ to $(0,1)$. It follows from Theorem \ref{thm:1-dim-2-layer} that this $f$ can be implemented by a $\mathbb{R} \rightarrow \mathbb{R}$ Sum-of-ReLU with at most $2$ ReLU gates hinged at the points $-p$ and $0$ on the domain. Because the affine transformation $\langle a, x\rangle + b$ can be implemented by the wires connecting the $n$ input nodes to the layer of ReLUs it follows that there exists a $\mathbb{R}^n \rightarrow \mathbb{R}$ Sum-of-ReLU with at most $2$ ReLU gates implementing the function $g(x) = f(\langle a, x\rangle + b) : \mathbb{R}^n \rightarrow \mathbb{R}$. Its clear that $g(\x) = \text{LTF}(\x)$ for all $\x \in \{-1,1\}^n$. 
\end{proof}

\section{PARITY on $k-$bits can be implemented by a $O(k)$ Sum-of-ReLU circuit}\label{sec:Parity-ReLU}

For this proof its convenient to think of the PARITY function as the following map, 

\begin{align}\label{parity}
\text{PARITY} : \{0,1\}^k &\rightarrow \{0,1\}\\
\x &\mapsto \left (\sum_{i=1}^k x_i \right )\mod 2
\end{align}

Its clear that that in the evaluation of the PARITY function as stated above the required sum over the coordinates of the input Boolean vector will take as value every integer in the set, $\{0,1,2,..,k\}$. The PARITY function can then be lifted to a $f :\R \rightarrow \R$ function such that, $f(y) = 0$ for all $y \leq 0$, $f(y) = y \mod 2$ for all $y \in {1,2,..,k}$, $f(y) = k \mod 2$ for all $y > k$ and for any $y \in (p,p+1)$ for $p \in \{0,1,..,k-1\}$ $f$ is the straight line function connecting the points, $(p, p \mod 2)$ and $(p +1, (p+1) \mod 2)$. Thus $f$ is a continuous piecewise linear function on $\R$ with $k+2$ linear pieces. Then it follows from Theorem \ref{thm:Lpapprox}   that this $f$ can be implemented by a $\mathbb{R} \rightarrow \mathbb{R}$ Sum-of-ReLU circuit with at most $k+1$ ReLU gates hinged at the points $\{0,1,2,..,k\}$ on the domain. The wires from the $k$ inputs of the ReLU gates can implement the linear function $\sum_{i=1}^k x_i$. Thus it follows that there exists a $\R^k \rightarrow \R$ Sum-of-ReLU circuit (say C) such that, $\text{C}(\x) = {\text PARITY}(\x)$ for all $\x \in \{0,1\}^k$.  

\section{Proof of Theorem \ref{deepReLU} (Proving smallness of the  sign-rank of LTF-of-(ReLU)$^{\mathbf d-1}$ with weight restrictions only on the bottom most layer)}\label{app:signranksmall} 

For a $\{-1,1\}^M \rightarrow \{-1,1\}$ LTF-of-ReLU circuit with any given weights on the network the inputs to the threshold function of the top LTF gate are some set of $2^M$ real numbers (one for each input). Over all these inputs let $p>0$ be the distance from $0$ of the largest negative number on which the LTF gate ever gets evaluated. Then by increasing the bias at this last LTF gate by a quantity less then $p$ we can ensure that no input to this LTF gate is $0$ while the entire circuit still computes the same Boolean function as originally. So we can assume without loss of generality that the input to the threshold function at the top LTF gate is never $0$. We also recall that the weights at the bottom most layer are constrained to be integers of magnitude at most $W>0$. 

~\\ 
Let this depth $d$ LTF-of-(ReLU)$^{d-1}$ circuit map $\{-1,1\}^{m}\times\{-1,1\}^{m} \rightarrow \{-1,1\}$. Let $\{w_{k}\}_{k=1}^{d-1}$ be the widths of the ReLU layers at depths indexed by increasing $k$ with increasing distance from the input. Thus, the output LTF gate gets $w_{d-1}$ inputs; the $j$-th input, for $j=1,2,..,w_{d-1}$, is the output of a circuit $C_j$ of depth $d-1$ composed of only ReLU gates. Let $f_j(\x,\y):\{-1,1\}^{m}\times\{-1,1\}^{m} \rightarrow \R$ be the pseudo-Boolean function implemented by $C_j$.

~\\
Thus the output of the overall LTF-of-(ReLU)$^{d-1}$ circuit is, 
\begin{align}\label{top} 
 f(\x,\y):=\text{LTF} \left [ \beta + \sum_{j=1}^{w_{d-1}} \alpha_j f_j(\x,\y)\right]
\end{align}

\begin{lemma}\label{ReLU-circuit-rank}
 Let $k \geq 1$ and $w_1, \ldots, w_k \geq 1$ be natural numbers.  Consider a family of depth $k+1$ circuits (say indexed by $i \in I$ for some index set $I$) with $2m$ inputs and a single output and consisting of only ReLU gates.  Let all of them have $w_j$ ReLU gates at depth $j$, with $j=1$ corresponding to the layer closest to the input (note that single output ReLU gate is not counted here). Moreover, let all of the circuits in the family have the same weights in all their layers except for the layer closest to the output. We restrict the inputs to $\{-1,1\}^m \times \{-1,1\}^m$ and let the $i^{th}$ circuit ($i\in I$) implement a pseudo-Boolean function $g_i: \{-1,1\}^m \times \{-1,1\}^m\to \R$. Assume that the weights of the $w_1$ ReLU gates in the layer closest to the input are restricted as per Definition~\ref{def:weight-restriction}. For every $i\in I$, define the $2^m\times 2^m$ matrix $G_i(\x,\y)$ whose rows are indexed by $\x \in \{-1,1\}^m$ and columns are indexed by $\y \in \{-1,1\}^m$ as follows: $$G_i(\x,\y) = g_i(\x,\y).$$ Then there exists a fixed way to order the rows and columns such that for each $G_i$ there exists a {\em contiguous} partitioning (which can depend on $i$) of its rows and columns into $O\big((\prod_{i=1}^k w_i)(mW)\big)$ blocks (thus, $G_i$ has $O\big((\prod_{i=1}^k w_i)^2(mW)^2\big)$ blocks), and within each block $G_i$ is constant valued.
\end{lemma}

~\\
Before we prove the above lemma, let us see why it implies Theorem~\ref{deepReLU}. 
\begin{proof} {(of Theorem ~\ref{deepReLU})} Let $F_j(\x,\y)$ be the matrix obtained from the ReLU circuit outputs $f_j(\x,\y)$ from~\eqref{top}, and let $F(\x,\y)$ be the matrix obtained from $f(\x,\y)$. Let $J_{2^m\times 2^m}$ be the matrix of all ones. Then

\begin{align*}
\text{sign-rank}(F(\x,\y))
= &  \;\text{sign-rank} \left ( \text{sign} \left [ \beta J_{2^m\times 2^m} + \sum_{j=1}^{w_{d-1}} \alpha_j F_j(\x,\y) \right] \right )\\
\leq &\; \text{rank} \left ( \beta J_{2^m\times 2^m} + \sum_{j=1}^{w_{d-1}} \alpha_j F_j(\x,\y) \right )\\
\leq &\; 1 + \sum_{j=1}^{w_{d-1}}\text{rank}(F_j(\x,\y))\\
= & O\left (\left ( \prod_{k=1}^{d-1} w_k\right )^2 (mW)^2\right )
\end{align*}
where the first inequality follows from the definition of sign-rank, the second inequality follows from the subadditivty of rank and the last inequality is a consequence of using Lemma~\ref{ReLU-circuit-rank} at depth $k+1 = d-1$. Indeed, a matrix with block structure as in the conclusion of Lemma~\ref{ReLU-circuit-rank} has rank at most $O\big((\prod_{i=1}^k w_i)^2(mW)^2\big)$ by expressing it as a sum of these many matrices of rank one and using subaddivity of rank.
~\\
Now we recall that the Chattopadhyay-Mande function $g$ (which is linear sized depth $2$ LTF) on $2m = 2(n^{\frac{4}{3}}+n\log n)$ bits has sign-rank $\Omega(2^{\frac{n^{\frac{1}{3}}}{81}})$.  It follows that we can find a constant $C >1$ s.t for all large enough $n$ we have,  $C^4n^{\frac{4}{3}} \geq m$. Then we would have, $\text{sign-rank}(g) = \Omega(2^{ \frac{m^{\frac 1 4}}{81C} })$. From the above upper bound on the sign-rank of our bottom layer weight restricted LTF-of-(ReLU)$^{d-1}$ with widths $\{w_k\}_{k=1}^{d-1}$ it follows that for this to represent this Chattopadhyay-Mande function it would need, $\left (\left ( \prod_{k=1}^{d-1} w_k\right )^2 (mW)^2\right ) = \Omega(2^{\frac{m^{\frac{1}{4}}}{81 C}})$.  Hence it follows by the ``AM$\geq$GM" inequality that the size ($1+\sum_{k=1}^{d-1}w_i$) required for such  LTF-of-(ReLU)$^{d-1}$ circuits to represent the Chattopadhyay-Mande function is $\Omega \left ( (d-1) \left [ \frac{2^{m^{\frac 1 8}}}{mW} \right ]^{\frac 1 {(d-1)}} \right )$.
~\\ \\
The statement about LTF circuits is a straightforward consequence of the above result and Claim~\ref{LLbyLR} in Appendix~\ref{sec:LTF-ReLU} which says that any LTF gate can be simulated by 2 ReLU gates.
\end{proof} 

~\\
Towards proving Lemma~\ref{ReLU-circuit-rank} we first make the following observation, 

\begin{claim}\label{claim:block-matrices}
Let $w, M, D$ be fixed natural numbers. Let $A_1, \ldots, A_w$ be any $M\times M$ matrices such that there exists a fixed way to order the rows and columns for each of the $A_i$ such that they get partitioned contiguously into $D$ blocks (not necessarily equal in size) and this partitioning is such that $A_i$ is constant valued within each of the $D^2$ blocks. Then $A := A_1 + \ldots + A_w$ is an $M\times M$ matrix whose rows and columns can be partitioned contiguously into $w(D-1) + 1$ groups such that $A$ is constant valued within each block defined by this partition of the rows and columns.
\end{claim}

\begin{proof} The partition of the rows of $A_i$ into $D$ contiguous blocks is equivalent to a choice of $D-1$ lines out of $M-1$ lines. (Potentially a different set of $D-1$ lines for each $A_i$) But the guarantee that this partitioning is induced in each of the $A_i$ by the same ordering of the rows means that When we sum the matrices, the refined partition in the sum corresponds to some selection of $w(D-1)$ lines out of the $M-1$ lines. This gives us at most $w(D-1) + 1$ contiguous blocks among the rows of the sum matrix. The same argument holds for the columns.
\end{proof}

\begin{proof}[Proof of Lemma~\ref{ReLU-circuit-rank}]
We will prove this Lemma by induction on $k$. 

\paragraph{The base case of the induction $k=1$ (i.e depth $2$)}  A single ReLU gate in the bottom most layer of the net which receives a tuple of vectors $(\x,\y)$ as input gives as output the number, $\max\{0, \langle \a^1, \x\rangle + \langle \a^2,\y\rangle + b\}$, for some $\a^1, \a^2 \in \R^m$ and $b \in \R$. Since the entries of $\a^1, \a^2$ and $b$ are assumed to be integers bounded by $W >0$ and $\x, \y \in \{-1,1\}^m$, the terms $\langle \a^1, \x\rangle$ and $\langle \a^2,\y\rangle$ can each take at most $O(mW)$ different values. So we can arrange the rows and columns of the $2^m \times 2^m$ dimensional output matrix of this gate in increasing order of $\langle \a^1,\x\rangle$ and $\langle \a^2,\y\rangle$ and then partition the rows and columns contiguously according to these values. And we note that because of the weight restriction as in Definition \ref{def:weight-restriction} that applies to each of ReLU gates in the bottom most layer, the ordering in increasing value of the inner-products as said above induces the same ordering of the rows for each of these output matrices at the different ReLU gates. Similarly, the same ordering is induced on the columns (note that the orderings for the rows may be different from the ordering for the columns; what is important is that the rows have the same ordering across the family and similarly for the columns.)
~\\ \\
Now we notice that the structure of the output matrices of the ReLU gates of the bottom most layer as described above is what is assumed in Claim \ref{claim:block-matrices}.
~\\ \\ 
Thus if $\{G_p(\x,\y)\}_{p=1,..,w_1}$ are the output matrices at each of the ReLU gates in the bottom most layer, then for some $\a \in \R^{w_1}$ and $b \in \R$ at depth $2$ the output matrix of any of the ReLU gates is given by, $\max\{0, bJ_{2^m\times 2^m} + \sum_{i=1}^{w_{1}} a_i G_i(\x,\y)\}$  where $J_{2^m\times 2^m}$ is the matrix of all ones and the ``max'' is taken entrywise. Then the base case of the induction is settled by applying Claim \ref{claim:block-matrices} on this matrix,  $bJ_{2^m\times 2^m} + \sum_{i=1}^{w_{1}} a_iG_i(\x,\y)$ with $D=O(mW)$ and $w=w_1$.
~\\ \\
We further note that the computations happening at the depth $2$ ReLU gates obviously do not change the ordering of the rows and columns frozen in at depth $1$, i.e., in the $G_i$ matrices. But with different depth $2$ gates, i.e., different choices of the vectors $\a$ and the number $b$, because of the linearity (in $\a$ and $b$) of the operation of forming, $bJ_{2^m\times 2^m} + \sum_{i=1}^{w_{1}} a_iG_i(\x,\y)$ they all have the same contiguous pattern of constant valued submatrices. Thus the depth $2$ output matrices continue to satisfy the hypothesis of Claim \ref{claim:block-matrices}. 

~\\
To complete the induction step, we consider a family of ReLU circuits with depth $k+1$ corresponding to different choices of, $b \in \R$ and $\a \in \R^{w_k}$ which can be seen as computing $g(\x, \y) = \max\{0, b + \sum_{p=1}^{w_{k}} a_pg_p(\x,\y)\}$ where $\{g_p(\x,\y)\}_{p=1,..,w_k}$ is a family of ReLU circuits of depth $k$ who by induction satisfy the lemma. Thus the corresponding output matrices of these depth $k+1$ circuits satisfy, 

\[ G(\x, \y) = \max\{0, bJ_{2^m\times 2^m} + \sum_{p=1}^{w_{k}} a_pG_p(\x,\y)\} \] 

~\\
where $G_p$ is the matrix form of $g_p$. Thus, the induction hypothesis applied to depth $k$ would then tell us that the rows and columns of each matrix $G_p$ can be partitioned contiguously into $O\big((\prod_{i=1}^{k-1} w_i)(mW)\big)$ such that $G_p$ is constant valued within each block. Then, by Claim~\ref{claim:block-matrices}, the rows and columns of the matrix $bJ_{2^m\times 2^m} + \sum_{p=1}^{w_{k}} a_pG_p(\x,\y)$ can be partitioned into $O\big((\prod_{i=1}^k w_i)(mW)\big)$ contiguous blocks. Moreover, this ordering of the rows and colums does not vary across the different circuits in the family, because they all have the same weights in the bottom most layer. Hence, the same ordering works for all the circuits in the family.

\end{proof}

\end{subappendices}

\chapter{\vspace{10pt} Provable Training of a $\relu$ gate}\label{chaptrain_new} 

\section{A review of provable neural training}

In this chapter we will prove results about trainability of a $\relu$ gate under more general settings than hitherto known till date.  To the best of our knowledge about the state-of-the-art in deep-learning both empirical and population risk minimization questions are typically solvable in either of the following two mutually exclusive scenarios : {\it  Scenario $1$ : Semi-Realizable Data} i.e the data comes as tuples $\z = (\x,\y)$ with $\y$ being the noise corrupted output of a net (of known architecture) when given $\x$ as an input. And {\it Scenario $2$ : Semi-Agnostic Data} i.e data comes as tuples $\z = (\x,\y)$ with no obvious functional relationship between $\x$ and $\y$ but there could be geometrical or statistical assumptions about the $\x$ and $\y$. 

We note that its not very interesting to work in the fully agnostic setting as in that case training even a single ReLU gate can be SPN-hard as shown in \cite{goel2016reliably}  On the other hand the simplifications that happen for infinitely large networks have been discussed since \cite{neal1996priors} and this theme has had a recent resurgence in works like \cite{chizat2018global,jacot2018neural}. Eventually this lead to an explosion of literature getting linear time training of various kinds of neural nets when their width is a high degree polynomial in training set size, inverse accuracy and inverse confidence parameters (a very {\it unrealistic} regime),  \citep{lee2018deep,wu2019global,dugradient,su2019learning,kawaguchi2019gradient,huang2019dynamics,allen2019convergenceDNN,allen2019learning,allen2019convergenceRNN,du2018power,zou2018stochastic,zou2019improved,arora2019exact,arora2019harnessing,li2019enhanced,arora2019fine,lee2018deep}. The essential proximity of this regime to kernel methods have been thought of separately in works like  \cite{allen2019can,wei2019regularization} 

Even in the wake of this progress, it remains unclear as to how any of this can help establish rigorous guarantees about ``smaller" neural networks or more pertinently for constant size neural nets which is a regime closer to what is implemented in the real world. Thus motivated we can summarize what is open about training depth $2$ nets into the following two questions, 

\begin{enumerate}
    \item {\bf Question $1$} Can {\it any} algorithm train a ReLU gate to $\epsilon-$accuracy in $\poly({\rm input{-}dimension},\frac 1 \epsilon)$ time using  neither symmetry nor compact support assumptions on the distribution? 
    \begin{itemize}
        \item {\bf Question $1.5$} Can a single ReLU gate be trained using (Stochastic) Gradient Descent with (a) random/arbitrary initialization and (b) weakly constrained data distribution - at least allowing it to be non-Gaussian and preferably non-compactly supported?
    \end{itemize}
    \item {\bf Question $2$} Can a neural training algorithm work with the following naturally wanted properties being simultaneously true? 
    \begin{enumerate} 
    {\small 
    \item Nets of depth $2$ with a constant/small number of gates. 
    \item The training data instances (and maybe also the noise) would have non-Gaussian non-compactly supported distributions. 
    \item Less structural assumptions on the weight matrices than being of the single filter convolutional type.
    \item $\epsilon-$approximate answers be obtainable in at most $\poly({\rm input{-}dimension},\frac {1}{\epsilon})$ time.
    }
    \end{enumerate} 
\end{enumerate}

\section{A summary of our results}

We make progress on some of the above fronts by drawing inspiration from two distinct streams of literature and often generalizing and blending techniques from them. First of them are the different avatars of the iterative stochastic non-gradient ``Tron" algorithms analyzed in the past like, \cite{rosenblatt1958perceptron,pal1992multilayer,freund1999large,kakade2011efficient,klivans2017learning,goel2017learning,goel2018learning}. The second kind of historical precedence that we are motivated by are the different works which have shown how some of the desired theorems about gradient descent can be proven if designed noise is injected into the algorithm in judicious ways, \citep{raginsky2017non,xu2018global,zhang2017hitting,durmus2019analysis,lee2019online,jin2018local,mou2018generalization,li2019generalization}. Here we will be working with the simplest neural net which is just a single $\relu$ gate mapping $\R^n \ni \x \mapsto \max \{0,\w^\top \x \} \in \R$ for $\w \in \R^n$ being its weight. In here already the corresponding empirical or the population risk is neither convex nor smooth in how it depends on the weights. Thus to the best of our knowledge none of the convergence results among these provable noise assisted algorithms cited above can be directly applied to this case because these proofs crucially leverage either convexity or very strong smoothness assumptions on the optimization objective. 
\bigskip 
We show $3$ kinds results in this chapter. 

In Section \ref{sec:almostSGDReLU} we have shown a very simple iterative stochastic algorithm to recover the underlying parameter $\w_*$ of the $\relu$ gate when realizable data allowed to be sampled online is of the form $(\x, \max \{ 0, \w_*^\top\x \} )$. The distributional condition is very mild which essentially just captures the intuition that enough of our samples are such that $\w_*^\top\x >0$.
~\\ \\
Not only is our algorithm's run-time near-optimal, but to the best of our knowledge the previous attempts at this problem have solved this only for the Gaussian distribution
(\cite{soltanolkotabi2017learning,kalan2019fitting}). Some results like \cite{goel2018learning} included a solution to this above problem as a special case of their result while assuming that the the data distribution is having a p.d.f symmetric about the origin. Thus in contrast to all previous attempts our assumptions on the distribution are significantly milder.  
~\\ \\
In Section \ref{sec:GDNoiseReLU} we show the first-of-its-kind analysis of gradient descent on a $\relu$ gate albeit when assisted with the injection of certain kinds of noise. We assume that the labels in the data are realizable but we make no assumptions on the distribution of the domain.  We make progress by showing that such a noise assisted GD in such a situation has a ``diffusive" behaviour about the global minima i.e  after {\rm T} steps of the algorithm starting from {\it anywhere}, w.h.p {\it all} the steps of the algorithm have been within $\sqrt{T}$ distance of the global minima of the function.  The key idea here is that of ``coupling" which shows that from the iterates of noise injected gradient descent on the squared loss of a $\relu$ gate one can create a discrete bounded difference super-martingale. 

\remark We would like to emphasize to the reader that in such a distribution free regime as above,  {\it no} algorithm is expected to provably train. Also note that the result is parametric in the magnitude of the added noise and hence one can make the algorithm be arbitrarily close to being a pure gradient descent. 

~\\
In Section \ref{sec:GLMTron} we re-analyze a known algorithm called ``GLM-Tron" under more general conditions than previously to show how well it can do (empirical) risk minimization on any Lipschitz gate with Lipschitz constant $<2$ (in particular a $\relu$ gate) in the noisily realizable setting while no assumptions are being made on the distribution of the noise beyond their boundedness - hence the noise can be { ``adversarial"}. We also point out how the result can be improved under some  assumptions on the noise making it more benign. Note that in contrast to the training result in Section \ref{sec:almostSGDReLU} which used a stochastic algorithm, here we are using full-batch iterative updates to gain these extra abilities to deal with more general gates, (adversarial) noise and essentially no distributional assumptions on training data.

\section{Almost distribution free learning of a $\relu$ gate}\label{sec:almostSGDReLU}

If data, $\x$, is being sampled from a distribution ${\cal D}$ and the corresponding true labels are being generated from a $\relu$ gate as $\max \{0, \w_*^\top \x \}$ for some $\w_* \in \R^n$ unknown to us,  then the question of learning this $\relu$ gate in this realizable setting is essentially the task of trying to solve the following optimization problem while having only sample access to ${\cal D}$, $\min_{\w \in \R^n} \E_{\x \sim {\cal D}} \Big [ \Big ( \max \{0,\w^\top \x\} - \max \{0,\w_{*}^\top \x\} \Big )^2\Big]$

 In contrast to all previous work we show the following simple algorithm which solves this learning problem to arbitrarily good accuracy assuming only very mild conditions on ${\cal D}$. We leverage the simple intuition that if we can get to see enough labels $y = \max \{0, \w_*^\top \x \}$ where $y >0$ then $\w^*$ is just the answer to the linear regression problem on those samples. 

\begin{algorithm}[]
\caption{Modified SGD for a $\relu$ gate}
\label{dadushrelu}
\begin{algorithmic}[1]
\State {\bf Input:} Sampling access to a distribution ${\cal D}$ on $\R^n$.
\State {\bf Input:} Oracle access to the true labels  when queried with some $\x \in \R^n$  
\State {\bf Input:} An arbitrarily chosen starting point of $\w_1 \in \R^n$ and a constant $\alpha <0$
\For {$t = 1,\ldots$}
    \State Sample $\x_t \sim {\cal D}$ and query the oracle with it. 
    \State The oracle replies back with $y_t = \max \{0, \w_*^{\top}\x_t \}$
    \State Form the gradient-proxy, 
    \[\g_t := \alpha {\bf 1}_{y_t > 0}( y_t  - \w_t^\top \x_t)\x_t\]
    \State $\w_{t+1} := \w_t - \eta \g_t$
\EndFor 
\end{algorithmic}
\end{algorithm}

\begin{theorem}\label{thm:dadushrelu}
We assume that the the data distribution ${\cal D}$ is s.t $\mathbb{E}\Big[  \norm{\x}^4 \Big ]$ and the covariance matrix $\E \Big [ \x  \x^\top \Big ] $ exist. Suppose $\w_*$ is s.t $ \mathbb{E}\Big[ {\bf 1}_{\w_*^\top \x >0 } \x  \x^\top \Big]$ is positive definite. Then if Algorithm \ref{dadushrelu} is run with $\alpha <0$ and  $\eta = \frac{\lambda_{\min} \Big ( \mathbb{E}\Big[ {\bf 1}_{\w_*^\top \x >0 } \x  \x^\top \Big] \Big )}{\vert \alpha \vert \mathbb{E}\Big[
{\bf 1}_{\w_*^\top \x >0}  \cdot \norm{\x}^4 \Big ] }$ starting from starting from $\w_1 \in \R^n$ then for ${\rm T} = O \Big ( \log \frac{\norm{\w_1 - \w_*}^2}{\epsilon^2 \delta}  \Big)$ we would have, 

\[ \mathbb{P} \Big [  \norm{\w_{\rm T} - \w_*}^2   \leq \epsilon^2 \Big ] \geq 1 - \delta \]
\qedsymbol
\end{theorem}
~\\
It's clear that $\norm{\w_{\rm T} - \w_*}^2  \leq \epsilon^2 \implies \E_{\x \sim {\cal D}} \Big [ \Big ( \max \{0,\w_{\rm T}^\top \x\} - \max \{0,\w_{*}^\top \x\} \Big )^2\Big] \leq \epsilon^2 \E \Big [ \norm{\x}^2 \Big ]$ and hence Algorithm \ref{dadushrelu} is in effect approximately solving the risk minimization problem that we set out to solve. Also note that (a) the above convergence hold starting from arbitrary initialization $\w_1$, (b) the proof will establish along the way that the assumptions being made in the theorem are enough to ensure that the choice of $\eta$ above is strictly positive and (c) for ease of interpretation we can just set $\alpha = -1$ in the above and observe how closely the choice of $\g_t$ in Algorithm \ref{dadushrelu} resembles the stochastic gradient that is commonly used and is known to have great empirical success.

\begin{proof}[{\bf Proof of Theorem \ref{thm:dadushrelu}}]
Let the training data sampled till the iterate $t$ be $S_t = \{(x_1,y_1),\ldots,(x_t,y_t)\}$
From the algorithm we know that the weight vector update at $t$-th iteration is $\w_{t+1} = \w_t + \eta \g_t$. 
Thus,
\begin{align}\label{term0}
\nonumber &\norm{\w_{t+1}-\w_*}^2 
= \norm{\w_t-\eta\g_t - \w_*}^2\\
&= \norm{\w_t-\w_*}^2 + \eta^2\norm{\g_t}^2 - 2\eta\langle\w_t-\w_*,\g_t\rangle 
\end{align}

We overload the notation to also denote by $S_t$ the $\sigma-$algebra generated by the random variables $\mathbf{x}_1, ..., \mathbf{x}_t$. Conditioned on $S_{t-1}$, $\w_t$ is determined while $\w_{t+1}$ and $\g_t$ are random and dependent on the random choice of $\x_t$. 

\begin{align}\label{DBothTerms}
\nonumber \E_{(\x_t,y_t)} \Big [ \norm{\w_{t+1} - \w_*}^2 \mid S_{t-1} \Big ] &=   \E_{(\x_t,y_t)} \Big [ \norm{\w_t - \w_*}^2 \mid S_{t-1} \Big ]\\
    \nonumber &+   \underbrace{(-2\alpha \eta) \E_{(\x_t,y_t)} \Big [ \Big \langle \w_t - \w_* ,  {\bf 1}_{y_t >0} \Big (y_t - \w_t^\top\x_t \Big ) \x_t \Big \rangle \mid S_{t-1} \Big ]}_{\text{Term }1}\\
    &+ \underbrace{\eta^2 \E_{(\x_t,y_t)} \Big [ \norm{\g_t}^2 \mid S_{t-1} \Big ]}_{\text{Term } 2}    
\end{align}

Now we simplify the last two terms of the RHS above, starting from the rightmost, 

\begin{align}
\nonumber &\text{Term } 2 = \mathbb{E} \Bigg [ \norm{\eta \g_t}^2 \mid S_{t-1} \Bigg ] 
= \eta^2 \alpha^2 \mathbb{E}\Bigg[ {\bf 1}_{y_t >0} (y_t - {\w_t}^\top \x_t)^2 \cdot \norm{\x_t}^2  \mid S_{t-1} \Bigg ]\\
\nonumber & = \eta^2 \alpha^2  \cdot \mathbb{E}\Bigg[ {\bf 1}_{y_t >0}  \big(\max \{0, \w_*^{\top}\x_t \}  - \w_t^\top \x_t \big)^2 \cdot \norm{\x_t }^2 \mid S_{t-1}\Bigg]\\
\nonumber & \leq \eta^2\alpha^2 \mathbb{E}\Bigg[
{\bf 1}_{y_t >0}  \norm{\w_* -\w_t}^2\cdot \norm{\x_t}^4 \mid S_{t-1}\Bigg] \\
\nonumber & \leq \eta^2\alpha^2   \norm{\w_* -\w_t}^2 \times \mathbb{E}\Bigg [
  {\bf 1}_{\w_*^\top \x_t >0}  \cdot \norm{\x_t}^4 \Bigg ]
\end{align}

Note that in the above step the quantity, $\mathbb{E}\Big[ {\bf 1}_{\w_*^\top \x >0}  \cdot \norm{\x}^4 \Big ]$ is finite and it is easy to see why this is true given that $\forall ~\w_*, \mathbb{E}\Big[ {\bf 1}_{\w_*^\top \x >0}  \cdot \norm{\x}^4 \Big ] \leq \mathbb{E}\Big[  \norm{\x}^4 \Big ]$ and we recall that the quantity in this upperbound has been assumed to be finite in the hypothesis of the theorem.    

Now we simplify {\rm Term} $1$ to get, 

\begin{align}
\nonumber {\rm Term} 1& = -2\eta \alpha \mathbb{E}\Bigg[ {\bf 1}_{y_t >0} \Big (y_t - \w_t^\top \x_t \Big)\cdot (\w_t-\w_*)^\top  \x_t\bigg|S_{t-1}\Bigg]\\
\nonumber & = -2\eta \alpha \mathbb{E}\Bigg[ {\bf 1}_{y_t >0} \Big (\max \{0, \w_*^{\top}\x_t \}  - \w_t^\top \x_t \Big) \times  (\w_t-\w_*)^\top  \x_t \bigg| S_{t-1}\Bigg] \\
\nonumber & \leq  -2\eta \alpha \mathbb{E}\Bigg[ (\w_*  - \w_t  )^\top {\bf 1}_{y_t >0} \x_t  \x_t^\top (\w_t-\w_*)   \bigg| S_{t-1}\Bigg] \\
\nonumber & \leq -2 \eta \vert \alpha \vert \mathbb{E}\Bigg [ 
(\w_t - \w_* )^\top {\bf 1}_{y_t >0} \x_t  \x_t^\top (\w_t-\w_*)   \bigg| S_{t-1}\Bigg ] \\
& \leq -2 \eta \vert \alpha \vert \lambda_{\min} \Big ( \mathbb{E}\Bigg [ {\bf 1}_{\w_*^\top \x_t >0 >0} \x_t  \x_t^\top    \Bigg ] \Big ) \norm{\w_t - \w^*}^2 
\end{align} 

In the above step we invoked that the quantity $\mathbb{E}\Bigg [ {\bf 1}_{\w_*^\top \x_t >0 >0} \x_t  \x_t^\top  \Bigg ]$ exists and its easy to see why this is true given that $\forall ~\w_*, \mathbb{E}\Bigg [ {\bf 1}_{\w_*^\top \x >0 } \x  \x^\top  \Bigg ] \leq \mathbb{E}\Bigg [ \x  \x^\top  \Bigg ]$ and we recall that the covariance occurring of the distribution has been assumed to be finite in the hypothesis of the theorem.


We can combine both the upper bounds obtained above into the RHS of equation \eqref{DBothTerms} to get,  

\begin{align}
\nonumber &\E_{(\x_t,y_t)} \Big [ \norm{\w_{t+1} - \w_*}^2 \mid S_{t-1} \Big ]\\
&\leq   \Bigg (  1 -2 \eta \vert \alpha \vert \times \lambda_{\min} \Big ( \mathbb{E}\Big[ {\bf 1}_{\w_*^\top \x_t >0 } \x_t  \x_t^\top \Big] \Big ) + \eta^2\alpha^2  \times \mathbb{E}\Big[
  {\bf 1}_{\w_*^\top \x_t >0}  \cdot \norm{\x_t}^4  \Big ] \Bigg ) \norm{\w_t - \w_*}^2
\end{align} 

We note that the two expectations on the RHS are properties of the distribution of the data $\x$ i.e ${\cal D}$ and $\w_*$ and we make the notation explicitly reflect that. $\x_t$ is a random variable that is independent of $\w_t$ since $\x_t$ is independent of $\x_1,\ldots,\x_{t-1}$. Hence by taking total expectation of the above we have,

\begin{align} \label{eq:expConv}
\E \Big [ \norm{\w_{t+1} - \w_*}^2 \Big ] \leq  \Bigg (  1 -2 \eta \vert \alpha \vert \lambda_{\min} \Big ( \mathbb{E}\Big[ {\bf 1}_{\w_*^\top \x >0 } \x \x^\top \Big] \Big ) + \eta^2\alpha^2  \times \mathbb{E}\Big[
  {\bf 1}_{\w_*^\top \x >0}  \cdot \norm{\x}^4 \Big ] \Bigg ) \E \Big [  \norm{\w_t - \w_*}^2 \Big ] 
  \end{align} 

Now we see that for $X_t := \E \Big [  \norm{\w_t - \w_*}^2 \Big ]$ the above is a recursion of the form given in Lemma \ref{recurse} with $c_2 =0$, $C = \norm{\w_1 - \w_*}^2$, $\eta'=\eta \vert \alpha \vert$, $b = 2\lambda_{\min} \Big ( \mathbb{E}\Big[ {\bf 1}_{\w_*^\top \x >0 } \x  \x^\top \Big] \Big )$ and $c_1 = \mathbb{E}\Big[
{\bf 1}_{\w_*^\top \x >0}  \cdot \norm{\x}^4 \Big ]$  

Now we note the following inequality,  

\begin{align*}
    \mathbb{E}\Big[ {\bf 1}_{\w_*^\top \x >0}  \cdot \norm{\x}^4 \Big ] &=  \mathbb{E}\Big[ {\bf 1}_{\w_*^\top \x >0}  \cdot (\x^\top \x)^2  \Big ] = \mathbb{E}\Big[  \Big (( {\bf 1}_{\w_*^\top \x >0}^{\frac 1 4}\x)^\top ( {\bf 1}_{\w_*^\top \x >0}^{\frac 1 4}\x) \Big )^2  \Big ]\\
    &=  \mathbb{E}\Big[ \Big ( \Tr \Big ( ({\bf 1}_{\w_*^\top \x >0}^{\frac 1 4}\x)^\top ( {\bf 1}_{\w_*^\top \x >0}^{\frac 1 4}\x) \Big ) \Big )^2 \Big ]\\
    &= \mathbb{E}\Big[  \Big ( \Tr \Big ( {\bf 1}_{\w_*^\top \x >0}^{\frac 1 2}\x \x^\top  \Big ) \Big )^2 \Big ]\
\end{align*}

We note that the function $\R^{n \times n} \ni {\bf Y} \mapsto \Tr^2({\bf Y}) \in \R $ is convex and hence by Jensen's inequality we have,  

\begin{align*}
    \mathbb{E}\Big[ {\bf 1}_{\w_*^\top \x >0}  \cdot \norm{\x}^4 \Big ]  \geq \Tr^2(\mathbb{E} \Big[ {\bf 1}_{\w_*^\top \x >0}^{\frac 1 2}\x \x^\top \Big ] ) = \Big ( \sum_{i=1}^n \lambda_i \Big ( \mathbb{E}\Big[ {\bf 1}_{\w_*^\top \x >0 } \x  \x^\top \Big ] \Big )  \Big )^2 \geq n^2 \lambda_{\min}^2 \Big ( \mathbb{E}\Big[ {\bf 1}_{\w_*^\top \x >0 } \x  \x^\top \Big] \Big )
\end{align*}

In the above $\lambda_i$ indicates the $i^{th}$ largest eigenvalue of the PSD matrix in its argument. And in particular the above inequality implies that $c_1 > \frac{b^2}{4}$. Now we recall that the assumptions in the theorem which ensure that $b >0$ and hence now we have $\frac{b}{c_1} >0$ and hence the step-length prescribed in the theorem statement is strictly positive. 

 Thus by invoking the first case of Lemma \ref{recurse} we have that  for $\eta' = \eta \vert \alpha \vert = \frac{b}{2c_1}$ we have $\forall \epsilon' >0$, $ \E \Big [  \norm{\w_{\rm T} - \w_*}^2 \Big ] \leq \epsilon'^2 $ for ${\rm T} = O \Big ( \log \frac{\norm{\w_1 - \w_*}^2}{\epsilon'^2}  \Big)$

Thus given a $\epsilon >0, \delta \in (0,1)$ we choose $\epsilon'^2 = \epsilon^2 \delta$ and then by Markov inequality we have what we set out to prove, 
\[ \mathbb{P} \Big [  \norm{\w_{\rm T} - \w_*}^2 \Big ] \leq \epsilon^2 \Big ] \geq 1 - \delta \]
\end{proof}

\newpage 
\section{Dynamics of noise assisted gradient descent on a single {\rm ReLU} gate}\label{sec:GDNoiseReLU}

As noted earlier it remains a significant challenge to prove the convergence of SGD or GD for a $\relu$ gate except for Gaussian data distributions. Towards this open question, we draw inspiration from ideas in \cite{lee2019online} and we focus on analyzing a noise assisted version of gradient dynamics on a ReLU gate in the realizable case as given in Algorithm \ref{GLDReLUParam}. In this setting we will see that we have some non-trivial control on the behaviour of the iterates despite making no distributional assumptions about the training data beyond realizability.  

\begin{algorithm}[h!]
\caption{Noise Assisted Gradient Dynamics on a single ReLU gate (realizable data)}
\label{GLDReLUParam}
\begin{algorithmic}[1]
\State {\bf Input:} We assume being given a step-length sequence $\{ \eta_t\}_{t=1,2,\ldots}$ 
and $\{ (\x_i,y_i) \}_{i=1,\ldots,S}$ tuples where $y_i = f_{\w_*}(\x_i)$ for some $\w_* \in \R^n$ where $f_\w$ is s.t \[ \R^n \ni \x \mapsto  f_\w (\x)= \textrm{ReLU}(\w^\top \x) = \max \{0, \w^\top \x \} \in \R \]
\State Start at $\w_0$ 
\For {$t = 0,\ldots$}
    \State Choice of Sub-Gradient $:= \g_t = -\frac{1}{S} \sum_{i = 1}^S \mathbf{1}_{\w_t^\top \x_i \geq 0} \Big (y_i - f_{\w_t}(\x_i) \Big )  \x_i$
    \State $\w_{t+1} := \w_t - \eta_t (\g_t + \xi_{t, 1}) +  \sqrt{\eta_t} \xi_{t, 2}$\Comment {\tt  $\xi_{t, 1}$ is $0$ mean bounded random variable}
    \State \Comment {\tt  $\xi_{t, 2}$ is a $0$ mean random variable s.t $ \E \Big [ \norm{\xi_{t,2}}^2 \Big ] < n$}
\EndFor 
\end{algorithmic}
\end{algorithm}

Note that in the above algorithm the indicator functions occurring in the definition of $\g_t$ are for the condition $\w_t^\top \x_i \geq 0$ for the $i^{th}{-}$data point. Whereas for the $\g_t$ used in Algorithm \ref{dadushrelu} in the previous section the indicator was for the condition $y_t > 0$ and hence dependent on $\w_*$ rather than $\w_t$. 

\begin{theorem}\label{confine:GLDReLUParam}
 We analyze Algorithm \ref{GLDReLUParam} with constant step length $\eta_t = \eta$ 
 Let $C := \max_{i=1,\ldots,S} \norm{\x_i}$, $S_1 >0$ be s.t $\forall t=1,\ldots, \norm{\xi_{t,1}} \leq S_1$ and  $\{\xi_{t,2}\}_{t=1,\ldots}$ be mean $0$, i.i.d as say $\xi_2$ s.t $\E \Big [ \norm{\xi_{t,2}}^2 \Big ] < n, ~\forall t = 1,\ldots$ 

Then for any $i_{\max} \in \Z^+$, $\lambda >0$, $C_L \in (0, \sqrt{n}), ~0 < \eta < \frac{1}{C^2 \sqrt{2i_{\max}}}$ and $r_*^2 \geq \lambda + \norm{\w_0 - \w_*}^2 + i_{\max}\Big \{ 2\eta^2 (C^4 r_*^2 + S_1^2) + \eta n   \Big \}$ we have, 

\begin{align}
    \nonumber &\mathbb{P} \Bigg [  \exists i \in \{ 1,\ldots,i_{\max} \} \mid \norm{\w_i - \w_*} > r_* \Bigg ]\\
    \nonumber &\leq  i_{\max} \Bigg (  \mathbb{P} \Big [ \norm{\xi_2} > C_L \Big ] +  \exp \Bigg \{ - \frac{\lambda^2}{2i_{\max}} \times \frac{1}{2 \sqrt{\eta} C_L \Big ( r_* + \eta (C^2r_* + S_1 ) \Big ) + \eta \Big ( 2S_1r_* + n + 2C^2r_*^2 +  2\eta ( C^4r_*^2 + S_1^2 ) \Big )}  \Bigg \} \Bigg ) 
\end{align}

\end{theorem}

\remark 
Thus for $\eta$ small enough and if $\mathbb{P} \Big [ \norm{\xi_2} > C_L \Big ]$ is small then with significant probability the noise assisted gradient dynamics on a single $\relu$ gate in its first $i_{\max}$ steps remains confined inside a ball around the true parameter of radius, 

\[ r_* \geq \sqrt { \frac{\lambda + \norm{\w_0 - \w_*}^2 + i_{\max}(\eta d + 2\eta^2 S_1^2)}{1 - 2i_{\max} \eta^2C^4} } \]

Larger the $\lambda >0$ we choose greater the (exponential) suppression in the probability that we get of finding the iterates outside the ball of radius $r_*$ around the origin whereby $r_*$ scales as $\sqrt{\lambda}$.

\bigskip 
Also we note that for the above two natural choices of the distribution for $\{ \xi_{t,2}, t=1,\ldots \}$ are (a)  $\{ \xi_{t,2} = 0, t=1,\ldots \}$   and  (b) $\{ (\xi_{t,2})_i  \sim {\cal N}(0,\sigma_i), i = 1, \ldots,n, t=1,\ldots \}$ where the $\{ \sigma_i, i = 1, \ldots,n \}$ can be chosen as follows : corresponding to this choice of distributions we invoke Equation $3.5$ from \cite{ledoux2013probability} to note that $\mathbb{P} \Big [ \norm{\xi_2} > C_L \Big ] \leq 4 e^{-\frac{C_L^2}{8 \times \E \Big [ \norm{\xi_2}^2 \Big ]}}$. Thus for the guarantee in the theorem to be non-trivial we need, $e^{-\frac{C_L^2}{8 \times \E \Big [ \norm{\xi_2}^2 \Big ]}} < \frac{1}{4i_{\max}}$. Now note that $\E \Big [ \norm{\xi_2}^2 \Big ] = \sum_{i=1}^n \sigma_i^2$ and hence the above condition puts a smallness constraint on the variances of the Gaussian noise depending on how large an $i_{\max}$ we want, $\sum_{i=1}^n \sigma_i^2 < \frac{C_L^2}{8 \log (4i_{\max})}$





\begin{proof}[{\bf Proof of Theorem \ref{confine:GLDReLUParam}}]  
~\\
For convenience we will use the notation, $\tilde{\g}_t := \g_t + \xi_{t,1}$. Suppose that at the $t^{th}$ iterate we have that $\norm{\w_t - \w_*}\leq r_*$. Given this we will get an upperbound on how far can $\w_{t+1}$ be from $\w_*$. Towards this we observe that, 
\begin{align}
    \nonumber \|\w_{t + 1} - \w_*\|^2 
    &= \|\w_t - \eta_t \tilde \g_t + \sqrt{\eta_t} \xi_{t, 2} - \w_*\|^2\\
    &= \|\w_t - \w_*\|^2 + \|-\eta_t \tilde \g_t + \sqrt{\eta_t} \xi_{t, 2}\|^2 + 2 \langle \w_t - \w_*, -\eta_t \tilde \g_t + \sqrt{\eta_t} \xi_{t, 2}\rangle \label{combine1}
 \end{align}

Expanding the second term above as, $\|-\eta_t \tilde \g_t + \sqrt{\eta_t} \xi_{t, 2}\|^2 
    = \eta^2_t \|\tilde \g_t\|^2 + \eta_t \|\xi_{t, 2}\|^2 - 2 \eta_t^{3/2} \langle \tilde \g_t, \xi_{t, 2} \rangle$
    
and combining into \ref{combine1}, we have, 

\begin{align}\label{stepdecrease1}
\nonumber  &\|\w_{t + 1} - \w_*\|^2 - \|\w_t - \w_*\|^2 \\
\nonumber &= \langle \xi_{t, 2}, -2 \eta_t^{3/2} \tilde \g_t + 2 \sqrt{\eta_t} \w_t - 2 \sqrt{\eta_t} \w_* \rangle + \eta_t \|\xi_{t, 2}\|^2 + \eta_t^2 \| \tilde \g_t \|^2 - 2 \eta_t \langle \w_t - \w_*, \tilde \g_t \rangle\\
\nonumber &= \langle \xi_{t, 2}, -2 \eta_t^{3/2} \tilde \g_t + 2 \sqrt{\eta_t} \w_t - 2 \sqrt{\eta_t} \w_* \rangle + \eta_t \|\xi_{t, 2}\|^2 + \eta_t^2 \| \tilde \g_t \|^2 - 2 \eta_t \langle \w_t - \w_*, \xi_{t, 1} \rangle\\
\nonumber &\quad - 2 \eta_t \langle \w_t - \w_*, \g_t \rangle\\
 &\leq -2\eta_t \langle \w_t - \w_* , \xi_{t,1} \rangle + \eta_t^2 \norm{\tilde \g_t}^2 + 2 \sqrt{\eta_t} \langle \xi_{t,2}, -\eta_t \tilde \g_t + \w_t - \w_* \rangle + \eta_t \norm{\xi_{t,2}}^2 
\end{align}

In the last line we have used the Lemma \ref{corr} which shows this critical fact that $\langle \w_t - \w_*, \g_t \rangle \geq 0$. 

Now we use the definition $\tilde \g_t = \g_t + \xi_{t, 1}$ on the $2^{nd}$ term in the RHS of equation \ref{stepdecrease1} to get, 

\begin{align}\label{stepdecrease2} 
\nonumber &\|\w_{t + 1} - \w_*\|^2 - \|\w_t - \w_*\|^2\\
\nonumber &\leq -2\eta_t \langle \w_t - \w_* , \xi_{t,1} \rangle + 2\eta_t^2 (\norm{\g_t}^2 + \norm{\xi_{t,1}}^2) + 2 \sqrt{\eta_t} \langle \xi_{t,2}, -\eta_t \tilde{\g_t} + \w_t - \w_* \rangle + \eta_t \norm{\xi_{t,2}}^2\\
\nonumber & \leq -2\eta_t \langle \w_t - \w_* , \xi_{t,1} \rangle + 2\eta_t^2 (\norm{\g_t}^2 + S_1 ^2) -  2 \sqrt{\eta_t} \langle \xi_{t,2} , \eta_t \tilde{\g}_t \rangle + 2 \sqrt{\eta_t} \langle \xi_{t,2}, \w_t - \w_* \rangle + \eta_t \norm{\xi_{t,2}}^2\\ 
&\leq 2\eta_t^2 (\norm{\g_t}^2 + S_1 ^2) + \eta_t n  + \Big [  -2\eta_t \langle \w_t - \w_* , \xi_{t,1} \rangle - 2 \sqrt{\eta_t} \langle \xi_{t,2} , \eta_t \tilde{\g}_t \rangle + 2 \sqrt{\eta_t} \langle \xi_{t,2}, \w_t - \w_* \rangle + \eta_t (\norm{\xi_{t,2}}^2 - n) \Big ] 
\end{align}


Now we will get a finite bound on $\norm{\g_t}^2$ by invoking the definition of $r_*$ and $C$ as follows,

\begin{align}\label{gradnorm}
\g_t = \frac{1}{S} \sum_{i = 1}^S \Big (y_i - {\rm ReLU}(\w_t^\top \x_i) \Big ) \mathbf{1}_{\w_t^\top \x_i \geq 0} (-\x_i)
\implies \norm{\g_t} \leq \frac{1}{S} \times C \sum_{i=1}^S \vert (\w_* - \w_t) ^\top \x_i \vert \leq C^2r_*   
\end{align}

Substituting this back into equation \ref{stepdecrease2} we have, 

\begin{align}\label{stepdecrease3}
\nonumber &\|\w_{t + 1} - \w_*\|^2 - \|\w_t - \w_*\|^2\\
\nonumber &\leq \Big \{ 2\eta_t^2 (C^4 r_*^2 + S_1^2) + \eta_t n   \Big \}\\ 
&+ \Big [  -2\eta_t \langle \w_t - \w_* , \xi_{t,1} \rangle - 2 \sqrt{\eta_t} \langle \xi_{t,2} , \eta_t \tilde{\g}_t \rangle + 2 \sqrt{\eta_t} \langle \xi_{t,2}, \w_t - \w_* \rangle + \eta_t (\norm{\xi_{t,2}}^2 - n) \Big ] 
\end{align}

We define $\w'_0 = \w_0$ and $\xi_{t,2}' = \min \Big \{ C_L, \norm{\xi_{t,2}} \Big \} \frac{\xi_{t,2}}{\norm{\xi_{t,2}}}$ and $C_L \in (0,\sqrt{n})$. 

Now we define a delayed stochastic process associated to the given algorithm, 

\[ \w'_{t+1} = \w'_t {\mathbf 1}_{\norm{\w'_t - \w_*} \geq r_*} + \Big (\w'_t - \eta_t \tilde{\g}_t + \sqrt{\eta_t}\xi_{t,2}' \Big ){\mathbf 1}_{\norm{\w'_t - \w_*} < r_*}  \]

In the above we note that whenever the primed iterate steps out of the $r_*$ ball it is made to stop. Associated to the above we define another stochastic process as follows, 

\begin{align}\label{def:z} 
z_t := \norm{\w_t' - \w_*}^2 - t\Big \{ 2\eta_t^2 (C^4 r_*^2 + S_1^2) + \eta_t n   \Big \}
\end{align} 

In Lemma \ref{bounded} we prove the crucial property that for $\eta_t = \eta >0$ a constant, the stochastic process $\{z_t\}_{t=0,1,\ldots}$ is a bounded difference process i.e $\vert z_{t+1} - z_t \vert \leq k$ for all $t=1,\ldots$ and 

\[ k = 2 \sqrt{\eta} C_L \Big ( r_* + \eta (C^2r_* + S_1 ) \Big ) + \eta \Big ( 2S_1r_* + n + 2C^2r_*^2 +  2\eta ( C^4r_*^2 + S_1^2 ) \Big ) \] 

Now note that $z_0$ is a constant since $\w_0$ is so. The proof of Lemma \ref{bounded} splits the analysis into two cases which we revisit again : in {\bf Case $1$} in there we have $z_{t+1} - z_t <0$ for $\eta_t = \eta >0$. And in {\bf Case $2$} therein  we take a conditional expectation of the RHS of equation \ref{diff01} w.r.t the sigma-algebra ${\cal F}_t$ generated by $\{ z_0,\ldots,z_t\}$. Then the first two terms will go to $0$ and the last term will give a negative contribution since $\norm{\xi'_{t,2}} \leq C_L$ and $C_L^2 < n$ by definition.
~\\ \\
Thus the stochastic process $z_0,\ldots$ satisfies the conditions of the concentration of measure Theorem \ref{AHSM} and thus we get that for any $\lambda >0$ and $t >0$ and $k$ as defined above, 

\[ \mathbb{P} \Big [ z_t - z_0 \geq \lambda \Big ] \leq e^{-\frac{\lambda^2}{2tk}}  \]

And explicitly the above is equivalent to,

\begin{align}\label{basic_prob}
    \nonumber &\mathbb{P} \Big [ \norm{\w_t' - \w_*}^2 - t\Big \{ 2\eta^2 (C^4 r_*^2 + S_1^2) + \eta n   \Big \} - \norm{\w_0 - \w_*}^2 \geq \lambda \Big ] \\ 
    &\leq \exp \Bigg \{ - \frac{\lambda^2}{2t} \times \frac{1}{2 \sqrt{\eta} C_L \Big ( r_* + \eta (C^2r_* + S_1 ) \Big ) + \eta \Big ( 2S_1r_* + n + 2C^2r_*^2 +  2\eta ( C^4r_*^2 + S_1^2 ) \Big )}  \Bigg \} 
\end{align}

The definition of $r_*$ given in the theorem statement is that it satisfies, $r_*^2 \geq \lambda + \norm{\w_0 - \w_*}^2 + t\Big \{ 2\eta^2 (C^4 r_*^2 + S_1^2) + \eta n   \Big \}$. 
Then the following is implied by equation \ref{basic_prob},

\begin{align}\label{escape1}
    \nonumber &\mathbb{P} \Big [ \norm{\w_t' - \w_*}^2 \geq r_*^2 \Big ] \\
    &\leq \exp \Bigg \{ - \frac{\lambda^2}{2t} \times \frac{1}{2 \sqrt{\eta} C_L \Big ( r_* + \eta (C^2r_* + S_1 ) \Big ) + \eta \Big ( 2S_1r_* + n + 2C^2r_*^2 +  2\eta ( C^4r_*^2 + S_1^2 ) \Big )}  \Bigg \} 
\end{align}

For the given positive integer $i_{\max}$ consider the event, 
\begin{align}\label{def:E}
E := \Big \{ \exists i \in \{ 1,\ldots,i_{\max} \} \mid \norm{\w_i - \w_*} > r_* \Big \}
\end{align}


Define the event $E_t := \{ \norm{\xi_{t,2}} > C_L \}$. Thus, if $E_t$ never happens, then the primed and the unprimed sequences both evolve the same unless  $\w'_t$ leaves the $r_*$ ball around $\w_*$ i.e., $\w'_t$ and $\w_t$ both leave the $r_*$ ball around $\w_*$.

The sample space can be written as a disjoint union of events $A := \cup_{t=1}^{i_{\max}} E_t$ and $B := \cap_{t=1}^{i_{\max}} E_t^c$. 

Let $L_t$ be the event that $t$ is the first time instant when $\w_t$ leaves the ball. Let $L'_t$ be the event that $t$ is the first time instant when $\w'_t$ leaves the ball. And we have argued above that when $B$ happens the two sequences evolve the same which in turn implies $B \cap L_t = B \cap L'_t$.  Thus we have, $\mathbb{P} \Big [ L_t \Big ] =  \mathbb{P} \Big [ L_t \cap A \Big ] +  \mathbb{P} \Big [ L_t \cap B \Big ] = \mathbb{P} \Big [ L_t \cap A \Big ] +  \mathbb{P} \Big [ L'_t \cap B \Big ]$

And combining this with $E$ defined in \ref{def:E} we have, 
\[ \mathbb{P} \Big [  E  \Big ] = \sum_{t=1}^{i_{\max}} \mathbb{P} \Big [ L_t \Big ] = \sum_{t=1}^{i_{\max}} \left ( \mathbb{P} \Big [ L_t \cap A \Big ] +  \mathbb{P} \Big [ L'_t \cap B \Big ] \right ) \leq \mathbb{P} \Big [ A \Big ] +  \sum_{t=1}^{i_{\max}} \mathbb{P} \Big [ L'_t \Big ] \leq \sum_{t=1}^{i_{\max}} \Big ( \mathbb{P} \Big [ E_t \Big ] + \mathbb{P} \Big [ L'_t \Big ] \Big ) \] 

The first equality and the first inequality above are true because $L_t$ are disjoint events.  

We further note that, $\mathbb{P} \Big [ L'_t \Big ] \leq \mathbb{P} \Big [ \norm{\w'_t - \w_*} > r_* \Big ]$

Hence combining the above two inequalities we have,  

\[ \mathbb{P} \Bigg [  \exists i \in \{ 1,\ldots,i_{\max} \} \mid \norm{\w_i - \w_*} > r_* \Bigg ] \leq \sum_{t=1}^{i_{\max}} \Big ( \mathbb{P} \Big [ \norm{\xi_{t,2}} > C_L \Big ] +  \mathbb{P} \Big [ \norm{\w'_t - \w_*} > r_* \Big ] \Big ) \] 

We invoke (a) the definition of the random variable $\xi_2$ and (b) equation \ref{escape1} on each of the summands in the RHS above and we can infer that, 

\begin{align}
    \nonumber &\mathbb{P} \Bigg [  \exists i \in \{ 1,\ldots,i_{\max} \} \mid \norm{\w_i - \w_*} > r_* \Bigg ]\\
    \nonumber &\leq  i_{\max} \Bigg (  \mathbb{P} \Big [ \norm{\xi_2} > C_L \Big ] +  \exp \Bigg \{ - \frac{\lambda^2}{2i_{\max}} \times \frac{1}{2 \sqrt{\eta} C_L \Big ( r_* + \eta (C^2r_* + S_1 ) \Big ) + \eta \Big ( 2S_1r_* + n + 2C^2r_*^2 +  2\eta ( C^4r_*^2 + S_1^2 ) \Big )}  \Bigg \} \Bigg ) 
\end{align}

This proves the theorem we wanted.
\end{proof} 

\newpage 
\begin{lemma}\label{corr}
$\langle \w_t - \w_*, \g_t \rangle \geq 0$
\end{lemma}
\begin{proof} 
We can obtain a (positive) lower bound on the inner product term $\langle \w_t - \w_*, \g_t \rangle$, 
\begin{align*}
    &\langle \w_t - \w_*, \g_t \rangle \\
    &= -\frac{1}{S} \sum_{i = 1}^S \Big \langle \w_t - \w_*, \Big(y_i - {\rm ReLU}(\w_t^\top \x_i) \Big) \mathbf{1}(\w_t^\top \x_i \geq 0) \x_i \Big \rangle\\
    &= -\frac{1}{S} \sum_{i = 1}^S \Big (\w_t^\top \x_i - \w^{*\top}\x_i \Big) \Big({\rm ReLU}(\w^{*\top}\x_i) - {\rm ReLU}(\w_t^\top \x_i) \Big) \mathbf{1}(\w_t^\top \x_i \geq 0)\\
    &= \frac{1}{S} \sum_{i = 1}^S \Big (\w^{*\top}\x_i - \w_t^\top \x_i \Big) \Big({\rm ReLU}(\w^{*\top}\x_i) - {\rm ReLU}(\w_t^\top \x_i) \Big) \mathbf{1}(\w_t^\top \x_i \geq 0)\\
    &\geq \frac{1}{S} \sum_{i = 1}^S \Big({\rm ReLU}(\w^{*\top}\x_i) - {\rm ReLU}(\w_t^\top \x_i) \Big)^2 \mathbf{1}(\w_t^\top \x_i \geq 0) \\
\end{align*}
\qed
\end{proof}

\begin{lemma}\label{bounded} 
If for all $t=0,\ldots$ we have $\eta_t = \eta$ a constant $>0$ then the stochastic process $\{z_t\}_{t=0,1,\ldots}$ defined in equation \ref{def:z} is a bounded difference stochastic process i.e there exists a constant $k >0$ s.t for all $t=1,\ldots$, $\vert z_{t+1} - z_t \vert \leq k$
\end{lemma}

\begin{proof}
We have 2 cases to consider. 

\paragraph{Case 1 : $\norm{\w_t' - \w_*} \geq r_*$}
~\\
\begin{align}\label{noshift}
z_{t+1} - z_t &= -(t+1)(2\eta_{t+1}^2(C^4r_*^2 + S_1^2 )+\eta_{t+1}n) + t(2\eta_t^2(C^4r_*^2 + S_1^2) + \eta_t n)\\
\nonumber &= 2(C^4r_*^2 + S_1^2)\Big \{  t \eta_t^2 - (t+1)\eta_{t+1}^2 \Big \} + n \Big \{ t\eta_t - (t+1)\eta_{t+1}\Big \} 
\end{align}


\paragraph{Case 2 : $\norm{\w_t' - \w_*} < r_*$}
~\\
Repeating the calculations as used to get equation \ref{stepdecrease3} but with $\xi'_{t,2}$ instead of $\xi_{t,2}$ we will get,  
\begin{align}\label{diff01}
    z_{t+1} - z_t &\leq -2\eta_t \langle \w_t - \w_* , \xi_{t,1} \rangle - 2 \sqrt{\eta_t} \langle \xi'_{t,2} , \eta_t \tilde{\g}_t \rangle + 2 \sqrt{\eta_t} \langle \xi_{t,2}', \w_t - \w_* \rangle + \eta_t (\norm{\xi_{t,2}'}^2 - n)
\end{align}

And by Cauchy-Schwartz the above implies,  
\begin{align}\label{diff1}    
    z_{t+1} - z_t \leq 2\eta_tS_1r_* + 2\sqrt{\eta_t}(r_* + \eta_t (C^2r_* + S_1) )C_L + \eta_t (C_L^2 - d)
\end{align}

Further repeating the calculations as used to get equation \ref{stepdecrease3} but with $\xi'_{t,2}$ instead of $\xi_{t,2}$ we will get,  

\begin{align*}
\| \w'_t - \w_*\|^2 - \|\w'_{t + 1} - \w_*\|^2 &\leq 2\eta_t r_* (\norm{\g_t} + \norm{\xi_{t,1}} ) + 2\sqrt{\eta_t} \Big \langle \xi'_{t,2}  , - \eta_t \tilde{\g}_t + (\w_t - \w_*)  \Big \rangle 
\end{align*}

In the above we invoke the definition of $S_1$ and equation \ref{gradnorm} to get, 

\begin{align}
    \nonumber \|\w'_t - \w_*\|^2 - \|\w'_{t + 1} - \w_*\|^2 &\leq 2\eta_t r_* (C^2r_* + S_1 ) + 2\sqrt{\eta_t} \Big \langle \xi'_{t,2}  , - \eta_t \tilde{\g}_t + (\w_t - \w_*)  \Big \rangle \\
    \nonumber &\leq 2\eta_t r_* (C^2r_* + S_1 ) + 2\sqrt{\eta_t} C_L (\eta_t (C^2r_* + S_1) + r_*  )
\end{align}

Hence we have,

\begin{align}\label{diff2}
    z_t - z_{t+1} &\leq 2\eta_t^2 (C^4 r_*^2 + S_1^2) + \eta_t n  + 2\eta_t r_* (C^2r_* + S_1 ) + 2\sqrt{\eta_t} C_L (\eta_t (C^2r_* + S_1) + r_*  )
\end{align}

Combining equations \ref{diff1} and \ref{diff2} we have,  

\begin{align}
    \nonumber \vert z_t - z_{t+1} \vert &\leq 2 \eta_t S_1 r_* + 2 \sqrt{\eta_t} C_L \Big ( r_* + \eta_t (C^2r_* + S_1 ) \Big )\\
     &+ \max \Big \{ \eta_t (C_L^2 - n), \eta_t \Big [ (n + 2C^2r_*^2) + 2\eta_t ( C^4r_*^2 + S_1^2 )\Big ]  \Big \} 
\end{align}

If we now invoke the that $\eta_t = \eta$, a positive constant then the above and the previous equation \ref{noshift} can be further combined to get for all $t=0,\ldots$, 

\begin{align}
    \nonumber \vert z_t - z_{t+1} \vert &\leq \max \Bigg \{ n \eta + 2(C^4r_*^2 + S_1^2) \eta^2 \\
\nonumber     &, 2 \eta S_1 r_* + 2 \sqrt{\eta} C_L \Big ( r_* + \eta (C^2r_* + S_1 ) \Big ) + \max \Big \{ \eta (C_L^2 - n), \eta \Big [ (n + 2C^2r_*^2) + 2\eta ( C^4r_*^2 + S_1^2 )\Big ]  \Big \} \Bigg \}\\ 
\nonumber     &\leq \max \Bigg \{ n \eta + 2(C^4r_*^2 + S_1^2) \eta^2 \\
\nonumber     &, 2 \sqrt{\eta} C_L \Big ( r_* + \eta (C^2r_* + S_1 ) \Big ) + \eta \Big ( 2S_1r_* + n + 2C^2r_*^2 +  2\eta ( C^4r_*^2 + S_1^2 ) \Big )  \Bigg \}\\
    &\leq  2 \sqrt{\eta} C_L \Big ( r_* + \eta (C^2r_* + S_1 ) \Big ) + \eta \Big ( 2S_1r_* + n + 2C^2r^2 +  2\eta ( C^4r_*^2 + S_1^2 ) \Big ) 
\end{align}

In the second inequality above we are using invoking our assumption that $C_L^2 < n$. And this proves the boundedness of the stochastic process $\{z_t\}$ as we set out to prove and a candidate $k$ is the RHS above. 
\end{proof}

\section{GLM-Tron converges on certain Lipschitz gates with no symmetry assumption on the data}\label{sec:GLMTron}

\begin{algorithm}[H]
\caption{GLM-Tron}
\label{GLMTron} 
\begin{algorithmic}[1]
\State {\bf Input:} $\{ (\x_i,y_i) \}_{i=1,\ldots,m}$ and an ``activation function"  $\sigma : \R \rightarrow \R$
\State $\w_1 =0$
\For {$t = 1,\ldots$}
    \State $\w_{t+1} := \w_t + \frac {1}{m} \sum_{i=1}^m \Big ( y_i - \sigma(\langle \w_t , \x_i\rangle) \Big)\x_i$ \Comment {\rm Define} $h_t(\x) := \sigma \Big ( \langle \w_t , \x\rangle \Big )$
\EndFor 
\end{algorithmic}
\end{algorithm}

First we state the following crucial lemma, 

\begin{lemma}\label{stepdecrease} 
Assume that for all $i = 1,\ldots,S$ $\norm{\x_i} \leq 1$ and in Algorithm \ref{GLMTron}, $\sigma$ is a $L-$Lipschitz non-decreasing function. Given any $\w$ and $W$ s.t at iteration $t$, we have $\norm{\w_t - \w} \leq W$, define $\eta >0$ s.t $\norm{\frac {1}{S} \sum_{i=1}^S \Big ( y_i - \sigma (\langle \w, \x_i \rangle) \Big ) \x_i} \leq \eta $. Then it follows that $\forall t = 1,2,\ldots$, 

\[  \norm{\w_{t+1} - \w}^2 \leq \norm{\w_t - \w}^2 -  \Big ( \frac {2}{L} - 1 \Big ) \tilde{L}_S(h_t) + \Big (\eta^2 + 2\eta W(L + 1) \Big ) \]

where we have defined, $\tilde{L}_S(h_t) := \frac {1}{S} \sum_{i=1}^S \Big (h_t(x_i) - \sigma (\langle  \w , \x_i\rangle) \Big)^2 = \frac {1}{S} \sum_{i=1}^S \Big ( \sigma (\langle \w_t, x_i\rangle ) - \sigma (\langle  \w , \x_i\rangle) \Big)^2$
\end{lemma}
\bigskip 
The above algorithm was introduced in \cite{kakade2011efficient} for bounded activations. Here we show the applicability of that idea for more general activations and also while having adversarial attacks on the labels. We give the proof of the above Lemma in Appendix \ref{app:glmtronstep}. Now we will see in the following theorem and its proof as to how the above Lemma leads to convergence of the ``effective-ERM", $\tilde{L}_S$ by GLM-Tron on a single gate. 

\begin{theorem}\label{GLMTronReLU} [{\bf GLM-Tron (Algorithm \ref{GLMTron}) solves the effective-ERM on a ReLU gate upto noise bound with minimal distributional assumptions}]
Assume that for all $i = 1,\ldots,S$ $\norm{\x_i} \leq 1$ and the label of the $i^{th}$ data point $y_i$ is generated as, $y_i = \sigma (\langle \w_*, \x_i \rangle ) + \xi_i$ s.t $\forall i, \vert \xi_i \vert \leq \theta$ for some $\theta >0$ and $\w_* \in \R^n$.  If $\sigma$ is a $L-$Lipschitz non-decreasing function for $L < 2$ then in at most $T = \frac {\norm{\w_*}}{\epsilon}$ GLM-Tron steps we would attain parameter value $\w_T$ s.t,  

\[ \tilde{L}_S(h_T) =  \frac {1}{S} \sum_{i=1}^S \Big ( \sigma(\langle \w_T, x_i\rangle ) - \sigma (\langle  \w_* , \x_i\rangle) \Big)^2 < \frac{L}{2-L} \Big ( \epsilon +  (\theta^2 + 2\theta W(L + 1) ) \Big ) \]
\qed 
\end{theorem}


\remark {\it Firstly} Note that in the realizable setting i.e when $\theta =0$, the above theorem is giving an upperbound on the number of steps needed to solve the ERM on say a $\relu$ gate to $O(\epsilon)$ accuracy. {\it Secondly} observe that the above theorem does not force any distributional assumption on the $\xi_i$ beyond the assumption of its boundedness. Thus the noise could as well be chosen ``adversarially" upto the constraint on its norm.  

The above Theorem is proven in Appendix \ref{app:GLMTronReLU}. If we make some assumptions on the noise being somewhat benign then we can get the following. 

\begin{theorem}[{\bf Performance guarantees on the GLM-Tron (Algorithm \ref{GLMTron}) in solving the ERM problem with data labels being output of a $\relu$ gate corrupted by benign noise}]\label{GLMTronReLUNoise}
Assume that the noise random variables $\xi_i, i = 1,\ldots,S$ are identically distributed as a centered random variable say $\xi$. Then for $T = \frac{\norm{\w}}{\epsilon}$, we have the following guarantee on the (true) empirical risk after $T$ iterations of GLM-Tron (say $\tilde{L}_S(h_T)$), 

\[ \mathbb{E}_{\{ (\x_i , \xi_i)\}_{i=1,\ldots S}} \Big [ L_S(h_T) \Big ] \leq \E_{\xi} [\xi^2] + \frac{L}{2-L} \Big ( \epsilon +  (\theta^2 + 2\theta W(L + 1) ) \Big ) \]  
\qed 
\end{theorem}

The above is proven in Appendix \ref{app:GLMTronReLUNoise}. Here we note a slight generalization of the above that can be easily read off from the above.

\begin{corollary}
Suppose that instead of assuming $\forall i = 1,\ldots,S$ $\vert \xi_i \vert \leq \theta$ we instead assume that the joint distribution of $\{ \xi_i \}_{i=1,\ldots,S}$ is s.t $\mathbb{P} \Big [ \vert \xi_i \vert \leq \theta ~\forall i \in \{1,\ldots,S \} \Big ] \geq 1 - \delta$ Then it would follow that the guarantee of the above Theorem \ref{GLMTronReLUNoise} still holds but now with probability $1-\delta$ over the noise distribution.  
\end{corollary}


\newpage 
\section{Conclusion} 
In this chapter we have initiated a number of directions of investigation towards understanding the trainability of finite sized nets while making minimal assumptions about the distribution of the data. A lot of open questions emanate from here which await answers. Of them we would like to particularly emphasize the issue of seeking a generalization of the results of Section \ref{sec:almostSGDReLU} and Section \ref{sec:GDNoiseReLU} to single filter depth $2$ nets as given in Definition \ref{net} below, which in many ways can be said to be the next more complicated case to consider, 

\begin{definition}[{\bf Single Filter Neural Nets Of Depth $2$}]\label{net} 
Given a set of $k$ matrices $\A_i \in \R^{r\times n}$, a $\w \in \R^r$ and an activation function $\sigma : \R \rightarrow \R$ we call the following depth $2$, width $k$ neural net to be a ``single filter neural net" defined by the matrices $\A_1,\ldots,\A_k$
\[
\R^n \ni \x \mapsto f_{\w}(\x) = \frac{1}{k} \sum_{i=1}^k \sigma \Big ( \w^\top \A_i\x \Big ) \in \R 
\]
and where $\sigma$ is the ``Leaky-$\relu$" which maps as, $\R \ni y \mapsto \sigma(y) = y {\mathbf 1}_{y \geq 0} + \alpha y {\mathbf 1}_{y <0} $ for some $\alpha \geq 0$
\end{definition}
Note that the above class of nets includes any single $\relu$ gate for $\alpha =0,k=1, \A_1 = I_{n\times n}$ and it also includes any depth $2$ convolutional neural net with a single filter by setting the $\A_i's$ to be $0/1$ matrices such that each row has exactly one $1$ and each column has at most one $1$. 

We would like to point out that towards this goal it would be interesting to settle a critical intermediate problem which is to know whether the sequence of random variables generated by noisy gradient descent on a $\relu$ gate as given in Algorithm \ref{GLDReLUParam} have distributional convergence and if they do then to find the corresponding rate. 


\begin{subappendices}
 
\chapter*{\vspace{10pt} Appendix To Chapter \ref{chaptrain_new}}\label{app:train_new} %


\section{Proof of Lemma \ref{stepdecrease}}\label{app:glmtronstep}
\begin{proof}
We observe that, 
\begin{align}\label{main} 
\nonumber \norm{\w_t - \w}^2 - \norm{\w_{t+1} - \w}^2    &= \norm{\w_t - \w}^2 - \norm{\Big ( \w_{t} + \frac {1}{S} \sum_{i=1}^S \Big (y_i - \sigma (\langle \w_t , \x_i\rangle) \Big )\x_i \Big ) - \w}^2\\
\nonumber &= -\frac{2}{S} \sum_{i=1}^S \Big \langle \Big (y_i - \sigma (\langle \w_t , \x_i\rangle) \Big)\x_i, \w_t - \w \Big  \rangle - \norm{\frac{1}{S}\sum_{i=1}^S \Big ( y_i - \sigma (\langle \w_t , \x_i\rangle ) \Big )\x_i}^2\\
&= \frac{2}{S} \sum_{i=1}^S \Big ( y_i - \sigma (\langle \w_t , \x_i\rangle) \Big ) \Big (\langle \w,\x_i \rangle  - \langle \w_t , \x_i \rangle \Big ) - \norm{\frac{1}{S}\sum_{i=1}^S \Big (y_i - \sigma (\langle \w_t , \x_i\rangle) \Big ) \x_i}^2
\end{align}

Analyzing the first term in the RHS above we get, 

\begin{align*} 
&\frac{2}{S} \sum_{i=1}^S \Big ( y_i - \sigma (\langle \w_t , \x_i\rangle) \Big ) \Big (\langle \w,\x_i \rangle  - \langle \w_t , \x_i \rangle \Big )\\
&= \frac{2}{S} \sum_{i=1}^S \Big (y_i - \sigma (\langle \w, \x_i \rangle ) + \sigma  (\langle \w, \x_i \rangle) - \sigma (\langle \w_t , \x_i\rangle) \Big ) \Big (\langle \w,\x_i \rangle  - \langle \w_t , \x_i \rangle \Big )\\
&= \frac{2}{S} \sum_{i=1}^S \Big \langle \Big (y_i - \sigma (\langle \w, \x_i \rangle ) \Big )\x_i, \w - \w_t \Big \rangle 
+   \frac{2}{S} \sum_{i=1}^S \Big (\sigma (\langle \w, \x_i \rangle) - \sigma (\langle \w_t , \x_i\rangle) \Big ) \Big ( \langle \x_i, \w \rangle - \langle \x_i , \w_t \rangle \Big )\\
&\geq - 2\eta W + \frac{2}{S} \sum_{i=1}^S  \Big ( \sigma (\langle \w, \x_i \rangle)  - \sigma (\langle \w_t , \x_i\rangle) \Big ) \Big ( \langle \x_i, \w \rangle - \langle \x_i , \w_t \rangle \Big)
\end{align*} 

In the first term above we have invoked the definition of $\eta$ and $W$ given in the Lemma. Further since we are given that $\sigma$ is non-decreasing and $L-$Lipschitz, we have for the second term in the RHS above, 

$\frac{2}{S} \sum_{i=1}^S  \Big ( \sigma (\langle \w, \x_i \rangle)  - \sigma (\langle \w_t , \x_i\rangle) \Big ) \Big ( \langle \x_i, \w \rangle - \langle \x_i , \w_t \rangle \Big)  \geq \frac{2}{SL} \sum_{i=1}^S  \Big ( \sigma (\langle \w, \x_i \rangle) - \sigma (\langle \w_t , \x_i\rangle) \Big )^2 =: \frac{2}{L} \tilde{L}_S (h_t)$

Thus together we have, 

\begin{eqnarray}\label{main1}
\frac{2}{S} \sum_{i=1}^S \Big (y_i - \sigma (\langle \w_t , \x_i\rangle) \Big ) \Big (\langle \w,\x_i \rangle  - \langle \w_t , \x_i \rangle \Big ) \geq -2\eta W  + \frac{2}{L} \tilde{L}_S (h_t) 
\end{eqnarray}

Now we look at the second term in the RHS of equation \ref{main} and that gives us, 

\begin{align}\label{main2}
\nonumber &\norm{\frac{1}{S}\sum_{i=1}^S \Big (y_i - \sigma (\langle \w_t , \x_i\rangle) \Big )\x_i}^2 = \norm{\frac{1}{S}\sum_{i=1}^S \Big (y_i - \sigma (\langle \w , \x_i\rangle) + \sigma(\langle \w , \x_i\rangle) - \sigma(\langle \w_t , \x_i\rangle ) \Big )\x_i}^2\\
\nonumber &\leq \norm{\frac{1}{S}\sum_{i=1}^S \Big (y_i - \sigma (\langle \w , \x_i\rangle) \Big )\x_i}^2 + 2 \norm{\frac{1}{S}\sum_{i=1}^S \Big (y_i - \sigma(\langle \w , \x_i\rangle) \Big)\x_i} \times \norm{\frac{1}{S}\sum_{i=1}^S \Big (\sigma(\langle \w , \x_i\rangle) - \sigma(\langle \w_t , \x_i\rangle) \Big)\x_i}\\
\nonumber &+ \norm{\frac{1}{S}\sum_{i=1}^S \Big (\sigma(\langle \w , \x_i\rangle) - \sigma(\langle \w_t , \x_i\rangle) \Big )\x_i}^2\\
 &\leq \eta^2 + 2\eta  \norm{\frac{1}{S}\sum_{i=1}^S \Big (\sigma(\langle \w , \x_i\rangle) - \sigma(\langle \w_t , \x_i\rangle) \Big)\x_i} + \norm{\frac{1}{S}\sum_{i=1}^S \Big (\sigma(\langle \w , \x_i\rangle) - \sigma(\langle \w_t , \x_i\rangle) \Big )\x_i}^2
\end{align}

Now by Jensen's inequality we have, 

$\norm{\frac{1}{S}\sum_{i=1}^S \Big (\sigma(\langle \w , \x_i\rangle) - \sigma(\langle \w_t , \x_i\rangle) \Big)\x_i}^2 \leq \frac{1}{S}\sum_{i=1}^S \Big (\sigma(\langle \w , \x_i\rangle) - \sigma(\langle \w_t , \x_i\rangle) \Big )^2 = \tilde{L}_S(h_t)$

And we have from the definition of $L$ and $W$,

\[ \norm{\frac{1}{S}\sum_{i=1}^S \Big ( \sigma(\langle \w , \x_i\rangle) - \sigma(\langle \w_t , \x_i\rangle) \Big )\x_i} \leq \frac{L}{S}\sum_{i=1}^S \norm{ \w - \w_t} \leq L \times W  \]

Substituting the above two into the RHS of equation \ref{main2} we have,

\begin{align}\label{main3} 
\norm{\frac{1}{S}\sum_{i=1}^S \Big (y_i - \sigma (\langle \w_t , \x_i\rangle) \Big )\x_i}^2 \leq \eta^2 + 2\eta LW + \tilde{L}_S(h_t)
\end{align} 

Now we substitute equations \ref{main1} and \ref{main3} into equation \ref{main} to get, 

\[  \norm{\w_t - \w}^2 - \norm{\w_{t+1} - \w}^2  \geq \Big ( -2\eta W  + \frac{2}{L} \tilde{L}_S (h_t) \Big ) -(\eta^2 + 2\eta LW + \tilde{L}_S(h_t))    \] 

The above simplifies to the inequality we claimed in the lemma i.e,

\[ \norm{\w_{t+1} - \w}^2  \leq  \norm{\w_t - \w}^2 - \Big ( \frac 2 L -1  \Big )\tilde{L}_S (h_t) + \Big ( \eta^2 + 2\eta W (L +1 )  \Big )  \] 
\end{proof} 

\section{Proof of Theorem \ref{GLMTronReLU}}\label{app:GLMTronReLU} 

\begin{proof}
The equation defining the labels in the data-set i.e $y_i = \sigma (\langle \w_*, \x_i \rangle ) + \xi_i$ with $\vert \xi_i \vert \leq \theta$ along with our assumption that, $\norm{\x_i} \leq 1$ implies that , $\norm{\frac {1}{S} \sum_{i=1}^S \Big ( y_i - \sigma (\langle \w_*, \x_i \rangle) \Big )\x_i} \leq \theta$. Thus we can invoke the above Lemma \ref{stepdecrease} between the $t^{th}$ and the $t+1^{th}$ iterate with $\eta = \theta$ and $W$ as defined there to get,   

\[  \norm{\w_{t+1} - \w_*}^2 \leq \norm{\w_t - \w_*}^2 - \left [   \Big ( \frac {2}{L} - 1 \Big ) \tilde{L}_S(h_t) - (\theta^2 + 2\theta W(L + 1) ) \right ] \]

If $\tilde{L}_S(h_t) \geq \frac{L}{2-L} \Big ( \epsilon +  (\theta^2 + 2\theta W(L + 1) ) \Big )$  then, $\norm{\w_{t+1} - \w}^2 \leq \norm{\w_t - \w}^2 - \epsilon$.  
Thus if the above lowerbound on $\tilde{L}_s(h_t)$ holds in the $t^{th}$ step then at the start of the $(t+1)^{th}$ step we still satisfy, $\norm{\w_{t+1}- \w} \leq W$. Since the iterations start with $\w_1=0$, in the first step we can choose $W = \norm{\w_*}$. Thus in at most $\frac{\norm{\w}}{\epsilon}$ steps of the above kind we can have a decrease in distance of the iterate to $\w$.   

Thus in at most $T = \frac {\norm{\w}}{\epsilon}$ steps we have attained, 

\[ \tilde{L}_S(h_T) = \frac {1}{S} \sum_{i=1}^S \Big ( \sigma (\langle \w_T, x_i\rangle ) - \sigma (\langle  \w_* , \x_i\rangle) \Big)^2 < \frac{L}{2-L} \Big ( \epsilon +  (\theta^2 + 2\theta W(L + 1) ) \Big ) \]

And that proves the theorem we wanted.  
\qed
\end{proof}

\section{Proof of Theorem \ref{GLMTronReLUNoise}}\label{app:GLMTronReLUNoise}

\begin{proof}
Let the true empirical risk at the $T^{th}-$iterate be defined as, 

\[ L_S(h_T) = \frac {1}{S} \sum_{i=1}^S \Big (\sigma(\langle \w_T, x_i\rangle ) - \sigma (\langle  \w_* , \x_i\rangle) - \xi_i \Big )^2 \] 

Then it follows that,  

\begin{align*}
    &\tilde{L}_S(h_T) - L_S(h_T) = \frac {1}{S} \sum_{i=1}^S \Big ( \sigma (\langle \w_T, x_i\rangle ) - \sigma (\langle  \w_* , \x_i\rangle) \Big )^2 - \frac {1}{S} \sum_{i=1}^S \Big ( \sigma (\langle \w_T, x_i\rangle ) - \sigma(\langle  \w_* , \x_i\rangle) - \xi_i \Big)^2\\
    =& \frac {1}{S} \sum_{i=1}^S \xi_i \Big ( -\xi_i + 2 \sigma (\langle \w_T, x_i\rangle ) - 2 \sigma (\langle  \w_* , \x_i\rangle) \Big )
    = -\frac {1}{S} \sum_{i=1}^S \xi_i^2 +  \frac {2}{S} \sum_{i=1}^S \xi_i \Big ( \sigma (\langle \w_T, x_i\rangle ) - \sigma (\langle  \w_* , \x_i\rangle) \Big )
\end{align*}

By the assumption of $\xi_i$ being an unbiased noise the second term vanishes when we compute,\\ $\mathbb{E}_{\{ (\x_i , \xi_i)\}_{i=1,\ldots S}} \Big [ \tilde{L}_S(h_T) - L_S(h_T) \Big ]$ Thus we are led to,
\[ \mathbb{E}_{\{ (\x_i , \xi_i)\}_{i=1,\ldots S}} \Big [ \tilde{L}_S(h_T) - L_S(h_T) \Big ] = -\frac {1}{m}  \mathbb{E}_{\{ \xi_i\}_{i=1,\ldots S}} \Big [ \sum_{i=1}^m \xi_i^2 \Big ] = -\frac{1}{m} \sum_{i=1}^m \mathbb{E}_{\{ \xi_i\}} \Big [ \xi_i^2 \Big ] = - \E_{\xi} [\xi^2] \] 

For $T = \frac{\norm{\w}}{\epsilon}$, we invoke the upperbound on $\tilde{L}_S(h_T)$ from the previous theorem and we can combine it with the above to say, 

\[ \mathbb{E}_{\{ (\x_i , \xi_i)\}_{i=1,\ldots S}} \Big [ L_S(h_T) \Big ] \leq \E_{\xi} [\xi^2] + \frac{L}{2-L} \Big ( \epsilon +  (\theta^2 + 2\theta W(L + 1) ) \Big ) \] 

This proves the theorem we wanted. 
\qed
\end{proof}

\section{Reviewing a variant of the Azuma-Hoeffding Inequality}

\begin{theorem}\label{AHSM} 
Suppose we have a real valued discrete stochastic process given as, $\{X_0,X_i,\ldots\}$ and the following properties hold,

\begin{itemize}
    \item $X_0$ is a constant  
    \item (The bounded difference property) $\forall i =0,1,\ldots$ $\exists c_i >0$ s.t $\vert X_i - X_{i-1}  \vert \leq c_i$
    \item (The super-martingale property) $\forall i =0,1,\ldots$, $\E \Big [ X_i - X_{i-1} \mid {\cal F}_{i-1} \Big ] \leq 0$ with ${\cal F}_{i-1} = \sigma \Big ( \Big \{ X_0,\ldots,X_{i-1} \Big \} \Big )$
\end{itemize}

Then for any $\lambda >0$ and a positive integer $n$ we have the following concentration inequality, 

\[ \mathbb{P} \Big [ X_n - X_0 \geq \lambda  \Big ] \leq e^{-\frac{1}{2} \frac{\lambda^2}{\sum_{i=1}^n c_i^2} } \]
\end{theorem}

\begin{proof}
We note that for any $c,t >0$, the function $f(x) = e^{tx}$ lies below the straight line connecting the two points $(-c,f(-c))$ and $(c,f(c))$. This gives the inequality, $e^{tx} \leq e^{-tc} + \Big ( \frac{e^{tc} - e^{-tc}}{2c} \Big ) (x + c)$. This simplifies to, 

\begin{align}\label{basic}
e^{tx} \leq \frac{1}{2c} (e^{tc} - e^{-tc})x + \Big ( \frac{e^{tc} + e^{-tc}}{2} \Big )    
\end{align}

Note that the above inequality holds only when $\vert x \vert \leq c$ Now we can invoke the bounded difference property of $\vert X_i - X_{i-1} \vert \leq c_i$ and use equation \ref{basic} with $x = X_i - X_{i-1} $ and $c = c_i$ to get,  

\[ \E \Big [ e^{t(X_i - X_{i-1})} \mid {\cal F}_{i-1}  \Big ] \leq \E \Big [ \frac{e^{tc_i} - e^{-tc_i}}{2c_i} \Big (X_i - X_{i-1}\Big )  + \Big ( \frac{e^{tc_i} + e^{-tc_i}}{2} \Big )     \mid {\cal F}_{i-1} \Big ] \leq \frac{e^{tc_i} + e^{-tc_i}}{2} \] 

The last inequality follows from the given property that, $\E \Big [ X_i - X_{i-1} \mid {\cal F}_{i-1} \Big ] \leq 0$

Now we invoke the inequality $\frac{e^x + e^{-x}}{2} \leq e^{\frac {x^2}{2}}$ on the RHS above to get, 

\[ \E \Big [ e^{t(X_i - X_{i-1})} \mid {\cal F}_{i-1}  \Big ] \leq e^{\frac {t^2 c_i^2}{2} } \]

Further since $X_{i-1}$ is ${\cal F}_{i-1}$ measurable we can write the above as, 
$\E \Big [ e^{tX_i}  \mid {\cal F}_{i-1}  \Big ] \leq e^{tX_{i-1}} e^{\frac {t^2 c_i^2}{2} }$

Now we recurse the above as follows, 

\[ \E \Big [ e^{tX_n} \Big ] = \E \Big [ \E \Big [ e^{tX_n} \mid {\cal F}_{n-1} \Big ] \Big ] \leq \E \Big [  e^{\frac{t^2 c_n^2}{2}} e^{tX_{n-1}} \Big ] = e^{\frac{t^2 c_n^2}{2}} \E \Big [ e^{tX_{n-1}} \Big ] \ldots \leq \prod_{i=1}^n e^{\frac {t^2 c_i^2}{2}} \E [ e^{tX_0} ]  \]

Now invoking that $X_0$ is a constant we can rewrite the above as, $\E \Big [ e^{t(X_n-X_0)} \Big ] \leq e^{\frac{t^2}{2} \sum_{i=1}^n c_i^2}$

Hence for any $\lambda >0$ we have by invoking the above,

\[ \mathbb{P} \Big [ X_n - X_0 \geq \lambda  \Big ] = \mathbb{P} \Big [ e^{t(X_n - X_0)} \geq e^{t\lambda} \Big ] \leq e^{-t\lambda} \E \Big [ e^{t(X_n - X_0)} \Big ] \leq e^{-t\lambda}e^{\frac{t^2}{2} \sum_{i=1}^n c_i^2}  \]

Now choose $t = \frac{\lambda}{\sum_{i=1}^n c_i^2}$ and we get, 
$\mathbb{P} \Big [ X_n - X_0 \geq \lambda  \Big ] \leq e^{-\frac{1}{2} \frac{\lambda^2}{\sum_{i=1}^n c_i^2} }$

\end{proof}

\newpage 
\section{A recursion estimate}\label{app:recurse}

\begin{lemma}\label{recurse} 
 Given constants $\eta', b,c_1, c_2 >0$ suppose one has a sequence of real numbers $X_1 = C,X_2,..$ s.t, 

\[ X_{t+1} \leq (1-\eta' b + \eta'^2 c_1)X_t + \eta'^2 c_2  \]

Given any $\epsilon' >0$ in the following two cases we have, $X_{\rm T} \leq \epsilon'^2$ 

\begin{itemize}
    \item If $c_2 =0, c_1 > \frac{b^2}{4}, C>0, \delta >0$,\\
    $\eta'= \frac {b}{2c_1}$ and ${\rm T} = O \Big ( \log \frac{C}{\epsilon'^2 } \Big )$
    
    
    \item If $0 < c_2 \leq c_1,  \epsilon'^2  \leq C, \frac{b^2}{c_1} \leq \Big ( \sqrt{\epsilon'}  + \frac{1}{\sqrt{\epsilon'}}  \Big )^2$,\\
    $\eta' = \frac{b}{c_1}\cdot\frac{\epsilon'^2}{(1+\epsilon'^2)}$ and 
${\rm T}  = O\Bigg(\frac{\log{\bigg(\frac{\epsilon'^2(c_1-c_2)}{Cc_1-c_2\epsilon'^2}}\bigg)}{\log{\bigg(1-\frac{b^2}{c_1}\cdot\frac{\epsilon'^2}{(1+\epsilon'^2)^2}\bigg)}}\Bigg)$
    . 
\end{itemize}
\end{lemma}

\begin{proof}
Suppose we define $\alpha = 1 - \eta' b + \eta'^2 c_1$ and $\beta = \eta'^2 c_2$. Then we have by unrolling the recursion, 
\begin{align*}
X_t &\leq \alpha X_{t-1} + \beta \leq \alpha (\alpha X_{t-1} + \beta ) + \beta \leq  ...\leq \alpha^{t-1}X_1 + \beta \frac{1-\alpha^{t-1}}{1-\alpha}.
\end{align*} 

We recall that $X_1 = C$ to realize that our Lemma gets proven if we can find ${\rm T}$ s.t, 
\[\alpha^{{\rm T}-1} C + \beta \frac{1-\alpha^{{\rm T}-1}}{1-\alpha} = \epsilon'^{2} \] 



Thus we need to solve the following for ${\rm T}$ s.t,  $\alpha ^{{\rm T}-1} = \frac{\epsilon'^2  (1-\alpha) - \beta}{C(1-\alpha) - \beta }$

{\bf Case 1 : $\beta = 0$}
In this case we see that if $\eta >0$ is s.t $\alpha \in (0,1)$ then
$\alpha ^{{\rm T}-1} = \frac{\epsilon'^2 }{C} \implies {\rm T} = 1 + \frac{\log \frac{C}{\epsilon^2 \delta} }{\log \frac {1}{\alpha} }$

But $\alpha = \eta'^2 c_1 - \eta' b + 1 = \Big ( \eta' \sqrt{c_1} - \frac{b}{2\sqrt{c_1}} \Big )^2 + \Big (  1 - \frac{b^2}{4c_1} \Big )$ Thus $\alpha \in (0,1)$ is easily ensured by choosing $\eta' = \frac{b}{2c_1}$ and ensuring $c_1 > \frac{b^2}{4}$.
This gives us the first part of the theorem.  

{\bf Case $2$ : $\beta >0$}

This time we are solving, 

\begin{align}\label{betanonz}
\alpha^{{\rm T}-1} = \frac{\epsilon'^2  (1-\alpha)-\beta}{C(1-\alpha)-\beta}
\end{align} 

Towards showing convergence, we want to set $\eta'$ such that $ \alpha^{t-1} \in (0,1)$ for all $t$. Since $\epsilon'^2  < C$, it is sufficient to require, 

\begin{align*}
    \beta < \epsilon'^2\delta(1-\alpha) &\implies \alpha < 1 - \frac{\beta}{\epsilon'^2}
    \Leftrightarrow 1-\frac{b^2}{4c_1} + \Big ( \eta' \sqrt{c_1} - \frac{b}{2\sqrt{c_1}} \Big )^2 \leq 1 - \frac{\beta}{\epsilon'^2}\\ 
    &\Leftrightarrow \frac{\eta'^2c_2}{\epsilon'^2} \leq \frac{b^2}{4c_1} - \Big ( \eta' \sqrt{c_1} - \frac{b}{2\sqrt{c_1}} \Big )^2 \Leftrightarrow \frac{c_2}{\epsilon'^2} \leq \frac{b^2}{4c_1\eta'^2} - \Big ( \sqrt{c_1} - \frac{b}{2\sqrt{c_1}\eta'} \Big )^2 
\end{align*}

Set $\eta' = \frac{b}{\theta c_1}$ for some constant $\theta>0$ to be chosen such that,
\begin{align*}
    \frac{c_2}{\epsilon'^2} \leq \frac{b^2}{4c_1 \cdot\frac{b^2}{\theta^2c_1^2}} - \Big ( \sqrt{c_1} - \frac{b}{2\sqrt{c_1}\cdot\frac{b}{\theta c_1}} \Big )^2 \implies \frac{c_2}{\epsilon'^2} \leq c_1\frac{\theta^2}{4} - c_1\cdot\Big (\frac{\theta}{2} - 1 \Big )^2 \implies  c_2 \leq \epsilon'^2 \cdot c_1 (\theta-1)
\end{align*}

Since $c_2 \leq c_1$ we can choose, $\theta = 1+\frac{1}{\epsilon'^2}$ and we have  $\alpha^{t-1} < 1$.  
Also note that, 
\begin{align*}
\alpha & = 1+ \eta'^2c_1 - \eta'b = 1+ \frac{b^2}{\theta^2c_1^2} - \frac{b^2}{\theta c_1}  = 1 - \frac{b^2}{c_1}\cdot\big(\frac{1}{\theta}-\frac{1}{\theta^2}\big). \\
& = 1 - \frac{b^2}{c_1}\cdot \frac{\epsilon'^2}{(1+\epsilon'^2)^2} = 1 - \frac{b^2}{c_1} \cdot \frac{1}{\Big ( \epsilon'  + \frac{1}{\epsilon' }  \Big )^2} 
\end{align*}

And here we recall that the condition that the lemma specifies on the ratio $\frac{b^2}{c_1}$ which ensures that the above equation leads to $\alpha >0$ 

Now in this case we get the given bound on ${\rm T}$ in the Lemma by solving equation \ref{betanonz}. To see this, note that,
\begin{align*}
\alpha = 1-\frac{b^2}{c_1}\cdot\frac{\epsilon'^2}{(1+\epsilon'^2)^2}
\text{ and }
\beta = \eta'^2 c_2 = \frac{b^2}{\theta^2 c_1} \cdot c_2 = \frac{b^2c_2}{c_1}\cdot \frac{(\epsilon'^2)^2}{(1+\epsilon'^2)^2}.
\end{align*}

Plugging the above into equation \ref{betanonz} we get, 
$\alpha^{{\rm T}-1} = \frac{\epsilon'^2\delta(c_1-c_2)}{Cc_1-c_2\epsilon'^2} 
\implies {\rm T} = 1 + \frac{\log{\bigg(\frac{\epsilon'^2(c_1-c_2)}{Cc_1-c_2\epsilon'^2}}\bigg)}{\log{\bigg(1-\frac{b^2}{c_1}\cdot\frac{\epsilon'^2}{(1+\epsilon'^2)^2}\bigg)}}$
\end{proof}

\end{subappendices}

\chapter{\vspace{10pt} Sparse Coding and Autoencoders}\label{chapaut} 

\section{Introduction}
One of the fundamental themes in learning theory is to consider data being sampled from a generative model and to provide efficient methods to recover the original model parameters exactly or with tight approximation guarantees. Classic examples include learning a mixture of gaussians \citep{moitra2010settling}, certain graphical models \citep{anandkumar2014tensor}, full rank square dictionaries \citep{spielman2012exact,blasiok2016improved} and overcomplete dictionaries \citep{agarwal2014learning,arora2014more,arora2015simple, arora2014new} The problem is usually distilled down to a non-convex optimization problem whose solution can be used to obtain the model parameters. With these hard non-convex problems it has been difficult to find any universal view as to why sometimes gradient descent gives very good and sometimes even exact recovery. In recent times progress has been made towards achieving a geometric understanding of the landscape of such non-convex optimization problems \citep{ge2017no}, \citep{mei2016landscape}, \citep{wu2017towards}. The corresponding question of parameter recovery for neural nets with one layer of activation has been solved in some special cases, \citep{du2017convolutional, allen2017natasha2, janzamin2015beating,sedghi2014provable, li2017convergence, tian2017analytical,zhang2017electron}. Almost all of these cases are in the supervised setting where it has also been assumed that the labels are being generated from a net of the same architecture as is being trained. In contrast to these works we address an unsupervised learning problem, and possibly more realistically, we do not tie the data generation model (sensing of sparse vectors by an overcomplete incoherent dictionary) to the neural architecture being analyzed except for assuming knowledge of a few parameters about the ground truth. 

Here we specialize to the generative model of  {\it dictionary learning/sparse coding} where one receives samples of vectors $y \in \mathbb{R}^n$ that have been generated as $y = A^*x^*$ where $A^* \in \mathbb{R}^{n \times h}$ and $x^* \in \mathbb{R}^h$. We typically assume that the number of non-zero entries in $x^*$ to be no larger than some function of the dimension $h$ and that $A^*$ satisfies certain incoherence properties. The question now is to recover $A^*$ from samples of $y$. There have been renewed investigations into the hardness of this problem \citep{tillmann2015computational} and many former results have recently been reviewed in these lectures \cite{GilbertLectures}. This question has been a cornerstone of learning theory ever since the ground-breaking paper by Olshausen and Field (\cite{olshausen1997sparse}) (a recent review by the same authors can be found in \cite{olshausen2005close}). Over the years many algorithms have been developed to solve this problem and a detailed comparison among these various approaches can be found in \cite{blasiok2016improved}.

~\\ 
{\em Autoencoder} neural networks that map $\mathbb{R}^n \rightarrow \mathbb{R}^n$ were defined in Section \ref{def:autoencoder}. These networks have been used extensively (\cite{baldi2012autoencoders,bengio2013generalized, rifai2011contractive, vincent2008extracting,   vincent2010stacked}) in the past for unsupervised feature learning tasks, and have been found to be successful in generating discriminative features \citep{coates2011analysis}. A number of different autoencoder architectures and regularizers have been proposed which purportedly induce sparsity, at the hidden layer \citep{arpit2016regularized,coates2011importance,li2016sparseness, ng2011sparse}. There has also been some investigation into what autoencoders learn about the data distribution \citep{alain2014regularized}.
~\\ \\
Olshausen and Field had, as early as $1996$, already made the connection between sparse coding and training neural architectures and in today's terminology this problem is very naturally reminiscent of the architecture of an autoencoder \citep{olshausen1996emergence}. However, to the best of our knowledge, there has not been sufficient progress to rigorously establish whether autoencoders can do sparse coding. 
~\\ \\
In this work, we present our progress towards bridging the above mentioned mathematical gap. To the best of our knowledge, there is no theoretical evidence (even under the usual generative assumptions of sparse coding) that the stationary points of any of the usual squared loss functions (with or without any of the usual regularizers) have any resemblance to the original dictionary that is being sought to be learned. {\bf The main point of this paper is to rigorously prove that for autoencoders with ReLU activation, the standard squared loss function has a neighborhood around the dictionary $A^*$ where the norm of the expected gradient is very small (for large enough sparse code dimension $h$). Thus, all points in a neighborhood of $A^*$, including $A^*$, are all asymptotic critical points of this standard squared loss.} 
We supplement our theoretical result with experimental evidence for it in Section~\ref{sec:experiments}, which also strongly suggests that the standard squared loss function has a local minimum in a neighborhood around $A^*$. We believe that our results provide theoretical and experimental evidence that the sparse coding problem can be tackled by training autoencoders.

\subsection {A motivating experiment on MNIST using TensorFlow}

We used TensorFlow \citep{abadi2016tensorflow} to train two ReLU autoencoders mapping $\R^{784} \rightarrow \R^{784}$ (since the MNIST images vectorize to elements in $\R^{784}$). These networks were trained on a subset of the MNIST dataset of handwritten digits. One of the nets had a single hidden layer of size $10000$ and the other one had two hidden layers of size $5000$ and $784$ (and a fixed identity matrix giving the output from the second layer of activations). In both the cases the weights of the encoder and decoder were maintained as transposes of each other. We trained the autoencoders on the standard squared loss function using RMSProp \cite{tieleman2015rmsprop}. The training was done on $6000$ images of the digits $6$ and $7$ from the MNIST dataset. In the following panel we show four pairs (two for each net) of ``reconstructed" image i.e output of the trained net when its given as input the ``actual" photograph as input. 

\begin{figure}[h] 
\centering
\includegraphics[scale=0.30]{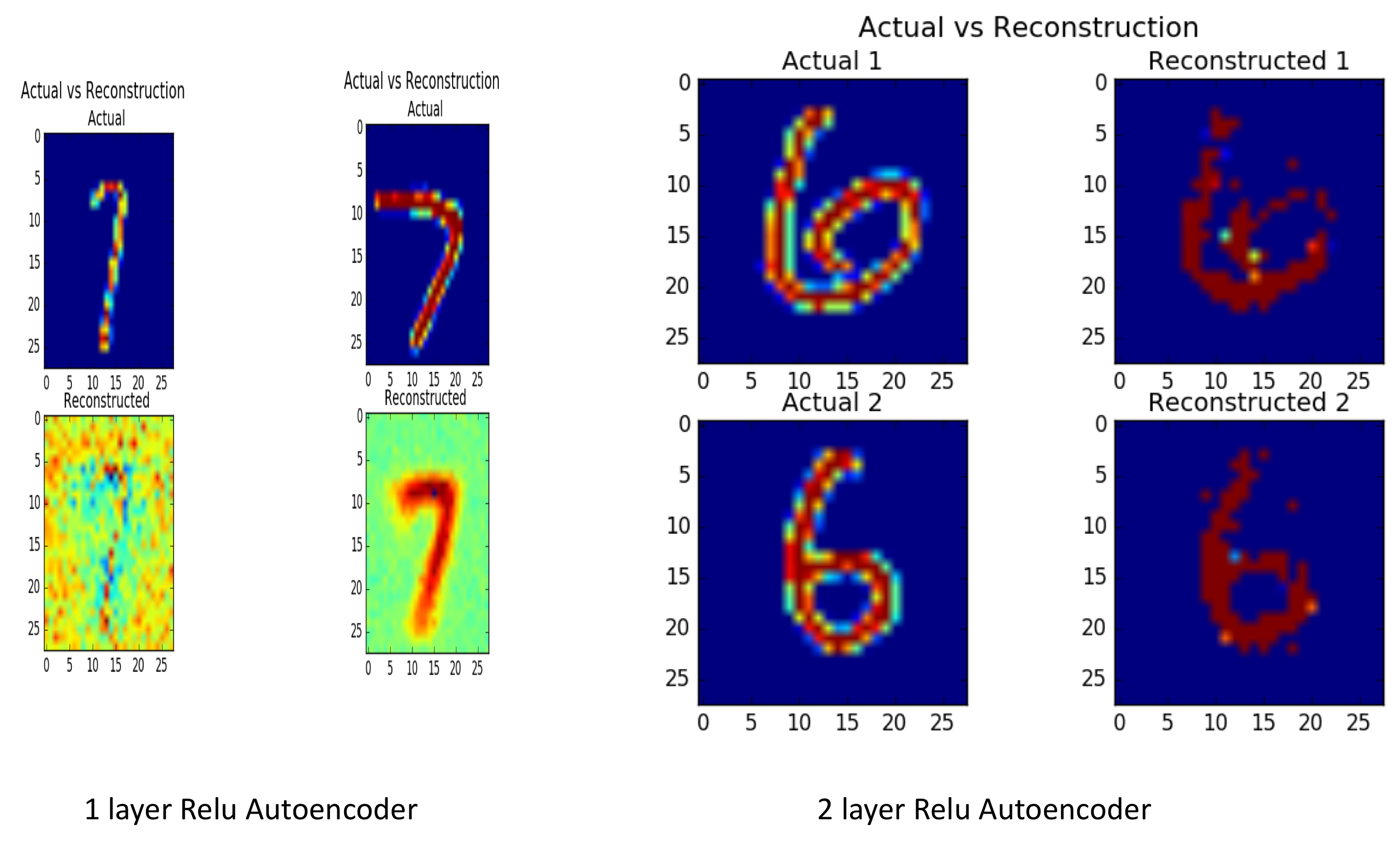}
\end{figure}\label{demo}
\newpage 
In our opinion, the above figures add support to the belief that a single and a double layer ReLU activated $\R^n \rightarrow \R^n$ network can learn an implicit high dimensional structure about the handwritten digits dataset. In particular this demonstrates that though adding more hidden layers obviously helps enhance the reconstruction ability, the single hidden layer autoencoder do hold within them significant power for unsupervised learning of representations. Unfortunately analyzing the RMSProp update rule used in the above experiment seems to be currently beyond our analytic means - though in the next chapter we shall make some progress about understanding this algorithm. However, we take inspiration from these experiments to devise a different mathematical set-up which is much more amenable to analysis taking us towards a better understanding of the power of autoencoders.

\section{Introducing the neural architecture and the distributional assumptions}\label{sec:defn}

For the autoencoders we continue to use the same variables as defined in equation \ref{eqn:autoencoder}.


\paragraph{Assumptions on the dictionary and the sparse code.}\label{assumptions} We assume that our signal $y$ is generated using sparse linear combinations of atoms/vectors of an overcomplete dictionary, i.e., $y = A^* x^*$, where $A^* \in \mathbb{R}^{n \times h}$ is a dictionary, and $x^* \in (\mathbb{R}^{\geq 0})^h$ is a non-negative sparse vector, with at most $k=h^p$ (for some $0 < p < 1$) non zero elements. The columns of the original dictionary $A^*$ (labeled as $\{ A^*_i \}_{i=1}^{h}$) are assumed to be normalized and we parameterize its incoherence property as,  $\max_{\substack{i,j=1,..,h \\ i \neq j}} \vert \langle A^*_i, A^*_j \rangle \vert \leq \frac{\mu}{\sqrt{n}} = h^{-\xi}$ for some $\xi >0$.  
\newline \newline 
We assume that the sparse code $x^*$ is sampled from a distribution with the following properties. We fix a set of possible supports of $x^*$, denoted by $\mathbb{S} \subseteq 2^{[h]}$, where each element of $\mathbb{S}$ has at most $k = h^p$ elements. We consider any arbitrary discrete probability distribution $D_{\mathbb{S}}$ on $\mathbb{S}$ such that the probability $q_1 := \mathbb{P}_{S \sim \mathbb{S}}[i \in S]$ is independent of $i \in [h]$, and the probability $q_{2}  := \mathbb{P}_{S \in \mathbb{S}}[i,j \in S]$ is independent of $i,j \in [h]$. A special case is when $\mathbb{S}$ is the set of all subsets of size $k$, and $D_{\mathbb{S}}$ is the uniform distribution on $\mathbb{S}$. For every $S \in \mathbb{S}$ there is a distribution say $D_S$ on $(\mathbb{R}^{\geq 0})^{h}$ which is supported on vectors whose support is contained in $S$ and which is uncorrelated for pairs of coordinates $i,j \in S$. Further, we assume that the distributions $D_S$ are such that each coordinate $x^*_i$ is compactly supported over an interval $[a(h),b(h)]$, where $a(h)$ and $b(h)$ are independent of both $i$ and $S$ but will be functions of $h$. Moreover, $m_1(h) := \mathbb{E}_{x^*\sim D_S}[x_i^*]$, and $m_2(h) :=\mathbb{E}_{x^*\sim D_S}[x_i^{*2}]$ are assumed to be independent of both $i$ and $S$ but allowed to depend on $h$. For ease of notation henceforth we will keep the $h$ dependence of these variables implicit and refer to them as $a, b, m_1$ and $m_2$. All of our results will hold in the special case when $a,b, m_1, m_2$ are constants (no dependence on $h$).

\section{Main Results}
 
\subsection{Recovery of the support of the sparse code by a layer of ReLUs}\label{sec:results-support}

First we prove the following theorem which precisely quantifies the sense in which a layer of ReLU gates is able to recover the support of the sparse code when the weight matrix of the deep net is close to the original dictionary. We recall that the size of the support of the sparse vector $x^*$ is $k=h^p$ for some $0 < p < 1$. We also recall the parameters $a,b$ as defining the support of the marginal distribution of each coordinate of $x^*$ and $m_1$ is the expected value of this marginal distribution (recall that none of these depend on the coordinate or the actual support). These parameters will be referenced in the results below.

\begin{theorem}\label{sec:theorem:support}
~\\
We recall from  equation \ref{eqn:autoencoder} that our autoencoding neural net under consideration is mapping, 
\begin{align*}
    \R^n &\rightarrow \R^n\\  
     \y  &\mapsto \W^T \rr \text{ where } \rr = \relu \left( \W \y - \epsilon \right)
\end{align*}
where the $h$ columns of $W^\top$ are denoted as $\{ W_i \in \R^n \mid i = 1,\ldots,h \}$

Let each column of $W^\top$ be within a $\delta$-ball of the corresponding column of $A^*$, where $\delta = O \left( h^{-p - \nu^2} \right)$ for some $\nu>0$, such that $p + \nu^2 < \xi$ (where $h^{-\xi}$ is the coherence parameter). We further assume that $a = \Omega \left( b h^{ - \nu^2} \right)$. Let the bias of the hidden layer of the autoencoder as given above, be $\epsilon = 2 m_1 k \left( \delta + \frac{\mu}{\sqrt{n}} \right)$. Then $r_i \neq 0$ if $i \in \textrm{supp}(x^*)$, and $r_i =0$ if $i \notin \textrm{supp}(x^*)$ with probability at least $1 - \exp \left ( -\frac{2 h^p m_1^2}{(b-a)^2} \right )$ (with respect to the distribution on $x^*$). \end{theorem}
~\\ As long as $\frac{h^p m_1^2}{(b-a)^2}$ is large, i.e., an increasing function of $h$, we can interpret this as saying that the probability of the adverse event is small, and we have successfully achieved support recovery at the hidden layer in the limit of large sparse code dimension.

\subsection{Asymptotic Criticality of the Autoencoder around $A^*$}\label{sec:result-loss}

In this work we analyze the following standard squared loss function for the autoencoder, 
\begin{align}\label{eqn:loss}
L = \frac{1}{2} \vert \vert \hat{y} - y \vert \vert^2
\end{align}
 If we consider a generative model in which $A^*$ is a square, orthogonal matrix and $x^*$ is a non-negative vector (not necessarily sparse), it is easily seen that the standard squared reconstruction error loss function for the autoencorder has a global minimum at $W = A^{*\top}$. In our generative model, however, $A^*$ is an incoherent and overcomplete dictionary.

\begin{theorem}\label{sec:theorem:critical} ({\bf The Main Theorem})
Assume that the hypotheses of Theorem \ref{sec:theorem:support} hold, and $p < \min \{ \frac{1}{2}, \nu^2\}$ (and hence $\xi > 2p$). Further, assume the distribution parameters satisfy $\textrm{exp} \left(\frac{h^p m_1^2}{2(b-a)^2}\right)$ is superpolynomial in $h$ (which holds, for example, when $m_1, a, b$ are $O(1)$). Then for $i=1, \ldots, h$, 
 $$
 \bigg\Vert \mathbb{E}\left [ \frac {\partial L}{\partial W_i}\right ] \bigg\Vert_2 \leq o\bigg(\frac{\max\{m_1^2, m_2\}}{h^{1-p}}\bigg). 
 $$
\end{theorem}

\paragraph{Roadmap.} We present the proof of the support recovery result, i.e., Theorem~\ref{sec:theorem:support}, in Section~\ref{proof:support}. Section~\ref{proof:asymptotic} gives the proof of our main result, Theorem~\ref{sec:theorem:critical}. The argument rests on Lemmas~\ref{good_proxy} and~\ref{bounds}), whose proofs appear in Appendix \ref{app:Auto} In Section~\ref{sec:experiments}, we run simulations to verify Theorem~\ref{sec:theorem:critical}. We also run experiments that strongly suggest that the standard squared loss function has a local minimum in a neighborhood around $A^*$.

\section{A Layer of ReLU Gates can Recover the Support of the Sparse Code (Proof of Theorem \ref{sec:theorem:support})}\label{proof:support}

Most sparse coding algorithms are based on an alternating minimization approach, where one iteratively finds a sparse code based on the current estimate of the dictionary, and then uses the estimated sparse code to update the dictionary. The analogue of the sparse coding step in an autoencoder, is the passing through the hidden layer of activations of a certain affine transformation ($W$ which behaves as the current estimate of the dictionary) of the input vectors. We show that under certain stochastic assumptions, the hidden layer of ReLU gates in an autoencoder recovers with high probability the support of the sparse vector which corresponds to the present input.

\begin{proof}[Proof of Theorem~\ref{sec:theorem:support}]
From the model assumptions, we know that the dictionary $A^*$ is incoherent, and has unit norm columns. So, $ \vert \langle A_i^* , A_j^* \rangle \vert \leq \frac{\mu}{\sqrt{n}}$ for all $i \neq j$, and $||A^*_i||=1$ for all $i$. This means that for $i \neq j$,

 \begin{align}\label{eq:bound-inner-prod-WiAj} 
 \nonumber \vert \langle W_i , A_j^*\rangle \vert  &= \vert \langle W_i - A_i^*, A_j^*\rangle  \vert + \vert \langle A_i^* , A_j^* \rangle \vert \\
 &\leq || W_i - A_i^* ||_2 ||A_j^*||_2 + \frac{\mu}{\sqrt{n}} \leq  (\delta + \frac{\mu}{\sqrt{n}})
 \end{align}
 
Otherwise for $i = j$, 
\[\langle W_i , A_i^*\rangle   =  \langle W_i - A_i^*, A_i^*\rangle  +  \langle A_i^* , A_i^* \rangle  = \langle W_i - A_i^*, A_i^*\rangle + 1, \] 
and thus,
\begin{equation}\label{eq:bound-inner-prod-WiAi} 1 - \delta \leq \langle W_i , A_i^*\rangle \leq 1 + \delta, \end{equation}
where we use the fact that $\vert \langle W_i - A_i^*, A_i^*\rangle \vert \leq \delta$.

Let $y=A^*\x^*$ and let $S$ be the support of $\x^*$. Then we define the input to the ReLU activation $Q - \epsilon = W\y - \epsilon$ as
\begin{align*}
Q_i = \sum_{j \in S} \langle W_i, A^*_j \rangle x^*_j = \langle W_i, A^*_i \rangle x^*_i \mathfrak{1}_{i\in S}+ \sum_{j \in S \setminus i} \langle W_i, A^*_j \rangle x^*_j = \langle W_i, A^*_i \rangle x^*_i\mathfrak{1}_{i\in S} + Z_i.
\end{align*}
~\\
First we try to get bounds on $Q_i$ when $i \in \textrm{supp}(x^*)$. From our assumptions on the distribution of $x^*_i$ we have, $ 0 \leq a \leq x_i^* \leq b$ and $\mathbb{E}[x^*_i] = m_1$ for all $i$ in the support of $x^*$. For $i \in \textrm{supp}(x^*)$, 
\begin{align*}
Q_i &= \langle W_i, A^*_i \rangle x^*_i + Z_i
\implies Q_i &\geq (1-\delta)a + Z_i
\end{align*} 
where we use~\eqref{eq:bound-inner-prod-WiAi}. Using~\eqref{eq:bound-inner-prod-WiAj}, $Z_i$ has the following bounds:
\[ -b k \left( \delta + \frac{\mu}{\sqrt{n}} \right) \leq Z_i \leq b k \left( \delta + \frac{\mu}{\sqrt{n}} \right) \]
Plugging in the lower bound for $Z_i$ and the proposed value for the bias, we get
\begin{align*}
Q_i - \epsilon &\geq (1-\delta) a - bk \left( \delta + \frac{\mu}{\sqrt{n}} \right) - 2 m_1 k \left( \delta + \frac{\mu}{\sqrt{n}} \right)
\end{align*}
~\\
For $Q_i - \epsilon \geq 0$, we need:
\[ a \geq \frac{(b+2 m_1) \left( \delta + \frac{\mu}{\sqrt{n}} \right) k}{1- \delta}\]
Now plugging in the values for the various quantities, $\frac{\mu}{\sqrt{n}} = h^{- \xi}$ and $k = h^p$ and $\delta = O \left( h^{-p -\nu^2} \right)$, if we have $a = \Omega \left( b h^{-\nu^2} \right)$, then $Q_i - \epsilon \geq 0$. 

~\\
Now, for $i \notin \textrm{supp}(x^*)$ we would like to analyze the following probability:
\begin{equation*}
\textrm{Pr}[ Q_i -\epsilon \geq 0 \vert i \notin \textrm{supp}(x^*)]
\end{equation*}
We first simplify the quantity $\textrm{Pr}[ Q_i -\epsilon \geq 0 \vert i \notin \textrm{supp}(x^*)]$ as follows

\begin{align*}
\textrm{Pr}[ Q_i \geq \epsilon \vert i \notin \textrm{supp}(x^*) ] = \textrm{Pr} [ Z_i \geq \epsilon] 
= \textrm{Pr} \left[ \sum_{j \in S\setminus i} \langle W_i, A_j^* \rangle x_j^* \geq \epsilon \right] 
\end{align*}
~\\
We recall that we had assumed that for every possible support $S$ (of $x^*$) the distribution $D_S$ on $(\mathbb{R}^{\geq 0})^{h}$, which is supported on vectors whose support is contained in $S$, is s.t the random variables corresponding to coordinates $i,j \in S$ are uncorrelated. Now using the Chernoff's bound, we can obtain
\begin{align*}
\textrm{Pr} [ Z_i \geq \epsilon] & \leq \underset{t \geq 0}{\textrm{inf}} e^{ -t\epsilon}\mathbb{E} \left[ \prod_{j \in S \setminus i} \left[ e^{ t \langle W_i, A_j^* \rangle x_j^*} \right] \right] = \underset{t \geq 0}{\textrm{inf}} e^{-t\epsilon} \prod_{j \in S \setminus i} \mathbb{E} \left[ e^{  t\langle W_i, A_j^* \rangle x_j^*} \right]  \\ 
&\leq  \underset{t \geq 0}{\textrm{inf}} e^{-t\epsilon } \mathbb{E}^k \left[ e^{  t \left(\delta + \frac{\mu}{\sqrt{n}} \right)  x^*_j} \right] \\
&\leq \underset{t \geq 0}{\textrm{inf}} e^{-t\epsilon} \left( e^{t \left( \delta + \frac{\mu}{\sqrt{n}} \right) m_1 }e^{  \frac{ t^2 \left(\delta + \frac{\mu}{\sqrt{n}} \right)^2 (b-a)^2 }{8}} \right)^k
\end{align*}
~\\
where the second inequality follows from ~\eqref{eq:bound-inner-prod-WiAj} and the fact that $t$ and $x^*_i$ are both nonnegative, and the third inequality follows from Hoeffding's Lemma. 
Next, we also have 
\begin{align*}
\textrm{Pr} [ Z_i \geq \epsilon]  &\leq \underset{t \geq 0}{\textrm{inf}} e^ { - t \left( \epsilon - k \left( \delta + \frac{\mu}{\sqrt{n}} \right) m_1 \right)  + t^2 \frac{k}{8} \left(\delta + \frac{\mu}{\sqrt{n}} \right)^2 (b-a)^2 } \\
&= e^{ - \frac{ (\epsilon -k(\delta + \frac{\mu}{\sqrt{n}}) m_1 )^2 }{ \frac{k }{2} (\delta + \frac{\mu}{\sqrt{n}})^2 (b-a)^2}}.
\end{align*}
~\\ Finally, since $k = h^p$ and $\epsilon = 2 m_1 k \left( \delta + \frac{\mu}{\sqrt{n}} \right)$, we have 
\begin{align*}
\exp \left( - \frac{2(\epsilon - k m_1 (\delta + \frac{ \mu}{\sqrt{n} } ) )^2 }{h^p (\delta + \frac{\mu}{\sqrt{n}} )^2(b -a )^2 } \right) = \exp \left( - \frac{2 h^p m_1^2}{(b-a)^2} \right) 
\end{align*}
\end{proof}

\section{Criticality of a neighborhood of $A^*$ (Proof of Theorem \ref{sec:theorem:critical})} \label{proof:asymptotic}

It turns out that the expectation of the full gradient of the loss function~\eqref{eqn:loss} is difficult to analyze directly. Hence corresponding to the true gradient with respect to the $i^{\textrm{th}}-$column of $W^\top$ we create a proxy, denoted by $\widehat{\nabla_i L}$, by replacing in the expression for the true expectation $\nabla_i L = \mathbb{E} \left[ \frac{\partial L}{\partial W_i} \right]$ every occurrence of the random variable $\mathbf{1}_{W^\top_i y - \epsilon_i \geq 0} = \textrm{Th}(W^\top_i y - \epsilon_i) = \textrm{Th}(W^\top_i A^*x^* - \epsilon_i)$ by the indicator random variable $\mathbf{1}_{i\in \textrm{supp}(x^*)}$. This proxy is shown to be a good approximant of the expected gradient in the following lemma. 

\begin{lemma}\label{good_proxy}
Assume that the hypotheses of Theorem \ref{sec:theorem:support} hold and additionally let $b$ be bounded by a polynomial in $h$.  Then we have for each $i$ (indexing the columns of $W^\top$),
\[\Bigg \vert \Bigg \vert \widehat{\nabla_i L} - \mathbb{E} \left[ \frac{\partial L}{\partial W_i} \right] \Bigg \vert \Bigg \vert_2 \leq \textrm{poly}(h) \textrm{exp} \left(- \frac{h^p m_1^2}{2(b-a)^2} \right) \]
\end{lemma}
\begin{proof}
This lemma has been proven in Section~\ref{app:proxy} of the Appendix. 
\end{proof}

\begin{lemma}\label{bounds}
~\\
Assume that the hypotheses of Theorem \ref{sec:theorem:support} hold, and $p < \min \{ \frac{1}{2}, \nu^2\}$ (and hence $\xi > 2p$). 
Then for each $i$ indexing the columns of $W^\top$, there exist real valued functions $\alpha_i$ and $\beta_i$, and a vector $e_i$ such that $\widehat{\nabla_i L} = \alpha_i W_i - \beta_i A^*_i + e_i$, and
\begin{align*}
\alpha_i =  \Theta(m_2h^{p-1})+o(m_1^2h^{p-1})\\
\beta_i =  \Theta(m_2h^{p-1})+o(m_1^2h^{p-1})\\
\alpha_i - \beta_i = o(\max\{m_1^2,m_2\}h^{p-1})\\
\vert \vert e_i \vert \vert _2 = o(\max\{m_1^2,m_2\}h^{p-1})
\end{align*}
\end{lemma}

\begin{proof}
In subsection \ref{subsec:grad-coeff} we first get explicit forms of the above defined quantities $\alpha_i,\beta_i$ and $e_i$. Then the proof is completed by estimating them which is done in Appendix~\ref{app:asymptotics}
\end{proof}
~\\
With the above asymptotic results, we are in a position to assemble the proof of Theorem~\ref{sec:theorem:critical}.

\begin{proof}[Proof of Theorem~\ref{sec:theorem:critical}] Consider any $i$ indexing the columns of $W^\top$. Recall the definition of the proxy gradient $\widehat{\nabla_i L}$ at the beginning of this section. Let us define $\gamma_i = \widehat{\nabla_i L} - \mathbb{E}\left [ \frac {\partial L}{\partial W_i}\right ]$. Using $\alpha_i, \beta_i$ and $e_i$ as defined in Lemma~\ref{bounds}, we can write the expectation of the true gradient as, $\mathbb{E}\left [ \frac {\partial L}{\partial W_i}\right ] = \alpha_i W_i - \beta_i A_i^* + e_i - \gamma_i$. Further, by Lemma~\ref{good_proxy},  
\[ \Vert \gamma_i \Vert \leq \textrm{poly}(h) \textrm{exp} \left(- \frac{h^p m_1^2}{2(b-a)^2} \right).\] 

Since $\textrm{exp} \left(\frac{h^p m_1^2}{2(b-a)^2}\right)$ is superpolynomial in $h$, we obtain 

\begin{align*}
\bigg\Vert \mathbb{E}\left [ \frac {\partial L}{\partial W_i}\right ] \bigg\Vert_2 &= \vert \vert \alpha_i W_i - \beta_i A_i^* + e_i - \gamma_i \vert \vert_2\\
&= \vert \vert \alpha_i (W_i - A_i^*) +(\alpha_i - \beta_i) A_i^* + e_i - \gamma_i \vert \vert_2\\
&\leq \vert \alpha_i \vert \Vert W_i - A_i^* \Vert_2 + \vert \alpha_i - \beta_i \vert + \vert \vert e_i - \gamma_i \vert \vert_2\\
&\leq \frac{\Theta(m_2h^{p-1})}{h^{2p+\theta^2}} + o (\max\{m_1^2,m_2\}h^{p-1})\\
&+ o(\max\{m_1^2,m_2\}h^{p-1})\\
& = o (\max\{m_1^2,m_2\}h^{p-1}) 
\end{align*}
\end{proof}

\subsection{Simplifying the proxy gradient of the autoencoder under the sparse-coding generative model - to get explicit forms of the coefficients $\alpha$, $\beta$ and $e$ as required towards proving Lemma \ref{bounds} }\label{subsec:grad-coeff}

To recap we imagine being given as input signals $y \in \mathbb{R}^n$ (imagined as column vectors), which are generated from an overcomplete dictionary $A^* \in \mathbb{R}^{n \times h}$ of fixed incoherence. Let $x^* \in \mathbb{R}^h$ (imagined as column vectors) be the sparse code that generates $y$.
The model of the autoencoder that we now have is $\hat{y} = W^\top \textrm{ReLU}(Wy - \epsilon)$. $W$ is a $h \times n$ matrix and the $i^{th}$ column of $W^\top$ is to be denoted as the column vector $W_i$. 

Using the above notation the squared loss of the autoencoder is $\frac{1}{2} \vert \vert \hat{y} -  y \vert \vert^2$. But we introduce a dummy constant $D=1$ to be multiplied to $y$ because this helps read the complicated equations that would now follow. This marker helps easily spot those terms which depend on the sensing of $x^*$ (those with a factor of $D$) as opposed to the terms which are ``purely'' dependent on the neural net (those without the factor of $D$). Thus we think of the squared loss $L$ of our autoencoder as, 
\[ L = \frac{1}{2} \vert \vert \hat{y} - D y \vert \vert^2 = \frac {1}{2} (W^\top \textrm{ReLU}(Wy - \epsilon) - D y)^\top (W^\top \textrm{ReLU}(Wy - \epsilon) - D y) = \frac {1}{2} f^T f\] 
~\\
where we have defined $f \in \R^n$ as, 
\[ f = W^\top \textrm{ReLU}(Wy - \epsilon) - D y\]
Then we have, 
\[ J_{W_i}(f)_{ab} = \frac {\partial f_a}{\partial W_{ib}} = \textrm{ReLU}(W_i^\top y - \epsilon)\delta_{ab} + \textrm{Th} (W_i^Ty - \epsilon) W_{ia}y_b\]
In the form of a $n \times n$ derivative matrix this means,
\[ J_{W_i}(f) = \left[ \frac {\partial f_a}{\partial W_{ib}} \right] = \textrm{ReLU}(W_i^\top y - \epsilon)I + \textrm{Th} (W_i^\top y - \epsilon)W_iy^\top\]

~\\
This helps us write, 

\begin{align*}
\frac {\partial L}{\partial W_i} &= J_{W_i}(f))^\top f\\
&= (\textrm{ReLU}(W_i^\top y - \epsilon)I + \textrm{Th} (W_i^\top y - \epsilon)W_iy^\top)^\top [W^\top \textrm{ReLU}(Wy - \epsilon) - D y ]\\
&= \textrm{Th}(W^\top_i y - \epsilon_i) \left[ (W_i^\top y - \epsilon_i)I + y W_i^\top \right] \left( \sum_{j=1}^{h} \textrm{ReLU}(W_j^\top y - \epsilon_j) W_j - D y \right) \\
\end{align*}

Now going over to the proxy gradient $\widehat{\nabla_i L}$ corresponding to this term and we define the vector $G_i$ as, 

\begin{align*}
\widehat{\nabla_i L} &= \mathbb{E}_{S \in \mathbb{S}} \left[ \mathbf{1}_{i \in S} \times \mathbb{E}_{x^*_S} \left[ \left[ (W_i^{\top} y - \epsilon_i)I + y  W_i^\top \right] \left( \sum_{j \in S} (W_j^{\top} y - \epsilon_j) W_j - D y \right) \right] \right]\\
&= \mathbb{E}_{S \in \mathbb{S}} \left[ \mathbf{1}_{i \in S} \times G_i \right]
\end{align*}

Thus we have, 
\begin{align*}
G_i &= \mathbb{E}_{x^*_S} \left[ \left[ (W_i^{\top} A^* x^* - \epsilon_i)I + (A^* x^*) W_i^\top \right] \left( \sum_{j \in S} (W_j^{\top} A^* x^* - \epsilon_j) W_j - D A^* x^* \right) \right]\\
&= \underbrace{\mathbb{E}_{x^*_S} \left[ (W_i^{\top} A^* x^* - \epsilon_i)\left( \sum_{j \in S} (W_j^{\top} A^* x^* - \epsilon_j) W_j - D A^* x^* \right) \right]}_{\textrm{Term 1}} \\ 
&+ \underbrace{\mathbb{E}_{x^*_S} \left[ (A^* x^*) W_i^\top \left( \sum_{j \in S} (W_j^{\top} A^* x^* - \epsilon_j) W_j - D A^* x^* \right) \right]}_{\textrm{Term 2}}\\
\end{align*}

which can be decomposed into the following convenient parts, 

\begin{align*}
G_i 
&= \underbrace{\mathbb{E}_{x^*_S} \left[ \sum_{j \in S} \epsilon_i \epsilon_j W_j - \sum_{j , k \in S} \epsilon_i (W_j^{\top} A^*_k) W_j  x_k^* - \sum_{j ,k \in S} \epsilon_j (W_i^{\top} A^*_k) W_j  x_k^* + \sum_{j,k,l \in S} ( W_i^{\top} A^*_k)( W_j^{\top} A^*_l) W_j x_l^* x_k^* \right]}_{\textrm{From Term 1}} \\
&+ \underbrace{\mathbb{E}_{x^*_S} \left[ - D \sum_{j,k \in S} ( W_i^{\top} A_k^*) A_j^* x_k^* x_j^* + D \sum_{j \in S} \epsilon_i A^*_j x^*_j \right]}_{\textrm{From Term 1}} + \underbrace{\mathbb{E}_{x^*_S} \left[ - D \sum_{j,k \in S} (A_k^{* \top}W_i) A^*_j x^*_k x^*_j \right]}_{\textrm{From Term 2}} \\ 
&+ \underbrace{\mathbb{E}_{x^*_S} \left[ - \sum_{j,k \in S} \epsilon_j A_k^* (W_i^\top W_j)  x^*_k \right]}_{\textrm{From Term 2}} + \underbrace{\mathbb{E}_{x^*_S} \left[ \sum_{j,k,l \in S}  (W_i^\top W_j) (W_j^{\top} A_l^*)A_k^* x_k^* x_l^* \right]}_{\textrm{From Term 2}}
\end{align*}

~\\
Now we invoke the distributional assumption about i.i.d sampling of the coordinates for a fixed support and the definition of $m_1$ and $m_2$ to write, $\mathbb{E}_{x^*_S}[x^*_ix^*_j] = \mathbb{E}^2_{x^*_S}[x^*_i] = m_1^2$ for all $i \neq j$ and for $i=j$, $m_2 = \mathbb{E}_{x^*_S}[x^*_ix^*_j]$. Thus we get, 

\begin{align*}
G_i &= \underbrace{\sum_{j \in S} \epsilon_i \epsilon_j W_j - m_1 \sum_{j,k \in S}  ( W_j^\top A^*_k)  W_j \epsilon_i  - m_1 \sum_{j,k \in S} \epsilon_j (W_i^\top A^*_k) W_j}_{G^1_i\textrm{ From Term 1}}   \\
&+ \underbrace{ m_2 \sum_{j , k \in S} ( W_i ^\top A^*_k) (W_j^\top A^*_k)  W_j  + m_1^2\sum_{\substack{j, k, l \in S \\ k \neq l}} ( W_i ^\top A^*_k ) ( W_j ^\top A^*_l)  W_j}_{G^2_i \textrm{ From Term 1}} \\
&+ \underbrace{\left[ - D m_1^2\sum_{\substack{j, k \in S \\ j \neq k}} ( W_i^{\top} A_k^*) A_j^*  - D m_2 \sum_{j\in S} ( W_i^{\top} A_j^*) A_j^*  + m_1 D \sum_{j \in S} \epsilon_i A^*_j  \right]}_{G^3_i \textrm{ From Term 1}} \\ 
&-\underbrace{\left[ D m_1^2\sum_{\substack{j, k \in S \\ j \neq k}} (A_k^{* \top}W_i) A^*_j  +D m_2\sum_{j \in S} (A_j^{* \top}W_i) A^*_j \right]}_{G^4_i \textrm{ From Term 2}} \\
&-\underbrace{m_1\left [ \sum_{j,k \in S} \epsilon_j (W_i^\top W_j) A_k^*   \right] +  \left[ m_2 \sum_{j,k \in S}  (W_i^\top W_j) (W_j^{\top} A_k^*)A^*_k + m_1^2 \sum_{\substack{j, k, l \in S \\ k \neq l}}  (W_i^\top W_j) (W_j^{\top} A_l^*)A^*_k \right]}_{G^5_i \textrm{ From Term 2}}
\end{align*}
~\\
Each term in the above sum is a vector. Now we separate out from the sums the terms which are in the directions of $W_i$ or $A_i^*$ and the rest. We remember that this is being under the condition that $i \in S$. To make this easy to read we do this separation for each line of the above equation separately in a different equation block. Also inside every block we do the separation for each summation term in a separate line.

\begin{align*}
G^1_i &= \sum_{j \in S} \epsilon_i \epsilon_j W_j - m_1 \sum_{j,k \in S}  ( W_j^\top A^*_k)  W_j \epsilon_i  - m_1 \sum_{j,k \in S} \epsilon_j (W_i^\top A^*_k) W_j \\ 
&= \left [ \epsilon_i^2 W_i + \sum_{\substack{j \in S \\ j \neq i}} \epsilon_i \epsilon_j W_j \right ]\\ 
&-m_1 \left [ \sum_{k \in S}  \epsilon_i  ( W_i^\top A^*_k)  W_i + \sum_{\substack{j, k \in S \\ j \neq i}}  ( W_j^\top A^*_k)  W_j \epsilon_i\right ]\\ 
&-m_1 \left [ \sum_{k \in S} \epsilon_i (W_i^\top A^*_k) W_i + \sum_{\substack{j, k \in S \\ j \neq i}} \epsilon_j (W_i^\top A^*_k) W_j \right ]\\ 
~\\ \\
G^2_i &= m_2 \sum_{j , k \in S} ( W_i ^\top A^*_k) (W_j^\top A^*_k)  W_j  + m_1^2\sum_{\substack{j, k, l \in S \\ k \neq l}}  ( W_i ^\top A^*_k ) ( W_j ^\top A^*_l)  W_j\\
&= m_2 \left [ \sum_{ k \in S} ( W_i ^\top A^*_k) (W_i^\top A^*_k)  W_i + \sum_{\substack{j, k \in S \\ j \neq i}} ( W_i ^\top A^*_k) (W_j^\top A^*_k)  W_j \right ] \\ 
&+ m_1^2 \left [ \sum_{\substack{k, l \in S \\ k \neq l}}  ( W_i ^\top A^*_k ) ( W_i ^\top A^*_l)  W_i + \sum_{\substack{j, k, l \in S \\ j \neq i \\ k \neq l}}  ( W_i ^\top A^*_k ) ( W_j ^\top A^*_l)  W_j \right ]\\
\end{align*}

\begin{align*}
G^3_i &= -D\left[ m_1^2\sum_{\substack{j, k \in S \\ j \neq k}} ( W_i^{\top} A_k^*) A_j^*  + m_2 \sum_{j\in S} ( W_i^{\top} A_j^*) A_j^*  - m_1  \sum_{j \in S} \epsilon_i A^*_j  \right] \\
&=-D \left [ m_1^2 \sum_{\substack{k \in S \\ k \neq i}} ( W_i^{\top} A_k^*) A_i^* +  m_1^2 \sum_{\substack{j, k \in S \\ j \neq i \\ j \neq k}} ( W_i^{\top} A_k^*) A_j^* \right ]\\
&-D \left [ m_2 ( W_i^{\top} A_i^*) A_i^* + m_2 \sum_{\substack{j \in S \\ j \neq i}} ( W_i^{\top} A_j^*) A_j^* \right ] \\
&-D \left [ -m_1 \epsilon_i A_i^* - m_1 \sum_{\substack{j \in S \\ j \neq i}} \epsilon_i A_j^* \right ]\\
~\\ \\
G^4_i &= -\left[ D m_1^2\sum_{\substack{j,k \in S \\ j \neq k}} (A_k^{* \top}W_i) A^*_j  +D m_2\sum_{j \in S} (A_j^{* \top}W_i) A^*_j \right] \\
&=-D \left [ m_1^2  \sum_{\substack{k \in S \\ k \neq i}} (A_k^{* \top}W_i) A^*_i  +  m_1^2 \sum_{\substack{j ,k\in S \\ j \neq k \\ j \neq i}} (A_k^{* \top}W_i) A^*_j \right ]\\
&-D \left [ m_2 (A_i^{* \top}W_i) A^*_i + m_2 \sum_{\substack{j \in S \\ j \neq i}} (A_j^{* \top}W_i) A^*_j \right ]\\
\end{align*}

\begin{align*}
G^5_i &= -m_1\left [ \sum_{j,k \in S} \epsilon_j (W_i^\top W_j) A_k^*   \right] +  \left[ m_2 \sum_{j,k \in S}  (W_i^\top W_j) (W_j^{\top} A_k^*)A^*_k + m_1^2 \sum_{\substack{j, k, l \in S \\ k \neq l}}  (W_i^\top W_j) (W_j^{\top} A_l^*)A^*_k \right] \\
&= -m_1 \sum_{j \in S} \epsilon_j (W_i^\top W_j) A^*_i - m_1 \sum_{\substack{j,k \in S \\ k \neq i}} \epsilon_j (W_i^\top W_j) A^*_k \\
&+ m_2 \sum_{j \in S} (W_i^\top W_j) (W_j^\top A^*_i) A^*_i + m_2 \sum_{\substack{j,k \in S \\ k \neq i}} (W_i^\top W_j) (W_j^\top A^*_k) A^*_k \\
&+ m_1^2 \sum_{\substack{j,l \in S \\ l \neq i}} (W_i^\top W_j) (W_j^{\top} A_l^*)A^*_i + m_1^2 \sum_{\substack{j,k,l \in S \\ k \neq i,l}} (W_i^\top W_j) (W_j^{\top} A_l^*)A^*_k 
\end{align*}

~\\
Thus combining the $G_i^1,\ldots,G_i^5$ above we have, $\widehat{\nabla_i L} = \alpha_i W_i - \beta_i A^*_i + e_i$ where, 

\begin{align*}
\alpha_i &=  \mathbb{E}_{S \in \mathbb{S}} \Bigg[ \mathbf{1}_{i \in S} \times \Bigg \{ m_2 \sum_{ k \in S} ( W_i ^\top A^*_k) (W_i^\top A^*_k) + m_1^2 \sum_{\substack{k, l \in S \\ k \neq l}}  ( W_i ^\top A^*_k ) ( W_i ^\top A^*_l) -2m_1 \sum_{k \in S}  \epsilon_i  ( W_i^\top A^*_k) + \epsilon_i^2 \Bigg \} \Bigg]\\ 
\beta_i &=  \mathbb{E}_{S \in \mathbb{S}} \Bigg[ \mathbf{1}_{i \in S} \times \Bigg \{  2 D m_1^2 \sum_{\substack{k \in S \\ k \neq i}} ( W_i^{\top} A_k^*) + 2 D m_2 ( W_i^\top A^*_i) - D m_1 \epsilon_i + m_1 \sum_{j \in S} \epsilon_j (W_i^\top W_j)\\
&- m_2 \sum_{j \in S} (W_i^\top W_j) (W_j^\top A^*_i) - m_1^2 \sum_{\substack{j,l \in S \\ l \neq i}} (W_i^\top W_j) (W_j^{\top} A_l^*) \Bigg \} \Bigg ] \\ 
e_i &= \mathbb{E}_{S \in \mathbb{S}} \Bigg[ \mathbf{1}_{i \in S} \times \Bigg \{ \sum_{\substack{j \in S \\ j \neq i}} \epsilon_i \epsilon_j W_j - m_1 \sum_{\substack{j, k \in S \\ j \neq i}}  \epsilon_i ( W_j^\top A^*_k)  W_j - m_1 \sum_{\substack{j, k \in S \\ j \neq i}} \epsilon_j (W_i^\top A^*_k) W_j \\
&+ m_2 \sum_{\substack{j, k \in S \\ j \neq i}} ( W_i ^\top A^*_k) (W_j^\top A^*_k)  W_j + m_1^2 \sum_{\substack{j, k, l \in S \\ j \neq i \\ k \neq l}}  ( W_i ^\top A^*_k ) ( W_j ^\top A^*_l)  W_j \\
&-2D m_1^2 \sum_{\substack{j, k \in S \\ j \neq i \\ j \neq k}} ( W_i^{\top} A_k^*) A_j^* -2D m_2 \sum_{\substack{j \in S \\ j \neq i}} ( W_i^{\top} A_j^*) A_j^* + D m_1 \sum_{\substack{j \in S \\ j \neq i}} \epsilon_i A_j^* \\
&- m_1 \sum_{\substack{j,k \in S \\ k \neq i}} \epsilon_j (W_i^\top W_j) A^*_k + m_2 \sum_{\substack{j,k \in S \\ k \neq i}} (W_i^\top W_j) (W_j^\top A^*_k) A^*_k + m_1^2 \sum_{\substack{j,k,l \in S \\ k \neq i,l}} (W_i^\top W_j) (W_j^{\top} A_l^*)A^*_k \Bigg \} \Bigg ]
\end{align*}
Thus we have laid the groundwork of finding a convenient decomposition of the proxy-gradient in terms of the quantities $\alpha_i,\beta_i$ and $e_i$. Now we can go over to  Appendix~\ref{app:asymptotics} where their magnitudes are estimated towards completing the proof of Lemma \ref{bounds}.

\section{Simulations}\label{sec:experiments}
We conduct some experiments on synthetic data in order to check whether the gradient norm is indeed small within the columnwise $\delta$-ball of $A^*$. We also make some observations about the landscape of the squared loss function, which has implications for being able to recover the ground-truth dictionary $A^*$.
\paragraph*{Data Generation Model} We generate random gaussian dictionaries ($A^*$) of size $n \times h$ where $n=50$, and $h=256, 512, 1024, 2048$ and $4096$. For each $h$, we generate a dataset containing $N=5000$ sparse vectors with $h^p$ non-zero entries, for various $p \in [0.01,0.5]$. In our experiments, the coherence parameter $\xi$ was approximately $0.1$. The support of each sparse vector $x^*$ is drawn uniformly from all sets of indices of size $h^p$, and the non-zero entries in the sparse vectors are drawn from a uniform distribution between $a = 1$ and $b = 10$. Once we have generated the sparse vectors, we collect them in a matrix $X^* \in \mathbb{R}^{h \times N}$ and then compute the signals $Y = A^* X^*$. We set up the autoencoder as defined through equation \ref{eqn:autoencoder}. We analyze the squared loss function in~\eqref{eqn:loss} and its gradient with respect to a column of $W$ through their empirical averages over the signals in $Y$.

\paragraph*{Results} Once we have generated the data, we compute the empirical average of the gradient of the loss function in~\eqref{eqn:loss} at $200$ random points which are columnwise $\frac{\delta}{2} = \frac{1}{2h^{2p}}$ away from $A^*$. We average the gradient over the $200$ points which are all at the same distance from $A^*$, and compare the average column norm of the gradient to $h^{p-1}$. Our experimental results shown in Table \ref{tab:grad-norm} demonstrate that the average column norm of the gradient is of the order of $h^{p-1}$ (and thus falling with $h$ for any fixed $p$) as expected from Theorem \ref{sec:theorem:critical}.

\begin{table*}[h]
\centering
\begin{tabular}{|l||*{7}{c|}}\hline
\backslashbox{$h$}{$p$} & 0.01 & 0.02 & 0.05 & 0.1 & 0.2 \\\hline\hline
256 & (0.0137, 0.0041) & (0.0138, 0.0044) & (0.0126, 0.0052) & (0.0095, 0.0068) & (0.0284, 0.0118) \\\hline
512 & (0.0058, 0.0021) & (0.0058, 0.0022) & (0.0054, 0.0027) & (0.0071, 0.0036) & (0.0104, 0.0068) \\\hline
1024 & (0.0025, 0.0010) & (0.0024, 0.0011) & (0.0026, 0.0014) & (0.0079, 0.0020) & (0.0078, 0.0039) \\\hline
2048 & (0.0011, 0.0005) & (0.0012, 0.0006) & (0.0025, 0.0007) & (0.0031, 0.0010) & (0.0032, 0.0022) \\\hline
4096 & (0.0006, 0.0003) & (0.0012, 0.0003) & (0.0013, 0.0004) & (0.0026, 0.0006) & (0.0020, 0.0013) \\\hline
\end{tabular}
\newline 
\begin{tabular}{|l||*{7}{c|}}\hline
\backslashbox{$h$}{$p$} & 0.3 & 0.5 \\\hline\hline
256 &  (0.0464, 0.0206) & (0.0343, 0.0625) \\\hline
512 &  (0.0214, 0.0127) & (0.0028, 0.0442) \\\hline
1024 & (0.0099, 0.0078) & (0.00, 0.0313)\\\hline
2048 & (0.0036, 0.0048) & (0.00, 0.0221) \\\hline
4096 & (0.0008, 0.0030) & (0.00, 0.0156) \\\hline
\end{tabular}
\caption{Average gradient norm for points that are columnwise $\frac{\delta}{2}$ away from $A^*$. For each $h$ and $p$ we report $\left( \vert \vert \mathbb{E} \left[ \frac{\partial L}{\partial W_i} \right] \vert \vert, h^{p-1}\right)$. We note that the gradient norm and $h^{p-1}$ are of the same order, and for any fixed $p$ the gradient norm is decreasing with $h$ as expected from Theorem \ref{sec:theorem:critical}}
\label{tab:grad-norm}
\end{table*}

We also plot the squared loss of the autoencoder along a randomly chosen direction to understand the geometry of the landscape of the loss function around $A^*$. We draw a matrix $\Delta W$ from a standard normal distribution, and normalize its columns. We then plot $f(t) = L( (A^* + t \Delta W )^\top )$, as well as the gradient norm averaged over all the columns. For purposes of illustration, we show these plots for $p=0.01,0.1, 0.3$. The plots for $h=256$ are in Figure \ref{fig:loss_grad_256}, and those for $h=4096$ in Figure \ref{fig:loss_grad_4096}. From the plots for $p=0.01$ and $0.1$, we can observe that the loss function value, and the gradient norm keep decreasing as we get close to $A^*$. Figure \ref{fig:loss_grad_256} and \ref{fig:loss_grad_4096} are representative of the shapes obtained for every direction, $\Delta W$ that we checked. This suggests that $A^*$ might conveniently lie at the bottom of a well in the landscape of the loss function. For the value of $p=0.3$, (which is much larger than the coherence parameter $\xi$), Theorem \ref{sec:theorem:support} is no longer valid. We see that the value of the loss function decreases a little as we move away from $A^*$, and then increases. We suspect that $A^*$ is here in a region where $\textrm{ReLU}(A^{*\top}y - \epsilon ) = 0$, which means the function is flat in a small neighborhood of $A^*$.

\begin{figure}[!htbp]
\center
\includegraphics[scale=0.30]{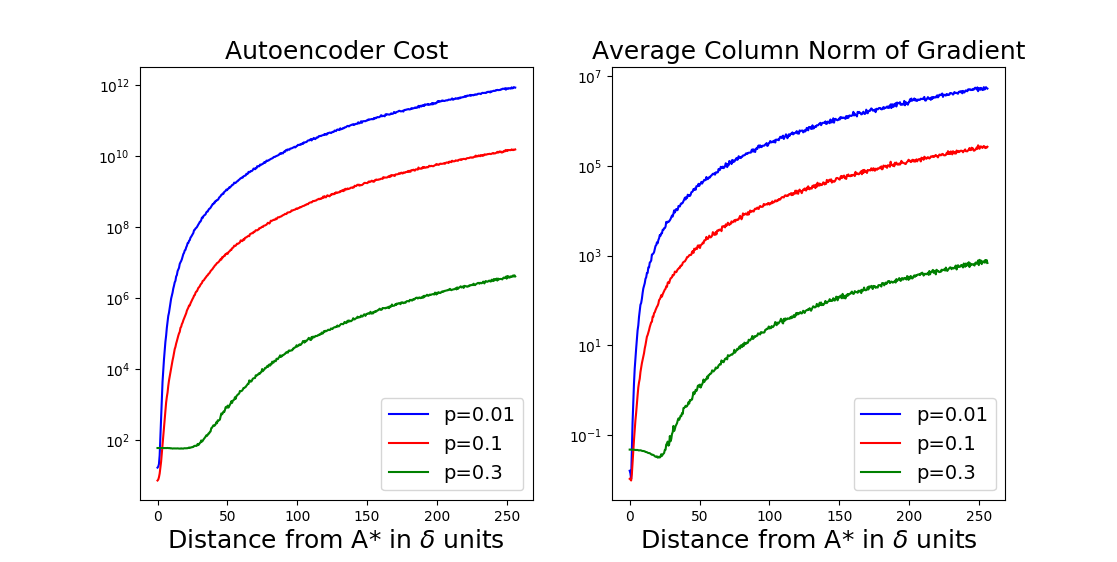}
\caption{Loss function plot for $h=256,n=50$}
\label{fig:loss_grad_256}
\end{figure} 

\begin{figure}[!htbp]
\centering
\includegraphics[scale=0.30]{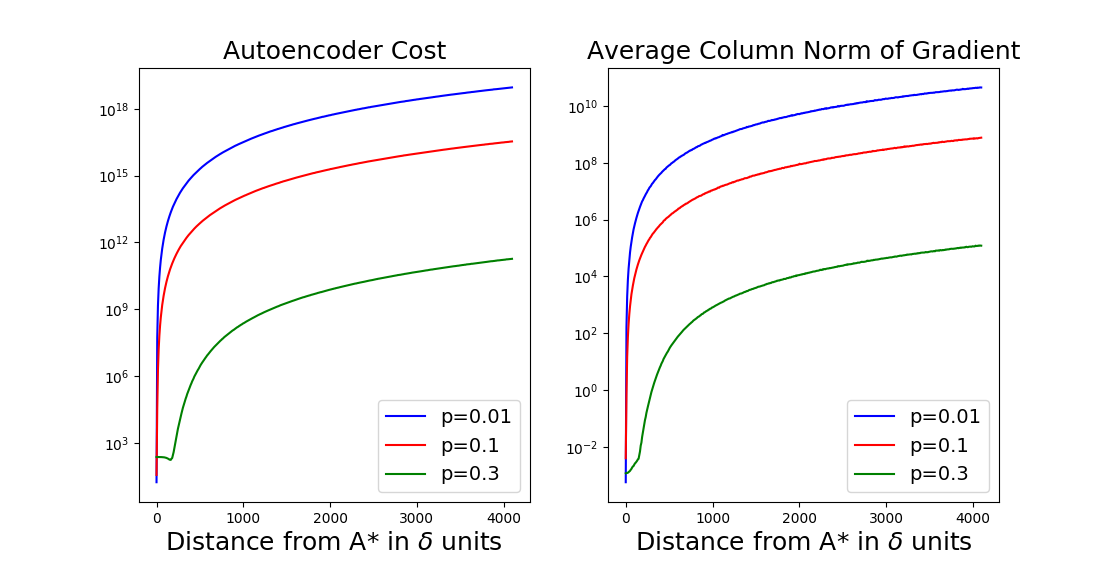}
\caption{Loss function plot for $h=4096,n=50$}\label{fig:loss_grad_4096}
\end{figure}

We also tried to minimize the squared loss of the autoencoder using gradient descent. In these experiments, we initialized $W^\top$ far away from $A^*$ (precisely at a columnwise distance of ${\frac h 5} \times \delta$), and did gradient descent until the gradient norm dropped below a factor of $2 \times 10^{-5}$ of the initial norm of the gradient. We then computed the average columnwise distance between $W^\top_{\textrm{final}}$ and $A^*$, and report the $\%$ decrease in the average columnwise distance from the initial point. These results are reported in Table \ref{tab:grad-des} below. These experiments suggest that there is a neighborhood of $A^*$ (the radius of which is increasing with $h$), such that gradient descent initialized at the edge of that neighborhood, greatly reduces the average columnwise distance between $W^\top$ and $A^*$.

\begin{table}[!htbp]
\centering
\begin{tabular}{|l||*{2}{c|}}\hline
$h$ & $p=0.05$ & $p=0.1$ \\\hline\hline
256 & 97.7\% & 96.9\%  \\\hline
512 & 98.6\% & 98.2\%  \\\hline
1024 & 99\% & 98.8\% \\\hline
2048 & 99.2\% & 99\% \\\hline
4096 & 99.4\% & 99.2\% \\\hline
\end{tabular}
\caption{Fraction of initial columnwise distance covered by the gradient descent procedure}\label{tab:grad-des}
\end{table}

\section{Conclusion}
In this chapter we have undertaken a rigorous analysis of the loss function of the squared loss of an autoencoder when the data is assumed to be generated by sensing of sparse high dimensional vectors by an overcomplete dictionary. {\bf We have shown that the expected gradient of this loss function is very close to zero in a neighborhood of the generating overcomplete dictionary.} 
~\\ \\
Our simulations complement this theoretical result by providing further empirical support. Firstly, they show that the gradient norm in this $\delta-$ball of $A^*$ indeed falls with $h$ and is of the same order as $\frac{1}{h^{1-p}}$ as expected from our proof. Secondly, the experiments also strongly suggest ranges of values of $h$ and $p$ where $A^*$ is a local minima of this loss function and that it has a neighborhood where the reconstruction error is low.
~\\ \\
This suggests sparse coding problems can be solved by training autoencoders using gradient descent based algorithms. 
Further, recent investigations have led to the conjecture/belief that many important unsupervised learning tasks, e.g. recognizing handwritten digits, are sparse coding problems in disguise~\citep{makhzani2013k,makhzani2015winner}. Thus, our results could shed some light on the observed phenomenon that gradient descent based algorithms train autoencoders to low reconstruction error for natural data sets, like MNIST.
~\\ \\
It remains to rigorously show whether a gradient descent algorithm can be initialized randomly (may be far away from $A^*$) and still be shown to converge to this neighborhood of critical points around the dictionary. Towards that it might be helpful to understand the structure of the Hessian outside this neighborhood. Since our analysis applies to the expected gradient, it remains to analyze the sample complexities where these nice results will become prominent. 
~\\ \\
The possibility also remains open that this standard loss or some other loss functions exist for the autoencoder with the provable property of having a global minima/minimum at the ground truth dictionary. We have mentioned one example of such in a special case (when $A^*$ is square orthogonal and $x^*$ is nonnegative) and even in this special case it remains open to find a provable optimization algorithm. 
~\\ \\
On the simulation front we have a couple of open challenges yet to be tackled. Firstly, it is left to find efficient implementations of the iterative update rule based on the exact gradient of the proposed loss function which has been given in~\eqref{eqn:loss}. This would open up avenues for testing the power of this loss function on real data rather than the synthetic data used here. Secondly, a simulation of the main Theorem~\ref{sec:theorem:critical} that can probe deeper into its claim would need to be able to sample $A^*$ for different $h$ at a fixed value of the incoherence parameter $\xi$. This sampling question of $A^*$ with these constraints is an unresolved one that is left for future work. 
~\\ \\Autoencoders with more than one hidden layer have been used for unsupervised feature learning \citep{le2013building} and recently there has been an analysis of the sparse coding performance of convolutional neural networks with one layer~\citep{gilbert2017towards} and two layers of nonlinearities \citep{vardan2016convolutional}. The connections between neural networks and sparse coding has also been recently explored in~\cite{bora2017compressed}. It remains an exciting open avenue of research to try to do a similar study as in this work to determine if and how deeper architectures under the same generative model might provide better means of doing sparse coding.

\begin{subappendices}

\chapter*{\vspace{10pt} Appendix To Chapter \ref{chapaut}}\label{app:Auto} 

\section {The proxy gradient is a good approximation of the true expectation of the gradient (Proof of Lemma 5.1)} \label{app:proxy}

\begin{proof}
To make it easy to present this argument let us abstractly think of the function $f$ (defined for any $i \in \{1,2,3,..,h\}$) as $f(y,W,X) = \frac{\partial L}{\partial W_i} $ where we have defined the random variable $X = \text{Th}[W_i^Ty-\epsilon_i]$. It is to be noted that because of the  ReLU term and its derivative this function $f$ has a dependency on $y= A^*x^*$ even outside its dependency through $X$. Let us define another random variable $Y = \mathbf{1}_{i \in \text{Support}(x^*)}$. Then we have,
\begin{align*}
&\big\Vert \mathbb{E}_{x^*}[f(y,W,X)] - \mathbb{E}_{x^*}[f(y,W,Y)] \big\Vert_{\ell_2} \\
\leq  &\mathbb{E}_{x^*} [ \vert f(y,W,X) -  f(y,W,Y) \vert_{\ell_2} ]\\
\leq &\mathbb{E}_{x^*} [ \vert f(y,W,X) (\mathbf{1}_{X=Y} + \mathbf{1}_{X \neq Y}) -  f(y,W,Y)(\mathbf{1}_{X=Y} + \mathbf{1}_{X \neq Y}) \vert_{\ell_2} ] \\
\leq & \mathbb{E}_{x^*} [\vert (f(y,W,X) - f(y,W,Y)) \vert_{\ell_2}\mathbf{1}_{X \neq Y} ]\\
\leq &\sqrt{\mathbb{E}_{x^*}[ \big\vert f(y,W,X) - f(y,W,Y) \big\vert_{2}^2 ]} \sqrt{\mathbb{E}_{x^*} [\mathbf{1}_{X \neq Y}]}
\end{align*}
~\\
In the last step above we have used the Cauchy-Schwarz inequality for random variables. We recognize that $\mathbb{E}_{x^*}[f(y,W,Y)]$ is precisely what we defined as the proxy gradient $\widehat{\nabla_i L}$. Further for such $W$ as in this lemma the support recovery theorem (Theorem 3.1) holds and that is precisely the statement that the term, $\mathbb{E}_{x^*} [\mathbf{1}_{X \neq Y}]$ is small. So we can rewrite the above inequality as, 
\begin{align*}
\bigg\Vert \mathbb{E}_{x^*}[\frac{\partial L}{\partial W_i}] - \widehat{\nabla_i L} \bigg\Vert_{2} \leq \sqrt{\mathbb{E}_{x^*}[ \big\vert f(y,W,X) - f(y,W,Y) \big\vert_{2}^2 ]} \exp \left ( - \frac{h^pm_1^2}{2(b-a)^2} \right )
\end{align*}
~\\
We remember that $f$ is a polynomial in $h$ because its $h$ dependency is through Frobenius norms of submatrices of $W$ and $\ell_2$ norms of projections of $Wy$. But the $\ell_\infty$ norm of the training vectors $y$ (that is $b$) have been assumed to be bounded by $\text{poly}(h)$. Also we have the assumption that the columns of $W^\top$ are within a $\frac{1}{h^{p+\nu^2}}-$ball of the corresponding columns of $A^*$ which in turn is a $n \times h$ dimensional matrix of bounded norm because all its columns are normalized. So summarizing we have, 

\begin{align*}
\bigg\Vert \mathbb{E}_{x^*}[\frac{\partial L}{\partial W_i}] - \widehat{\nabla_i L} \bigg\Vert_{2} \leq  \text{poly}(h)\exp \left ( - \frac{h^pm_1^2}{2(b-a)^2} \right )
\end{align*}
~\\
The above inequality immediately implies the claimed lemma.
\end{proof}

\newpage 
\section {The asymptotics of the coefficients of the gradient of the squared loss (Proof of Lemma $5.2$)}\label{app:asymptotics}

~\\  
We will pick up from where subsection \ref{subsec:grad-coeff} left and will now estimate bounds on each of the terms $\alpha_i, \beta_i, ||e_i||$, which were defined at the end of that segment. We will separate them as $\alpha_i = \tilde{\alpha_i} + \hat{\alpha_i}$ (similarly for the other terms). Where the tilde terms are those that come as a coefficient of $m_2$, and the hat terms are the ones that come as coefficient of $m_1$ or $\epsilon$ or both. (Note : Given the previous definitions of $q_1$ and $q_2$ it is obvious from context as to how the quantities $q_i,q_{ij},q_{ijk}$ and $q_S$ mean and we shall use this notation in this Appendix.)

\subsection {Estimating the $m_2$ dependent parts of the derivative}

Since $||A^*_i||=1$ and $W_i$ is being assumed to be within a $0 < \delta <1$ ball of $A^*_i$ we can use the following inequalities:
\begin{align*}
||W_i|| &= ||W_i - A^*_i + A^*_i|| \leq ||W_i - A^*_i|| + ||A^*_i|| = \delta + 1\\ 
||W_i|| &\geq 1-\delta \\
\langle W_i, A^*_i \rangle &= \langle W_i - A^*_i, A^*_i \rangle + \langle A^*_i, A^*_i \rangle \leq ||W_i - A^*_i||||A^*_i|| +  1 \leq \delta + 1\\ 
\langle W_i, A^*_i \rangle &\geq 1-\delta \\
|\langle W_j, A^*_i \rangle| &= |\langle W_j - A^*_j, A^*_i \rangle + \langle A^*_j, A^*_i \rangle| \leq \frac{\mu}{\sqrt{n}}+ ||W_j - A^*_j||||A^*_i|| = \frac{\mu}{\sqrt{n}}+\delta\\
\vert \langle W_i, W_j  \rangle \vert &= \vert \langle W_i - A_i^*,W_j \rangle + \langle A_i^*,W_j \rangle \vert \leq \delta(1+\delta) + (\delta + \frac {\mu}{\sqrt{n}}) = \delta ^2 + 2\delta + \frac{\mu}{\sqrt{n}}\\
\langle W_i, W_i \rangle &= ||W_i||^2 \geq (1-\delta)^2 \\
\langle W_i, W_i \rangle &= ||W_i||^2 \leq (1+\delta)^2
\end{align*}

\paragraph{Bounding $\tilde{\beta_i}$}
\begin{align*}
\tilde{\beta_i} &= \mathbb{E}_{S \in \mathbb{S}}\left [ \mathbf{1}_{i \in S} \left \{ 2D   m_2 ( W_i^{\top} A_i^*) - m_2 \sum_{j \in S} (W_i^\top W_j) (W_j^\top A^*_i)\right \} \right ]\\
&= \mathbb{E}_{S \in \mathbb{S}}\left [ \mathbf{1}_{i \in S} \left \{ 2D   m_2 \langle W_i, A_i^* \rangle  - m_2 ||W_i||^2 \langle W_i, A^*_i\rangle - m_2 \sum_{\substack{j \in S \\ j \neq i}} \langle W_i, W_j \rangle  \langle W_j, A^*_i\rangle \right \} \right ]
\end{align*}

Evaluating the outer expectation we get,
\begin{align}\label{beta_m2}
\nonumber \tilde{\beta_i} &= \sum_{\{S \in \mathbb{S} : i \in S\}} q_S 2D m_2 \langle W_i, A_i^* \rangle  - \sum_{\{S \in \mathbb{S} : i \in S\}} q_S m_2 ||W_i||^2 \langle W_i, A^*_i\rangle - m_2 \sum_{\substack{j =1 \\ j \neq i}}^h \langle W_i, W_j \rangle  \langle W_j , A^*_i\rangle \sum_{\{S \in \mathbb{S} : i,j \in S, i \neq j\}} q_S\\ 
\nonumber &= 2D q_i m_2 \langle W_i, A_i^* \rangle - q_i m_2 ||W_i||^2 \langle W_i, A^*_i\rangle - m_2 \sum_{\substack{j =1 \\ j \neq i}}^h q_{ij} \langle W_i, W_j \rangle  \langle W_j , A^*_i\rangle\\
\nonumber &\text{Upper bounding the above we get,}\\
\nonumber \tilde{\beta_i} &\leq 2D m_2 h^{p-1} (1+\delta) - m_2 h^{p-1} (1-\delta)^3 + m_2 h^{2p-1} \left( \delta + \frac{\mu}{\sqrt{n}} \right) \left( \delta^2 + 2\delta + \frac{\mu}{\sqrt{n}} \right)\\
\nonumber &= 2D m_2 h^{p-1} (1+h^{-p-\nu^2}) - m_2 h^{p-1} (1-3h^{-p-\nu^2} + 3 h^{-2p-2\nu^2} - h^{-3p-3\nu^2}) \\
&+ m_2 h^{2p-1} (h^{-3p-3\nu^2} + 2h^{-2p-2\nu^2} + h^{-2p-2\nu^2 - \xi} + 3h^{-p-\nu^2 - \xi} + h^{-2 \xi}) \\
\nonumber &\text{Similarly for the lower bound on $\beta_i$ we get,}\\
\nonumber \tilde{\beta_i} &\geq 2D m_2 h^{p-1} (1-\delta) - m_2 h^{p-1} (1+\delta)^3 - m_2 h^{2p-1} \left( \delta + \frac{\mu}{\sqrt{n}} \right) \left( \delta^2 + 2\delta + \frac{\mu}{\sqrt{n}} \right)\\
\nonumber &= 2D m_2 h^{p-1} (1-h^{-p-\nu^2}) - m_2 h^{p-1} (1+3h^{-p-\nu^2} + 3 h^{-2p-2\nu^2} + h^{-3p-3\nu^2}) \\
&- m_2 h^{2p-1} (h^{-3p-3\nu^2} + 2h^{-2p-2\nu^2} + h^{-2p-2\nu^2 - \xi} + 3h^{-p-\nu^2 - \xi} + h^{-2 \xi})
\end{align}
~\\
Thus for $0<p<2\xi$ and $D=1$, we have $\beta = \Theta \left( m_2 h^{p-1} \right)$

\paragraph{Bounding $\tilde{\alpha_i}$}
\begin{align*}
\tilde{\alpha_i} &= \mathbb{E}_{S \in \mathbb{S}}\left [ \mathbf{1}_{i \in S} \left \{ m_2  \sum_{k \in S}  (W_i^{\top}A_k^*)^2  \right \} \right ]\\
&= \mathbb{E}_{S \in \mathbb{S}}\left [ \mathbf{1}_{i \in S} \left \{ m_2  \langle W_i, A_i^* \rangle^2 + m_2 \sum_{\substack{k \in S \\ k \neq i}}  \langle W_i, A_k^*\rangle^2  \right \} \right ] \\
&= \sum_{\{S \in \mathbb{S} : i \in S\}} m_2 \langle W_i, A_i^* \rangle^2 q_S +  \sum_{\substack{k=1 \\ k \neq i}}^h \sum_{\{S \in \mathbb{S} : i,k \in S\}} \langle W_i, A_k^*\rangle^2 q_S \\
&= m_2 \langle W_i, A_i^* \rangle^2 \sum_{\{S \in \mathbb{S} : i \in S\}} q_S + m_2 \sum_{\substack{k=1 \\ k \neq i}}^h \langle W_i, A_k^*\rangle^2 \left( \sum_{\{S \in \mathbb{S} : i,k \in S, i \neq k\}} q_S\right)  \\
&= q_i m_2 \langle W_i, A_i^* \rangle^2  + m_2 \sum_{\substack{k=1 \\ k \neq i}}^h q_{ik} \langle W_i, A_k^*\rangle^2 \\
&= h^{p-1} m_2 \langle W_i, A_i^* \rangle^2  + m_2 h^{2p-1} \textrm{ max } \langle W_i, A_k^*\rangle^2
\end{align*}

~\\
The above implies the following bounds,
\begin{align}\label{alpha_m2}
h^{p-1} m_2 (1 - h^{-p-\nu^2})^2 \leq \tilde{\alpha_i} \leq h^{p-1} m_2 (1 + h^{-p-\nu^2})^2 + m_2 h^{2p-1} (h^{-p-\nu^2} + h^{-\xi})^2
\end{align}
As long as $0< p < 2\xi$, $\tilde{\alpha_i} = \Theta \left( m_2 h^{p-1}\right)$

\paragraph{Bounding $\vert \vert \tilde{e_i} \vert \vert_2$}

\begin{align*}
\tilde{e_i} &= \mathbb{E}_{S \in \mathbb{S}}\left [ \mathbf{1}_{i \in S} \times \left \{    m_2\sum_{\substack{j, k \in S \\ j \neq i}} ( W_i ^\top A^*_k) (W_j^\top A^*_k)  W_j
+ (-2D)m_2 \sum_{\substack{j \in S \\ j \neq i}} ( W_i^{\top} A_j^*) A_j^* \right \} \right ]\\
&+ \mathbb{E}_{S \in \mathbb{S}}\left [ \mathbf{1}_{i \in S} \times \left \{ m_2 \sum_{\substack{j,k \in S \\ k \neq i}} (W_i^\top W_j) (W_j^\top A^*_k) A^*_k \right  \} \right ]\\
\end{align*}

Expanding further over the summation of the $j$ and the $k$ indices we have,

\begin{align*}
\tilde{e_i} &= \mathbb{E}_{S \in \mathbb{S}}\Bigg [ \mathbf{1}_{i \in S} \times m_2 \Bigg \{    \sum_{j (=k) \in S\setminus i } ( W_i^{\top} A_j^*)(  W_j^{\top}  A_j^*) W_j + \sum_{\substack{j \in S\setminus i\\ k \in S \setminus i,j}} ( W_i^{\top}  A_k^*)( W_j^{\top}  A_k^*) W_j \\
&+ \sum_{\substack{j \in S\setminus i \\ k=i}} ( W_i^{\top}  A_i^*)(  W_j^{\top}  A_i^*) W_j \Bigg \} \Bigg ]\\
&+ \mathbb{E}_{S \in \mathbb{S}}\left [ \mathbf{1}_{i \in S} \times (-2D)m_2 \left \{ \sum_{\substack{j \in S \\ j \neq i}} ( W_i^{\top} A_j^*) A_j^*  \right \} \right ]\\
&+\mathbb{E}_{S \in \mathbb{S}}\Bigg [ \mathbf{1}_{i \in S} \times m_2 \Bigg \{ \sum_{\substack{k(=j) \in S \setminus i}} (W_i^\top W_k) (W_k^\top A^*_k) A^*_k  + \sum_{\substack{k \in S\setminus i\\ j \in S \setminus i,k}} (W_i^\top W_j) (W_j^\top A^*_k) A^*_k\\
&+ \sum_{\substack{k \in S\setminus i \\ j=i}} (W_i^\top W_i) (W_i^\top A^*_k) A^*_k \Bigg \} \Bigg ]
\end{align*}

~\\
Expanding the above in terms of $q_S$ we have, 

\begin{align*}
\tilde{e_i} &= m_2 \Bigg \{ \sum_{j=1, j \neq i}^h ( W_i^{\top} A_j^*)(  W_j^{\top}  A_j^*) W_j \sum_{\{S \in \mathbb{S}: i,j \in S, i \neq j\}} q_S + \sum_{\substack{j,k=1 \\j \neq k \neq i}}^h ( W_i^{\top}  A_k^*)( W_j^{\top}  A_k^*) W_j \sum_{\{S \in \mathbb{S}: i,j,k \in S, i \neq j \neq k\}} q_S \\
&+ \sum_{\substack{j =1 \\ j \neq i}}^h ( W_i^{\top}  A_i^*)(  W_j^{\top}  A_i^*) W_j \sum_{\{S \in \mathbb{S}: i,j \in S, i \neq j\}} q_S \Bigg \}\\
&+ (-2D)m_2 \left \{ \sum_{\substack{j =1 \\ j \neq i}}^h ( W_i^{\top} A_j^*) A_j^* \sum_{\{S \in \mathbb{S}: i,j \in S, i \neq j\}} q_S \right \} \\
&+ m_2 \Bigg \{ \sum_{\substack{k=1 \\ k \neq i}}^h (W_i^\top W_k) (W_k^\top A^*_k) A^*_k \sum_{\{S \in \mathbb{S}: i,k \in S, i \neq k\}} q_S  + \sum_{\substack{j,k =1\\ j \neq i \neq k}}^h (W_i^\top W_j) (W_j^\top A^*_k) A^*_k \sum_{\{S \in \mathbb{S}: i,j,k \in S, i \neq j \neq k\}} q_S \\
&+ \sum_{\substack{k =1 \\ k \neq i}}^h (W_i^\top W_i) (W_i^\top A^*_k) A^*_k \sum_{\{S \in \mathbb{S}: i,k \in S, i \neq k\}} q_S \Bigg \}\\
\end{align*}

Expanding the $q_S$ dependency in terms of $q_{ij}$ and $q_{ijk}$ we have,

\begin{align*}
\tilde{e_i}
&= m_2 \Bigg \{ \sum_{j=1, j \neq i}^h q_{ij} ( W_i^{\top} A_j^*)(  W_j^{\top}  A_j^*) W_j + \sum_{\substack{j,k=1 \\j \neq k \neq i}}^h q_{ijk} ( W_i^{\top}  A_k^*)( W_j^{\top}  A_k^*) W_j \\
&+ \sum_{\substack{j =1 \\ j \neq i}}^h q_{ij} ( W_i^{\top}  A_i^*)(  W_j^{\top}  A_i^*) W_j \Bigg \} + (-2D)m_2 \left \{ \sum_{\substack{j =1 \\ j \neq i}}^h q_{ij} ( W_i^{\top} A_j^*) A_j^*  \right \} \\
&+ m_2 \Bigg \{ \sum_{\substack{k=1 \\ k \neq i}}^h q_{ik} (W_i^\top W_k) (W_k^\top A^*_k) A^*_k  + \sum_{\substack{j,k =1\\ j \neq i \neq k}}^h q_{ijk} (W_i^\top W_j) (W_j^\top A^*_k) A^*_k \\
&+ \sum_{\substack{k =1 \\ k \neq i}}^h q_{ik} (W_i^\top W_i) (W_i^\top A^*_k) A^*_k \Bigg \}
\end{align*}


~\\
Upper bounding the norm of this vector $\tilde{e}_i$ we get, 

\begin{align}
\label{ei_m2}
\nonumber ||\tilde{e_i}|| &\leq m_2 h^{2p-1} \left( \delta + \frac{\mu}{\sqrt{n}} \right) (1+\delta)^2 + m_2 h^{3p-1} \left( \delta + \frac{\mu}{\sqrt{n}} \right)^2 (1+\delta) \\
\nonumber &+ m_2 h^{2p-1} \left( \delta + \frac{\mu}{\sqrt{n}} \right) (1+\delta)^2 + 2D m_2 h^{2p-1} \left( \delta + \frac{\mu}{\sqrt{n}} \right) \\
\nonumber &+ m_2 h^{2p-1} \left( \delta^2 + 2\delta + \frac{\mu}{\sqrt{n}} \right) (1+\delta) + m_2 h^{3p-1} \left( \delta^2 + 2\delta + \frac{\mu}{\sqrt{n}} \right) \left(\delta + \frac{\mu}{\sqrt{n}} \right) \\
\nonumber &+ m_2 h^{2p-1} \left( \delta + \frac{\mu}{\sqrt{n}} \right) (1+\delta)^2\\
\nonumber &\leq m_2 h^{2p-1} (h^{-p-\nu^2} +2h^{-2p-2\nu^2} + h^{-3p-3\nu^2} + 2h^{-p-\nu^2 -\xi} + h^{-2p-2\nu^2-\xi} + h^{-\xi}) \\
\nonumber &+ m_2 h^{3p-1} (h^{-2p-2\nu^2} + h^{-3p-3\nu^2} + 2h^{-p-\nu^2 -\xi} + 2h^{-2p-2\nu^2-\xi} + h^{-2\xi} + h^{-p -\nu^2 -2\xi}) \\
\nonumber &+ m_2 h^{2p-1} (h^{-p-\nu^2} +2h^{-2p-2\nu^2} + h^{-3p-3\nu^2} + 2h^{-p-\nu^2 -\xi} + h^{-2p-2\nu^2-\xi} + h^{-\xi})\\
\nonumber &+ 2D m_2 h^{2p-1} (h^{-p-\nu^2} +h^{-\xi}) \\
\nonumber &+ m_2 h^{2p-1} (2h^{-p-\nu^2} +3h^{-2p-2\nu^2} + h^{-3p-3\nu^2} + h^{-p-\nu^2 -\xi} + h^{-\xi}) \\
\nonumber &+ m_2 h^{3p-1} (2h^{-2p-2\nu^2} + h^{-3p-3\nu^2} + 3h^{-p-\nu^2 -\xi} + h^{-2p-2\nu^2-\xi} + h^{-2\xi}) \\
&+ m_2 h^{2p-1} (h^{-p-\nu^2} +2h^{-2p-2\nu^2} + h^{-3p-3\nu^2} + 2h^{-p-\nu^2 -\xi} + h^{-2p-2\nu^2-\xi} + h^{-\xi})
\end{align}

~\\
If $D=1$ and $0<p<\xi$, we get $||\tilde{e_i}|| = o(m_2h^{p-1})$


\subsection {Estimating the $m_1$ dependent parts of the derivative}

We continue working in the same regime for the $W$ matrix as in the previous subsection. Hence the same inequalities as listed at the beginning of the previous subsection continue to hold and we use them  to get the following bounds, 

\paragraph {Bounding $\hat{\alpha_i}$}

\begin{align*}
\hat{\alpha_i} &= \mathbb{E}_{S \in \mathbb{S}} \Bigg[ \mathbf{1}_{i \in S} \times \Bigg \{ m_1^2 \sum_{\substack{k, l \in S \\ k \neq l}}  ( W_i ^\top A^*_k ) ( W_i ^\top A^*_l) -2m_1 \sum_{k \in S}  \epsilon_i  ( W_i^\top A^*_k) + \epsilon_i^2 \Bigg \} \Bigg] \\
&= \mathbb{E}_{S \in \mathbb{S}} \Bigg[ \mathbf{1}_{i \in S} \times \Bigg \{ m_1^2 \sum_{\substack{k \in S \\ k \neq i}} \langle W_i, A^*_k \rangle \langle W_i, A^*_i\rangle + m_1^2 \sum_{\substack{l \in S \\ l \neq i}} \langle W_i, A^*_i \rangle \langle W_i, A^*_l\rangle + m_1^2 \sum_{\substack{k, l \in S \\ k \neq l \\ k \neq i \\ l \neq i}}  \langle W_i, A^*_k \rangle \langle W_i, A^*_l \rangle \\
&- 2m_1 \epsilon_i \langle W_i, A^*_i \rangle - 2m_1 \sum_{\substack{k \in S \\ k \neq i}} \epsilon_i \langle W_i, A^*_k \rangle + \epsilon_i^2 \Bigg \} \Bigg] \\
&= 2m_1^2 \sum_{\substack{k = 1 \\ k \neq i}}^h \langle W_i, A^*_k \rangle \langle W_i, A^*_i\rangle \sum_{\{S \in \mathbb{S} : i,k \in S, k \neq i \}}q_S + m_1^2 \sum_{\substack{k, l = 1 \\ k \neq l \\ k \neq i \\ l \neq i}}^h  \langle W_i, A^*_k \rangle \langle W_i, A^*_l \rangle \sum_{\{S \in \mathbb{S} : i,k,l \in S, k \neq i \neq l \}}q_S \\
&- 2m_1 \epsilon_i \langle W_i, A^*_i \rangle \sum_{\{S \in \mathbb{S} : i \in S \}}q_S - 2m_1 \sum_{\substack{k = 1 \\ k \neq i}}^h \epsilon_i \langle W_i, A^*_k \rangle \sum_{\{S \in \mathbb{S} : i,k \in S, k \neq i \}}q_S + \epsilon_i^2 \sum_{\{S \in \mathbb{S} : i \in S \}}q_S\\
\implies \hat{\alpha_i} &= 2m_1^2 \sum_{\substack{k = 1 \\ k \neq i}}^h q_{ik} \langle W_i, A^*_k \rangle \langle W_i, A^*_i\rangle + m_1^2 \sum_{\substack{k, l = 1 \\ k \neq l \\ k \neq i \\ l \neq i}}^h q_{ikl} \langle W_i, A^*_k \rangle \langle W_i, A^*_l \rangle \\
&- 2m_1 q_i \epsilon_i \langle W_i, A^*_i \rangle - 2m_1 \sum_{\substack{k = 1 \\ k \neq i}}^h q_{ik} \epsilon_i \langle W_i, A^*_k \rangle + q_i \epsilon_i^2
\end{align*}

~\\
We plugin $\epsilon_i = 2m_1 h^p \left( \delta + \frac{\mu}{\sqrt{n}} \right)$ for $i = 1, \ldots, h$

\begin{align*}
|\hat{\alpha_i}| &\leq 2m_1^2 h^{2p-1} \left( \delta + \frac{\mu}{\sqrt{n}}\right) (1+\delta) + m_1^2 h^{3p-1} \left( \delta + \frac{\mu}{\sqrt{n}} \right)^2 + 4m_1^2 h^{2p-1} (1+\delta) \left( \delta + \frac{\mu}{\sqrt{n}} \right) \\
&+ 4m_1^2 h^{3p-1} \left( \delta + \frac{\mu}{\sqrt{n}} \right)^2 + 4m_1^2 h^{3p-1} \left( \delta + \frac{\mu}{\sqrt{n}} \right)^2\\
&= 2m_1^2 h^{2p-1} (h^{-p-\nu^2} + h^{-2p-2\nu^2} + h^{-p-\nu^2 -\xi} + h^{-\xi}) + m_1^2 h^{3p-1} (h^{-2p-2\nu^2} + 2h^{-p-\nu^2 -\xi} + h^{-2\xi}) \\
&+4m_1^2 h^{2p-1} (h^{-p-\nu^2} + h^{-2p-2\nu^2} + h^{-\xi} + h^{-p-\nu^2 -\xi}) + 4m_1^2 h^{3p-1} (h^{-2p-2\nu^2} + 2h^{-p-\nu^2 -\xi} + h^{-2\xi}) \\
&+ 4m_1^2 h^{3p-1} (h^{-2p-2\nu^2} + 2h^{-p-\nu^2 -\xi} + h^{-2\xi})
\end{align*}
~\\
This means that if $p < \xi$, $|\hat{\alpha_i}| = o( m_1^2 h^{p-1} )$. Putting  this together with the bounds obtained below equation \ref{alpha_m2}, we get that $\alpha_i = \Theta(m_2h^{p-1}) + o( m_1^2 h^{p-1} )$.

\paragraph{Bounding $\hat{\beta_i}$}

\begin{align*}
\hat{\beta_i} &= \mathbb{E}_{S \in \mathbb{S}} \Bigg[ \mathbf{1}_{i \in S} \times \Bigg \{  2 D m_1^2 \sum_{\substack{k \in S \\ k \neq i}} ( W_i^{\top} A_k^*) - D m_1 \epsilon_i + m_1 \sum_{j \in S} \epsilon_j (W_i^\top W_j) - m_1^2 \sum_{\substack{j,l \in S \\ l \neq i}} (W_i^\top W_j) (W_j^{\top} A_l^*) \Bigg \} \Bigg ] \\
&= 2D m_1^2 \sum_{\substack{k = 1 \\ k \neq i}}^h \langle W_i, A_k^* \rangle \sum_{\{S \in \mathbb{S} : i,k \in S, k \neq i \}}q_S -Dm_1 \epsilon_i \sum_{\{S \in \mathbb{S} : i \in S \}}q_S + m_1 \epsilon_i ||W_i||^2 \sum_{\{S \in \mathbb{S} : i \in S \}}q_S \\
&+m_1\sum_{j =1, j\neq i}^h \epsilon_j \langle W_i, W_j \rangle \sum_{\{S \in \mathbb{S} : i,j \in S, j \neq i \}}q_S - m_1^2 \sum_{\substack{l =1 \\ l \neq i}}^h ||W_i||^2 \langle W_i, A_l^* \rangle \sum_{\{S \in \mathbb{S} : i,l \in S, l \neq i \}}q_S \\
&-m_1^2 \sum_{\substack{l =1 \\ l \neq i}}^h \langle W_i, W_l \rangle \langle W_l, A_l^*\rangle \sum_{\{S \in \mathbb{S} : i,l \in S, l \neq i \}}q_S - m_1^2 \sum_{\substack{j,l =1 \\ l \neq i \\ j \neq l,i}}^h \langle W_i, W_j\rangle \langle W_j, A_l^* \rangle \sum_{\{S \in \mathbb{S} : i,j,l \in S, l \neq i \neq i \}}q_S \\
&= 2D m_1^2 \sum_{\substack{k = 1 \\ k \neq i}}^h q_{ik} \langle W_i, A_k^* \rangle -Dm_1 \epsilon_i q_i + m_1 \epsilon_i ||W_i||^2 q_i +m_1\sum_{j =1, j\neq i}^h \epsilon_j q_{ij} \langle W_i, W_j \rangle\\
&- m_1^2 \sum_{\substack{l =1 \\ l \neq i}}^h ||W_i||^2 \langle W_i, A_l^* \rangle q_{il} -m_1^2 \sum_{\substack{l =1 \\ l \neq i}}^h \langle W_i, W_l \rangle \langle W_l, A_l^*\rangle q_{il} - m_1^2 \sum_{\substack{j,l =1 \\ l \neq i \\ j \neq l,i}}^h \langle W_i, W_j\rangle \langle W_j, A_l^* \rangle q_{ijl}
\end{align*}
~\\
We plugin $\epsilon_i = 2m_1 h^p \left( \delta + \frac{\mu}{\sqrt{n}} \right)$ for $i = 1, \ldots, h$

\begin{align*}
|\hat{\beta_i}| &\leq 4Dm_1^2 h^{2p-1} \left(\delta + \frac{\mu}{\sqrt{n}} \right) + 2m_1^2 h^{2p-1} \left( \delta + \frac{\mu}{\sqrt{n}} \right) (1+\delta)^2 +  2m_1^2 h^{3p-1} \left( \delta + \frac{\mu}{\sqrt{n}} \right) \left( \delta^2 + 2\delta + \frac{\mu}{\sqrt{n}}\right)\\
&+ m_1^2 h^{2p-1}(1+\delta)^2\left( \delta + \frac{\mu}{\sqrt{n}} \right) + m_1^2 h^{2p-1} \left( \delta^2 + 2\delta + \frac{\mu}{\sqrt{n}}\right) (1+\delta) \\
&+ m_1^2 h^{3p-1} \left( \delta^2 + 2\delta + \frac{\mu}{\sqrt{n}}\right) \left(\delta + \frac{\mu}{\sqrt{n}}\right)\\
&= 4Dm_1^2 h^{2p-1} (h^{-p-\nu^2} + h^{-\xi} )\\
&+ 2m_1^2 h^{2p-1} (h^{-p -\nu^2} + 2h^{-2p -2\nu^2} + h^{-3p -3\nu^2} + h^{-\xi} + 2h^{-p -\nu^2-\xi} + h^{-2p -2\nu^2 - \xi})\\
&+  2m_1^2 h^{3p-1} (2h^{-2p -2\nu^2} + h^{-3p -3\nu^2} + 3h^{-p -\nu^2-\xi} + h^{-2p -2\nu^2 - \xi} + h^{-2\xi})\\
&+ m_1^2 h^{2p-1} (h^{-p -\nu^2} + 2h^{-2p -2\nu^2} + h^{-3p -3\nu^2} + h^{-\xi} + 2h^{-p -\nu^2-\xi} + h^{-2p -2\nu^2 - \xi})\\
&+ m_1^2 h^{2p-1}  (3h^{-2p -2\nu^2} + h^{-3p -3\nu^2} + h^{-p -\nu^2-\xi} + 2h^{-p -\nu^2} + h^{-\xi}) \\
&+ m_1^2 h^{3p-1}  (2h^{-2p -2\nu^2} + h^{-3p -3\nu^2} + 3h^{-p -\nu^2-\xi} + h^{-2p -2\nu^2 - \xi} + h^{-2\xi})
\end{align*}
~\\
This means that if $p < \xi$, $|\hat{\beta_i}| = o( m_1^2 h^{p-1} )$. Putting  this together with the bounds obtained below \ref{beta_m2}, we get that $\beta_i = \Theta(m_2h^{p-1}) + o( m_1^2 h^{p-1} )$.

\paragraph{Bounding $\vert \vert \hat{e_i} \vert \vert_2$}

\begin{align*}
\hat{e_i} &= \underbrace{\mathbb{E}_{S \in \mathbb{S}}\left [ \mathbf{1}_{i \in S} \times \left \{ \sum_{\substack{j \in S \\ j \neq i}} \epsilon_i \epsilon_j W_j
-m_1 \sum_{\substack{j, k \in S \\ j \neq i}}  ( W_j^\top A^*_k)  W_j \epsilon_i
-m_1 \sum_{\substack{j, k \in S \\ j \neq i}} \epsilon_j (W_i^\top A^*_k) W_j \right \}  \right ]}_{\hat{e_{i1}}} \\
&+ \underbrace{\mathbb{E}_{S \in \mathbb{S}}\left [ \mathbf{1}_{i \in S} \times \left \{
m_1^2 \sum_{\substack{j, k, l \in S \\ j \neq i \\ k \neq l}}  ( W_i ^\top A^*_k ) ( W_j ^\top A^*_l)  W_j\right \} \right ]}_{\hat{e_{i2}}}\\
&+ \underbrace{\mathbb{E}_{S \in \mathbb{S}}\left [ \mathbf{1}_{i \in S} \times \left \{ -2D m_1^2 \sum_{\substack{j, k \in S \\ j \neq i \\ k \neq i}} ( W_i^{\top} A_k^*) A_j^*
+ D m_1 \sum_{\substack{j \in S \\ j \neq i}} \epsilon_i A_j^*\right \} \right ]}_{\hat{e_{i3}}}\\
&+\underbrace{\mathbb{E}_{S \in \mathbb{S}}\left [ \mathbf{1}_{i \in S} \times \left \{ - m_1 \sum_{\substack{j,k \in S \\ k \neq i}} \epsilon_j  (W_i^\top W_j) A^*_k
+ m_1^2 \sum_{\substack{j,k,l \in S \\ k \neq i,l}} (W_i^\top W_j) (W_j^{\top} A_l^*)A^*_k\right \} \right ]}_{\hat{e_{i4}}}
\end{align*}

~\\
We estimate the different summands separately. 

\begin{align*}
\hat{e_{i1}} &= \mathbb{E}_{S \in \mathbb{S}}\left [ \mathbf{1}_{i \in S} \times \left \{ \sum_{\substack{j \in S \\ j \neq i}} \epsilon_i \epsilon_j W_j \right \}  \right ]\\
&+\mathbb{E}_{S \in \mathbb{S}}\left [ \mathbf{1}_{i \in S} \times (-m_1)\left \{ \sum_{\substack{j(=k) \in S \setminus  i}}  ( W_j^\top A^*_j)  W_j \epsilon_i + \sum_{\substack{j \in S \setminus i \\ k \in S \setminus i,j}}  ( W_j^\top A^*_k)  W_j \epsilon_i  + \sum_{\substack{j \in S \setminus i \\ k =i}}  ( W_j^\top A^*_i)  W_j \epsilon_i \right \}  \right ]\\
&+\mathbb{E}_{S \in \mathbb{S}}\left [ \mathbf{1}_{i \in S} \times (-m_1)\left \{ \sum_{\substack{j(=k) \in S \setminus i}} \epsilon_j (W_i^\top A^*_j) W_j + \sum_{\substack{ j \in S \setminus i \\ k \in S \setminus i,j}} \epsilon_j (W_i^\top A^*_k) W_j + \sum_{\substack{j \in S \setminus i \\ k=i }} \epsilon_j (W_i^\top A^*_i) W_j \right \} \right ]
\end{align*}

~\\
We substitute, $\epsilon = 2m_1h^p(h^{-p-\nu^2}+h^{-\xi})$ and for any two vectors $\x$ and $\y$ and any two scalars $a$ and $b$ we use the inequality, $\vert \vert a \x + b \y \vert \vert_2 \leq \vert a\vert_{max}\vert \vert \x \vert \vert _{2,max} + \vert b \vert _{max} \vert \vert \y \vert \vert _{2,max} $to get, 

~\\
\begin{align*}
\vert \vert \hat{e_{i1}} \vert \vert_2 &\leq 4m_1^2 h^{2p} \left( \delta+\frac{\mu}{\sqrt{n}} \right)^2 \sum_{j=1, j \neq i}^h q_{ij} ||W_j|| \\
&+ 2m_1^2 h^p \left( \delta+\frac{\mu}{\sqrt{n}} \right) \Bigg ( \sum_{j=1, j \neq i}^h q_{ij} \langle W_j, A^*_j\rangle W_j + \sum_{j,k=1, j \neq i, k \neq i,j}^h q_{ijk} \langle W_j, A^*_k\rangle W_j \\
&+ \sum_{j=1, j \neq i}^h q_{ij} \langle W_j, A^*_i\rangle W_j \Bigg ) \\
&+ 2m_1^2 h^p \left( \delta+\frac{\mu}{\sqrt{n}} \right) \Bigg ( \sum_{j=1, j \neq i}^h q_{ij} \langle W_i, A^*_j\rangle W_j + \sum_{j,k=1, j \neq i, k \neq i,j}^h q_{ijk} \langle W_i, A^*_k\rangle W_j \\
&+ \sum_{j=1, j \neq i}^h q_{ij} \langle W_i, A^*_i\rangle W_j \Bigg )\\
\implies \vert \vert \hat{e_{i1}} \vert \vert_2 &\leq 4m_1^2 h^{2p} h^{2p-1}(1+\delta)\left( \delta+\frac{\mu}{\sqrt{n}} \right)^2 \\
&+ 2m_1^2 h^p \left( \delta+\frac{\mu}{\sqrt{n}} \right) \left( h^{2p-1}(1+\delta)^2 + h^{3p-1}\left(\delta +\frac{\mu}{\sqrt{n}} \right) (1+\delta)  + h^{2p-1}\left(\delta +\frac{\mu}{\sqrt{n}} \right) (1+\delta) \right) \\
&+ 2m_1^2 h^p \left( \delta+\frac{\mu}{\sqrt{n}} \right) \left( h^{2p-1}\left(\delta +\frac{\mu}{\sqrt{n}} \right) (1+\delta) + h^{3p-1}\left(\delta +\frac{\mu}{\sqrt{n}} \right) (1+\delta)  + h^{2p-1}(1+\delta)^2 \right)\\
\implies \vert \vert \hat{e_{i1}} \vert \vert_2 &\leq 4m_1^2 h^{4p-1}(1+\delta) \left( \delta+\frac{\mu}{\sqrt{n}} \right)^2 \\
&+ 2m_1^2 h^{3p-1} \left( \delta+\frac{\mu}{\sqrt{n}} \right) (1+\delta)^2 + 2m_1^2 h^{4p-1}\left(\delta +\frac{\mu}{\sqrt{n}} \right)^2 (1+\delta)\\
&+ 2m_1^2 h^{3p-1}\left(\delta +\frac{\mu}{\sqrt{n}} \right)^2 (1+\delta) \\
&+ 2m_1^2 h^{3p-1}\left(\delta +\frac{\mu}{\sqrt{n}} \right)^2 (1+\delta) + 2m_1^2 h^{4p-1}\left(\delta +\frac{\mu}{\sqrt{n}} \right)^2 (1+\delta)\\
&+ 2m_1^2 h^{3p-1}\left( \delta+\frac{\mu}{\sqrt{n}} \right)(1+\delta)^2\\
\implies \vert \vert \hat{e_{i1}} \vert \vert_2 &\leq 8m_1^2 h^{4p-1}(1+\delta) \left( \delta+\frac{\mu}{\sqrt{n}} \right)^2 + 4m_1^2 h^{3p-1} \left( \delta+\frac{\mu}{\sqrt{n}} \right) (1+\delta)^2\\
&+ 4m_1^2 h^{3p-1}\left(\delta +\frac{\mu}{\sqrt{n}} \right)^2 (1+\delta)\\
\end{align*}

\begin{align*} 
\implies \vert \vert \hat{e_{i1}} \vert \vert_2 &\leq 8m_1^2 h^{4p-1} (h^{-2p-2\nu^2} + h^{-3p-3\nu^2} + 2h^{-p-\nu^2 -\xi} + 2h^{-2p-2\nu^2 -\xi} + h^{-p-\nu^2 -2\xi} + h^{-2\xi}) \\
&+ 4m_1^2 h^{3p-1} (h^{-p-\nu^2} + h^{-3p-3\nu^2} + 2h^{-2p-2\nu^2} + h^{-\xi} + h^{-2p-2\nu^2 -\xi} + 2h^{-p-\nu^2 -\xi}) \\
&+ 4m_1^2 h^{3p-1}(h^{-2p-2\nu^2} + h^{-3p-3\nu^2} + 2h^{-p-\nu^2 -\xi} + 2h^{-2p-2\nu^2 -\xi} + h^{-p-\nu^2 -2\xi} + h^{-2\xi}) \\
&= 8m_1^2 h^{p-1} (h^{p-2\nu^2} + h^{-3\nu^2} + 2h^{p-\nu^2 + p-\xi} + 2h^{-2\nu^2 + p -\xi} + h^{-\nu^2 + 2p -2\xi} + h^{3p-2\xi}) \\
&+ 4m_1^2 h^{p-1} (h^{p-\nu^2} + h^{-p-3\nu^2} + 2h^{-2\nu^2} + h^{2p-\xi} + h^{-2\nu^2 -\xi} + 2h^{-\nu^2 +p-\xi}) \\
&+ 4m_1^2 h^{p-1}(h^{-2\nu^2} + h^{-p-3\nu^2} + 2h^{-\nu^2 + p -\xi} + 2h^{-2\nu^2 -\xi} + h^{-\nu^2 +p-2\xi} + h^{2p-2\xi})
\end{align*}
~\\
From the above it follows that, $\vert \vert \hat{e_{i1}} \vert \vert_2 = o(m_1^2h^{p-1})$ for $p < \nu^2$ and $2p < \xi$\\
And now we start to estimate $\hat{e_{i2}}$

\begin{align*}
\hat{e_{i2}} &= \mathbb{E}_{S \in \mathbb{S}}\left [ \mathbf{1}_{i \in S} \times m_1^2 \left \{ \sum_{\substack{j, k, l \in S \\ j \neq i \\ k \neq l}}  ( W_i ^\top A^*_k ) ( W_j ^\top A^*_l)  W_j\right \} \right ]\\
&= \mathbb{E}_{S \in \mathbb{S}}\Bigg [ \mathbf{1}_{i \in S} \times m_1^2 \Bigg \{ \sum_{\substack{j \in S \\ j \neq i}}  ( W_i ^\top A^*_j ) ( W_j ^\top A^*_i)  W_j + \sum_{\substack{j,k \in S \\ k \neq j \neq i}}  ( W_i ^\top A^*_k ) ( W_j ^\top A^*_i)  W_j\\
&+ \sum_{\substack{j \in S \\ j \neq i}}  ( W_i ^\top A^*_i ) ( W_j ^\top A^*_j)  W_j \\
&+ \sum_{\substack{j,l \in S \\ l \neq j \neq i}}  ( W_i ^\top A^*_i ) ( W_j ^\top A^*_l)  W_j + \sum_{\substack{j,l \in S \\ l \neq j \neq i}}  ( W_i ^\top A^*_j ) ( W_j ^\top A^*_l)  W_j + \sum_{\substack{j,k \in S \\ k \neq j \neq i}}  ( W_i ^\top A^*_k ) ( W_j ^\top A^*_j)  W_j \\
&+ \sum_{\substack{j,k,l \in S \\ l \neq k \neq j \neq i}}  ( W_i ^\top A^*_k ) ( W_j ^\top A^*_l)  W_j \Bigg \} \Bigg ]\\
\implies \hat{e_{i2}} &= m_1^2 \Bigg \{ \sum_{\substack{j =1 \\ j \neq i}}^h q_{ij} ( W_i ^\top A^*_j ) ( W_j ^\top A^*_i)  W_j + \sum_{\substack{j,k =1 \\ k \neq j \neq i}}^h q_{ijk} ( W_i ^\top A^*_k ) ( W_j ^\top A^*_i)  W_j\\
&+ \underbrace{\sum_{\substack{j =1 \\ j \neq i}}^h q_{ij} ( W_i ^\top A^*_i ) ( W_j ^\top A^*_j)  W_j}_{\mathbf{a}} \\
&+ \sum_{\substack{j,l =1 \\ l \neq j \neq i}}^h q_{ijl} ( W_i ^\top A^*_i ) ( W_j ^\top A^*_l)  W_j + \sum_{\substack{j,l =1 \\ l \neq j \neq i}}^h q_{ijl} ( W_i ^\top A^*_j ) ( W_j ^\top A^*_l)  W_j + \sum_{\substack{j,k =1 \\ k \neq j \neq i}}^h q_{ijk} ( W_i ^\top A^*_k ) ( W_j ^\top A^*_j)  W_j \\
&+ \sum_{\substack{j,k,l \in S \\ l \neq k \neq j \neq i}} q_{ijkl} ( W_i ^\top A^*_k ) ( W_j ^\top A^*_l)  W_j \Bigg \}\\
\implies ||\hat{e_{i2}}|| &\leq m_1^2 \Bigg \{ h^{2p-1} \left(\delta + \frac{\mu}{\sqrt{n}} \right)^2(1+\delta) + h^{3p-1} \left(\delta + \frac{\mu}{\sqrt{n}} \right)^2 (1+\delta) + ||\mathbf{a}|| \\
&+ h^{3p-1} \left(\delta + \frac{\mu}{\sqrt{n}} \right)(1+\delta)^2 + h^{3p-1} \left(\delta + \frac{\mu}{\sqrt{n}} \right)^2 (1+\delta) + h^{3p-1} \left(\delta + \frac{\mu}{\sqrt{n}} \right) (1+\delta)^2 \\
&+ h^{4p-1} \left(\delta + \frac{\mu}{\sqrt{n}} \right)^2 (1+\delta) \Bigg \}\\
\end{align*}

\begin{align*} 
\implies ||\hat{e_{i2}}|| &\leq m_1^2 \Bigg \{ h^{2p-1} (h^{-2p-2\nu^2} + h^{-3p-3\nu^2} + 2h^{-p-\nu^2 -\xi} + 2h^{-2p-2\nu^2 -\xi} + h^{-p-\nu^2 -2\xi} + h^{-2\xi}) \\
&+ h^{3p-1} (h^{-2p-2\nu^2} + h^{-3p-3\nu^2} + 2h^{-p-\nu^2 -\xi} + 2h^{-2p-2\nu^2 -\xi} + h^{-p-\nu^2 -2\xi} + h^{-2\xi}) \\
&+ ||\mathbf{a}||\\
&+ h^{3p-1} (h^{-p-\nu^2} + h^{-3p-3\nu^2} + 2h^{-2p-2\nu^2} + h^{-2p-2\nu^2 -\xi} + 2h^{-p-\nu^2 -\xi} + h^{-\xi})  \\
&+ h^{3p-1} (h^{-2p-2\nu^2} + h^{-3p-3\nu^2} + 2h^{-p-\nu^2 -\xi} + 2h^{-2p-2\nu^2 -\xi} + h^{-p-\nu^2 -2\xi} + h^{-2\xi}) \\
&+ h^{3p-1} (h^{-p-\nu^2} + h^{-3p-3\nu^2} + 2h^{-2p-2\nu^2} + h^{-2p-2\nu^2 -\xi} + 2h^{-p-\nu^2 -\xi} + h^{-\xi}) \\
&+ h^{4p-1} (h^{-2p-2\nu^2} + h^{-3p-3\nu^2} + 2h^{-p-\nu^2 -\xi} + 2h^{-2p-2\nu^2 -\xi} + h^{-p-\nu^2 -2\xi} + h^{-2\xi})  \Bigg \}\\
\implies ||\hat{e_{i2}}|| &\leq m_1^2 \Bigg \{ h^{p-1} (h^{-p-2\nu^2} + h^{-2p-3\nu^2} + 2h^{-\nu^2 -\xi} + 2h^{-p-2\nu^2 -\xi} + h^{-\nu^2 -2\xi} + h^{p-2\xi}) \\
&+ h^{p-1} (h^{-2\nu^2} + h^{-p-3\nu^2} + 2h^{-\nu^2 + p -\xi} + 2h^{-2\nu^2 -\xi} + h^{-\nu^2 + p -2\xi} + h^{2p-2\xi}) \\
&+ ||\mathbf{a}||\\
&+ h^{p-1} (h^{p-\nu^2} + h^{-p-3\nu^2} + 2h^{-2\nu^2} + h^{-2\nu^2 -\xi} + 2h^{-\nu^2 + p -\xi} + h^{2p-\xi})  \\
&+ h^{p-1} (h^{-2\nu^2} + h^{-p-3\nu^2} + 2h^{-\nu^2 + p -\xi} + 2h^{-2\nu^2 -\xi} + h^{-\nu^2 + p -2\xi} + h^{2p-2\xi}) \\
&+ h^{p-1} (h^{p-\nu^2} + h^{-2p-3\nu^2} + 2h^{-2\nu^2} + h^{-2\nu^2 -\xi} + 2h^{-\nu^2 + p -\xi} + h^{2p-\xi}) \\
&+ h^{p-1} (h^{p-2\nu^2} + h^{-3\nu^2} + 2h^{p-\nu^2 + p -\xi} + 2h^{-2\nu^2 + p -\xi} + h^{-\nu^2 + 2p -2\xi} + h^{3p-2\xi})  \Bigg \}
\end{align*}

Now let us find a bound for $||\mathbf{a}||$.
\begin{align*}
\mathbf{a} &= \sum_{\substack{j =1 \\ j \neq i}}^h q_{ij} ( W_i ^\top A^*_i ) ( W_j ^\top A^*_j)  W_j \\
&= \langle W_i, A_i^* \rangle q_{ij} W_{-j}^\top \textrm{diag} (W_{-j} A^*_{-j}) 
\end{align*}
Where $A^*_{-j}$ is the dictionary $A^*$ with the $j$th column set to zero, $W_{-j}$ is the dictionary $W$ with the $j$th row set to zero, and $\textrm{diag} (W_{-j} A^*_{-j})$ is the $h$-dimensional vector containing the diagonal elements of the matrix $W_{-j} A^*_{-j}$. We also make use of the distributional assumption that $q_{ij}$ is the same for all $i,j$ in order to pull $q_{ij}$ out of the sum.
\begin{align*}
||\mathbf{a}||_2 &= h^{2p-2} \langle W_i, A_i^* \rangle \vert \vert W_{-j}^\top \textrm{diag} (W_{-j} A^*_{-j}) \vert \vert_2 \\
&\leq h^{2p-2}(1+\delta) ||W_{-j}^\top||_2 ||\textrm{diag} (W_{-j} A^*_{-j})||_2 \\
&\leq h^{2p-2}(1+\delta)^2 h^{1/2} \sqrt{\lambda_{\textrm{max}} (W^{\top}_{-j} W_{-j})} \\
&\leq h^{2p-2}(1+\delta)^2 h^{1/2} \sqrt{h \left( \delta ^2 + 2\delta + \frac{\mu}{\sqrt{n}} \right) + (1+\delta)^2 } \\
&= h^{p-1} \sqrt{h^{2p-2} \times h \times (1+\delta)^4 \times \left( h \left( \delta ^2 + 2\delta + \frac{\mu}{\sqrt{n}} \right) + (1+\delta)^2 \right)} \\
&= h^{p-1} \sqrt{h^{2p-1} \times (1+h^{-p-\nu^2})^4 \times \left( h (h^{-2p-2\nu^2} + 2h^{-p-\nu^2} + h^{-\xi} ) + (1+h^{-p-\nu^2})^2 \right)} \\
&= h^{p-1} \sqrt{(1+h^{-p-\nu^2})^4 \times ( h^{-2\nu^2} + 2h^{p-\nu^2} + h^{2p-\xi} + h^{2p-1}(1+h^{-p-\nu^2})^2)}
\end{align*}
Here $||W_{-j}^\top||_2$ is the spectral norm of $W_{-j}^\top$, and is the top singular value of the matrix. We use Gershgorin's Circle theorem to bound the top eigenvalue of $W^{\top}_{-j}W_{-j}$ by its maximum row sum.

~\\
If $p < \frac{\xi}{2}$, $p < \frac{1}{2}$, and $p < \nu^2$, then $||\hat{e_{i2}}|| = o(m_1^2 h^{p-1})$\\
And now we start to estimate $\hat{e_{i3}}$ as follows. 

\begin{align*}
\hat{e_{i3}} &= \mathbb{E}_{S \in \mathbb{S}}\left [ \mathbf{1}_{i \in S} \times \left \{ D m_1 \sum_{\substack{j \in S \\ j \neq i}} \epsilon_i A_j^* - 2D m_1^2 \sum_{\substack{j, k \in S \\ j \neq i \\ k \neq i}} ( W_i^{\top} A_k^*) A_j^* \right \} \right ] \\
&= \mathbb{E}_{S \in \mathbb{S}}\left [ \mathbf{1}_{i \in S} \times \left \{ D m_1 \sum_{\substack{j \in S \\ j \neq i}} \epsilon_i A_j^* - 2D m_1^2 \sum_{\substack{j \in S \\ j \neq i}} ( W_i^{\top} A_j^*) A_j^* -2D m_1^2 \sum_{\substack{j, k \in S \\ k \neq j \neq i}} ( W_i^{\top} A_k^*) A_j^* \right \} \right ] \\
&= D m_1 \sum_{\substack{j =1 \\ j \neq i}}^h \epsilon_i A_j^* \sum_{\{S \in \mathbb{S}: i,j \in S, i\neq j \} } q_S - 2D m_1^2 \sum_{\substack{j =1 \\ j \neq i}}^h ( W_i^{\top} A_j^*) A_j^* \sum_{\{S \in \mathbb{S}: i,j \in S, i\neq j \} } q_S \\
&- 2D m_1^2 \sum_{\substack{j,k =1 \\ k \neq j \neq i}}^h ( W_i^{\top} A_k^*) A_j^* \sum_{\{S \in \mathbb{S}: i,j,k \in S, i\neq j \neq k \} } q_S \\
&= D m_1 \sum_{\substack{j =1 \\ j \neq i}}^h q_{ij} \epsilon_i A_j^* - 2D m_1^2 \sum_{\substack{j =1 \\ j \neq i}}^h q_{ij} ( W_i^{\top} A_j^*) A_j^* - 2D m_1^2 \sum_{\substack{j,k =1 \\ k \neq j \neq i}}^h q_{ijk} ( W_i^{\top} A_k^*) A_j^*
\end{align*}

~\\
We plugin $\epsilon_i = 2m_1 h^p \left( \delta + \frac{\mu}{\sqrt{n}} \right)$ for $i = 1, \ldots, h$

\begin{align*}
||\hat{e_{i3}}|| &\leq 2Dm_1^2 h^{3p-1} \left( \delta + \frac{\mu}{\sqrt{n}} \right) + 2Dm_1^2 h^{2p-1} \left( \delta + \frac{\mu}{\sqrt{n}} \right) + 2D m_1^2 h^{3p-1} \left( \delta + \frac{\mu}{\sqrt{n}} \right) \\
&= 4Dm_1^2 h^{3p-1} (h^{-p-\nu^2} + h^{-\xi}) + 2Dm_1^2 h^{2p-1} (h^{-p-\nu^2} + h^{-\xi}) \\
&= 4Dm_1^2 h^{p-1} (h^{p-\nu^2} + h^{2p-\xi}) + 2Dm_1^2 h^{p-1} (h^{-\nu^2} + h^{p-\xi})
\end{align*}

~\\
This means for $D=1$, $p < \nu^2$ and $p < \frac{\xi}{2}$, we have $||\hat{e_{i3}}|| = o(m_1^2 h^{p-1})$\\
And now we start to estimate $\hat{e_{i4}}$ as follows. 

\begin{align*}
\hat{e_{i4}} &= \mathbb{E}_{S \in \mathbb{S}}\left [ \mathbf{1}_{i \in S} \times \left \{ - m_1 \sum_{\substack{j,k \in S \\ k \neq i}} \epsilon_j  (W_i^\top W_j) A^*_k
+ m_1^2 \sum_{\substack{j,k,l \in S \\ k \neq i,l}} (W_i^\top W_j) (W_j^{\top} A_l^*)A^*_k\right \} \right ]\\
&= \mathbb{E}_{S \in \mathbb{S}}\left [ \mathbf{1}_{i \in S} \times (-m_1) \left \{  \sum_{k(=j) \in S \setminus i} \epsilon_k  (W_i^\top W_k) A^*_k + \sum_{\substack{j \in S \setminus i \\ k \in S \setminus i,j}} \epsilon_j  (W_i^\top W_j) A^*_k + \sum_{\substack{k \in S \setminus i \\ j = i}} \epsilon_j  (W_i^\top W_i) A^*_k \right \} \right ]\\
&+ \mathbb{E}_{S \in \mathbb{S}}\left [ \mathbf{1}_{i \in S} \times m_1^2 \left \{  \sum_{\substack{j,k,l \in S \\ k \neq i,l}} (W_i^\top W_j) (W_j^{\top} A_l^*)A^*_k\right \} \right ]\\
&=\mathbb{E}_{S \in \mathbb{S}}\left [ \mathbf{1}_{i \in S} \times (-m_1) \left \{  \sum_{k(=j) \in S \setminus i} \epsilon_k  (W_i^\top W_k) A^*_k + \sum_{\substack{j \in S \setminus i \\ k \in S \setminus i,j}} \epsilon_j  (W_i^\top W_j) A^*_k + \sum_{\substack{k \in S \setminus i \\ j = i}} \epsilon_j  (W_i^\top W_i) A^*_k \right \} \right ]\\
&+ \mathbb{E}_{S \in \mathbb{S}}\Bigg [ \mathbf{1}_{i \in S} \times m_1^2 \Bigg \{  \sum_{\substack{k \in S \\ k \neq i}} (W_i^\top W_i) (W_i^{\top} A_i^*)A^*_k  + \sum_{\substack{k \in S \\ k \neq i}} (W_i^\top W_k) (W_k^{\top} A_i^*)A^*_k + \sum_{\substack{j,k \in S \\ j \neq k \neq i}} (W_i^\top W_j) (W_j^{\top} A_i^*)A^*_k \\
&+ \sum_{\substack{k,l \in S \\ \ \neq k \neq i}} (W_i^\top W_i) (W_i^{\top} A_l^*)A^*_k + \sum_{\substack{k,l \in S \\ l \neq k \neq i}} (W_i^\top W_k) (W_k^{\top} A_l^*)A^*_k + \sum_{\substack{k,l \in S \\ l \neq k \neq i}} (W_i^\top W_l) (W_l^{\top} A_l^*)A^*_k \\
&+ \sum_{\substack{j,k,l \in S \\ j \neq k \neq l \neq i}} (W_i^\top W_j) (W_j^{\top} A_l^*)A^*_k
\Bigg \} \Bigg ]\\
\hat{e_{i4}} &= (-m_1) \left \{  \sum_{k=1, k \neq i}^h q_{ik} \epsilon_k  (W_i^\top W_k) A^*_k + \sum_{\substack{j,k =1 \\ j \neq k \neq i}}^h q_{ijk} \epsilon_j  (W_i^\top W_j) A^*_k + \sum_{\substack{k =1 \\ k \neq i}}^h q_{ik} \epsilon_i  (W_i^\top W_i) A^*_k \right \}\\
&+ m_1^2 \Bigg \{  \underbrace{\sum_{\substack{k =1 \\ k \neq i}}^h q_{ik} (W_i^\top W_i) (W_i^{\top} A_i^*)A^*_k}_{\mathbf{b}} + \sum_{\substack{k =1 \\ k \neq i}}^h q_{ik} (W_i^\top W_k) (W_k^{\top} A_i^*)A^*_k + \sum_{\substack{j,k =1 \\ j \neq k \neq i}}^h q_{ijk} (W_i^\top W_j) (W_j^{\top} A_i^*)A^*_k \\
&+ \sum_{\substack{k,l =1 \\ l \neq k \neq i}}^h q_{ikl} (W_i^\top W_i) (W_i^{\top} A_l^*)A^*_k + \sum_{\substack{k,l =1 \\ l \neq k \neq i}}^h q_{ikl} (W_i^\top W_k) (W_k^{\top} A_l^*)A^*_k + \sum_{\substack{k,l =1 \\ l \neq k \neq i}}^h q_{ikl} (W_i^\top W_l) (W_l^{\top} A_l^*)A^*_k \\
&+ \sum_{\substack{j,k,l =1 \\ j \neq k \neq l \neq i}}^h q_{ijkl} (W_i^\top W_j) (W_j^{\top} A_l^*)A^*_k
\Bigg \}
\end{align*}

~\\
We plugin $\epsilon_i = 2m_1 h^p \left( \delta + \frac{\mu}{\sqrt{n}} \right)$ for $i = 1, \ldots, h$ in the above to get,

\begin{align*}
||\hat{e_{i4}}|| &\leq 2m_1^2 h^{3p-1} \left( \delta + \frac{\mu}{\sqrt{n}} \right)^2 + 2m_1^2 h^{4p-1} \left( \delta + \frac{\mu}{\sqrt{n}} \right) \left( \delta^2 + 2\delta + \frac{\mu}{\sqrt{n}} \right)\\
&+ 2m_1^2 h^{3p-1} \left( \delta + \frac{\mu}{\sqrt{n}} \right) (1+\delta)^2 \\
&+m_1^2||\mathbf{b}|| + m_1^2 h^{2p-1} \left( \delta + \frac{\mu}{\sqrt{n}} \right) \left( \delta^2 + 2\delta + \frac{\mu}{\sqrt{n}} \right) + m_1^2 h^{3p-1} \left( \delta + \frac{\mu}{\sqrt{n}} \right) \left( \delta^2 + 2\delta + \frac{\mu}{\sqrt{n}} \right) \\
&+ m_1^2 h^{3p-1} (1+\delta)^2 \left( \delta + \frac{\mu}{\sqrt{n}} \right) + m_1^2 h^{3p-1} \left( \delta + \frac{\mu}{\sqrt{n}} \right) \left( \delta^2 + 2\delta + \frac{\mu}{\sqrt{n}} \right)\\
&+ m_1^2 h^{3p-1} (1+\delta) \left( \delta^2 + 2\delta + \frac{\mu}{\sqrt{n}} \right)\\
&+ m_1^2 h^{4p-1} \left( \delta + \frac{\mu}{\sqrt{n}} \right) \left( \delta^2 + 2\delta + \frac{\mu}{\sqrt{n}} \right)\\
\implies ||\hat{e_{i4}}|| &\leq 2m_1^2 h^{3p-1} \left( \delta + \frac{\mu}{\sqrt{n}} \right)^2 + 3m_1^2 h^{4p-1} \left( \delta + \frac{\mu}{\sqrt{n}} \right) \left( \delta^2 + 2\delta + \frac{\mu}{\sqrt{n}} \right)\\
&+ 3m_1^2 h^{3p-1} \left( \delta + \frac{\mu}{\sqrt{n}} \right) (1+\delta)^2 \\
&+m_1^2||\mathbf{b}|| + m_1^2 h^{2p-1} \left( \delta + \frac{\mu}{\sqrt{n}} \right) \left( \delta^2 + 2\delta + \frac{\mu}{\sqrt{n}} \right) +2 m_1^2 h^{3p-1} \left( \delta + \frac{\mu}{\sqrt{n}} \right) \left( \delta^2 + 2\delta + \frac{\mu}{\sqrt{n}} \right) \\
&  + m_1^2 h^{3p-1} (1+\delta) \left( \delta^2 + 2\delta + \frac{\mu}{\sqrt{n}} \right)\\
\implies ||\hat{e_{i4}}|| &\leq 2m_1^2 h^{3p-1} ( h^{-2p-2\nu^2} + 2h^{-p-\nu^2 - \xi} + h^{-2\xi}) \\
&+ 3m_1^2 h^{4p-1} (h^{-3p-3\nu^2} + 2h^{-2p-2\nu^2} + 3h^{-p-\nu^2 - \xi} + h^{-2p-2\nu^2 -\xi} + h^{-2\xi}) \\
&+ 3m_1^2 h^{3p-1} (h^{-3p-3\nu^2} + 2h^{-2p-2\nu^2} + 2h^{-p-\nu^2 - \xi} + h^{-2p-2\nu^2 -\xi} + h^{-\xi} + h^{-p-\nu^2}) \\
&+ m_1^2 ||\mathbf{b}|| \\
&+ m_1^2 h^{2p-1} (h^{-3p-3\nu^2} + 2h^{-2p-2\nu^2} + 3h^{-p-\nu^2 - \xi} + h^{-2p-2\nu^2 -\xi} + h^{-2\xi}) \\
&+2 m_1^2 h^{3p-1} (h^{-3p-3\nu^2} + 2h^{-2p-2\nu^2} + 3h^{-p-\nu^2 - \xi} + h^{-2p-2\nu^2 -\xi} + h^{-2\xi}) \\
&  + m_1^2 h^{3p-1} (h^{-3p-3\nu^2} + 3h^{-2p-2\nu^2} + h^{-p-\nu^2 - \xi} + h^{-\xi} + 2h^{-p-\nu^2})\\
\implies ||\hat{e_{i4}}|| &\leq 2m_1^2 h^{p-1} ( h^{-2\nu^2} + 2h^{-\nu^2 + p- \xi} + h^{2p-2\xi}) \\
&+ 3m_1^2 h^{p-1} (h^{-3\nu^2} + 2h^{-p-2\nu^2} + 3h^{p-\nu^2 + p- \xi} + h^{-2\nu^2 + p-\xi} + h^{3p-2\xi}) \\
&+ 3m_1^2 h^{p-1} (h^{-p-3\nu^2} + 2h^{-2\nu^2} + 2h^{-\nu^2 + p - \xi} + h^{-2\nu^2 -\xi} + h^{2p-\xi} + h^{p-\nu^2}) \\
&+ m_1^2 ||\mathbf{b}|| \\
&+ m_1^2 h^{p-1} (h^{-2p-3\nu^2} + 2h^{-p-2\nu^2} + 3h^{-\nu^2 - \xi} + h^{-p-2\nu^2 -\xi} + h^{p-2\xi}) \\
&+2 m_1^2 h^{p-1} (h^{-p-3\nu^2} + 2h^{-2\nu^2} + 3h^{-\nu^2 + p - \xi} + h^{-2\nu^2 -\xi} + h^{2p-2\xi}) \\
&  + m_1^2 h^{p-1} (h^{-p-3\nu^2} + 3h^{-2\nu^2} + h^{-\nu^2 + p - \xi} + h^{2p-\xi} + 2h^{p-\nu^2})
\end{align*}

~\\
Now let us find a bound for $||\mathbf{b}||$.
\begin{align*}
\mathbf{b} &= \sum_{\substack{k =1 \\ k \neq i}}^h q_{ik} (W_i^\top W_i) (W_i^{\top} A_i^*)A^*_k \\
&= \langle W_i, W_i \rangle \langle W_i, A_i^* \rangle q_{ik} A^*_{-i} \mathbf{1}_h 
\end{align*}
Where $A^*_{-i}$ is the dictionary $A^*$ with the $i$th column set to zero, and $\mathbf{1}_h \in \mathbb{R}^h$ is the $h$-dimensional vector of all ones. Here we make use of the distributional assumption that $q_{ik}$ is the same for all $i,k$ in order to pull $q_{ik}$ out of the sum.
\begin{align*}
||\mathbf{b}||_2 &= h^{2p-2} \langle W_i, W_i \rangle \langle W_i, A_i^* \rangle \vert \vert A^*_{-i} \mathbf{1}_h \vert \vert_2 \\
&\leq h^{2p-2}(1+\delta)^3 ||A^*_{-i}||_2 ||\mathbf{1}_h||_2 \\
&= h^{2p-2}(1+\delta)^3 h^{1/2} \sqrt{\lambda_{\textrm{max}} (A^{*\top}_{-i} A^{*}_{-i})} \\
&= h^{2p-2}(1+\delta)^3 h^{1/2} \sqrt{h\frac{\mu}{\sqrt{n}} + 1} \\
&= h^{p-1} \sqrt{h^{2p-2} \times h \times (1+\delta)^6 \times \left( h\frac{\mu}{\sqrt{n}} +1 \right)} \\
&= h^{p-1} \sqrt{h^{2p-1} \times (1+h^{-p-\nu^2})^6 \times \left( h^{1-\xi} +1 \right)} \\
&= h^{p-1} \sqrt{(1+h^{-p-\nu^2})^6 \times ( h^{2p-\xi} + h^{2p-1})}
\end{align*}
Here $||A^*_{-i}||_2$ is the spectral norm of $A^*_{-i}$, and is the top singular value of the matrix. We use Gershgorin's Circle theorem to bound the top eigenvalue of $A^{*\top}_{-i}A^*_{-i}$ by its maximum row sum.
~\\
If $p < \frac{\xi}{2}$, $p < \frac{1}{2}$, and $p < \nu^2$, then $||\hat{e_{i4}}|| = o(m_1^2 h^{p-1})$. Now we combine the above obtained bounds for $\Vert \hat{e_{it}}\Vert$ (for $t \in \{1,2,3,4\}$) with the bound obtained below equation \ref{ei_m2} to say that, $\Vert e_i \Vert = o(\max\{m_1^2,m_2\}h^{p-1})$

\subsection {About $\alpha_i - \beta_i$}

Remembering that $D=1$ and doing a close scrutiny of the terms in \ref{alpha_m2} and \ref{beta_m2} will indicate that the coefficients are the \emph {same} for the $m_2h^{p-1}$ term in each of them. (which is the term with the highest $h$ scaling in the $m_2$ dependent parts of $\alpha_i$ and $\beta_i$). So this largest term cancels off in the difference and we are left with the sub-leading order terms coming from both their $m_1^2$ as well as the $m_2$ parts and this gives us,

\[ \alpha_i - \beta_i = o(\max\{m_1^2,m_2\}h^{p-1}) \]

\end{subappendices}

\chapter{\vspace{10pt} Understanding Adaptive Gradient Algorithms}\label{chapadam} 

\section{Introduction}

Many optimization questions arising in machine learning can be cast as a finite sum optimization problem of the form: $\min_\x f(\x)$ where $f(\x) = \frac{1}{k} \sum_{i=1}^k f_i(\x)$. Most neural network problems also fall under a similar structure where each function $f_i$ is typically non-convex. A well-studied algorithm to solve such problems is Stochastic Gradient Descent (SGD), which uses updates of the form: 
$\x_{t+1} := \x_t - \alpha \nabla  f_{i_t} (\x_t)$, where $\alpha$ is a step size, and $f_{i_t}$ is a function chosen randomly from $\{f_1, f_2, \dots, f_k\}$ at time $t$.

Often in neural networks, ``momentum" is added to the SGD update to yield a two-step update process given as:
$\ve{v}_{t+1} = \mu \ve{v}_t - \alpha \nabla \tilde f_{i_t} (\x_t)$ followed by $\x_{t+1} = \x_t + \ve{v}_{t+1}$. This algorithm is typically called the Heavy-Ball (HB) method (or sometimes classical momentum), with $\mu > 0$ called the momentum parameter \citep{polyak1987introduction}.
In the context of neural nets, another variant of SGD that is popular is Nesterov's Accelerated Gradient (NAG), which can also be thought of as a momentum method \citep{sutskever2013importance}, and has updates of the form $\ve{v}_{t+1} = \mu \ve{v}_t - \alpha \nabla \tilde f_{i_t} (\x_t + \mu \ve{v}_t)$ followed by $\x_{t+1} = \x_t + \ve{v}_{t+1}$  (see Algorithm \ref{NAG_TF} for more details).

Momentum methods like HB and NAG have been shown to have superior convergence properties compared to gradient descent both for convex and non-convex functions \citep{nesterov1983method, polyak1987introduction}, \\ \citep{zavriev1993heavy, ochs2016local, o2017behavior, jin2017accelerated}. To the best of our knowledge, when using a stochastic gradient oracle there is no clear theoretical justification yet known of the benefits of NAG and HB over regular SGD in general \citep{yuan2016influence,kidambi2018on, wiegerinck1994stochastic, yang2016unified, gadat2018stochastic}, 
unless considering specialized function classes \citep{loizou2017momentum}. But in practice, these momentum methods, and in particular NAG, have been repeatedly shown to have good convergence and generalization on a range of neural net problems  \citep{sutskever2013importance,lucas2018aggregated, kidambi2018on}. 

The performance of NAG (as well as HB and SGD), however, are typically quite sensitive to the selection of its hyper-parameters: step size, momentum and batch size \citep{sutskever2013importance}. Thus,
``adaptive gradient" algorithms such as RMSProp (Algorithm \ref{RMSProp_TF}) \citep{tieleman2012lecture} and ADAM (Algorithm \ref{ADAM_TF}) \citep{kingma2014adam} have become very popular for optimizing deep neural networks 
\citep{melis2017state, denkowski2017stronger, gregor2015draw, radford2015unsupervised, bahar2017empirical}. 
The reason for their widespread popularity seems to be the fact that they are easier to tune than SGD, NAG or HB. Adaptive gradient methods use as their update direction a vector which is the image of a linear combination of all the gradients seen till now, under a linear transformation (often called the ``diagonal pre-conditioner") constructed out of the history of the gradients. It is generally believed that this ``pre-conditioning'' makes these algorithms much less sensitive to the selection of its hyper-parameters. A precursor to RMSProp and ADAM was the AdaGrad algorithm, \citep{duchi2011adaptive}.

Despite their widespread use in the deep-learning community, till our work, adaptive gradients methods like RMSProp and ADAM have lacked any theoretical justifications in the non-convex setting - even with exact/deterministic gradients \citep{bernstein2018signsgd}. On the contrary, intriguing recent works like \cite{wilson2017marginal} and \cite{keskar2017improving} have shown cases where SGD (no momentum) and HB (classical momentum) generalize much better than RMSProp and ADAM with stochastic gradients. In particular \cite{wilson2017marginal} showed that ADAM generalizes poorly for large enough nets and that RMSProp generalizes better than ADAM on a couple of neural network tasks (most notably in the character-level language modeling task). But in general it's not clear and no heuristics are known to the best of our knowledge to decide whether these insights about relative performances (generalization or training) between algorithms hold for other models or carry over to the full-batch setting.

Most notably in \cite{reddi2018convergence} the authors showed that in the setting of online convex optimization there are certain sequences of convex functions where ADAM and RMSprop fail to converge to asymptotically zero average regret. 


\subsection{A summary of our contributions}
In this work we shed light on the above described open questions about adaptive gradient methods in the following two ways.

\begin{itemize}
    \item To the best of our knowledge, this work gives the first convergence guarantees for RMSProp and ADAM under any setting. Specifically (a) in Section \ref{sec:theory-statement} we show stochastic gradient oracle conditions for which RMSProp can converge to approximate criticality for smooth non-convex objectives. Most interesting among these is the ``interpolating" oracle condition that we motivate and which we show helps stochastic RMSPRop converge at gradient descent speeds. (b) In Section \ref{sec:deterministic} we show run-time bounds s.t  for certain regimes of hyper-parameters and classes of smooth non-convex functions deterministic RMSProp and ADAM can reach approximate criticality.

    \item Our second contribution (in Section \ref{experiments}) is to undertake a detailed empirical investigation into adaptive gradient methods, targeted to probe the competitive advantages of RMSProp and ADAM. We compare the convergence and generalization properties of RMSProp and ADAM against NAG on (a) a variety of autoencoder experiments on MNIST data, in both full and mini-batch settings and (b) on image classification task on CIFAR-10 using a VGG-9 convolutional neural network in the mini-batch setting. 
    
    In the full-batch setting, we demonstrate that ADAM with very high values of the momentum parameter ($\beta_1 = 0.99$) matches or outperforms carefully tuned NAG and RMSProp, in terms of getting lower training and test losses. We show that as the autoencoder size keeps increasing, RMSProp fails to generalize pretty soon. In the mini-batch experiments we see exactly the same behaviour for large enough nets. 
    
     We also demonstrate the enhancement in ADAM's ability to get lower population risk values and gradient norms when the {\bf $\xi$} parameter is increased. Thus we conclude that this is a crucial hyperparameter that was incidentally not tuned in studies like \cite{wilson2017marginal}
\end{itemize}

\remark 
The counterexample to ADAM's convergence constructed in Theorem $3$ in \cite{reddi2018convergence} is in the stochastic optimization framework and is incomparable to our result about deterministic ADAM. Thus our result establishes a key conceptual point that for adaptive gradient algorithms one cannot transfer intuitions about convergence from online setups to their more common use case in offline setups.

On the experimental side we note that recently it has been shown by \cite{lucas2018aggregated}, that there are problems where NAG generalizes better than ADAM even after tuning $\beta_1$ (see Algorithm \ref{ADAM_TF}). In contrast our experiments reveal controlled setups where tuning ADAM's $\beta_1$ closer to $1$ than usual practice helps close the generalization gap with NAG and HB which exists at standard values of $\beta_1$. 

\newpage 
\subsection{Comparison with concurrent proofs in literature} Much after this work was completed we came to know of  \cite{li2018convergence} and  \cite{ward2019adagrad} which analyzed similar questions as us though none of them address RMSProp or ADAM. The latter of these two shows convergence on smooth non-convex objectives of a form of AdaGrad where adaptivity is limited to only rescaling the currently sampled stochastic gradient. In a similar setup the former reference analyzes convergence rates of a modification of AdaGrad where the currently sampled stochastic gradient does not affect the pre-conditioner. We emphasize that this is a conceptually significant departure from the framework of famously successful adaptive gradient algorithms and experimentally this modification can be shown to hurt the performance.  After the initial version of our work \cite{de2018convergence}  was made public, a flurry of activity happened in this field towards trying to prove better convergence results for ADAM and RMSProp like algorithms, \citep{chen2018convergence,zhou2018convergence,zou2018sufficient,zaheer2018adaptive} and \citep{chen2018closing}. Most recently in \cite{staib2019escaping} a massive modification of RMSProp has been shown to have the ability to converge to approximate second order critical points. 

For the convergence proofs to work the above papers have introduced one or more of the following modifications : {\bf (1)} while attempting to prove convergence of stochastic RMSProp and/or ADAM they have either forced the stochastic oracle to be a bounded random variable or they have introduced time-decay in the adaptivity parameters, $\beta_1$ (that controls the momentum adaptivity) and the $\beta_2$ (that controls the historical contribution of the squared gradients). {\bf (2)} they introduce many extra steps (like most notably in \cite{staib2019escaping} and \cite{chen2018closing}) than there are in the standard software implementations of ADAM or RMSProp which are successful in the real world.  

Unlike all the above results, in our following first-of-its-kind characterizations of different conditions for the convergence of RMSProp and ADAM, we do {\it not} modify the structure of the extremely successful implementations of RMSProp or ADAM (including keeping the adaptivity and the momentum parameters to constants) and in particular for stochastic RMSProp we demonstrate the first-of-its-kind  examples of stochastic oracles for which sub-linear rate of convergence to criticality is possible while also using constant step-sizes.  

\section{Pseudocodes}
\label{sec:pseudocodes}
Towards stating the pesudocodes used for NAG, RMSProp and ADAM in theory and experiments, we need the following definition of square-root of diagonal matrices, 

\begin{definition}{{\bf Square root of the Penrose inverse}}\label{penrose}
If $\v \in \R^d$ and $V = \textrm{diag}(\v)$ then we define, $V^{-\frac {1}{2}} := \sum_{i\in \textrm{Support}(\v)} \frac{1}{\sqrt{\v_i}} \e_i \e_i^T$, where $\{\e_i\}_{\{i=1,\ldots,d\}}$ is the standard basis of $\R^d$
\end{definition}

\begin{algorithm}[H]
\caption{\textbf{Nesterov's Accelerated Gradient (NAG)}}
\label{NAG_TF}
\begin{algorithmic}[1]
{\tt 
\State {\bf Input :} A step size $\alpha$, momentum $\mu \in [0,1)$, and an initial starting point $\x_1 \in \mathbb{R}^d$, and we are given query access to a (possibly noisy) Oracle for gradients of $f : \mathbb{R}^d \rightarrow \mathbb{R}$.
\Function{NAG}{$\x_1, \alpha, \mu$}
    \State {\bf Initialize :} $\v_1 = \0$ 
    \For{$t = 1, 2, \ldots$}
        \State When queried with $\x_t$, the Oracle replies with $\g_t$ s.t $\E[\g_t] = \nabla f(\x_t)$ 
        \State $\v_{t+1} = \mu \v_t + \nabla f(\x_t)$
        \State $\x_{t+1} = \x_t - \alpha (\g_t + \mu \v_{t+1})$
        \EndFor
\EndFunction
}
\end{algorithmic}
\end{algorithm}

\vspace{-1cm}
\begin{algorithm}[H]
\caption{{\bf RMSProp}}
\label{RMSProp_TF}
\begin{algorithmic}[1]
{\tt 
\State {\bf Input :} A constant vector $\mathbb{R}^d \ni \xi\mathbf{1}_d \geq 0$, parameter $\beta_2 \in [0,1)$, step size $\alpha$, initial starting point $\x_1 \in \mathbb{R}^d$, and we are given query access to a (possibly noisy) oracle for gradients of $f : \mathbb{R}^d \rightarrow \mathbb{R}$.
\Function{RMSProp}{$\x_1, \beta_2, \alpha, \xi$}
\State {\bf Initialize :} $\v_0 = \0$  
\For{$t = 1, 2, \ldots$}
   \State When queried with $\x_t$, the Oracle replies with $\g_t$ s.t $\E[\g_t] = \nabla f(\x_t)$
   \State $\v_t = \beta_2 \v_{t-1} + (1-\beta_2)(\g_t^2 + \xi\mathbf{1}_d )$
   \State $V_t = \text{diag}(\v_t)$
   \State $\x_{t+1} = \x_{t} - \alpha V_t^{-\frac 1 2} \g_t$
   \EndFor
\EndFunction
}
\end{algorithmic}
\end{algorithm}

\begin{algorithm}[H]
\caption{{\bf ADAM}}
\label{ADAM_TF}
\begin{algorithmic}[1]
{\tt 
\State {\bf Input :} A constant vector $\mathbb{R}^d \ni \xi\mathbf{1}_d > 0$, parameters $\beta_1, \beta_2 \in [0,1)$, a sequence of step sizes $\{ \alpha_t\}_{t=1,2..}$, initial starting point $\x_1 \in \mathbb{R}^d$, and we are given Oracle access to (possibly noisy) estimates of gradients of $f : \mathbb{R}^d \rightarrow \mathbb{R}$.
\Function{ADAM}{$\x_1, \beta_1, \beta_2, \alpha, \xi$}
\State {\bf Initialize :} $\m_0 = \0$, $\v_0 = \0$  
\For{$t = 1, 2, \ldots$}
   \State When queried with $\x_t$, the Oracle replies with $\g_t$ s.t $\E[\g_t] = \nabla f(\x_t)$
   \State $\m_t = \beta_1 \m_{t-1} + (1-\beta_1)\g_t$
   \State $\v_t = \beta_2 \v_{t-1} + (1-\beta_2)\g_t^2$
   \State $V_t = \text{diag}(\v_t)$
   \State $\x_{t+1} = \x_{t} - \alpha_t \Big( V_t^{\frac 1 2} + \text{diag} (\xi\mathbf{1}_d) \Big)^{-1} \m_t$
\EndFor
\EndFunction
}
\end{algorithmic}
\end{algorithm}

\section{Sufficient conditions for convergence to criticality for stochastic RMSProp}\label{sec:theory-statement}
Previously it has been shown in \cite{rangamani2017critical} that mini-batch RMSProp can off-the-shelf do autoencoding on depth $2$ autoencoders trained on MNIST data while similar results using non-adaptive gradient descent methods requires much tuning of the step-size schedule. Here we give the first results about convergence to criticality for stochastic RMSProp. Towards that we need the following definitions,

\begin{definition}{{\bf $L-$smoothness}}\label{smooth} 
If $f : \R^d \rightarrow \R$ is at least once differentiable then we call it $L-$smooth for some $L >0$ if for all $\x, \y \in \R^d$ the following inequality holds, $\textstyle f(\y) \leq f(\x) + \langle \nabla f(\x), \y - \x \rangle + \frac {L}{2} \norm{\y - \x}^2$
\end{definition}

\begin{definition}[{\bf $(\xi,c,f)-$Constrained Oracle}]\label{orcover} 
~\\
For some $\xi >0$ and $c>0$ and an atleast once differentiable objective function $f: \R^d \rightarrow \R$, a {\it $(\xi,c,f)-$Constrained Oracle} when queried at $\x_t \in \R^d$ replies with the vector $\g_t \in \R^d$ s.t it satisfies the following inequality, 
\[  \E \left [ \left ( \norm{\g_t} + \frac{\sqrt{d\xi}}{2} \right )^2  \right ]  \leq \left ( \sqrt{c} \norm{\nabla f (\x_t)} - \frac{\sqrt{d\xi}}{2} \right )^2 
 \] 
\end{definition}

\remark Seeing the stochastic algorithm as a stochastic process $\{\x_1,\ldots\}$, in the proof we will need the above inequality to hold only for the conditional expectation of $ \left ( \norm{\g_t} + \frac{\sqrt{d\xi}}{2} \right )^2$ w.r.t the the sigma algebra generated by $\{\x_1,\ldots,\x_t\}$. 

\paragraph*{Intuition for the above oracle condition} In a typical use-case of ADAM or RMSProp, $\g_t$ is an unbiased estimate of the gradient of the empirical loss $f$ given as, $f = \frac{1}{k} \sum_{i=1}^k f_i$ where $f_i$ is the $\R^d \rightarrow \R$ loss function evaluated on the $i^{th}-$data point. If one were training say neural nets then the {\it ``$d$" above would be the number of trainable parameters of the net} which is typically in tens of millions. When queried at parameter value $\x_t \in \R^d$, a standard instantiation of the oracle is that it returns, $\g_t = \nabla f_i(\x_t)$ after sampling $f_i$ uniformly at random from $\{ f_j \}_{j=1}^k$. Suppose $\x_c$ is a parameter value s.t it is critical to all the $\{ f_j \}_{j=1}^k$ then $\x_c$ is also a critical point of $f$. If the class of functions is large enough (like those corresponding to deep nets used in practice) that for some parameter values it can interpolate the training data and then for loss functions lowerbounded by $0$, such candidate $\x_c$s are these interpolating parameter values. 


By continuity of the gradient of $f_i$s, the above oracle when queried in a neighbourhood of the interpolating $\x'$s returns a vector of infinitesimal norm and in those neighbourhoods the true gradient is also infinitesimal. Thus if the algorithm is started in such a neighbourhood and if it never escapes such a neighbourhood then we can see that the oracle condition proposed in definition \ref{orcover} gives a way to abstractly capture this phenomenon.

\remark {\it In Section \ref{fast_rmsprop_new} we shall define a slightly different (and somewhat less intuitive) oracle condition and show how it can be explicitly instantiated and indicate that it also leads to the same theorem as given below.} 


Now we can demonstrate the power of this definition by proving the following theorem which leverages this condition gives the first proof of convergence of stochastic RMSProp. 

\begin{theorem}{\bf Fast Stochastic RMSProp with the  $(\xi,c,f)-$Constrained Oracle (Proof in Section \ref{proof1})}\label{fastRMSProp} 
Suppose $f : \R^d \rightarrow \R$ is $L-$smooth and $\exists$ $\sigma >0$ s.t $\norm{\nabla_i f(\x)} \leq \sigma$ for all $\x \in \R^d$ and $i \in \{1,\ldots,d\}$. Now suppose that we run the RMSProp algorithm as defined in Algorithm \ref{RMSProp_TF} (with query access to conditionally unbiased estimator of the gradient of $f$) and the oracle additionally satisfying the $(\xi,c,f)-$constraint condition given in Definition \ref{orcover} s.t $\sigma < \Big (\frac{\xi}{2c} \Big )^{\frac 2 3}$  and $\beta_2$ is chosen so that, $\frac{c\sigma^{1.5}}{\xi} < \sqrt{\beta_2(1-\beta_2)}$. \footnote{Since the constants $c, \sigma$ and $\xi$ are constrained s.t $\frac{c\sigma^{1.5}}{\xi} < \frac 1 2$, it follows that a choice of $\beta_2$ as required always exists.}  Then there exists a choice of constant step-size $\alpha$ for the algorithm such that for $T = O(\frac{1}{\epsilon^2})$ steps we have, 
\[ \E [ \min_{i=1,\ldots,T} \norm{\nabla f (\x_i)}^2  ] \leq O(\epsilon^2)\] 
\qed 
\end{theorem}



\remark {\bf (a)} Note that here we see a stochastic algorithm being able to converge at the same fast speed as is characteristic of SGD on differentiable convex functions with a global minimum. This result can be contrasted with corollary $3$ in \cite{zaheer2018adaptive} where similar speeds were motivated for RMSProp with mini-batch sizes being unrealistically large i.e as big as the number of steps to required to converge. In our above theorem such a convergence is seen to arise as a more general phenomenon because of a certain control being true on the expected value of the norm of the gradient oracle's reply. {\bf (b)} Long after this work was completed, we became aware of works like \cite{vaswani2018fast} where oracle conditions were introduced of the similar kind as above to show enhanced convergence speeds of much simpler algorithms like SGD.


Now we demonstrate yet another situation for which stochastic RMSProp can be shown to converge and this time we directly put constraints on the training data to get the convergence instead of using oracle conditions as above. Towards this we need the following definition, 

\begin{definition}[{\bf The sign function}]
We define the function $\text{sign} : \R^d \rightarrow \{-1,1\}^d$ s.t it maps $\v \mapsto (1 \text{ if } \v_i \geq 0 \text{ else } -1)_{i=1,\ldots,d}$.
\end{definition}

\begin{theorem}[{\bf Standard speed stochastic RMSProp with a sign constrained oracle (Proof in Appendix \ref{supp_rmspropS})}]\label{thm:RMSPropS-proof}

Let $f : \R ^d \rightarrow \R$ be $L-$smooth and be  of the form $f = \frac {1}{k} \sum_{p=1}^k f_p$ s.t.~(a) each $f_i$ is at least once differentiable, $(b)$ the gradients are s.t $\forall \x \in \R^d, \forall p,q \in \{1,\ldots,k\}$, $\text{sign}(\nabla f_p(\x)) = \text{sign}(\nabla f_q(\x))$ , (c) $\sigma_f < \infty$ is an upperbound on the norm of the gradients of $f_i$ and (d) $f$ has a minimizer, i.e., there exists $\x_*$ such that $f(\x_*) = \min_{\x \in \R^d} f(\x)$. 

Let the gradient oracle be s.t when invoked at some $\x_t \in \R^d$ it uniformly at random picks $i_t \sim \{1,2,..,k\}$ and returns, $\nabla f_{i_t}(\x_t) = \g_t$.  Then corresponding to any $\epsilon, \xi >0$ and a starting point $\x_1$ for Algorithm \ref{RMSProp_TF}, we can define, $T \leq \frac{1}{\epsilon^4} \left ( \frac{2L\sigma_f^2 (\sigma_f^2 + \xi)\left (f(\x_{1}) -f(\x_*) \right ) }{(1-\beta_2)\xi} \right ) $ s.t.~we are guaranteed that the iterates of Algorithm \ref{RMSProp_TF} using a constant step-length of, $\alpha = \frac{1}{\sqrt{T}} \sqrt{\frac{2 \xi (1-\beta_2)\left (f(\x_{1}) -f(\x_*)\right )}{\sigma_f ^2 L}}$ will find an $\epsilon-$critical point in at most $T$ steps in the sense that, $\min_{t =1,2\ldots,T}\mathbb{E} [ \norm{\nabla f (\x_t)}^2] \leq \epsilon^2$.
\qed
\end{theorem}

\remark We note that the theorem above continues to hold even if the constraint $(b)$ that we have about the signs of the gradients of the $\{f_p\}_{p=1,\ldots,k}$ holds {\it only} on the points in $\R^d$ that the stochastic RMSProp visits. Further we can say in otherwords that this constraint ensures that all the options for the gradient that this stochastic oracle has at any point, lie in the same orthant of $\R^d$ though this orthant itself may change from one iterate of the next. Note that the assumption also ensures that if for some coordinate $i$, $\nabla_i f = 0$ then for all $p \in \{1,\ldots,k\}$, $\nabla_i f_p =0$. 

\section{Later improvements to the proof of sub-linear convergence of stochastic RMSProp}\label{fast_rmsprop_new}

After this thesis was defended and submitted, in collaboration with Jiayao Zhang (at UPenn) we figured out that Theorem \ref{fastRMSProp} also holds for an oracle which is not only less constrained than the one given in Definition \ref{orcover} but also for which corresponding distributions (even certain heavy-tailed ones) for the stochastic gradient can be easily instantiated. This improved result, which we now record here as Theorem \ref{fastRMSProp_imp}, follows from essentially the same proof as given for Theorem \ref{fastRMSProp}. 

\begin{definition}[$(\alpha,\beta,\xi,X)$-distributed random variable.]\label{def:dist}
Given $X \in \R, \alpha \in \R^+, \beta \in \R^+, \xi \in \R^+$, we say that a real valued random variable $g$ is $(\alpha,\beta,\xi, X)$-distributed if it satisfies the following three conditions simultaneously,
\begin{itemize}
    \item $\E \left [ g \right ] = X$,
    \item $\E  [ \abs{g} \sqrt{g^2 + \xi}  ] \leq \alpha \abs{X}$,
    \item $\E \left [  g^2 \right ] \leq \beta X^2$.
\end{itemize}

\end{definition}

Now we shall give a way to construct families of distributions which satisfy the above.

\begin{lemma}\label{thm:condition}
   Define $\sigma > 0$ such that $\abs{X} \le \sigma$
   . Then for any $\beta > 1$, any random variable $g$ s.t
   \[
        \E[g] = X, \quad
        {\rm Var}[{g}] = (\beta-1) \min\{X^2, 1\},
    \]
    is $(\alpha,\beta,\xi, X)$-distributed for 
    \[
        \alpha \ge \sqrt{\beta\xi + \beta^2\sigma^2}, \quad
        \xi \ge 0.
    \]
\end{lemma}

It can be shown that {\it arbitrary} mixtures of distributions of the above kind also are of the type given in Definition \ref{def:dist} 

\begin{example}
For example, we may fix some parameter $\beta>1$ and take
\[
    g \sim {\cal N}\left(X, \sqrt{\beta-1} \min\{\abs{X}, 1\}\right),
\]
or
\[
    g \sim \mathrm{Laplace}(X, \sqrt{(\beta-1)/2} \min\{\abs{X},1\}).
\]
Then this oracle satisfies the condition given in Definition~\ref{def:dist}
with $\alpha \ge \sqrt{\beta(\beta X^2 + \xi)}, \quad
    \xi \ge 0.$
\end{example}

Note in particular, that we may take $g$ to be distributed as a  mixture of the Gaussians or the  Laplacians as specified in the examples above. Now we use the above definition to state the following theorem about sub-linear convergence of stochastic RMSProp. 

\begin{theorem}{\bf Fast Stochastic RMSProp.}\label{fastRMSProp_imp} 
Suppose $f : \R^d \rightarrow \R$ is $L-$smooth and $\exists$ $\sigma >0$ such that
$\abs{\nabla_i f(\x)} \leq \sigma$ for all $\x \in \R^d$ and $i \in \{1,\ldots,d\}$.
Now suppose that we execute the RMSProp algorithm as defined in Algorithm~\ref{RMSProp_TF}
with query access to a gradient oracle of $f$
which for every coordinate $i$ satisfies the condition given in Definition~\ref{def:dist} with
$X = \nabla_i f$ and
$\xi$ large enough. 
Then there exists a choice of constant step size $s$ for the algorithm such that with $T = O(1/\epsilon^2)$ steps we have,
\[ \E \left[ \min_{t=1,\ldots,T} \norm{\nabla f (\x_t)}^2  \right] = O(\epsilon^2).\]
\end{theorem} 





\section{Sufficient conditions for convergence to criticality for non-convex deterministic adaptive gradient algorithms}\label{sec:deterministic}

We note that there are important motivations to study the behavior of neural net training algorithms in the deterministic setting because of use cases where the amount of noise is controlled during optimization, either by using larger batches \citep{martens2015optimizing, de2017automated, babanezhad2015stop} or by employing variance-reducing techniques \citep{johnson2013accelerating, defazio2014saga}. Inspired by these we also investigate the full-batch RMSProp and ADAM in our controlled autoencoder experiments in Section \ref{main:fullbatch}. Towards that we will now demonstrate that such oracle conditions as in the previous section are not necessary to guarantee convergence of the deterministic RMSProp.

\begin{theorem}[\bf Convergence of deterministic RMSProp - the version with standard speeds (Proof in Appendix \ref{sec:supp_rmsprop1})]\label{thm:RMSProp1-proof}
Let $f : \R ^d \rightarrow \R$ be $L-$smooth and let $\sigma < \infty$ be an upperbound on the norm of the gradient of $f$. Assume also that $f$ has a minimizer, i.e., there exists $\x_*$ such that $f(\x_*) = \min_{\x \in \R^d} f(\x)$. Then the following holds for Algorithm~\ref{RMSProp_TF} when $\g_t = \nabla f(\x_t) ~\forall t$,

For any $\epsilon, \xi >0$, using a constant step length of $\alpha_t = \alpha = \frac{(1-\beta_2)\xi}{L\sqrt{\sigma^2 + \xi}}$ for $t=1,2,...$, guarantees that $\norm{\nabla f(\x_{t})} \leq \epsilon$ for some $t \leq   \frac{1}{\epsilon^2}\times \frac{2L(\sigma^2 + \xi)(f(\x_1) -  f(\x_*))}{(1-\beta_2)\xi} $, where $\x_1$ is the first iterate of the algorithm.\qed  
\end{theorem}

One might wonder if the $\xi$ parameter introduced  in all the algorithms above is necessary to get convergence guarantees for RMSProp. Towards that in the following theorem we show convergence of another variant of deterministic RMSProp which does not use the $\xi$ parameter and instead uses other assumptions on the objective function and step size modulation. But these tweaks to eliminate the need of $\xi$ come at the cost of the convergence rates getting weaker.

\begin{theorem}[\bf Convergence of deterministic RMSProp - the version with no $\xi$ shift (Proof in Appendix \ref{sec:supp_rmsprop2})]\label{thm:RMSProp2-proof}
 Let $f : \R ^d \rightarrow \R$ be $L-$smooth and let $\sigma < \infty$ be an upperbound on the norm of the gradient of $f$. Assume also that $f$ has a minimizer, i.e., there exists $\x_*$ such that $f(\x_*) = \min_{\x \in \R^d} f(\x)$, and the function $f$ be bounded from above and below by constants $B_\ell$ and $B_u$ as $B_l \leq f(\x) \leq B_u$ for all $\x \in \R^d$. Then for $\xi =0$ and any $\epsilon >0$, $\exists ~T = {\cal O} (\frac{1}{\epsilon^4})$ s.t.~the Algorithm \ref{RMSProp_TF} when $\g_t = \nabla f(\x_t) ~\forall t$ is guaranteed to reach a $t$-th iterate s.t.~$1 \leq t \leq T$ and $\norm{ \nabla f(\x_{t}) } \leq \epsilon$.\qed 
\end{theorem}


Next we analyze deterministic ADAM albeit in the small $\beta_1$ regime. We note that a small $\beta_1$ does not cut-off contributions to the update direction from gradients in the arbitrarily far past (which are typically significantly large), and neither does it affect the non-triviality of the pre-conditioner which does not depend on $\beta_1$ at all. 

\begin{theorem}{{\bf Deterministic ADAM converges to criticality (Proof in subsection \ref{sec:supp_adam})}}\label{thm:ADAM-proof}
Let $f : \R ^d \rightarrow \R$ be $L-$smooth and let $\sigma < \infty$ be an upperbound on the norm of the gradient of $f$. Assume also that $f$ has a minimizer, i.e., there exists $\x_*$ such that $f(\x_*) = \min_{\x \in \R^d} f(\x)$. Then the following holds for Algorithm~\ref{ADAM_TF} when $\g_t = \nabla f(\x_t) ~\forall t$,

\begin{itemize} 
\item For any $\epsilon >0$, $\beta_1 < \frac{\epsilon}{\epsilon +  \sigma}$ and $\xi >   \frac{\sigma^2 \beta_1}{- \beta_1\sigma + \epsilon(1-\beta_1)}$, there exist step sizes $\alpha_t > 0$, $t = 1, 2, \ldots$ and a natural number $T$ (depending on $\beta_1, \xi$) such that $\norm{\nabla f(\x_t)} \leq \epsilon$ for some $t \leq T$.

\item In particular if one sets $\beta_1 = \frac{\epsilon}{\epsilon + 2\sigma}$, $\xi = 2\sigma$, and $\alpha_t = \frac{\norm{\g_t} ^2}{L(1-\beta_1^t)^2}\frac{4\epsilon}{3(\epsilon + 2\sigma)^2}$, then $T$ can be taken to be $\frac{9L \sigma ^2}{\epsilon^6   }[f(\x_2) - f(\x_*)]$, where $\x_2$ is the second iterate of the algorithm. 
\end{itemize} 
\qed 
\end{theorem}

In other words in $T$ iterates the lowest norm of the gradient encountered by deterministic/``full-batch" ADAM for smooth non-convex objectives falls at least as fast as ${\cal O} \Big ( \frac{1}{T^{\frac 1 6}} \Big )$

Our motivations towards the above theorem were primarily rooted in trying to understand the situations where ADAM as an offline optimizer can converge at all (given the negative results about ADAM in the online setting as in \cite{reddi2018convergence}). But we point out that it remains open to tighten the analysis of deterministic ADAM and obtain faster rates than what we have shown in the theorem above and also to be able to characterize conditions when stochastic ADAM can converge.

\remark It is often believed that ADAM gains over RMSProp because of its so-called ``bias correction term" which refers to the step length of ADAM having an iteration dependence of the following form, $\sqrt{1 - \beta_2^t} / (1 - \beta_1^t)$. As a key success of the above theorem, we note that the $1/(1 - \beta_1^t)$ term of this ``bias correction term" naturally comes out from theory!

\section{The Experimental setup}
\label{sec:exp}
For testing the empirical performance of ADAM and RMSProp, we perform experiments on fully connected autoencoders using ReLU activations and shared weights and on CIFAR-10 using VGG-9, a convolutional neural network. The experiment on VGG-9 has been described in subsection \ref{vgg9_sec}. 
~\\ \\
To the best of our knowledge there have been very few comparisons of ADAM and RMSProp with other methods on a regression setting and that is one of the main gaps in the literature that we aim to fix by our study here.  In a way this also builds on our previous work \citep{rangamani2017critical} (Chapter \ref{chapaut}) where we had undertaken a theoretical analysis of autoencoders and in their experiments and had found RMSProp to have good reconstruction error for MNIST when used on even just $2$ layer ReLU autoencoders. 
~\\ \\
To keep our experiments as controlled as possible, we make all layers in a network have the same width (which we denote as $h$). Thus, we fix the dimensions of the weight matrices of the depth $2\ell-1$, $\R^d \rightarrow \R^d$ autoencoders (as defined in Chapter \ref{chapsum}) as : $W_1 \in \mathbb{R}^{h \times d}$, $W_i \in \mathbb{R}^{h \times h}, i = 2, \dots, \ell$.
This allowed us to study the effect of increasing depth $\ell$ or width $h$ without having to deal with added confounding factors. For all experiments, we use the standard ``Glorot initialization" for the weights \citep{glorot2010understanding}, where each element in the weight matrix is initialized by sampling from a uniform distribution with $[-\text{limit}, \text{limit}]$, $\text{limit} = \sqrt{6/(\text{fan}_{\text{in}} + \text{fan}_{\text{out}})}$, where $\text{fan}_{\text{in}}$ denotes the number of input units in the weight matrix, and $\text{fan}_{\text{out}}$ denotes the number of output units in the weight matrix. All bias vectors were initialized to zero. No regularization was used. 
~\\ \\
We performed autoencoder experiments on the MNIST dataset for various network sizes (i.e., different values of $\ell$ and $h$). We implemented all experiments using TensorFlow \citep{abadi2016tensorflow} using an NVIDIA GeForce GTX 1080 Ti graphics card. We compared the performance of ADAM and RMSProp with Nesterov's Accelerated Gradient (NAG). All experiments were run for $10^5$ iterations. We tune over the hyper-parameters for each optimization algorithm using a grid search as described in Appendix \ref{supp_sec:exp_details}. 
~\\ \\
To pick the best set of hyper-parameters, we choose the ones corresponding to the lowest loss on the training set at the end of $10^5$ iterations. Further, to cut down on the computation time so that we can test a number of different neural net architectures, we crop the MNIST image from $28\times 28$ down to a $22\times 22$ image by removing 3 pixels from each side (almost all of which is whitespace).

\paragraph{Full-batch experiments}
 We are interested in first comparing these algorithms in the full-batch setting. To do this in a computationally feasible way, we consider a subset of the MNIST dataset (we call this: mini-MNIST), which we build by extracting the first 5500 images in the training set and first 1000 images in the test set in MNIST. Thus, the training and testing datasets in mini-MNIST is 10\% of the size of the  MNIST dataset. Thus the training set in mini-MNIST contains 5500 images, while the test set contains 1000 images. This subset of the dataset is a fairly reasonable approximation of the full MNIST dataset (i.e., contains roughly the same distribution of labels as in the full MNIST dataset), and thus a legitimate dataset to optimize on.

\paragraph{Mini-batch experiments}
To test if our conclusions on the full-batch case extend to the mini-batch case, we then perform the same experiments in a mini-batch setup where we fix the mini-batch size at 100. For the mini-batch experiment, we consider the full training set of MNIST, instead of the mini-MNIST dataset considered for the full-batch experiments and we also test on CIFAR-10 using VGG-9, a convolutional neural network.

\section {Experimental Results}\label{experiments}

\subsection{RMSProp and ADAM are sensitive to choice of $\xi$}
\label{sec:main_xi}

The $\xi$ parameter is a feature of the default implementations of RMSProp and ADAM such  as in TensorFlow. Most interestingly this strictly positive parameter is crucial for our proofs. In this section we present experimental evidence that attempts to clarify that this isn't merely a theoretical artefact but its value indeed has visible effect on the behaviours of these algorithms. 
We see in Figure \ref{xi_best} that on increasing the value of this fixed shift parameter $\xi$, ADAM in particular, is strongly helped towards getting lower gradient norms and lower test losses though it can hurt its ability to get lower training losses. The plots are shown for optimally tuned values for the other hyper-parameters. 

\begin{figure*}[h]
\centering
\begin{subfigure}[t]{0.32\textwidth}
\includegraphics[width=\textwidth]{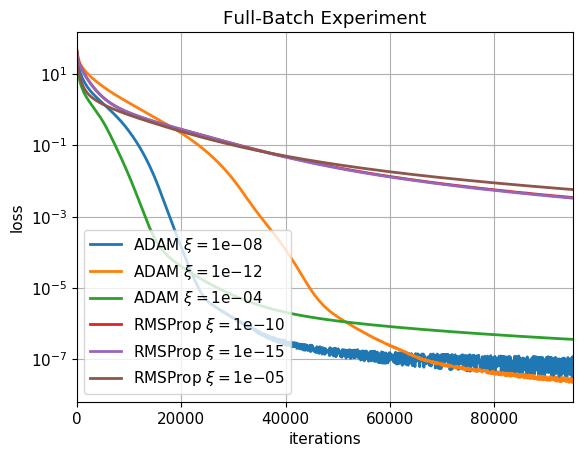}
\end{subfigure}
\begin{subfigure}[t]{0.32\textwidth}
\includegraphics[width=\textwidth]{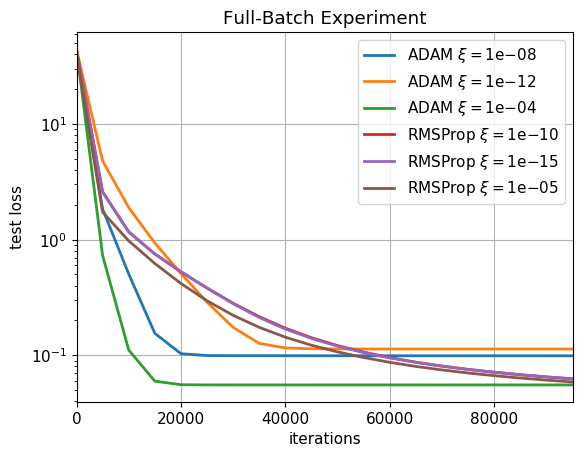}
\end{subfigure}
\begin{subfigure}[t]{0.32\textwidth}
\includegraphics[width=\textwidth]{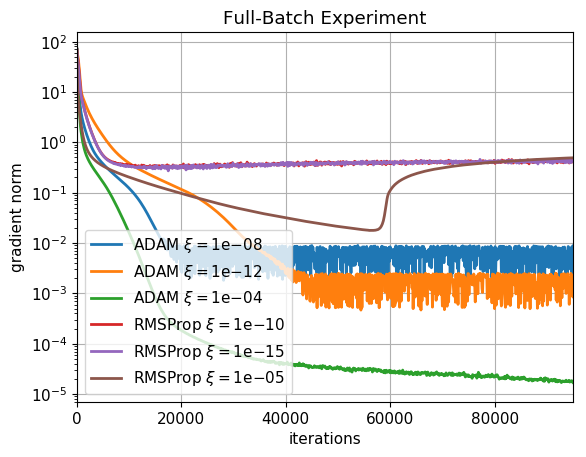}
\end{subfigure}
\caption{Optimally tuned parameters for different $\xi$ values. 1 hidden layer network of 1000 nodes; \emph{Left}: Loss on training set; \emph{Middle}: Loss on test set; \emph{Right}: Gradient norm on training set}
\label{xi_best}
\end{figure*}

\subsection{Tracking $\lambda_{min}(\text{Hessian})$ of the loss function}

To check whether NAG, ADAM or RMSProp is capable of consistently moving from a ``bad" saddle point to a ``good" saddle point region, we track the most negative eigenvalue of the Hessian $\lambda_{\min}(\text{Hessian})$.
Even for a very small neural network with around $10^5$ parameters, it is still intractable to store the full Hessian matrix in memory to compute the eigenvalues. Instead, we use the Scipy library function \texttt{scipy.sparse.linalg.eigsh} that can use a function that computes the matrix-vector products to compute the eigenvalues of the matrix \citep{lehoucq1998arpack}. Thus, for finding the eigenvalues of the Hessian, it is sufficient to be able to do Hessian-vector products. 
This can be done exactly in a fairly efficient way \citep{hvp}.

We display a representative plot in Figure \ref{min_eig} which shows that NAG in particular has a distinct ability to gradually, but consistently, keep increasing the minimum eigenvalue of the Hessian while continuing to decrease the gradient norm.  However unlike as in deeper autoencoders in this case the gradient norms are consistently bigger for NAG, compared to RMSProp and ADAM. In contrast, RSMProp and ADAM quickly get to a high value of the minimum eigenvalue and a small gradient norm, but somewhat stagnate there. In short, the trend looks better for NAG, but in actual numbers RMSProp and ADAM do better.

\begin{figure*}[h!]
\centering
\begin{subfigure}[t]{0.32\textwidth}
\includegraphics[width=\textwidth]{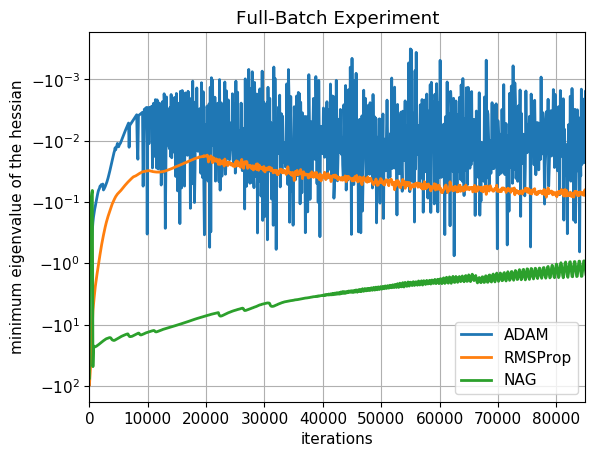}
\end{subfigure}
\begin{subfigure}[t]{0.32\textwidth}
\includegraphics[width=\textwidth]{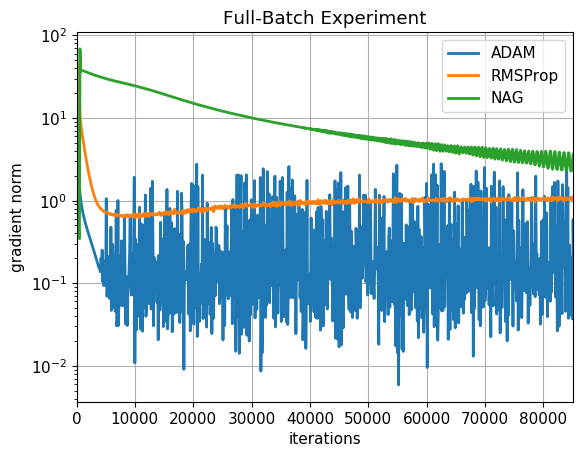}
\end{subfigure}
\caption{Tracking the smallest eigenvalue of the Hessian on a 1 hidden layer network of size 300. \emph{Left}: Minimum Hessian eigenvalue. \emph{Right}: Gradient norm on training set.}
\label{min_eig}
\end{figure*}

\subsection{Comparing performance in the full-batch setting}\label{main:fullbatch}

In Figure \ref{full_batch_3_1000}, we show how the training loss, test loss and gradient norms vary through the iterations for RMSProp, ADAM 
(at $\beta_1 = 0.9$ and $0.99$)
and NAG (at $\mu = 0.9$ and $0.99$)
on a $3$ hidden layer autoencoder with $1000$ nodes in each hidden layer trained on mini-MNIST. Appendix \ref{sec:supp_fullbatch} and \ref{sec:supp_varydim} have more such comparisons for various neural net architectures with varying depth and width and input image sizes, where the following qualitative results also extend.


\begin{figure*}[tb]
\centering
\begin{subfigure}[t]{0.32\textwidth}
\includegraphics[width=\textwidth]{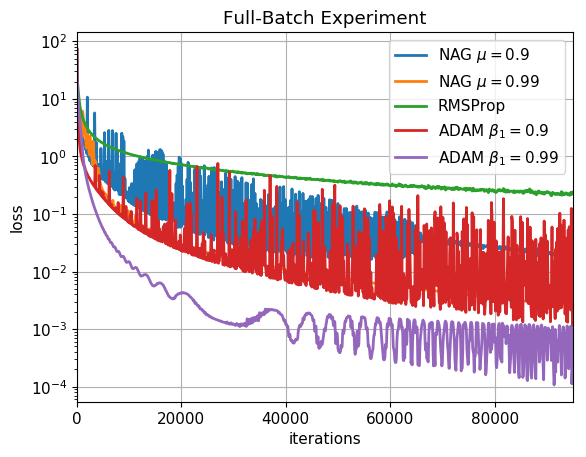}
\end{subfigure}
\begin{subfigure}[t]{0.32\textwidth}
\includegraphics[width=\textwidth]{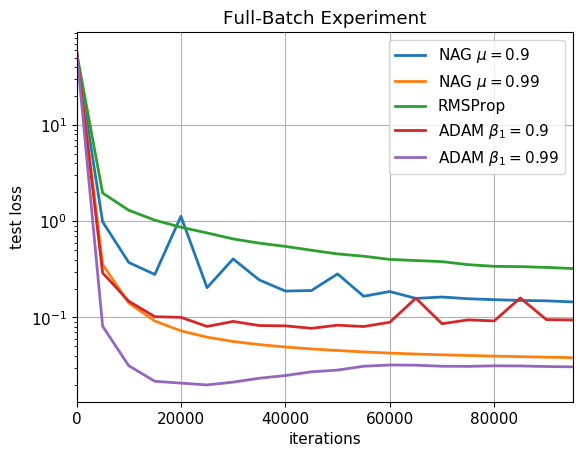}
\end{subfigure}
\begin{subfigure}[t]{0.32\textwidth}
\includegraphics[width=\textwidth]{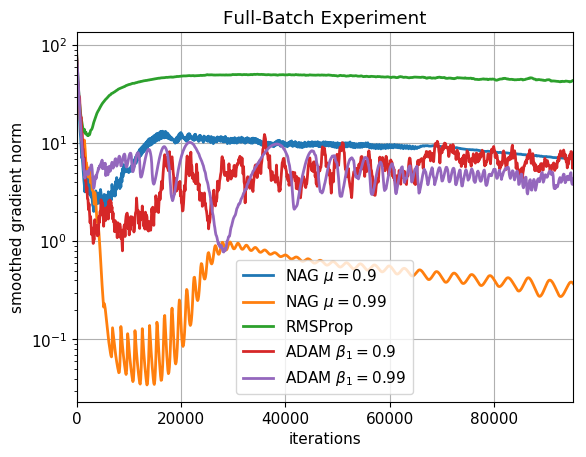}
\end{subfigure}
\caption{Full-batch experiments on a 3 hidden layer network with 1000 nodes in each layer; \emph{Left}: Loss on training set; \emph{Middle}: Loss on test set; \emph{Right}: Gradient norm on training set}
\label{full_batch_3_1000}
\end{figure*}

\paragraph{Conclusions from the full-batch experiments of training autoencoders on mini-MNIST}
\begin{itemize}
\item Pushing $\beta_1$ closer to $1$ significantly helps ADAM in getting lower training and test losses and at these values of $\beta_1$, it has better performance on these metrics than all the other algorithms.
One sees cases like the one displayed in Figure~\ref{full_batch_3_1000} where ADAM at $\beta_1 =0.9$ was getting comparable or slightly worse test and training errors than NAG. But once $\beta_1$ gets closer to $1$, ADAM's performance sharply improves and gets better than other algorithms.

\item Increasing momentum helps NAG get lower gradient norms though on larger nets it might hurt its training or test performance. NAG does seem to get the lowest gradient norms compared to the other algorithms, except for single hidden layer networks like in Figure~\ref{min_eig}.
\end{itemize}

\begin{figure*}[tb!]
\centering
\begin{subfigure}[t]{0.32\textwidth}
\includegraphics[width=\textwidth]{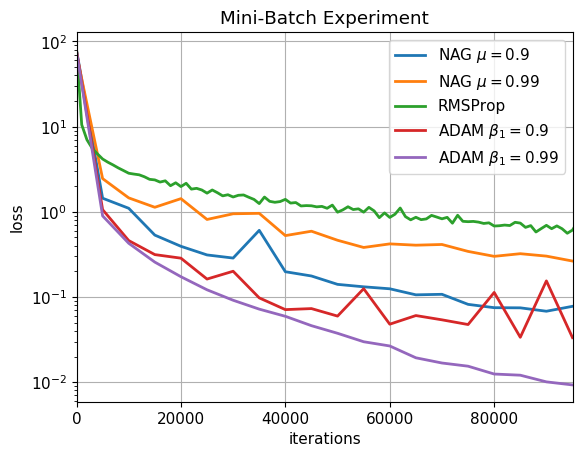}
\end{subfigure}
\begin{subfigure}[t]{0.32\textwidth}
\includegraphics[width=\textwidth]{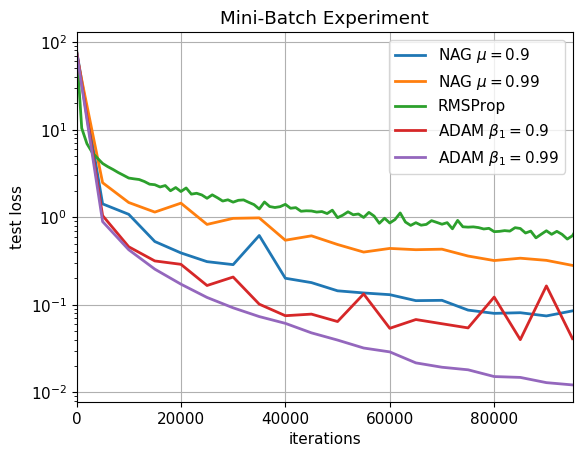}
\end{subfigure}
\begin{subfigure}[t]{0.32\textwidth}
\includegraphics[width=\textwidth]{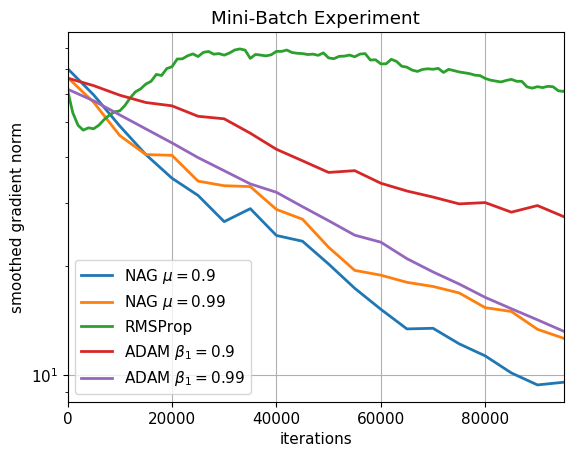}
\end{subfigure}
\caption{Mini-batch experiments on a network with 5 hidden layers of 1000 nodes each; \emph{Left}: Loss on training set; \emph{Middle}: Loss on test set; \emph{Right}: Gradient norm on training set}
\label{mini_batch_5_1000}
\end{figure*}

\subsection{Corroborating the full-batch behaviors in the mini-batch setting}

In Figure \ref{mini_batch_5_1000}, we show how training loss, test loss and gradient norms vary when using mini-batches of size 100, on a $5$ hidden layer autoencoder with $1000$ nodes in each hidden layer trained on the full MNIST dataset. The same phenomenon as here has been demonstrated in more such mini-batch comparisons on autoencoder architectures with varying depths and widths in Appendix \ref{sec:supp_minibatch} and on VGG-9 with CIFAR-10 in the next subsection \ref{vgg9_sec}. 


\paragraph{Conclusions from the mini-batch experiments of training autoencoders on the full MNIST dataset:}
\begin{itemize}
\item Mini-batching does seem to help NAG do better than ADAM on small nets. However, for larger nets, the full-batch behavior continues, i.e., when ADAM's momentum parameter $\beta_1$ is pushed closer to $1$, it gets better generalization (significantly lower test losses) than NAG at any momentum tested.
\item In general, for all metrics (test loss, training loss and gradient norm reduction) both ADAM as well as NAG seem to improve in performance when their momentum parameter ($\mu$ for NAG and $\beta_1$ for ADAM) is pushed closer to $1$. This effect, which was present in the full-batch setting, seems to get more pronounced here.
\item As in the full-batch experiments, NAG continues to have the best ability to reduce gradient norms while for larger enough nets, ADAM at large momentum continues to have the best training error.
\end{itemize}

\subsection{Image Classification on Convolutional Neural Nets}
\label{vgg9_sec}

\begin{figure}[h!]
\centering
\begin{subfigure}[t]{0.4\textwidth}
\includegraphics[width=\textwidth]{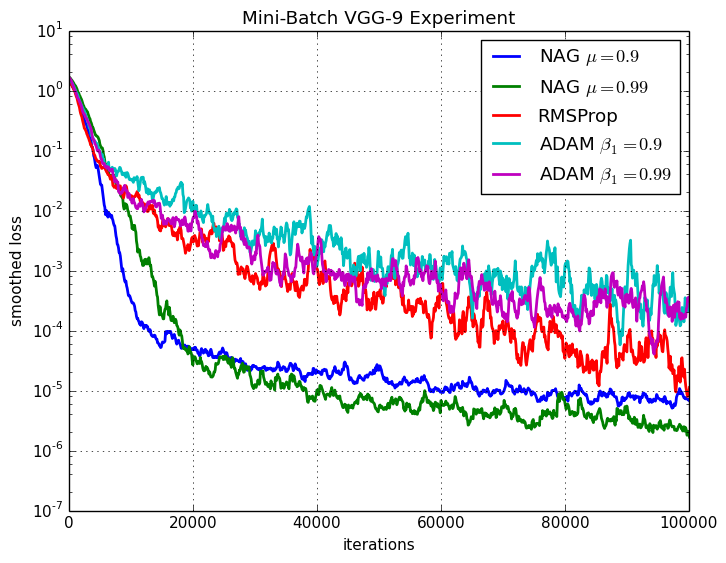}
\caption{Training loss}
\end{subfigure}
\begin{subfigure}[t]{0.4\textwidth}
\includegraphics[width=\textwidth]{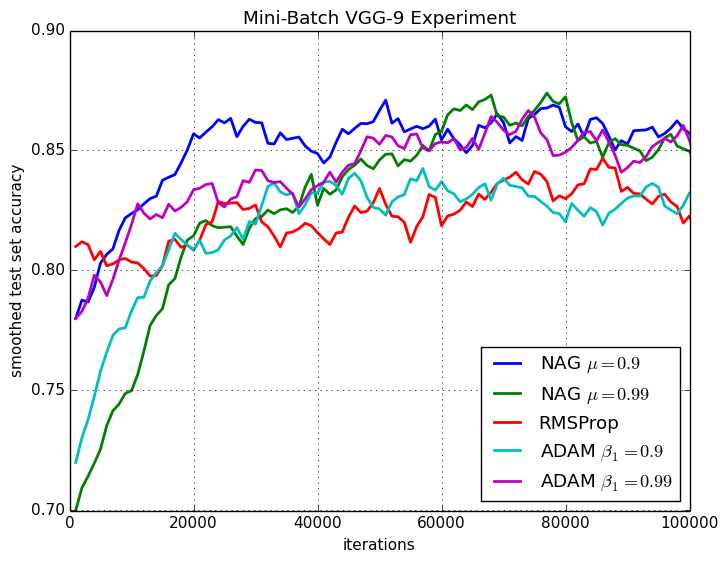}
\caption{Test set accuracy}
\end{subfigure}
\caption{Mini-batch image classification experiments with CIFAR-10 using VGG-9}
\label{vgg9}
\end{figure}

To test whether these results might qualitatively hold for other datasets and models, we train an image classifier on CIFAR-10 (containing 10 classes) using VGG-like convolutional neural networks \\  \citep{simonyan2014very}. In particular, we train VGG-9 on CIFAR-10, which contains 7 convolutional layers and 2 fully connected layers, a total of 9 layers. The convolutional layers contain 64, 64, 128, 128, 256, 256, 256 filters each of size $3 \times 3$, respectively. We use batch normalization \citep{ioffe2015batch} and ReLU activations after each convolutional layer, and the first fully connected layer. Table \ref{tab:vgg9} contains more details of the VGG-9 architecture. We use minibatches of size 100, and weight decay of $10^{-5}$. We use fixed step sizes, and all hyperparameters were tuned as indicated in Section \ref{supp_sec:exp_details}.
~\\ \\
We present results in Figure \ref{vgg9}. As before, we see that this task is another example where tuning the momentum parameter ($\beta_1$) of ADAM helps. While attaining approximately the same loss value, ADAM with $\beta_1 = 0.99$ generalizes as good as NAG and better than when $\beta_1 = 0.9$. Thus tuning $\beta_1$ of ADAM helped in closing the generalization gap with NAG.

\begin{table}[h]
\centering
\small
\caption{VGG-9 on CIFAR-10.}
\label{tab:vgg9}
\begin{tabular}{l|ccc}
\hline
layer type  & kernel size & input size              & output size \\ \hline
Conv\_1     & $3 \times 3$    & $~~3~~ \times 32  \times 32$  & $~64~  \times 32 \times 32$       \\
Conv\_2     & $3 \times 3$    & $~64~ \times 32  \times 32$  & $~64~  \times 32 \times 32$       \\
Max Pooling & $2 \times 2$    & $~64~ \times 32  \times 32$  & $~64~  \times 16 \times 16$       \\\hline
Conv\_3     & $3 \times 3$    & $~64~ \times 16  \times 16$  & $128 \times 16 \times 16$      \\
Conv\_4     & $3 \times 3$    & $128 \times 16  \times 16$  & $128 \times 16 \times 16$      \\
Max Pooling & $2 \times 2$    & $128 \times 16  \times 16$  & $128 \times ~8~ \times ~8~$      \\\hline
Conv\_5     & $3 \times 3$    & $128 \times ~8~   \times ~8~$   & $256 \times ~8~ \times ~8~$      \\
Conv\_6     & $3 \times 3$    & $256 \times ~8~   \times ~8~$   & $256 \times ~8~ \times ~8~$      \\
Conv\_7     & $3 \times 3$    & $256 \times ~8~   \times ~8~$   & $256 \times ~8~ \times ~8~$      \\
Max Pooling & $2 \times 2$    & $256 \times ~8~   \times ~8~$   & $256 \times ~4~ \times ~4~$ \\\hline
Linear      & $1 \times 1$    & $~1 \times 4096~~$           & $~1 \times 256~~~$      \\
Linear      & $1 \times 1$    & $~1 \times ~256~~~$            & $~1 \times ~10~~~~$       \\
\hline
\end{tabular}
\end{table}

\newpage 
\section{Proofs of convergence of (stochastic) RMSProp and ADAM}


\subsection{Fast convergence of stochastic RMSProp with ``Over Parameterization" (Proof of Theorem \ref{fastRMSProp})}
\label{proof1}   

\begin{proof} 
By $L-$smoothness of the objective we have the following relationship between the values at consecutive updates, 

\begin{align}
f(\x_{t+1}) &\leq f(\x_t) + \langle \nabla f(\x_t), \x_{t+1} - \x_t \rangle + \frac {L}{2} \norm{\x_{t+1} - \x_t}^2\\ 
&\leq f(\x_t) + \sum_{i=1}^d \nabla_i f(\x_t) (\x_{t+1} - \x_t)_i  + \frac {L}{2} \sum_{i=1}^d (\x_{t+1} - \x_t)_i^2
\end{align}

In the last step above we substitute the update rule for the $i^{th}-$coordinate of $\x_t$ as, $\x_{t+1,i} = \x_{t,i} - \frac{\alpha_t \g_{t,i}}{\sqrt{\ve{v}_{t,i}}}$ to get, 

\begin{align*}
f(\x_{t+1}) &\leq f(\x_t) - \alpha_t \sum_{i=1}^d \nabla_i f(\x_t) \frac{\g_{t,i}}{\sqrt{\ve{v}_{t,i}}}   + \frac {L \alpha_t^2}{2} \sum_{i=1}^d \frac {\g_{t,i}^2 }{(\sqrt{\ve{v}_{t,i}})^2 }\\
&\leq f(\x_t) - \alpha_t \sum_{i=1}^d \nabla_i f(\x_t) \left ( \frac{\g_{t,i}}{\sqrt{\ve{v}_{t,i}}} - \frac{\g_{t,i}}{\sqrt{\beta_2 \ve{v}_{t-1,i}}}   + \frac{\g_{t,i}}{\sqrt{\beta_2 \ve{v}_{t-1,i}}} \right)   
\\&+ \frac {L \alpha_t^2}{2} \sum_{i=1}^d 
\frac{\g_{t,i}^2}{\ve{v}_{t,i}}
\end{align*}

Now recall that $\E [ \g_t \mid \{\x_i\}_{i=1,\ldots,t} ] =  \nabla f (\x_t)$. We substitute this in the above to get, 


\begin{align}\label{ELsmooth} 
&\nonumber \E [ f(\x_{t+1}) \mid  \{\x_i\}_{i=1,\ldots,t} ] \\&\leq f(\x_t) - \alpha_t \sum_{i=1}^d \nabla_i f(\x_t) \Big ( \frac{\nabla_i f (\x_t) }{\sqrt{\beta_2 \ve{v}_{t-1,i}}}  
+   \E \left [  \frac{\g_{t,i}}{\sqrt{\ve{v}_{t,i}} } - \frac{\g_{t,i}}{\sqrt{\beta_2  \ve{v}_{t-1,i}}} \mid \{\x_i\}_{i=1,\ldots,t} \right ]  \Big ) \\
\nonumber &+ \frac {L \alpha_t^2}{2} \sum_{i=1}^d \E \left [ \frac {\g_{t,i}^2 }{\ve{v}_{t,i}  } \mid \{\x_i\}_{i=1,\ldots,t} \right ]\\
\nonumber &\leq f(\x_t) - \alpha_t \sum_{i=1}^d  \frac{ (\nabla_i f (\x_t))^2 }{\sqrt{\beta_2 \ve{v}_{t-1,i}}}\\
\nonumber &+   \alpha_t \sum_{i=1}^d \vert \nabla_i f(\x_t)\vert  \left  \vert \E \left[  \frac{\g_{t,i}}{\sqrt{\ve{v}_{t,i}}} - \frac{\g_{t,i}}{\sqrt{\beta_2 \ve{v}_{t-1,i}}} \mid \{\x_i\}_{i=1,\ldots,t} \right] \right \vert    \\ 
&+ \frac {L \alpha_t^2}{2} \sum_{i=1}^d \E \left [ \frac {\g_{t,i}^2 }{\ve{v}_{t,i}} \mid \{\x_i\}_{i=1,\ldots,t} \right]
\end{align}

Now observe that, 

\begin{align*}\label{T1} 
\frac{\g_{t,i}}{\sqrt{\ve{v}_{t,i}}} - \frac{\g_{t,i}}{\sqrt{\beta_2 \ve{v}_{t-1,i}}} &\leq \vert \g_{t,i} \vert \left \vert \frac{1}{\sqrt{\ve{v}_{t,i}}} - \frac{1}{\sqrt{\beta_2 \ve{v}_{t-1,i}}} \right \vert\\
&\leq \left \vert \frac{\g_{t,i}}{ \sqrt{\ve{v}_{t,i}} \sqrt{\beta_2 \ve{v}_{t-1,i}}} \right \vert \vert  \sqrt{\beta_2 \ve{v}_{t-1,i}}  - \sqrt{\ve{v}_{t,i}} \vert\\
&\leq \left \vert \frac{\g_{t,i}}{ \sqrt{\ve{v}_{t,i}} \sqrt{\beta_2 \ve{v}_{t-1,i}}} \right \vert \left \vert \frac{\beta_2 \ve{v}_{t-1,i}  - \ve{v}_{t,i}}{\sqrt{\beta_2 \ve{v}_{t-1,i}}  + \sqrt{\ve{v}_{t,i}}} \right \vert
\end{align*}

From the algorithm we have,  $\ve{v}_{t,i} = \beta_2 \ve{v}_{t-1,i} + (1-\beta_2) (\g_{t,i}^2 + \xi)$. We substitute this into the numerator and the denominator of the second factor of the RHS above to get, 

\begin{align*}
    &\frac{\g_{t,i}}{\sqrt{\ve{v}_{t,i}}} - \frac{\g_{t,i}}{\sqrt{\beta_2 \ve{v}_{t-1,i}}} \\&\leq  \left \vert \frac{\g_{t,i}}{ \sqrt{\ve{v}_{t,i}} \sqrt{\beta_2 \ve{v}_{t-1,i}}} \right \vert \left \vert \frac{ (1-\beta_2) (\g_{t,i}^2 + \xi)}{\sqrt{\beta_2 \ve{v}_{t-1,i}}  + \sqrt{\beta_2 \ve{v}_{t-1,i} + (1-\beta_2) (\g_{t,i}^2 + \xi)}}  \right \vert\\
    &\leq \left \vert \frac{\g_{t,i}}{ \sqrt{\ve{v}_{t,i}} \sqrt{\beta_2 \ve{v}_{t-1,i}}} \right \vert \left \vert \frac{ (1-\beta_2) (\g_{t,i}^2 + \xi)}{\sqrt{  (1-\beta_2) (\g_{t,i}^2 + \xi)}}  \right \vert = \sqrt{1-\beta_2} \left \vert \frac{\g_{t,i}\sqrt{\g_{t,i}^2 + \xi}    }{ \sqrt{\ve{v}_{t,i}} \sqrt{\beta_2 \ve{v}_{t-1,i}}} \right \vert
\end{align*} 

Now we substitute the above into equation \ref{ELsmooth} to get, 
\begin{align}
&\nonumber \E [ f(\x_{t+1}) \mid  \{\x_i\}_{i=1,\ldots,t} ] \\&\leq f(\x_t) - \alpha_t \sum_{i=1}^d  \frac{ (\nabla_i f (\x_t))^2 }{\sqrt{\beta_2 \ve{v}_{t-1,i}}} \\
\nonumber &+ \alpha_t \sqrt{1-\beta_2} \sum_{i=1}^d \vert \nabla_i f(\x_t)\vert   \E \left [  \left \vert \frac{\g_{t,i}\sqrt{\g_{t,i}^2 + \xi}    }{ \sqrt{\ve{v}_{t,i}} \sqrt{\beta_2 \ve{v}_{t-1,i}}} \right \vert  \mid \{\x_i\}_{i=1,\ldots,t} \right ]    \\ 
&+ \frac {L \alpha_t^2}{2} \sum_{i=1}^d \E \left [ \frac {\g_{t,i}^2 }{\ve{v}_{t,i}} \mid \{\x_i\}_{i=1,\ldots,t} \right ]
\end{align}

Now by definition we have, $\ve{v}_{t,i} \geq \beta_2 \ve{v}_{t-1,i}$ and the definition of $\sigma$ we infer from the above, 

\begin{align}
\nonumber \E [ f(\x_{t+1}) \mid  \{\x_i\}_{i=1,\ldots,t} ] &\leq f(\x_t) - \alpha_t \sum_{i=1}^d  \frac{ (\nabla_i f (\x_t))^2 }{\sqrt{\beta_2 \ve{v}_{t-1,i}}}  \\&+   \sigma \alpha_t \sqrt{1-\beta_2} \sum_{i=1}^d   \E \left [  \left \vert \frac{\g_{t,i}\sqrt{\g_{t,i}^2 + \xi}    }{ \beta_2 \ve{v}_{t-1,i}} \right \vert  \mid \{\x_i\}_{i=1,\ldots,t} \right ]    \\ 
&+ \frac {L \alpha_t^2}{2} \sum_{i=1}^d \E \left [ \frac {\g_{t,i}^2 }{\beta_2 \ve{v}_{t-1,i}} \mid \{\x_i\}_{i=1,\ldots,t} \right ]
\end{align}

We have, ${\ve v}_t = (1-\beta_2)\sum_{k=1}^t \beta_2^{t-k}(\g_k^2 + \xi)$ This implies,
${\ve v}_{t,i} \geq (1-\beta_2^t)\xi \geq (1-\beta_2)\xi$. The last inequality follows because we have, $\beta_2 \in (0,1)$ and $t \geq 1$
Substituting this in the above (along with the fact that $\ve{v}_{t,i} >0$ ) we get, 

\begin{align}\label{assum}
&\nonumber  \E [ f(\x_{t+1}) \mid  \{\x_j\}_{j=2,\ldots,t} ] \\
\nonumber &\leq f(\x_t) + \sum_{i=1}^d \Big(  - \alpha_t  \frac{ (\nabla_i f (\x_t))^2 }{\sqrt{\beta_2 \ve{v}_{t-1,i}}} +   \sigma \alpha_t \sqrt{1-\beta_2} \frac{ \E \left [ \sqrt{\g_{t,i}^4 + \xi\vert \g_{t,i}\vert^2 }    \mid \{\x_j\}_{j=2,\ldots,t} \right ]}{ \beta_2 \ve{v}_{t-1,i}} \\ 
\nonumber &+ \frac {L \alpha_t^2}{2}   \frac { \E \left [ \g_{t,i}^2 \mid \{\x_j\}_{j=2,\ldots,t} \right ]}{\beta_2 \ve{v}_{t-1,i}}  \Big )\\
\nonumber &\leq f(\x_t)   - \alpha_t  \frac{ \norm{\nabla f (\x_t)}^2 }{\sqrt{\beta_2 \sigma}}\\ 
\nonumber &+   \sum_{i=1}^d \Big ( \sigma \alpha_t \sqrt{1-\beta_2} \frac{ \E \left [ \sqrt{\g_{t,i}^4 + \xi\vert \g_{t,i}\vert^2 }    \mid \{\x_j\}_{j=2,\ldots,t} \right ]}{ \xi \beta_2(1-\beta_2)} + \frac {L \alpha_t^2}{2}   \frac { \E \left [ \g_{t,i}^2 \mid \{\x_j\}_{j=2,\ldots,t} \right ]}{\xi \beta_2(1-\beta_2)}  \Big )\\
\nonumber &\leq f(\x_t)   - \alpha_t  \frac{ \norm{\nabla f (\x_t)}^2 }{\sqrt{\beta_2 \sigma}}\\ 
&+   \sum_{i=1}^d \Big ( \sigma \alpha_t \sqrt{1-\beta_2}  + \frac {L \alpha_t^2}{2}    \Big )\frac{ \E \left [ \sqrt{\g_{t,i}^4 + \xi\vert \g_{t,i}\vert^2 }    \mid \{\x_j\}_{j=2,\ldots,t} \right ]}{ \xi \beta_2(1-\beta_2)}
\end{align}

Now we note that, 

\begin{align*}
&    \E \left [ \sum_{i=1}^d \sqrt{\g_{t,i}^4 + \xi\vert \g_{t,i}\vert^2 }    \mid \{\x_j\}_{j=2,\ldots,t} \right ] \\&\leq \E \left [ \sum_{i=1}^d ( \g_{t,i}^2 + \sqrt{\xi} \vert  \g_{t,i}\vert ) \mid \{\x_j\}_{j=2,\ldots,t} \right ] \leq \E \left [ \norm{\g_{t}}^2 + \sqrt{d\xi} \norm{\g_{t,i}} ) \mid \{\x_j\}_{j=2,\ldots,t} \right ]\\
    &\leq \E \left [ \Big ( \norm{\g_{t}} + \frac{\sqrt{d\xi}}{2} \Big )^2 - \frac{d\xi }{4} \mid \{\x_j\}_{j=2,\ldots,t} \right ]
\end{align*}

Now we invoke the the property of the oracle given in definition \ref{orcover} to say that, 

\begin{align*}
    &\E \left [ \sum_{i=1}^d \sqrt{\g_{t,i}^4 + \xi\vert \g_{t,i}\vert^2 }    \mid \{\x_j\}_{j=2,\ldots,t} \right ] \\&\leq \E \left [ \Big ( \sqrt{c} \norm{\nabla f(\x_t)} - \frac{\sqrt{d\xi}}{2} \Big )^2 - \frac{d\xi }{4} \mid \{\x_j\}_{j=2,\ldots,t} \right ]\\
    &\leq  c \norm{\nabla f(\x_t)}^2 
\end{align*}     

Thus substituting the above back into equation \ref{assum} we get, 

\begin{align}
\nonumber \E [ f(\x_{t+1}) \mid  \{\x_j\}_{j=2,\ldots,t} ]  &\leq f(\x_t)   
\\&+   \sum_{i=1}^d \Big \{ -\frac{\alpha_t}{\sqrt{\beta_2 \sigma}} + c\frac{\Big ( \sigma \alpha_t \sqrt{1-\beta_2}  + \frac {L \alpha_t^2}{2}    \Big )}{\xi \beta_2(1-\beta_2)}   \Big \}  \norm{\nabla_i f (\x_t)}^2
\end{align}

Further we make the optimal choice of $\alpha_t = \frac{\xi \beta_2 (1-\beta_2)}{cL} \Big ( \frac{1}{\sqrt{\beta_2 \sigma}} - \frac{c\sigma }{\xi \beta_2 \sqrt{1-\beta_2}} \Big )$ (which is positive by assumptions) and we get, 

\[ \E [ f(\x_{t+1}) \mid  \{\x_j\}_{j=2,\ldots,t} ]  \leq f(\x_t)   
- \frac{1}{2cL}   \Big ( \frac{c\sigma}{\sqrt{\beta_2 \xi}} - \sqrt{\frac{\xi (1-\beta_2)}{\sigma}} \Big )^2 \norm{\nabla f (\x_t)}^2 \] 

Taking expectation and rearranging we get, 


\[ \E [ \min_{t=1,\ldots,T} \norm{\nabla f(\x_t)}^2 ] \leq \frac{1}{T} \sum_{t=1}^ T \E [ \norm{\nabla f (\x_t)}^2 ] \leq \frac{ f(\x_1)  - f_*}{ \frac{T}{2cL}   \Big ( \frac{c\sigma}{\sqrt{\beta_2 \xi}} - \sqrt{\frac{\xi (1-\beta_2)}{\sigma}} \Big )^2}  \] 


From here the result follows.  

\end{proof} 

\subsection {Proving ADAM (Proof of Theorem~\ref{thm:ADAM-proof})}
\label{sec:supp_adam}

\begin{proof}
Let us assume to the contrary that  $\norm{g_t} > \epsilon$ for all $t=1,2,3.\ldots$. We will show that this assumption will lead to a contradiction. 
By $L-$smoothness of the objective we have the following relationship between the values at consecutive updates, 

\[ f(\x_{t+1}) \leq f(\x_t) + \langle \nabla f(\x_t), \x_{t+1} - \x_t \rangle + \frac {L}{2} \norm{\x_{t+1} - \x_t}^2 \]

~\\
Substituting the update rule using a dummy step length $\eta_t >0$ we have, 

\begin{align}\label{decrease1} 
\nonumber f(\x_{t+1}) &\leq f(\x_t) - \eta_t \langle \nabla f(\x_t), \Big( V_t^{\frac 1 2} + \text{diag} (\xi\mathbf{1}_d) \Big)^{-1} \m_t \rangle \\&+ \frac {L \eta_t ^2}{2} \norm{ \Big( V_t^{\frac 1 2} + \text{diag} (\xi\mathbf{1}_d) \Big)^{-1} \m_t}^2\\
\implies &f(\x_{t+1}) -  f(\x_t) \\&\leq \eta_t \left ( -  \langle \g_t , \Big( V_t^{\frac 1 2} + \text{diag} (\xi\mathbf{1}_d) \Big)^{-1} \m_t \rangle + \frac {L \eta_t }{2} \norm{ \Big( V_t^{\frac 1 2} + \text{diag} (\xi\mathbf{1}_d) \Big)^{-1} \m_t}^2 \right )
\end{align}

The RHS in equation \ref{decrease1} above is a quadratic in $\eta_t$ with two roots: $0$ and $\frac{\langle \g_t , \Big( V_t^{\frac 1 2} + \text{diag} (\xi\mathbf{1}_d) \Big )^{-1} \m_t \rangle}{\frac {L}{2} \norm{ \Big( V_t^{\frac 1 2} + \text{diag} (\xi\mathbf{1}_d) \Big )^{-1} \m_t}^2}$. So the quadratic's minimum value is at the midpoint of this interval, which gives us a candidate $t^{th}-$step length i.e $$\alpha_t^* := \frac{1}{2}\cdot\frac{\langle \g_t , \Big( V_t^{\frac 1 2} + \text{diag} (\xi\mathbf{1}_d) \Big )^{-1} \m_t \rangle}{\frac {L}{2} \norm{ \Big( V_t^{\frac 1 2} + \text{diag} (\xi\mathbf{1}_d) \Big )^{-1} \m_t}^2}$$ and the value of the quadratic at this point is $-\frac{1}{4}\cdot\frac{(\langle \g_t , \Big( V_t^{\frac 1 2} + \text{diag} (\xi\mathbf{1}_d) \Big )^{-1} \m_t \rangle)^2}{\frac {L}{2} \norm{ \Big( V_t^{\frac 1 2} + \text{diag} (\xi\mathbf{1}_d) \Big )^{-1} \m_t}^2}.$
That is with step lengths being this $\alpha_t^*$ we have the following guarantee of decrease of function value between consecutive steps,

\begin{align}\label{decrease} 
f(\x_{t+1}) -  f(\x_t) \leq -\frac{1}{2L}\cdot\frac{(\langle \g_t , \Big( V_t^{\frac 1 2} + \text{diag} (\xi\mathbf{1}_d) \Big )^{-1} \m_t \rangle)^2}{ \norm{ \Big( V_t^{\frac 1 2} + \text{diag} (\xi\mathbf{1}_d) \Big )^{-1} \m_t}^2}
\end{align}
Now we separately lower bound the numerator and upper bound the denominator of the RHS above. 
~\\
\paragraph{{\bf Upperbound on} $\norm{ \Big( V_t^{\frac 1 2} + \text{diag} (\xi\mathbf{1}_d) \Big )^{-1} \m_t}$}
~\\ 
We have, $\lambda_{max} \Big ( \Big( V_t^{\frac 1 2} + \text{diag} (\xi\mathbf{1}_d) \Big )^{-1} \Big ) \leq \frac{1}{\xi + \min_{i=1..d} \sqrt{(\v_t)_i}}$ Further we note that the recursion of $\v_t$ can be solved as, $\v_t = (1-\beta_2) \sum_{k=1}^t \beta_2 ^{t-k} \g_k^2$. Now we define, $\epsilon_t := \min_{k=1,..,t,i=1,..,d} (\g_k^2)_i$ and this gives us,

\begin{align}\label{max1}
\lambda_{max} \Big ( \Big( V_t^{\frac 1 2} + \text{diag} (\xi\mathbf{1}_d) \Big )^{-1} \Big ) \leq \frac {1}{\xi + \sqrt{(1-\beta_2^t)\epsilon_t}}
\end{align}
~\\
We solve the recursion for $\m_t$ to get, $\m_t = (1-\beta_1)\sum_{k=1}^t \beta_1^{t-k}\g_k$. Then by triangle inequality and defining $\sigma_t := \max_{i=1,..,t} \norm{\nabla f (\x_i)}$ we have, $\norm{\m_t} \leq (1-\beta_1^t)\sigma_t $. Thus combining this estimate of $\norm{\m_t}$ with equation \ref{max1} we have,

\begin{align}\label{denom1} 
\norm{ \Big( V_t^{\frac 1 2} + \text{diag} (\xi\mathbf{1}_d) \Big )^{-1} \m_t} \leq \frac{(1-\beta_1^t)\sigma_t}{\xi +\sqrt{\epsilon_t (1-\beta_2^t)}} \leq \frac{(1-\beta_1^t)\sigma_t}{\xi}
\end{align} 

\paragraph{ {\bf Lowerbound on} $\langle \g_t , \Big( V_t^{\frac 1 2} + \text{diag} (\xi\mathbf{1}_d) \Big )^{-1} \m_t \rangle$}

~\\
To analyze this we define the following sequence of functions for each $i = 0,1,2..,t$

\[ Q_i = \langle \g_t , \Big( V_t^{\frac 1 2} + \text{diag} (\xi\mathbf{1}_d) \Big )^{-1} \m_i \rangle \]

~\\
This gives us the following on substituting the update rule for $\m_t$,

\begin{align*} 
Q_i - \beta_1 Q_{i-1} &= \langle \g_t , \Big( V_t^{\frac 1 2} + \text{diag} (\xi\mathbf{1}_d) \Big )^{-1} (\m_i - \beta_1 \m_{i-1})\rangle\\
&= (1-\beta_1) \langle \g_t , \Big( V_t^{\frac 1 2} + \text{diag} (\xi\mathbf{1}_d) \Big )^{-1} \g_i \rangle
\end{align*}

~\\
At $i=t$ we have, $Q_t - \beta_1 Q_{t-1} \geq (1-\beta_1)  \norm{\g_t}^2 \lambda_{min} \Big ( \Big( V_t^{\frac 1 2} + \text{diag} (\xi\mathbf{1}_d) \Big )^{-1} \Big )$

Lets define, $\sigma_{t-1} := \max_{i=1,..,t-1} \norm{\nabla f (\x_i)}$ and this gives us for $i \in \{1,..,t-1\}$, 

\begin{align*} 
Q_i - \beta_1 Q_{i-1} &\geq -(1-\beta_1) \norm{\g_t} \sigma_{t-1}  \lambda_{max} \Big ( \Big( V_t^{\frac 1 2} + \text{diag} (\xi\mathbf{1}_d) \Big )^{-1} \Big )
\end{align*}

~\\
We note the following identity, 
\begin{align*}
    Q_t -\beta_1 ^t Q_0 &= (Q_t - \beta_1 Q_{t-1}) + \beta_1 (Q_{t-1}-\beta_1Q_{t-2}) + \beta_1^2 (Q_{t-2}-\beta_1Q_{t-3}) + .. \\&+ \beta_1^{t-1}(Q_{1} - \beta_1 Q_0)
\end{align*}

Now we use the lowerbounds proven on $Q_i - \beta_1 Q_{i-1}$ for $i\in \{1,..,t-1\}$ and $Q_t - \beta_1 Q_{t-1}$ to lowerbound the above sum as,

\begin{align}\label{telescope}
\nonumber Q_t -\beta_1 ^t Q_0  &\geq (1-\beta_1) \norm{\g_t}^2 \lambda_{min} \Big ( \Big( V_t^{\frac 1 2} + \text{diag} (\xi\mathbf{1}_d) \Big )^{-1} \Big )\\
\nonumber &- (1-\beta_1) \norm{\g_t}\sigma_{t-1} \lambda_{max} \Big ( \Big( V_t^{\frac 1 2} + \text{diag} (\xi\mathbf{1}_d) \Big )^{-1} \Big )\sum_{j=1}^{t-1} \beta_1^j\\
\nonumber &\geq (1-\beta_1) \norm{\g_t}^2 \lambda_{min} \Big ( \Big( V_t^{\frac 1 2} + \text{diag} (\xi\mathbf{1}_d) \Big )^{-1} \Big )\\
&-(\beta_1 - \beta_1^t) \norm{\g_t} \sigma_{t-1} \lambda_{max} \Big ( \Big( V_t^{\frac 1 2} + \text{diag} (\xi\mathbf{1}_d) \Big )^{-1} \Big )
\end{align}

We can evaluate the following lowerbound, 
$$\lambda_{min} \Big ( \Big( V_t^{\frac 1 2} + \text{diag} (\xi\mathbf{1}_d) \Big )^{-1} \Big ) \geq \frac{1}{\xi + \sqrt{\max_{i=1,..,d} (\v_t)_i}}$$ Next we remember that the recursion of $\v_t$ can be solved as, $\v_t = (1-\beta_2) \sum_{k=1}^t \beta_2 ^{t-k} \g_k^2$ and we define, $\sigma_t := \max_{i=1,..,t} \norm{\nabla f (\x_i)}$ to get, 

\begin{align}\label{min}
\lambda_{min} \Big ( \Big( V_t^{\frac 1 2} + \text{diag} (\xi\mathbf{1}_d) \Big )^{-1} \Big ) \geq \frac {1}{\xi + \sqrt{(1-\beta_2^t)\sigma_t^2}}
\end{align}

~\\
Now we combine the above and equation \ref{max1} and the known value of $Q_0 = 0$ (from definition and initial conditions) to get from the equation \ref{telescope},

\begin{align}\label{qtb}
\nonumber Q_t  &\geq -(\beta_1 - \beta_1^t)\norm{\g_t}\sigma_{t-1}\frac {1}{\xi + \sqrt{(1-\beta_2^t)\epsilon_t}}\\
\nonumber &+(1-\beta_1)\norm{\g_t}^2\frac {1}{\xi + \sqrt{(1-\beta_2^t)\sigma_t^2}}\\
&\geq \norm{\g_t}^2 \left ( \frac {(1-\beta_1)}{\xi + \sigma \sqrt{(1-\beta_2^t)}} -  \frac {(\beta_1 - \beta_1^t) \sigma}{\xi \norm{\g_t}} \right )
\end{align}

In the above inequalities we have set $\epsilon_t =0$ and we have set, $\sigma_t = \sigma_{t-1} =\sigma$. Now we examine the following part of the lowerbound proven above,

\begin{align*}
&\frac {(1-\beta_1)}{\xi + \sqrt{(1-\beta_2^t)\sigma^2}} - \frac{(\beta_1 - \beta_1^t) \sigma}{\xi \norm{\g_t}} \\&= \frac {\xi \norm{\g_t} (1-\beta_1) - \sigma (\beta_1 - \beta_1^t)(\xi + \sigma \sqrt{(1-\beta_2^t)})}{\xi \norm{\g_t} ( \xi + \sigma\sqrt{(1-\beta_2^t)} ) }\\
&= \sigma(\beta_1 - \beta_1^t)\frac {\xi \left ( \frac{\norm{\g_t} (1-\beta_1)}{\sigma (\beta_1 - \beta_1^t)} - 1 \right ) - \sigma \sqrt{(1-\beta_2^t)}}{\xi \norm{\g_t} ( \xi + \sigma\sqrt{(1-\beta_2^t)} ) }\\
&=\sigma(\beta_1 - \beta_1^t)\left ( \frac{\norm{\g_t} (1-\beta_1)}{\sigma(\beta_1 - \beta_1^t)} - 1 \right ) \frac {\xi - \left ( \frac{\sigma\sqrt{(1-\beta_2^t)}}{-1+  \frac {(1-\beta_1)\norm{\g_t}}{(\beta_1 - \beta_1^t) \sigma}} \right ) }{\xi \norm{\g_t} ( \xi + \sigma\sqrt{(1-\beta_2^t)} )}
\end{align*}

Now we remember the assumption that we are working under i.e  $\norm{\g_t} > \epsilon$. Also by definition $0 < \beta_1 <1$ and hence we have $0 <\beta_1 - \beta_1^t < \beta_1$. This implies, $\frac {(1-\beta_1)\norm{\g_t}}{(\beta_1 - \beta_1^t) \sigma} > \frac{(1-\beta_1)\epsilon}{\beta_1 \sigma} > 1$ where the last inequality follows because of our choice of $\epsilon$ as stated in the theorem statement. This allows us to define a constant, $ \frac {\epsilon(1-\beta_1)}{\beta_1 \sigma} - 1 := \theta_1 > 0$ s.t  $\frac {(1-\beta_1)\norm{\g_t}}{(\beta_1 - \beta_1^t) \sigma} - 1 > \theta_1$ Similarly our definition of $\xi$ allows us to define a constant $\theta_2 > 0$ to get,

\[ \left ( \frac{\sigma\sqrt{(1-\beta_2^t)}}{-1+  \frac {(1-\beta_1)\norm{\g_t}}{(\beta_1 - \beta_1^t) \sigma}} \right ) < \frac{\sigma}{\theta_1} = \xi - \theta_2 \]

Putting the above back into the lowerbound for $Q_t$ in equation \ref{qtb} we have, 

\begin{align}\label{qtfinal}
Q_t \geq \norm{\g_t} ^2 \left (  \frac{\sigma (\beta_1 - \beta_1^2) \theta_1  \theta_2}{\xi \sigma (\xi + \sigma)} \right )
\end{align}

Now we substitute the above and equation \ref{denom1} into equation \ref{decrease} to get, 
\begin{align}\label{findecrease}
\nonumber f(\x_{t+1}) -  f(\x_t) &\leq -\frac{1}{2L}\cdot\frac{(\langle \g_t , \Big( V_t^{\frac 1 2} + \text{diag} (\xi\mathbf{1}_d) \Big )^{-1} \m_t \rangle)^2}{ \norm{ \Big( V_t^{\frac 1 2} + \text{diag} (\xi\mathbf{1}_d) \Big )^{-1} \m_t}^2}\\
\nonumber &\leq -\frac {1}{2L} \frac {Q_t^2}{\norm{ \Big( V_t^{\frac 1 2} + \text{diag} (\xi\mathbf{1}_d) \Big )^{-1} \m_t}^2}\\
\nonumber &\leq -\frac {1}{2L}\frac {\norm{\g_t} ^4 \left (  \frac{ (\beta_1 - \beta_1^2) \theta_1  \theta_2}{\xi (\xi + \sigma)} \right )^2}{\left ( \frac{(1-\beta_1^t)\sigma}{\xi} \right ) ^2 }\\
 &\leq -\frac {\norm{\g_t} ^4}{2L} \left (    \frac{ (\beta_1 - \beta_1^2)^2 \theta_1^2  \theta_2^2}{ (\xi + \sigma)^2 (1-\beta_1^t)^2\sigma^2} \right )
\end{align}

Thus we get from the above, 

\begin{align*}
 &\left (    \frac{ (\beta_1 - \beta_1^2)^2 \theta_1^2  \theta_2^2}{ 2L (\xi + \sigma)^2 (1-\beta_1^t)^2\sigma^2} \right ) \norm{\nabla f (\x_t)}^4 \leq [ f(\x_t) - f(\x_{t+1})] \\
 &\implies \sum_{t=2}^T \left (    \frac{ (\beta_1 - \beta_1^2)^2 \theta_1^2  \theta_2^2}{ 2L (\xi + \sigma)^2 \sigma^2} \right )  \norm{\nabla f (\x_t)}^4 \leq  [f(\x_2) - f(\x_{T+1}) ]\\
&\implies \min_{t=2,..,T} \norm{\nabla f (\x_t)}^4 \leq \frac{2L (\xi + \sigma)^2 \sigma^2}{T(\beta_1 - \beta_1^2)^2 \theta_1^2  \theta_2^2 } [f(\x_2) - f(\x_*)]\\
\end{align*}

Observe that if, $T \geq \frac{2L \sigma^2 (\xi + \sigma)^2}{2 \epsilon^4 (\beta_1 - \beta_1^2)^2 \theta_1^2  \theta_2^2 } [f(\x_2) - f(\x_*)]$ then the RHS of the inequality above is less than or equal to $\epsilon^4$ and this would contradict the assumption that $\norm{\nabla f (\x_t)} > \epsilon$ for all $t=1,2, \ldots$. 

As a consequence we have proven the first part of the theorem which guarantees the existence of positive step lengths, $\alpha_t$ s.t ADAM finds an approximately critical point in finite time. 

Now choose $\theta_1 = 1$ i.e $\frac{\epsilon}{2} = \frac{\beta_1 \sigma}{1-\beta_1}$ i.e $\beta_1 = \frac{\epsilon}{\epsilon + 2\sigma} \implies \beta_1(1-\beta_1) = \frac{\epsilon}{\epsilon + 2\sigma} ( 1 - \frac{\epsilon}{\epsilon + 2\sigma}) = \frac {2\sigma \epsilon}{(\epsilon + 2\sigma)^2}$. This also gives a easier-to-read condition on $\xi$ in terms of these parameters i.e $\xi > \sigma$. Now choose $\xi = 2\sigma$ i.e $\theta_2 = \sigma$ and making these substitutions gives us,

\begin{align*}
T \geq \frac{18L\sigma^4}{2\epsilon^4  \left ( \frac {2\sigma \epsilon}{(\epsilon + 2\sigma)} \right ) ^2 \sigma^2 } [f(\x_2) - f(\x_*)] &\geq \frac{18L}{8\epsilon^6  \left ( \frac {1}{\epsilon + 2\sigma} \right ) ^2  } [f(\x_2) - f(\x_*)] \\&\geq \frac{9 L \sigma ^2}{\epsilon^6   } [f(\x_2) - f(\x_*)] 
\end{align*}

We substitute these choices in the step length found earlier to get, 

\begin{align*}
\alpha_t^* &= \frac{1}{L}\cdot\frac{\langle \g_t , \Big( V_t^{\frac 1 2} + \text{diag} (\xi\mathbf{1}_d) \Big )^{-1} \m_t \rangle}{\norm{ \Big( V_t^{\frac 1 2} + \text{diag} (\xi\mathbf{1}_d) \Big )^{-1} \m_t}^2} = \frac{1}{L}\cdot\frac{Q_t}{\norm{ \Big( V_t^{\frac 1 2} + \text{diag} (\xi\mathbf{1}_d) \Big )^{-1} \m_t}^2}\\
&\geq \frac{1}{L} \frac{\norm{\g_t} ^2 \left (  \frac{\sigma^2 (\beta_1 - \beta_1^2) }{\xi \sigma (\xi + \sigma)} \right )}{ \left (  \frac{(1-\beta_1^t)\sigma}{\xi} \right )^2} = \frac{\norm{\g_t} ^2}{L(1-\beta_1^t)^2}\frac{4\epsilon}{3(\epsilon + 2\sigma)^2} := \alpha_t
\end{align*}

In the theorem statement we choose to call as the final $\alpha_t$ the lowerbound proven above. We check below that this smaller value of $\alpha_t$ still guarantees a decrease in the function value that is sufficient for the statement of the theorem to hold.

\paragraph{{\bf A consistency check!}} Let us substitute the above final value of the step length $ \alpha_t = \frac{1}{L} \frac{\norm{\g_t} ^2 \left (  \frac{\sigma^2 (\beta_1 - \beta_1^2) }{\xi \sigma (\xi + \sigma)} \right )}{ \left (  \frac{(1-\beta_1^t)\sigma}{\xi} \right )^2} = \frac{\xi}{L(1-\beta_1^t)^2}\norm{\g_t} ^2 \left (  \frac{ (\beta_1 - \beta_1^2) }{\sigma (\xi + \sigma)} \right )$, the bound in  equation \ref{denom1} (with $\sigma_t$ replaced by $\sigma$), and the bound in equation \ref{qtfinal} (at the chosen values of $\theta_1 = 1$ and $\theta_2 = \sigma$) in the original equation \ref{decrease} to measure the decrease in the function value between consecutive steps, 

\begin{align*} 
&f(\x_{t+1}) -  f(\x_t) \\&\leq \alpha_t \left ( -  \langle \g_t , \Big( V_t^{\frac 1 2} + \text{diag} (\xi\mathbf{1}_d) \Big)^{-1} \m_t \rangle + \frac {L \alpha_t }{2} \norm{ \Big( V_t^{\frac 1 2} + \text{diag} (\xi\mathbf{1}_d) \Big)^{-1} \m_t}^2 \right )\\
&\leq \alpha_t \left ( -  Q_t + \frac {L \alpha_t }{2} \norm{ \Big( V_t^{\frac 1 2} + \text{diag} (\xi\mathbf{1}_d) \Big)^{-1} \m_t}^2 \right )\\
&\leq \frac{\xi}{L(1-\beta_1^t)^2}\norm{\g_t} ^2 \left (  \frac{ (\beta_1 - \beta_1^2) }{\sigma (\xi + \sigma)} \right ) \left ( -  \norm{\g_t} ^2 \left (  \frac{\sigma (\beta_1 - \beta_1^2) \theta_1  \theta_2}{\xi \sigma (\xi + \sigma)} \right ) \right )\\
&+ \frac{L}{2} \left ( \frac{\xi}{L(1-\beta_1^t)^2}\norm{\g_t} ^2 \left (  \frac{ (\beta_1 - \beta_1^2) }{\sigma (\xi + \sigma)} \right ) \frac{(1-\beta_1^t)\sigma}{\xi}\right )^2
\end{align*}
The RHS above can be simplified to be shown to be equal to the RHS in equation \ref{findecrease} at the same values of $\theta_1$ and $\theta_2$ as used above. And we recall that the bound on the running time was derived from this equation \ref{findecrease}. 
\end{proof}

\newpage 
\section{Conclusion}
To the best of our knowledge, we present the first theoretical guarantees of convergence to criticality for the immensely popular algorithms RMSProp and ADAM in their most commonly used setting of optimizing a non-convex objective. 
 
By our experiments, we have sought to shed light on the important topic of the interplay between adaptivity and momentum in training nets. By choosing to study textbook autoencoder architectures  where various parameters of the net can be changed controllably we  highlight the following two aspects that (a) the value of the gradient shifting hyperparameter $\xi$ has a significant influence on the performance of ADAM and RMSProp and (b) ADAM seems to perform particularly well (often supersedes Nesterov accelerated gradient method) when its momentum parameter $\beta_1$ is very close to $1$. On VGG-9 with CIFAR-10 and for the task of training autoencoders on MNIST we have verified these conclusions across different widths and depths of nets as well as in the full-batch and the mini-batch setting (with large nets) and also under compression of the input/output image size. 

Curiously enough, this regime of $\beta_1$ being close to $1$ is currently not within the reach of our proof techniques of showing convergence for ADAM. Our experiments give strong reasons to try to advance theory in this direction in future work. Though we note that it is still open to find a characterization of the class of objectives for which ADAM and RMSProp in their standard stochastic forms converge to criticality using just a bounded moment and unbiased gradient estimating oracle. Hence theoretically we are still far from being able to explain the unique advantages of the standard versions of RMSProp or ADAM, which in turn we have thoroughly demonstrated in the experiments in this work.  
\begin{subappendices}
\chapter*{\vspace{10pt}Appendix To Chapter \ref{chapadam}}\label{app:adam}
\renewcommand{\thesubsection}{\Alph{subsection}}

\section{Proving stochastic RMSProp (Proof of Theorem \ref{thm:RMSPropS-proof})}
\label{supp_rmspropS}

Now we give the proof of Theorem \ref{thm:RMSPropS-proof}.

\begin{proof}
We define $\sigma_t := \max_{k=1,..,t} \Vert \nabla f_{i_k}(x_k) \Vert$ and we solve the recursion for $v_t$ as, $v_t = (1-\beta_2)\sum_{k=1}^t \beta_2^{t-k}(g_k^2 + \xi)$. This lets us write the following bounds, 
\begin{align*}
\lambda_{min} (V_t^{-\frac 1 2})  &\geq \frac  { 1 } {\sqrt{\max_{i=1,..,d} (\v_t)_i}} \geq \frac  { 1 } {\sqrt{\max_{i=1,..,d} ((1-\beta_2)\sum_{k=1}^{t} \beta_2^{t-k}(\g_k^2 + \xi\mathfrak{1}_d)_i)}}\\
&\geq  \frac  {1} {\sqrt{1-\beta_2^t}\sqrt{\sigma_t^2 + \xi}}
\end{align*}

Now we define, $\epsilon_t := \min_{k=1,..,t ,i = 1,..,d} (\nabla f_{i_k}(\x_k))^2_i$ and this lets us get the following bounds,

\begin{align*}
\lambda_{max}(V_t^{-\frac 1 2}) \leq \frac {1}{\min_{i=1,..,d} (\sqrt{ (\v_t)_i})} \leq \frac {1}{\sqrt{(1-\beta_2^t)}\sqrt{(\xi+\epsilon_t)}}    
\end{align*}

Now we invoke the bounded gradient assumption about the $f_i$ functions and replace in the above equation the eigenvalue bounds of the pre-conditioner by worst-case estimates $\mu_{\max}$ and $\mu_{\min}$ defined as, 

\begin{align*}
\lambda_{min}(V_t^{-\frac {1}{2}}) &\geq  \frac  {1} {\sqrt{\sigma_f^2 + \xi}} := \mu_{\min}\\
\lambda_{max}(V_t^{-\frac {1}{2}}) &\leq \frac {1}{\sqrt{(1-\beta_2)}\sqrt{\xi}} := \mu_{max}
\end{align*}

Using the $L$-smoothness of $f$ between consecutive iterates $\x_t$ and $\x_{t+1}$ we have, 
\begin{align*}
    f(\x_{t+1}) & \leq f(\x_t) +\langle \nabla f(\x_t), \x_{t+1} - \x_t\rangle + \frac{L}{2}\Vert \x_{t+1} - \x_t\Vert^2  \\
\end{align*}

We note that the update step of stochastic RMSProp is $x_{t+1} = x_t - \alpha (V_t)^{-\frac{1}{2}}g_t$ where $g_t$ is the stochastic gradient at iterate $x_t$. Let $H_t = \{ \x_1, \x_2,..,\x_t\}$ be the set of random variables corresponding to the first $t$ iterates. The assumptions we have about the stochastic oracle give us the following relations, $\mathbb{E}[g_t \mid H_t ] = \nabla f(x_t)$  and $\mathbb{E}[\Vert g_t\Vert^2 \mid H_t] \leq \sigma_f^2$ 
. Now we can invoke these stochastic oracle's properties and take a conditional (on $H_t$) expectation over $\g_t$ of the $L-$smoothness in equation to get,

\begin{align}\label{history}
    \nonumber \mathbb{E}[f(\x_{t+1})\mid H_t] & \leq f(\x_t) -\alpha \mathbb{E}\left [\langle \nabla f(\x_t), (V_t)^{-\frac{1}{2}}\g_t  \rangle \mid H_t \right ]+ \frac{\alpha^2 L}{2}\mathbb{E}\left [\Vert (V_t)^{-\frac{1}{2}}\g_t  \Vert^2\mid H_t \right ]  \\
    \nonumber & \leq f(\x_t) - \alpha \mathbb{E}\left [\langle \nabla f(\x_t), (V_t)^{-\frac{1}{2}}\g_t  \rangle \mid H_t \right ]+ \mu_{\max}^2\frac{\alpha^2 L}{2}\mathbb{E} \left [\Vert \g_t  \Vert^2 \mid H_t \right ]  \\
    & \leq f(\x_t) - \alpha \mathbb{E} \left [\langle \nabla f(\x_t), (V_t)^{-\frac{1}{2}}\g_t  \rangle \mid  H_t \right ]+ \mu_{\max}^2 \frac{\alpha^2\sigma_f ^2  L}{2}
\end{align}

We now separately analyze the middle term in the RHS above. In Lemma \ref{lemma:ipbound} below and we get,

\begin{align*}
     \mathbb{E}[\langle \nabla f(\x_t), (V_t)^{-\frac{1}{2}}\g_t  \rangle \mid H_t] \geq  \mu_{min} \Vert \nabla f(\x_t)\Vert^2
\end{align*}

We substitute the above into equation \ref{history} and take expectations over $H_t$ to get, 

\begin{align}\label{enorm}
     \nonumber &\mathbb{E}[f(\x_{t+1}) -f(\x_t)] \leq - \alpha\mu_{min} \mathbb{E}[ \Vert \nabla f(\x_t)\Vert^2 ]  + \mu_{max}^2\frac{\alpha^2\sigma_f^2  L}{2}\\
     \implies &\mathbb{E}[\Vert \nabla f(\x_t)\Vert^2] \leq  \frac{1}{\alpha \mu_{min}}\mathbb{E}[f(\x_{t}) -f(\x_{t+1})] +\frac{\alpha \sigma_f^2 L}{2} \frac{\mu_{max}^2}{\mu_{min}}  
\end{align}


Doing the above replacements to upperbound the RHS of equation \ref{enorm} and summing the inequation over $t = 1$ to $t=T$ and taking the average and replacing the LHS by a lowerbound of it, we get,
\begin{align*}
     \min_{t = 1 \ldots T} \mathbb{E}[\Vert \nabla f(\x_t)\Vert^2] &\leq  \frac{1}{\alpha T \mu_{\min}}\mathbb{E}[f(\x_{1}) -f(\x_{T+1})] +\frac{\alpha  \sigma_f^2 L}{2} \frac{\mu_{\max}^2}{\mu_{\min}}\\
     &\leq \frac{1}{\alpha T \mu_{\min}}\left (f(\x_{1}) -f(\x_*) \right ) +\frac{\alpha  \sigma_f^2 L}{2} \frac{\mu_{\max}^2}{\mu_{\min}}
\end{align*}
Replacing into the RHS above the optimal choice of,  
\[ \alpha = \frac{1}{\sqrt{T}} \sqrt{\frac{2\left (f(\x_{1}) -f(\x_*)\right )}{\sigma_f ^2 L \mu_{\max}^2}} = \frac{1}{\sqrt{T}} \sqrt{\frac{2 \xi (1-\beta_2)\left (f(\x_{1}) -f(\x_*)\right )}{\sigma_f ^2 L}} \]  
we get,
\begin{align*}
\min_{t = 1 \ldots T} \mathbb{E}[\Vert \nabla f(\x_t)\Vert^2] &\leq 2\sqrt{ \frac{1}{ T \mu_{\min}}\left (f(\x_{1}) -f(\x_*) \right ) \times \frac{ L \sigma_f^2}{2} \frac{\mu_{\max}^2}{\mu_{\min}} } \\&= \frac{1}{\sqrt{T}}\sqrt{\frac{2L\sigma_f^2 (\sigma_f^2 + \xi)\left (f(\x_{1}) -f(\x_*) \right ) }{(1-\beta_2)\xi}} 
\end{align*}
Thus stochastic RMSProp with the above step-length is guaranteed is reach $\epsilon$ criticality in number of iterations given by, $T \leq \frac{1}{\epsilon^4} \left ( \frac{2L\sigma_f^2 (\sigma_f^2 + \xi)\left (f(\x_{1}) -f(\x_*) \right ) }{(1-\beta_2)\xi} \right )$
\end{proof}

\begin{lemma} At any time $t$, the following holds,
\label{lemma:ipbound}
\begin{align*}
    \mathbb{E}[\langle \nabla f(x_t),V_t^{-1/2} g_t\rangle \mid H_t] \geq \mu_{\text{min}} \norm{\nabla f(x_t)}^2 
\end{align*}
\end{lemma}




\begin{proof}
\begin{align}\label{A1E1}
    \nonumber \mathbb{E}\left [\langle \nabla f(x_t),V_t^{-1/2}g_t\rangle \mid H_t \right ] 
    &= \mathbb{E} \left [\sum_{i=1}^d \nabla_i f(x_t) (V_t^{-1/2})_{ii} (g_t)_i  \mid H_t \right] \\
    &=  \sum_{i=1}^d \nabla_i f(x_t)\mathbb{E} \left [(V_t^{-1/2})_{ii} (g_t)_i \mid H_t \right ]
\end{align}

Now we introduce some new variables to make the analysis easier to present. Let $a_{pi} := \left  [ \nabla f_p (\x_t) \right ]_i$ where $p$ indexes the training data set, $p \in  \{1,\ldots,k\}$. (conditioned on $H_t$, $a_{pi}$s are constants)  This implies, 
$\nabla_i f(x_t) = \frac{1}{k}\sum_{p=1}^k a_{pi}$ 
We recall that $\mathbb{E} \left [ (\g_t)_i\right ] = \nabla_i f(\x_t)$ 
where the expectation is taken over the oracle call at the $t^{th}$ update step. Further our instantiation of the oracle is equivalent to doing the uniformly at random sampling, $(\g_t)_i \sim \{a_{pi}\}_{p=1,\ldots,k}$. 

Given that we have, $V_t = \text{diag}(\v_t)$ with $\v_t = (1-\beta_2)\sum_{k=1}^t \beta_2^{t-k}(\g_k^2 + \xi\mathbf{1}_d)$ this implies, $(V_t^{-1/2})_{ii} = \frac{1}{\sqrt{(1-\beta_2)(\g_t)_i^2 + d_i}}$ where we have defined $d_i := (1-\beta_2)\xi + (1-\beta_2)\sum_{k=1}^{t-1} \beta_2^{t-k}((\g_k)_i^2 + \xi)$. (conditioned on $H_t$, $d_i$ is a constant) This leads to an explicit form of the needed expectation over the $t^{th}-$oracle call as, 

\begin{align*}
\mathbb{E} \left [(V_t^{-1/2})_{ii} (g_t)_i\mid H_t \right ] &= \mathbb{E} \left [(V_t^{-1/2})_{ii} (g_t)_i \mid H_t \right]\\
&= \mathbb{E}_{(\g_t)_i \sim \{a_{pi}\}_{p=1,\ldots,k}} \left [\frac{ (g_t)_i}{\sqrt{(1-\beta_2)(g_t)_i^2 + d_i}} \mid H_t \right]\\
&= \frac{1}{k}\sum_{p=1}^k \frac{ a_{pi}}{\sqrt{(1-\beta_2)a_{pi}^2  + d_{i}}}
\end{align*} 

Substituting the above (and the definition of the constants $a_{pi}$) back into equation \ref{A1E1} we have, 

\begin{align*}
     \mathbb{E} \left [\langle \nabla f(x_t),V_t^{-1/2} g_t\rangle \mid H_t \right ]  =  \sum_{i=1}^d \left(\frac{1}{k}\sum_{p=1}^k a_{pi}\right)\left(\frac{1}{k}\sum_{p=1}^k \frac{ a_{pi}}{\sqrt{(1-\beta_2)a_{pi}^2  + d_{i}}}\right)
\end{align*}

We define two vectors $\a_i, \mathbf{h}_i \in \R^k$ s.t $(\a_i)_p = a_{pi}$ and  $(\mathbf{h}_i)_p = \frac{1}{\sqrt{(1-\beta_2)a_{pi}^2  + d_{i}}}$

Substituting this, the above expression can  be written as,

\begin{align}
    \mathbb{E} \left [\langle \nabla f(x_t),V_t^{-1/2} g_t\rangle \mid H_t \right ] = \frac{1}{k^2} \sum_{i=1}^d \left ( \a_i^\top \mathbf{1}_k \right ) \left (  \mathbf{h}_i^\top \a_i \right ) =  \frac{1}{k^2} \sum_{i=1}^d \a_i^\top \left ( \mathbf{1}_k \mathbf{h}_i^\top \right ) \a_i
\end{align}

Note that with this substitution, the RHS of the claimed lemma becomes, 
\begin{align*}
    \mu_{\text{min}} \norm{\nabla f(x_t)}^2  &= \mu_{\text{min}} \sum_{i=1}^d  \left(\frac{1}{k}\sum_{p=1}^k  \nabla_p f(\x_t)\right)^2 \\
    & =  \frac{ \mu_{\text{min}} }{k^2} \sum_{i=1}^d (\a_i^\top \mathbf{1}_k)^2 \\
    & =  \frac{ \mu_{\text{min}} }{k^2} \sum_{i=1}^d  \a_i^\top \mathbf{1}_k \mathbf{1}_k^\top \a_i
\end{align*}

Therefore our claim is proved if we show that for all $i$, $$ \frac{1}{k^2}\a_i^\top \left ( \mathbf{1}_k \mathbf{h}_i^\top \right ) \a_i -   \frac{\mu_{\text{min}} }{k^2}  \a_i^\top \mathbf{1}_k \mathbf{1}_k^\top \a_i \geq 0$$. This can be simplified as,

\begin{align*}
    \frac{1}{k^2}\a_i^\top \left ( \mathbf{1}_k \mathbf{h}_i^\top \right ) \a_i -  \mu_{\text{min}}  \frac{1}{k^2}  \a_i^\top \mathbf{1}_k \mathbf{1}_k^\top \a_i & = \frac{1}{k^2} \a_i^\top\left(\mathbf{1}_k \left( \mathbf{h}_i - \mu_{\text{min}} \mathbf{1}_k  \right)^\top \right)\a_i
\end{align*}

To further simplify, we define $\mathbf{q}_i \in \mathbb{R}^k, (\mathbf{q}_i)_p =  (\mathbf{h}_i)_p - \mu_{\text{min}} = \frac{1}{\sqrt{(1-\beta_2)a_{pi}^2  + d_{i}}} - \mu_{\text{min}}$. We therefore need to show, 

\begin{align*}
\frac{1}{k^2} \a_i^\top\left(\mathbf{1}_k \mathbf{q}_i^\top \right)\a_i \geq 0
\end{align*}

We first bound $d_i$ by recalling the definition of $\sigma_f$ (from which it follows that $a_{pi}^2 \leq \sigma_f^2$), 

\begin{align}\label{mumin}
 \nonumber &d_i \leq (1-\beta_2) \left [ \xi  + \sum_{k=1}^{t-1} \beta_2^{t-k} ( \sigma_f^2 + \xi ) \right ] =  (1-\beta_2) \left [ \xi  +  \frac{\beta_2 - \beta_2^{t-1}}{1-\beta_2}( \sigma_f^2 + \xi )  \right ]\\
 \nonumber &\leq (1-\beta_2)\xi + (\beta_2 - \beta_2^{t-1})\xi + (\beta_2 - \beta_2^{t-1})\sigma_f^2 = (1-\beta_2^{t-1})\xi + (\beta_2 - \beta_2^{t-1})\sigma_f^2\\
 \nonumber &\implies \sqrt{(1-\beta_2)a_{pi}^2  + d_{i}} \leq \sqrt{(1-\beta_2)\sigma_f^2 + (1-\beta_2^{t-1})\xi + (\beta_2 - \beta_2^{t-1})\sigma_f^2} \\ \nonumber &= \sqrt{(1-\beta_2^{t-1})(\sigma_f^2 + \xi)}\\
 \nonumber &\implies -\mu_{\text{min}} + \frac{1}{\sqrt{(1-\beta_2)a_{pi}^2  + d_{i}}} \geq -\mu_{\text{min}} + \frac{1}{\sqrt{(1-\beta_2^{t-1})(\sigma_f^2 + \xi)}} \\
  \nonumber &= -\frac{1}{\sqrt{\sigma_f^2 + \xi}} + \frac{1}{\sqrt{(1-\beta_2^{t-1})(\sigma_f^2 + \xi)}}\\
&\implies -\mu_{\text{min}} + \frac{1}{\sqrt{(1-\beta_2)a_{pi}^2  + d_{i}}} \geq 0 
\end{align}

The inequality follows since $\beta_2 \in (0,1]$ 

Putting this all together, we get,
\begin{align*}
&(\a_i^\top \mathbf{1}_k)(\mathbf{q}_i^\top \a_i)\\
= & \left ( \sum_{p=1}^k a_{pi} \right)\left ( \sum_{p=1}^k \left [ -\mu_{\text{min}}a_{pi} + \frac{a_{pi}}{\sqrt{(1-\beta_2)a_{pi}^2 + d_i}} \right ] \right )\\ 
= &\sum_{p,q=1}^k \left [ -\mu_{\text{min}}a_{pi}a_{qi} + \frac{a_{pi}a_{qi}}{\sqrt{(1-\beta_2)a_{pi}^2 + d_i}} \right ]\\ 
= &\sum_{p,q=1}^k a_{pi}a_{qi} \left [ -\mu_{\text{min}} + \frac{1}{\sqrt{(1-\beta_2)a_{pi}^2 + d_i}} \right ]
\end{align*}

Now our assumption that for all $\x$, $\text{sign}(\nabla f_p (\x)) = \text{sign}(\nabla f_q (\x))$ for all $p,q \in \{1,\ldots,k\}$ leads to the conclusion that the term $a_{pi}a_{qi} \geq 0$. And we had already shown in equation \ref{mumin} that $ \left [ -\mu_{\text{min}} + \frac{1}{\sqrt{(1-\beta_2)a_{pi}^2 + d_i}} \right ] \geq 0$. Thus we have shown that $(\a_i^\top \mathbf{1}_k)(\mathbf{q}_i^\top \a_i) \geq 0$ and this finishes the proof.
\end{proof}

\section{Proving deterministic RMSProp - the version with standard speed (Proof of Theorem \ref{thm:RMSProp1-proof})}
\label{sec:supp_rmsprop1}

\begin{proof}
By the $L-$smoothness condition and the update rule in Algorithm \ref{RMSProp_TF} we have,https://www.overleaf.com/project/5f7fa365abb4250001d5d795
\begin{align}\label{step1}
\nonumber f(\x_{t+1}) &\leq f(\x_t) - \alpha_t \langle \nabla f(\x_t),  V_t^{-\frac 1 2}\nabla f(\x_t)\rangle +  \alpha_t ^2 \frac {L}{2}\Vert V_t^{-\frac 1 2}  \nabla f(\x_t) \Vert ^2 \\
\implies f(\x_{t+1}) - f(\x_t) &\leq \alpha_t \left (\frac {L \alpha_t}{2} \Vert V_t^{-\frac 1 2}  \nabla f(\x_t) \Vert ^2 - \langle \nabla f(\x_t),  V_t^{-\frac 1 2}\nabla f(\x_t)\rangle  \right ) 
\end{align}

For $0 < \delta_t^2 < \frac{1}{ \sqrt{1-\beta_2^t} \sqrt{\sigma_t^2 + \xi}}$ we now show a strictly positive lowerbound on the following function,
\begin{align}\label{step2}
 \frac{2}{L} \left (  \frac {\langle \nabla f(\x_t),  V_t^{-\frac 1 2}\nabla f(\x_t)\rangle - \delta_t^2 \norm{\nabla f(\x_t)}^2 } { \Vert V_t^{-\frac 1 2}  \nabla f(\x_t) \Vert ^2 } \right ) 
\end{align}
We define $\sigma_t := \max_{i=1,..,t} \Vert \nabla f(\x_i) \Vert$ and we solve the recursion for $\v_t$ as, $\v_t = (1-\beta_2)\sum_{k=1}^t \beta_2^{t-k}(\g_k^2 + \xi)$. This lets us write the following bounds, 

\begin{align}\label{numer} 
\nonumber \langle \nabla f(\x_t),  V_t^{-\frac 1 2}\nabla f(\x_t)\rangle &\geq \lambda_{min} (V_t^{-\frac 1 2}) \Vert \nabla f(\x_t) \Vert ^2 \geq \frac  { \Vert \nabla f(\x_t) \Vert ^2 } {\sqrt{\max_{i=1,..,d} (\v_t)_i}} \\
\nonumber &\geq \frac  { \Vert \nabla f(\x_t) \Vert ^2 } {\sqrt{\max_{i=1,..,d} ((1-\beta_2)\sum_{k=1}^{t} \beta_2^{t-k}(\g_k^2 + \xi\mathfrak{1}_d)_i)}} \\
&\geq  \frac  { \Vert \nabla f(\x_t) \Vert ^2 } {\sqrt{1-\beta_2^t}\sqrt{\sigma_t^2 + \xi}} 
\end{align}
Now we define, $\epsilon_t := \min_{k=1,..,t ,i = 1,..,d} (\nabla f(\x_k))^2_i$ and this lets us get the following sequence of inequalities,

\begin{align}\label{denom}
\Vert V_t^{-\frac 1 2}  \nabla f(\x_t) \Vert ^2 \leq \lambda_{max}^2(V_t^{-\frac 1 2}) \Vert \nabla f(\x_t) \Vert ^2 &\leq \frac {\Vert \nabla f(\x_t) \Vert ^2}{(\min_{i=1,..,d} (\sqrt{ (\v_t)_i}))^2} \\&\leq \frac {\Vert \nabla f(\x_t) \Vert ^2}{(1-\beta_2^t)(\xi+\epsilon_t)} 
\end{align}

So combining equations \ref{denom} and \ref{numer} into equation \ref{step2} and from the exit line in the loop we are assured that $\Vert \nabla f(\x_t) \Vert ^2 \neq 0$ and combining these we have, 

\begin{align*}
&\frac{2}{L} \left ( \frac{-\delta_t^2 \norm{\nabla f(\x_t)}^2  + \langle \nabla f(\x_t),  V_t^{-\frac 1 2}\nabla f(\x_t)\rangle} { \Vert V_t^{-\frac 1 2}  \nabla f(\x_t) \Vert ^2 } \right ) \\&\geq \frac{2}{L} \left ( \frac {- \delta_t^2 + \frac{ 1} {\sqrt{1-\beta_2^t}\sqrt{\sigma_t^2 + \xi}} }{ \frac {1}{(1-\beta_2^t)(\xi+\epsilon_t)} } \right )  \\
&\geq \frac{2 (1-\beta_2^t)(\xi + \epsilon_t)}{L} \left ( -\delta_t^2 + \frac{1}{ \sqrt{1-\beta_2^t} \sqrt{\sigma_t^2 + \xi}} \right )
\end{align*}

Now our definition of $\delta_t^2$ allows us to define a parameter $0 < \beta_t := \frac{1}{ \sqrt{1-\beta_2^t} \sqrt{\sigma_t^2 + \xi}} - \delta_t^2$ and rewrite the above equation as, 

\begin{align}\label{lower} 
\frac{2}{L} \left ( \frac{-\delta_t^2 \norm{\nabla f(\x_t)}^2  + \langle \nabla f(\x_t),  V_t^{-\frac 1 2}\nabla f(\x_t)\rangle} { \Vert V_t^{-\frac 1 2}  \nabla f(\x_t) \Vert ^2 } \right )  \geq \frac{2 (1-\beta_2^t)(\xi + \epsilon_t)\beta_t}{L} 
\end{align}
We can as well satisfy the conditions needed on the variables, $\beta_t$ and $\delta_t$ by choosing, 

\[ 
\delta_t^2  = \frac{1}{2} \min_{t=1,\ldots} \frac{1}{\sqrt{1-\beta_2^t}\sqrt{\sigma_t^2 + \xi}} =  \frac{1}{2} \frac{1}{\sqrt{\sigma^2 + \xi}} =: \delta^2\] 

and

\[ 
\beta_t  = \min_{t=1,..} \frac{1}{\sqrt{1-\beta_2^t}\sqrt{\sigma_t^2 + \xi}} - \delta^2 =  \frac{1}{2} \frac{1}{\sqrt{\sigma^2 + \xi}} 
\]  

Then the worst-case lowerbound in equation \ref{lower} becomes, 

\[  
\frac{2}{L} \left ( \frac{-\delta_t^2 \norm{\nabla f(\x_t)}^2  + \langle \nabla f(\x_t),  V_t^{-\frac 1 2}\nabla f(\x_t)\rangle} { \Vert V_t^{-\frac 1 2}  \nabla f(\x_t) \Vert ^2 } \right ) \geq \frac{2(1-\beta_2)\xi}{L} \times \frac{1}{2} \frac{1}{\sqrt{\sigma^2 + \xi}}  
\]

This now allows us to see that a constant step length $\alpha_t = \alpha >0$ can be defined as, $\alpha  = \frac{(1-\beta_2)\xi}{L\sqrt{\sigma^2 + \xi}}$ and this is 
such that the above equation can be written as, $\frac {L \alpha}{2} \Vert V_t^{-\frac 1 2}  \nabla f(\x_t) \Vert ^2 - \langle \nabla f(\x_t),  V_t^{-\frac 1 2}\nabla f(\x_t)\rangle \leq - \delta^2 \norm{\nabla f(\x_t)}^2$ . This when substituted back into equation \ref{step1} we have, 

\[ f(\x_{t+1}) - f(\x_t) \leq -\delta ^2 \alpha  \norm{\nabla f (\x_t)}^2 =  -\delta^2 \alpha \norm{\nabla f(\x_t)}^2 \]

This gives us, 
\begin{align}\label{average}
\nonumber \norm{\nabla f(\x_t)}^2 &\leq \frac{1}{\delta^2 \alpha} [ f(\x_t) -  f(\x_{t+1}) ]\\
\implies  \sum_{t=1}^T \norm{\nabla f(\x_t)}^2 &\leq \frac{1}{\delta^2 \alpha}  [ f(\x_1) -  f(\x_{*}) ]\\
\implies \min_{t=1,,.T}{\norm{\nabla f(\x_t)}^2} &\leq \frac{1}{T\delta^2 \alpha} [ f(\x_1) -  f(\x_*) ]
\end{align}

Thus for any given $\epsilon>0$, $T$ satisfying, $\frac{1}{T\delta^2 \alpha} [ f(\x_1) -  f(\x_*)] \leq \epsilon ^2$ is a sufficient condition to ensure that the algorithm finds a point $\x_{result} := \argmin_{t=1,,.T}{\norm{\nabla f(\x_t)}^2}$ with $ \norm{\nabla f(\x_{result})}^2 \leq \epsilon^2$. 

Thus we have shown that using a constant step length of $\alpha = \frac{(1-\beta_2)\xi}{L\sqrt{\sigma^2 + \xi}}$ deterministic RMSProp can find an $\epsilon-$critical point in $T = \frac{1}{\epsilon^2} \times \frac{ f(\x_1) -  f(\x_*)}{\delta^2 \alpha} =  \frac{1}{\epsilon^2}\times \frac{2L(\sigma^2 + \xi)(f(\x_1) -  f(\x_*))}{(1-\beta_2)\xi} $ steps. 

\end{proof}




\section{Proving deterministic RMSProp - the version with no added shift (Proof of Theorem \ref{thm:RMSProp2-proof})}
\label{sec:supp_rmsprop2}

\begin{proof}
From the $L-$smoothness condition on $f$ we have between consecutive iterates of the above algorithm, 

\begin{align}\label{step}
\nonumber f(\x_{t+1}) &\leq f(\x_t) - \alpha_t \langle \nabla f(\x_t),  V_t^{-\frac 1 2}\nabla f(\x_t)\rangle \\&+ \frac {L}{2} \alpha_t^2 \Vert V_t^{-\frac 1 2}\nabla f(\x_t) \Vert^2\\
\implies \langle \nabla f(\x_t), V_t^{-\frac 1 2} \nabla f(\x_t) \rangle &\leq \frac{1}{\alpha_t} \left(f(\x_t)-f(\x_{t+1}) \right) + \frac {L\alpha_t}{2} \Vert V_t^{-\frac 1 2}\nabla f(\x_t)  \Vert^2
\end{align}

Now the recursion for $\v_t$ can be solved to get, $\v_t = (1-\beta_2)\sum_{k=1}^t \beta_2^{t-k}\g_k^2$. Then
\begin{align*}
\norm{V_t^{\frac 1 2}} &\geq \frac{1}{\max_{ i \in \text{Support}(\v_t)} \sqrt{(\v_t)_i}} \\&= \frac{1}{\max_{ i \in \text{Support}(\v_t)} \sqrt{(1-\beta_2)\sum_{k=1}^t \beta_2^{t-k}(\g_k^2)_i}} \\&= \frac{1}{\max_{ i \in \text{Support}(\v_t)} \sigma \sqrt{(1-\beta_2)\sum_{k=1}^t \beta_2^{t-k}}} = \frac {1}{\sigma \sqrt{(1-\beta_2^t)}}
\end{align*}
Substituting this in a lowerbound on the LHS of equation \ref{step} we get,

\begin{align*}
\frac {1}{\sigma \sqrt{(1-\beta_2^t)}}\Vert \nabla f(\x_t) \Vert^2 &\leq  \langle \nabla f(\x_t), V_t^{-\frac 1 2} \nabla f(\x_t) \rangle \leq \frac{1}{\alpha_t} \left(f(\x_t)-f(\x_{t+1}) \right)
\\&+ \frac {L\alpha_t}{2} \Vert V_t^{-\frac 1 2}\nabla f(\x_t)  \Vert^2
\end{align*}

Summing the above we get,
\begin{align}\label{sum}
\sum_{t=1}^T\frac {1}{\sigma \sqrt{(1-\beta_2^t)}}\Vert \nabla f(\x_t) \Vert^2 \leq  \sum_{t=1}^T\frac{1}{\alpha_t} \left(f(\x_t)-f(\x_{t+1}) \right)
+ \sum_{t=1}^T\frac {L\alpha_t}{2} \Vert V_t^{-\frac 1 2}\nabla f(\x_t)  \Vert^2
\end{align}
Now we substitute $\alpha_t = \frac{\alpha}{\sqrt{t}}$ and invoke the definition of $B_\ell$ and $B_u$ to write the first term on the RHS of equation \ref{sum} as, 

\begin{align*}
\sum_{t=1}^T \frac {1}{\alpha_t}[f(\x_t) -f(\x_{t+1})] &=  \frac{ f(\x_1)}{\alpha} + \sum_{t=1}^T \left ( \frac{f(\x_{t+1})}{\alpha_{t+1}} - \frac{f(\x_{t+1})}{\alpha_{t}} \right ) - \frac{f(x_{T+1})}{\alpha_{T+1}} \\
&= \frac {f(\x_1)}{\alpha} - \frac{f(x_{T+1})}{\alpha_{T+1}} + \frac{1}{\alpha}\sum_{t=1}^T f(\x_{t+1}) (\sqrt{t+1} - \sqrt{t})\\
&\leq \frac{B_u}{\alpha} - \frac{B_\ell \sqrt{T+1}}{\alpha} + \frac {B_u}{\alpha} (\sqrt{T+1} -1)\\
\end{align*}

Now we bound the second term in the RHS of equation \ref{sum} as follows. Lets first define a function $P(T)$ as follows, $P(T) = \sum_{t=1}^T \alpha_t \Vert V_t^{-\frac 1 2} \nabla f(\x_t)\Vert^2$ and that gives us,  
\begin{align*}
P(T) - P(T-1) &= \alpha_T \sum_{i=1}^d \frac {\g_{T,i}^2}{\v_{T,i}} = \alpha_T \sum_{i=1}^d \frac { \g_{T,i} ^2 }{(1-\beta_2) \sum_{k=1}^T \beta_{2}^{T-k}\g^2_{k,i}}\\
&=  \frac {\alpha_T }{ (1-\beta_2)} \sum_{i=1}^d \frac { \g_{T,i} ^2 }{ \sum_{k=1}^T \beta_{2}^{T-k}\g^2_{k,i}} \leq \frac {d\alpha}{(1-\beta_2)\sqrt{T}}\\
\implies \sum_{t=2}^T [P(t) -P(t-1) ] &= P(T) - P(1) \\&\leq \frac {d\alpha}{(1-\beta_2)} \sum_{t=2}^T \frac {1}{\sqrt{t}} \leq \frac {d\alpha}{2(1-\beta_2)} (\sqrt{T} -2)\\
\implies P(T) &\leq P(1) + \frac {d\alpha}{2(1-\beta_2)} (\sqrt{T} -2) 
\end{align*} 

So substituting the above two bounds back into the RHS of the above inequality \ref{sum}and removing the factor of $\sqrt{1-\beta_2^T} <1$ from the numerator, we can define a point $\x_{result}$ as follows,

\begin{align*}
\norm{ \nabla f(\x_{result}) }^2 := \argmin_{t=1,..,T}  \Vert \nabla f(\x_t) \Vert^2 &\leq \frac {1}{T} \sum_{t=1}^T \Vert \nabla f(\x_t) \Vert^2 \\
&\leq \frac {\sigma }{T} \Bigg ( \frac {B_u}{\alpha} - \frac{B_l \sqrt{T+1}}{\alpha}+ \frac {B_u}{\alpha} (\sqrt{T+1} -1) \\
&+  \frac {L}{2} \left [ P(1) + \frac {d\alpha}{2(1-\beta_2)} (\sqrt{T} -2) \right ]  \Bigg ) 
\end{align*}

Thus it follows that for $T = O(\frac{1}{\epsilon^4})$ the algorithm \ref{RMSProp_TF} is guaranteed to have found at least one point $\x_{result}$ such that, $\norm{ \nabla f(\x_{result}) }^2 \leq \epsilon^2$
\end{proof}

\section{Hyperparameter Tuning}\label{supp_sec:exp_details}
Here we describe how we tune the hyper-parameters of each optimization algorithm. NAG has two hyper-parameters, the step size $\alpha$ and the momentum $\mu$. The main hyper-parameters for RMSProp are the step size $\alpha$, the decay parameter $\beta_2$ and the perturbation $\xi$. ADAM, in addition to the ones in RMSProp, also has a momentum parameter $\beta_1$. We vary the step-sizes of ADAM in the conventional way of $\alpha_t = \alpha \sqrt{1 - \beta_2^t} / (1 - \beta_1^t)$. 

For tuning the step size, we follow the same method used in \cite{wilson2017marginal}. We start out with a logarithmically-spaced grid of five step sizes. If the best performing parameter was at one of the extremes of the grid, we tried new grid points so that the best performing parameters were at one of the middle points in the grid. While it is computationally infeasible even with substantial resources to follow a similarly rigorous tuning process for all other hyper-parameters, we do tune over them somewhat as described below.

\paragraph{NAG}
The initial set of step sizes used for NAG were: $\{3\mathrm{e}{-3}, 1\mathrm{e}{-3}, 3\mathrm{e}{-4}, 1\mathrm{e}{-4}, 3\mathrm{e}{-5}\}$ We tune the momentum parameter over values $\mu \in \{0.9, 0.99\}$.

\paragraph{RMSProp}
The initial set of step sizes used were: $\{3\mathrm{e}{-4}, 1\mathrm{e}{-4}, 3\mathrm{e}{-5}, 1\mathrm{e}{-5}, 3\mathrm{e}{-6}\}$. We tune over $\beta_2 \in \{0.9, 0.99\}$. We set the perturbation value $\xi = 10^{-10}$, following the default values in TensorFlow, except for the experiments in Section \ref{sec:main_xi}. In Section \ref{sec:main_xi}, we show the effect on convergence and generalization properties of ADAM and RMSProp when changing this parameter $\xi$.


Note that ADAM and RMSProp uses an accumulator for keeping track of decayed squared gradient $\v_t$. For ADAM this is recommended to be initialized at $\v_0 = 0$. However, we found in the TensorFlow implementation of RMSProp that it sets $\v_0 = \mathbf{1}_d$. Instead of using this version of the algorithm, we used a modified version where we set $\v_0 = 0$. We typically found setting $\v_0 = 0$ to lead to faster convergence in our experiments.

\paragraph{ADAM}
The initial set of step sizes used were: $\{3\mathrm{e}{-4}, 1\mathrm{e}{-4}, 3\mathrm{e}{-5}, 1\mathrm{e}{-5}, 3\mathrm{e}{-6}\}$. 
For ADAM, we tune over $\beta_1$ values of $\{0.9, 0.99\}$. For ADAM, We set $\beta_2 = 0.999$ for all our experiments as is set as the default in TensorFlow. 
Unless otherwise specified we use for the perturbation value $\xi = 10^{-8}$ for ADAM, following the default values in TensorFlow. 

Contrary to what is the often used values of $\beta_1$ for ADAM (usually set to 0.9), we found that we often got better results on the autoencoder problem when setting $\beta_1 = 0.99$.

\newpage 
\section{Effect of the $\xi$ parameter on adaptive gradient algorithms}

In Figure \ref{xi_fixed}, we show the same effect of changing $\xi$ as in Section \ref{sec:main_xi} 
on a 1 hidden layer network of 1000 nodes, while keeping all other hyper parameters fixed (such as learning rate, $\beta_1$, $\beta_2$). These other hyper-parameter values were fixed at the best values of these parameters for the default values of $\xi$, i.e., $\xi = 10^{-10}$ for RMSProp and $\xi = 10^{-8}$ for ADAM.

\begin{figure}[h!]
\centering
\begin{subfigure}[t]{0.32\textwidth}
\includegraphics[width=\textwidth]{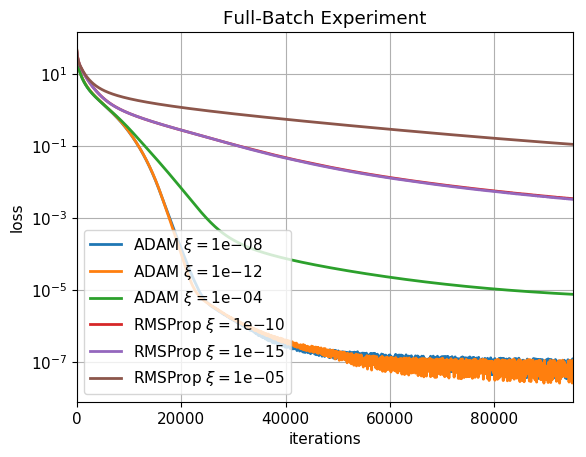}
\caption{Loss on training set}
\end{subfigure}
\begin{subfigure}[t]{0.32\textwidth}
\includegraphics[width=\textwidth]{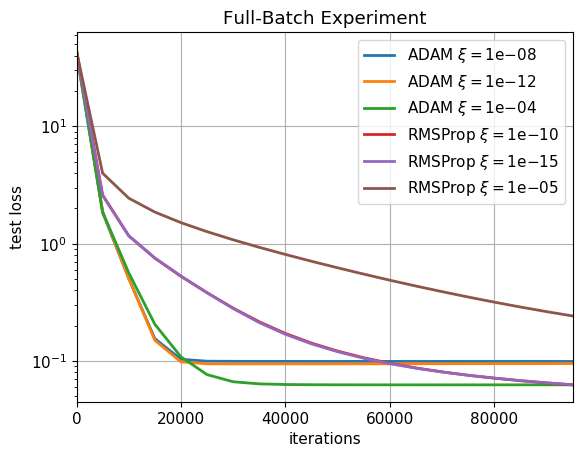}
\caption{Loss on test set}
\end{subfigure}
\begin{subfigure}[t]{0.32\textwidth}
\includegraphics[width=\textwidth]{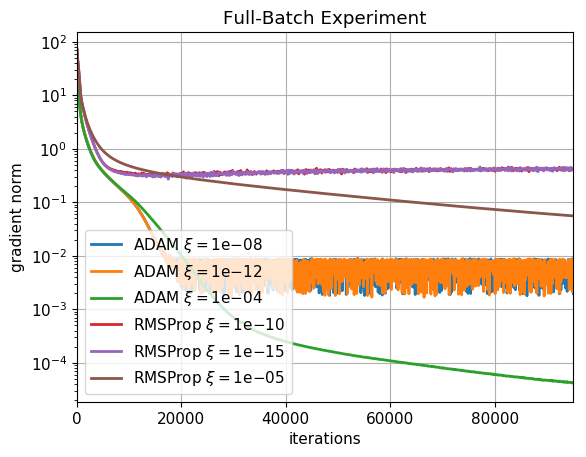}
\caption{Gradient norm on training set}
\end{subfigure}
\caption{Fixed parameters with changing $\xi$ values. 1 hidden layer network of 1000 nodes}
\label{xi_fixed}
\end{figure}

\newpage 
\section{Additional Experiments}

\subsection{Additional full-batch experiments on $22 \times 22$ sized images}
\label{sec:supp_fullbatch}

In Figures \ref{full_batch_trainloss_imsize_22}, \ref{full_batch_testloss_imsize_22} and \ref{full_batch_normgrad_imsize_22}, we show training loss, test loss and gradient norm results for a variety of additional network architectures. Across almost all network architectures, our main results remain consistent. ADAM with $\beta_1 = 0.99$ consistently reaches lower training loss values as well as better generalization than NAG.
\vspace{-3mm}
\begin{figure}[h!]
\centering
\begin{subfigure}[t]{0.32\textwidth}
\includegraphics[width=\textwidth]{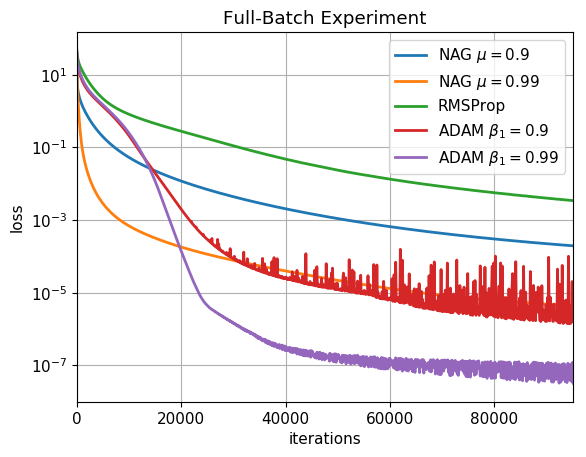}
\vspace{-6mm}\caption{1 hidden layer; 1000 nodes}
\end{subfigure}
\begin{subfigure}[t]{0.32\textwidth}
\includegraphics[width=\textwidth]{adam/1000_1000_loss.png}
\vspace{-6mm}\caption{3 hidden layers; 1000 nodes each}
\end{subfigure}
\begin{subfigure}[t]{0.32\textwidth}
\includegraphics[width=\textwidth]{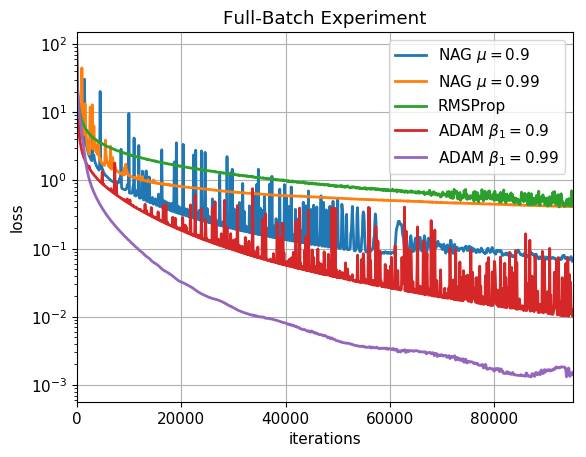}
\vspace{-6mm}\caption{5 hidden layers; 1000 nodes each}
\end{subfigure}
\begin{subfigure}[t]{0.32\textwidth}
\includegraphics[width=\textwidth]{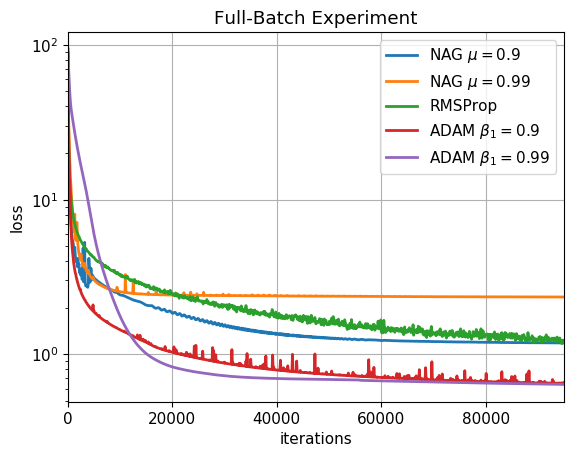}
\vspace{-6mm}\caption{3 hidden layers; 300 nodes}
\end{subfigure}
\begin{subfigure}[t]{0.32\textwidth}
\includegraphics[width=\textwidth]{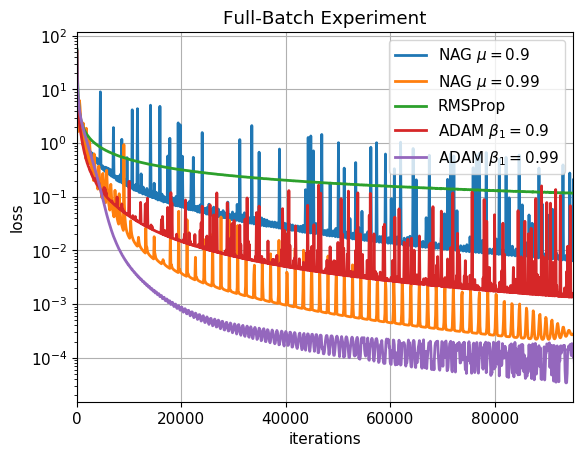}
\vspace{-6mm}\caption{3 hidden layer; 3000 nodes}
\end{subfigure}
\begin{subfigure}[t]{0.32\textwidth}
\includegraphics[width=\textwidth]{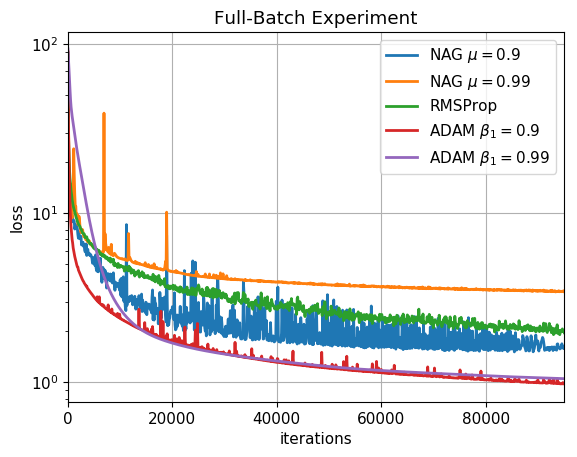}
\vspace{-6mm}\caption{5 hidden layer; 300 nodes}
\end{subfigure}
\vspace{-2mm}
\caption{Loss on training set; Input image size $22\times 22$}
\label{full_batch_trainloss_imsize_22}
\end{figure}

\begin{figure}[h!]
\centering
\begin{subfigure}[t]{0.32\textwidth}
\includegraphics[width=\textwidth]{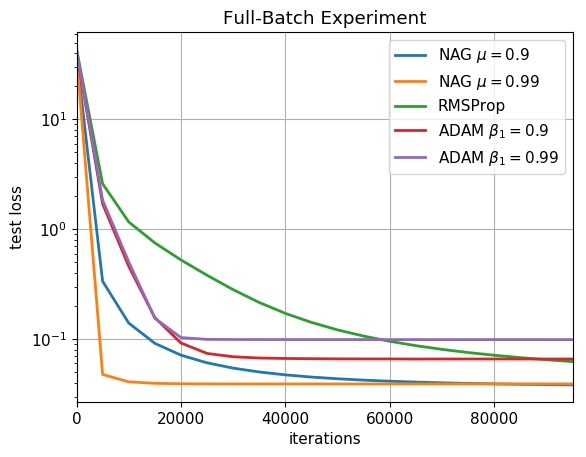}
\caption{1 hidden layer; 1000 nodes}
\end{subfigure}
\begin{subfigure}[t]{0.32\textwidth}
\includegraphics[width=\textwidth]{adam/1000_1000_test.png}
\caption{3 hidden layers; 1000 nodes each}
\end{subfigure}
\begin{subfigure}[t]{0.32\textwidth}
\includegraphics[width=\textwidth]{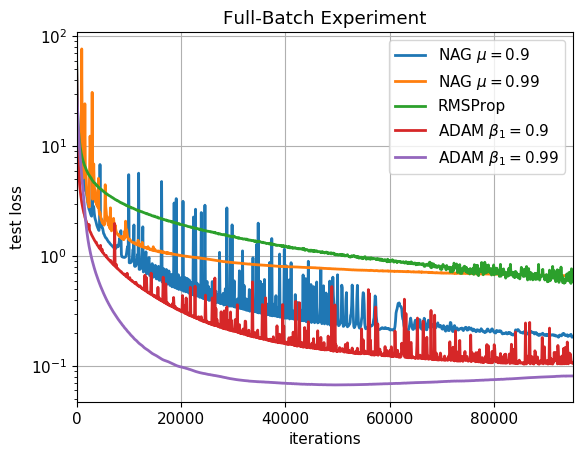}
\caption{5 hidden layers; 1000 nodes each}
\end{subfigure}
\begin{subfigure}[t]{0.32\textwidth}
\includegraphics[width=\textwidth]{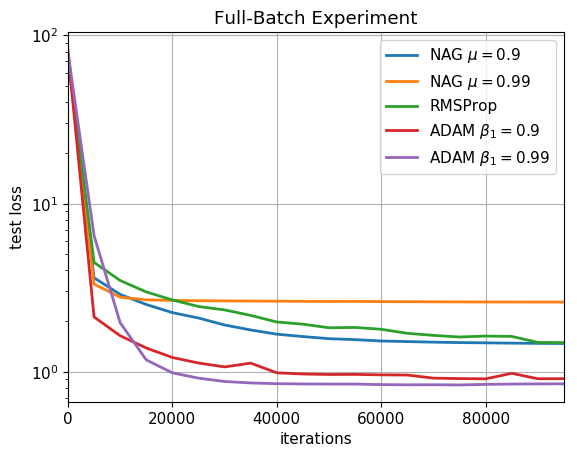}
\caption{3 hidden layers; 300 nodes}
\end{subfigure}
\begin{subfigure}[t]{0.32\textwidth}
\includegraphics[width=\textwidth]{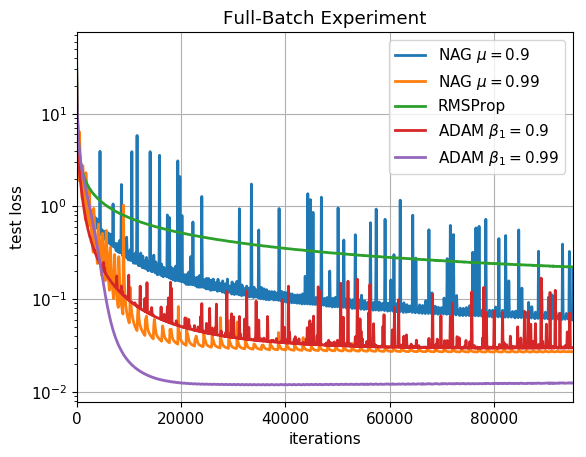}
\caption{3 hidden layer; 3000 nodes}
\end{subfigure}
\begin{subfigure}[t]{0.32\textwidth}
\includegraphics[width=\textwidth]{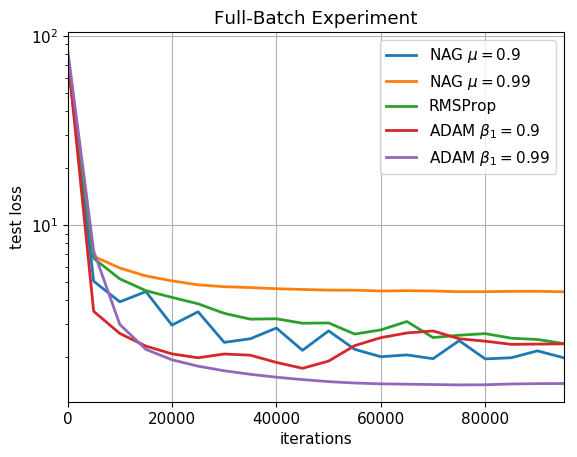}
\caption{5 hidden layer; 300 nodes}
\end{subfigure}
\caption{Loss on test set; Input image size $22\times 22$}
\label{full_batch_testloss_imsize_22}
\end{figure}

\begin{figure}[h!]
\centering
\begin{subfigure}[t]{0.32\textwidth}
\includegraphics[width=\textwidth]{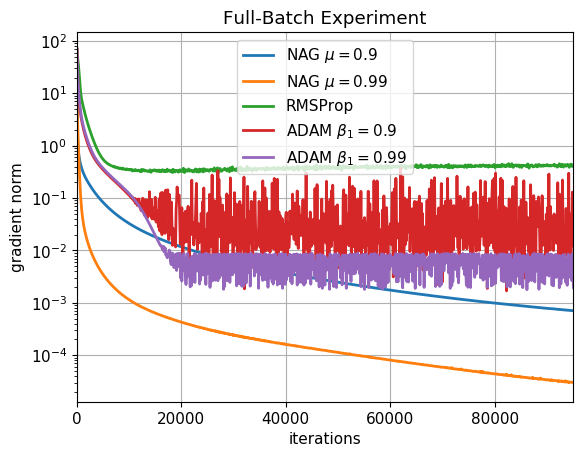}
\caption{1 hidden layer; 1000 nodes}
\end{subfigure}
\begin{subfigure}[t]{0.32\textwidth}
\includegraphics[width=\textwidth]{adam/1000_1000_grad_smooth.png}
\caption{3 hidden layers; 1000 nodes each}
\end{subfigure}
\begin{subfigure}[t]{0.32\textwidth}
\includegraphics[width=\textwidth]{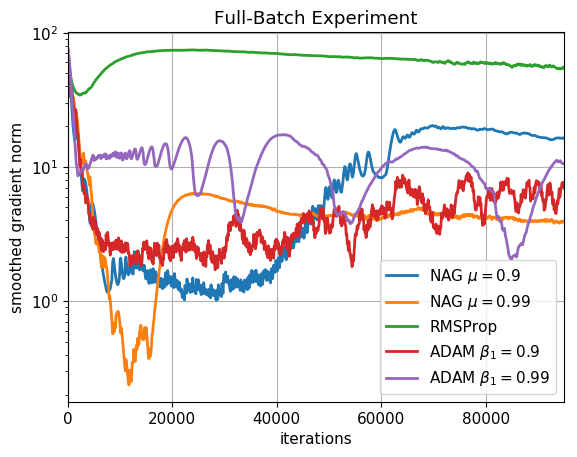}
\caption{5 hidden layers; 1000 nodes each}
\end{subfigure}
\begin{subfigure}[t]{0.32\textwidth}
\includegraphics[width=\textwidth]{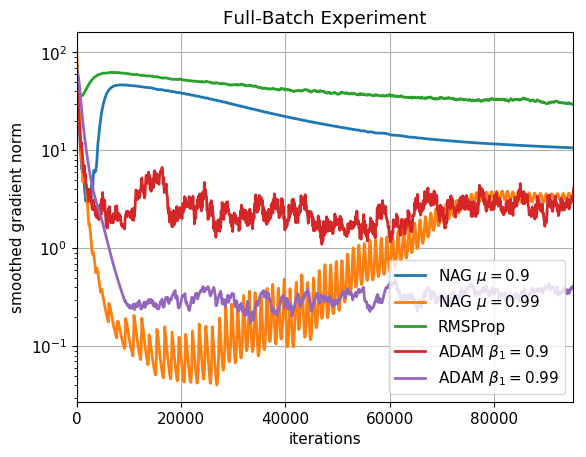}
\caption{3 hidden layers; 300 nodes}
\end{subfigure}
\begin{subfigure}[t]{0.32\textwidth}
\includegraphics[width=\textwidth]{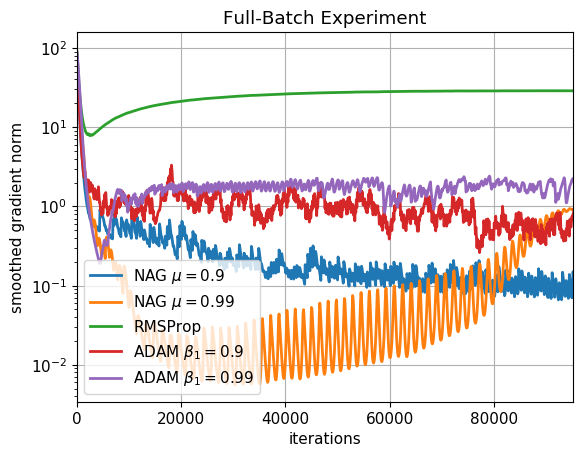}
\caption{3 hidden layer; 3000 nodes}
\end{subfigure}
\begin{subfigure}[t]{0.32\textwidth}
\includegraphics[width=\textwidth]{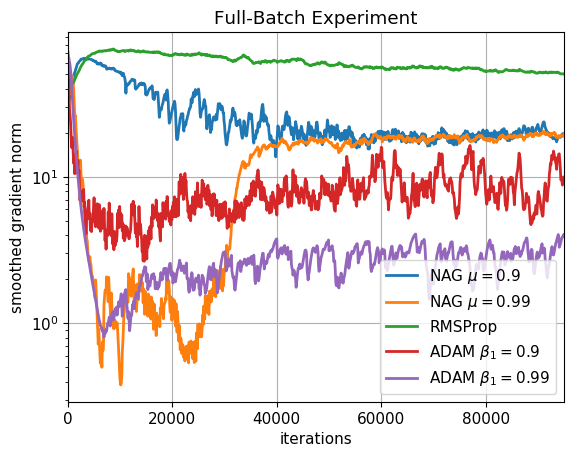}
\caption{5 hidden layer; 300 nodes}
\end{subfigure}
\caption{Norm of gradient on training set; Input image size $22\times 22$}
\label{full_batch_normgrad_imsize_22}
\end{figure}

\clearpage
\subsection{Are the full-batch results consistent across different input dimensions?}
\label{sec:supp_varydim}

To test whether our conclusions are consistent across different input dimensions, we do two experiments where we resize the $22\times 22$ MNIST image to $17\times 17$ and to $12\times 12$. Resizing is done using TensorFlow's 
\texttt{tf.image.resize\_images} 
method, which uses bilinear interpolation.

\subsubsection{Input images of size $17 \times 17$}

Figure \ref{dim_17} shows results on input images of size $17\times 17$ on a 3 layer network with 1000 hidden nodes in each layer. Our main results extend to this input dimension, where we see ADAM with $\beta_1 = 0.99$ both converging the fastest as well as generalizing the best, while NAG does better than ADAM with $\beta_1 = 0.9$.

\begin{figure}[h!]
\centering
\begin{subfigure}[t]{0.32\textwidth}
\includegraphics[width=\textwidth]{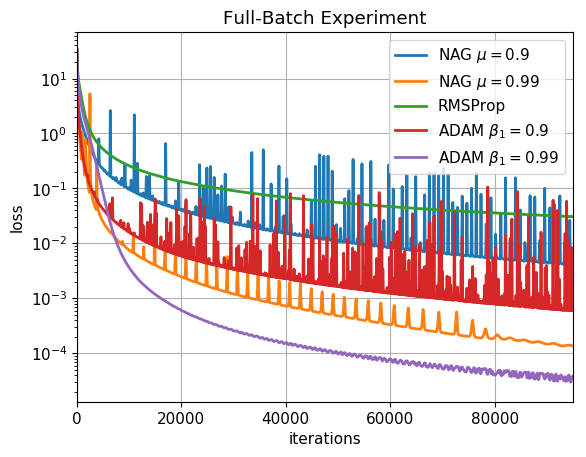}
\caption{Training loss}
\end{subfigure}
\begin{subfigure}[t]{0.32\textwidth}
\includegraphics[width=\textwidth]{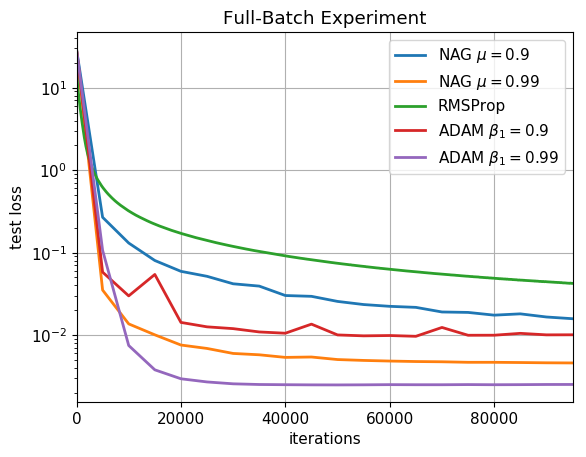}
\caption{Test loss}
\end{subfigure}
\begin{subfigure}[t]{0.32\textwidth}
\includegraphics[width=\textwidth]{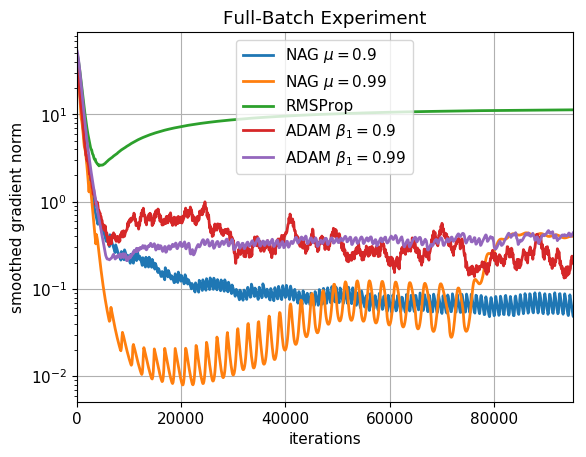}
\caption{Gradient norm}
\end{subfigure}
\caption{Full-batch experiments with input image size $17\times 17$}
\label{dim_17}
\end{figure}

\subsubsection{Input images of size $12 \times 12$}

Figure \ref{dim_12} shows results on input images of size $12\times 12$ on a 3 layer network with 1000 hidden nodes in each layer. Our main results extend to this input dimension as well. ADAM with $\beta_1 = 0.99$ converges the fastest as well as generalizes the best, while NAG does better than ADAM with $\beta_1 = 0.9$.

\begin{figure}[h!]
\centering
\begin{subfigure}[t]{0.32\textwidth}
\includegraphics[width=\textwidth]{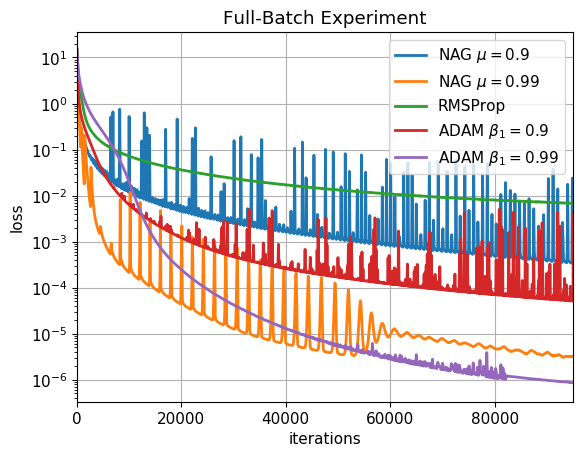}
\caption{Training loss}
\end{subfigure}
\begin{subfigure}[t]{0.32\textwidth}
\includegraphics[width=\textwidth]{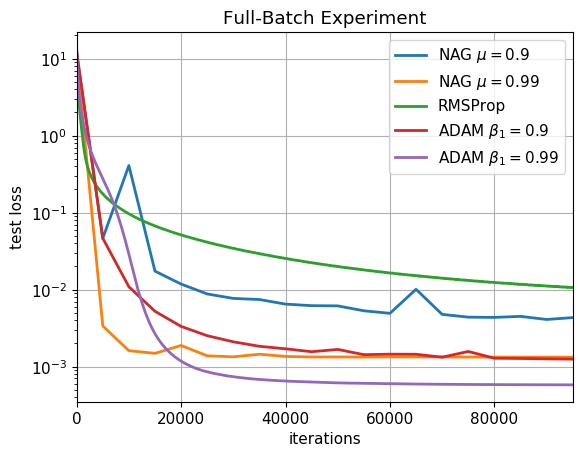}
\caption{Test loss}
\end{subfigure}
\begin{subfigure}[t]{0.32\textwidth}
\includegraphics[width=\textwidth]{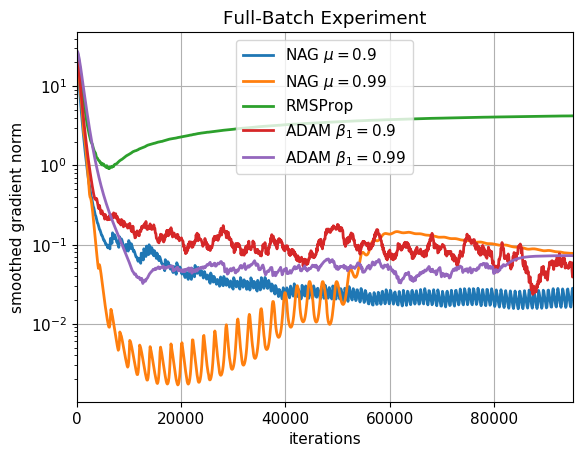}
\caption{Gradient norm}
\end{subfigure}
\caption{Full-batch experiments with input image size $12\times 12$}
\label{dim_12}
\end{figure}

\clearpage

\subsection{Additional mini-batch experiments on $22 \times 22$ sized images}
\label{sec:supp_minibatch}

In Figure \ref{minibatch_imsize_22}, we present results on additional neural net architectures on mini-batches of size 100 with an input dimension of $22 \times 22$. We see that most of our full-batch results extend to the mini-batch case.

\begin{figure}[h!]
\centering
\begin{subfigure}[t]{0.32\textwidth}
\includegraphics[width=\textwidth]{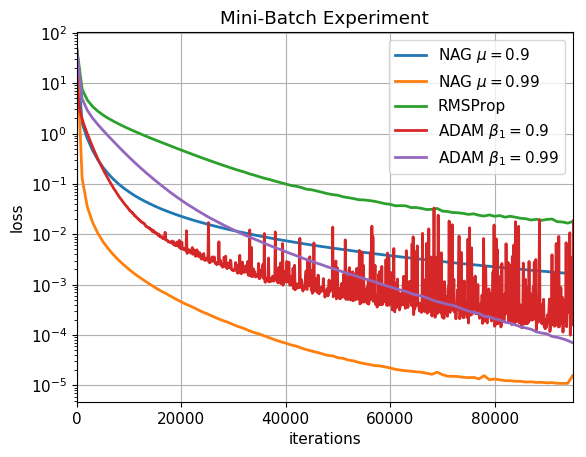}
\caption{1 hidden layer; 1000 nodes}
\end{subfigure}
\begin{subfigure}[t]{0.32\textwidth}
\includegraphics[width=\textwidth]{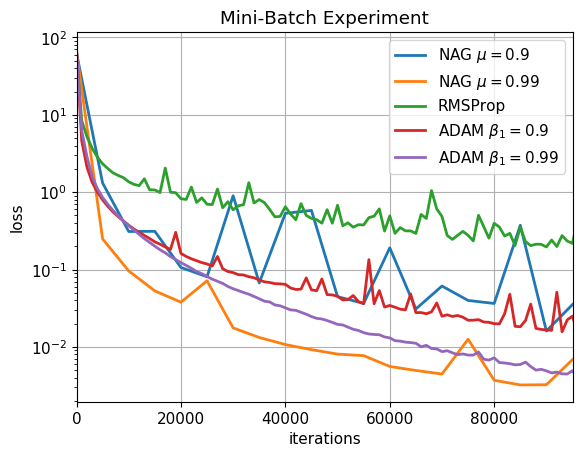}
\caption{3 hidden layers; 1000 nodes each}
\end{subfigure}
\begin{subfigure}[t]{0.32\textwidth}
\includegraphics[width=\textwidth]{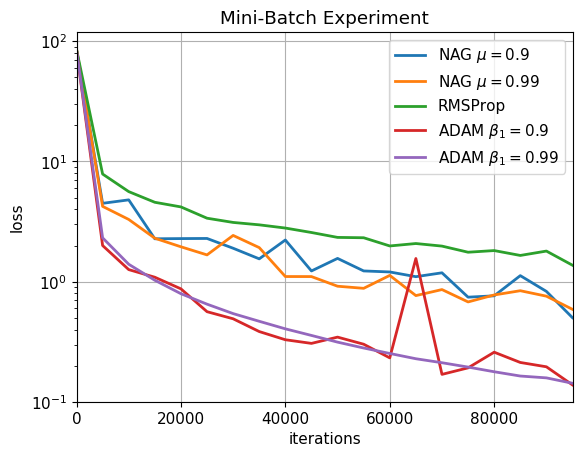}
\caption{9 hidden layers; 1000 nodes each}
\end{subfigure}
\begin{subfigure}[t]{0.32\textwidth}
\includegraphics[width=\textwidth]{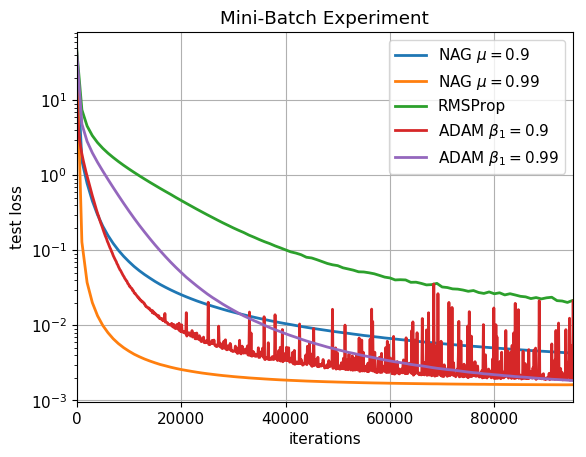}
\caption{1 hidden layer; 1000 nodes}
\end{subfigure}
\begin{subfigure}[t]{0.32\textwidth}
\includegraphics[width=\textwidth]{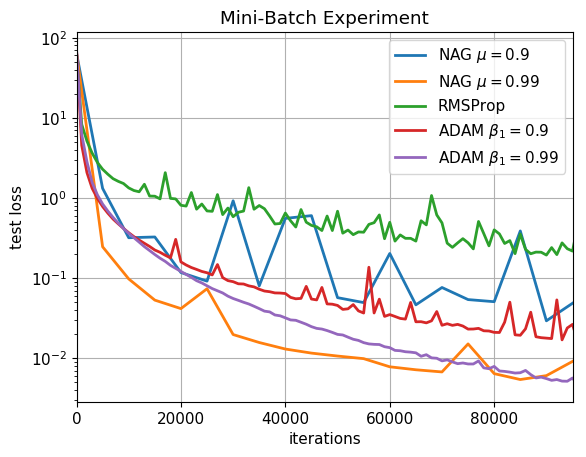}
\caption{3 hidden layers; 1000 nodes each}
\end{subfigure}
\begin{subfigure}[t]{0.32\textwidth}
\includegraphics[width=\textwidth]{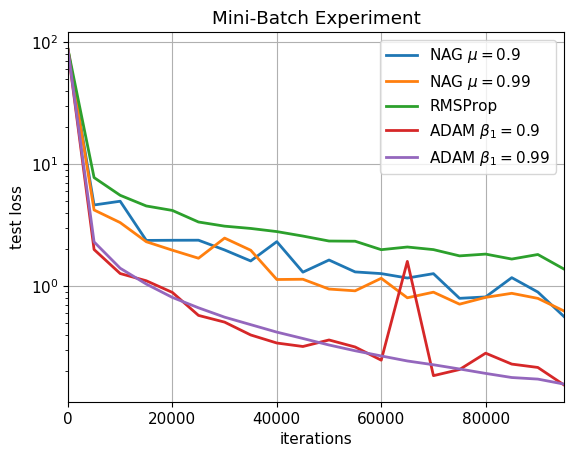}
\caption{9 hidden layers; 1000 nodes each}
\end{subfigure}
\begin{subfigure}[t]{0.32\textwidth}
\includegraphics[width=\textwidth]{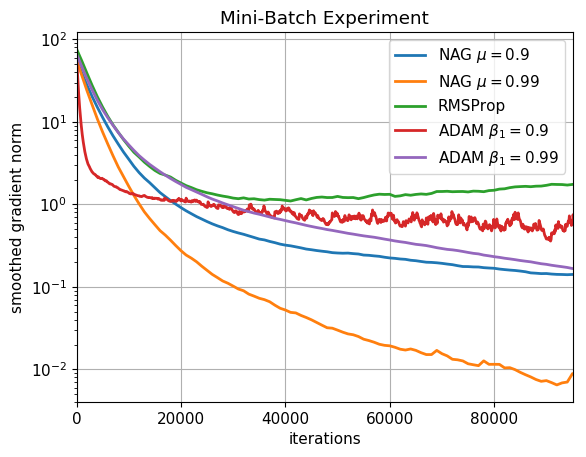}
\caption{1 hidden layer; 1000 nodes}
\end{subfigure}
\begin{subfigure}[t]{0.32\textwidth}
\includegraphics[width=\textwidth]{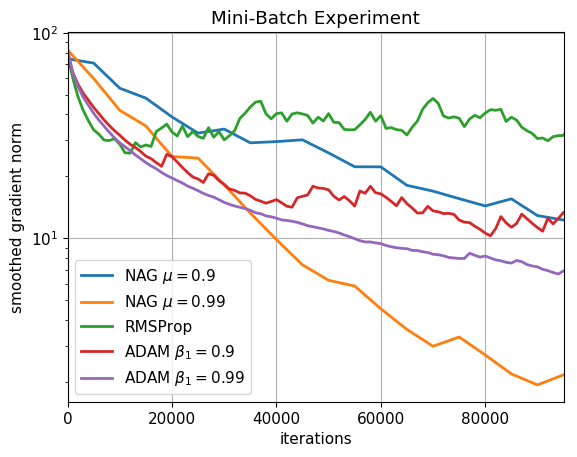}
\caption{3 hidden layers; 1000 nodes each}
\end{subfigure}
\begin{subfigure}[t]{0.32\textwidth}
\includegraphics[width=\textwidth]{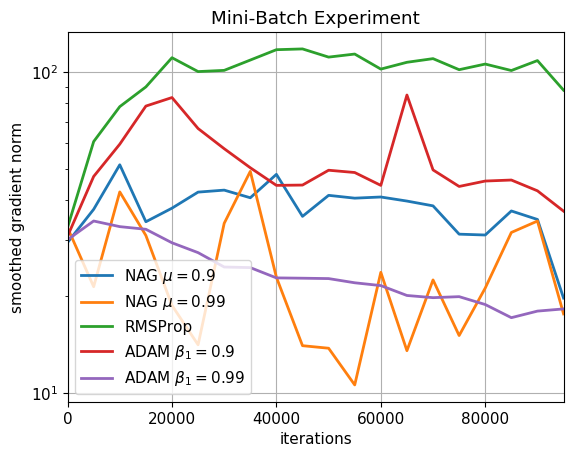}
\caption{9 hidden layers; 1000 nodes each}
\end{subfigure}
\caption{Experiments on various networks with mini-batch size 100 on full MNIST dataset with input image size $22\times 22$. First row shows the loss on the full training set, middle row shows the loss on the test set, and bottom row shows the norm of the gradient on the training set.}
\label{minibatch_imsize_22}
\end{figure}

\end{subappendices}

\chapter{\vspace{10pt} PAC-Bayesian Risk Bounds for Neural Nets}\label{chapPAC}

\section{Introduction}
At the end of the thesis we finally arrive to explore what is possibly the deepest and the hardest question about neural nets and that is to understand their risk function. A long standing open question in deep-learning is to be able to theoretically explain as to when and why do neural nets which are massively over-parameterized happen to also minimize the risk even when they fit the training data arbitrarily accurately. Attempts to explain this have led to obtaining of risk bounds which do not scale with the number of parameters being trained on, \citep{bartlett1998sample,golowich2018size,harvey2017nearly}. Recently it has been increasingly realized that good risk bounds can be obtained by making the bounds more sensitive to the training algorithm as well as the training data,\citep{arora2019fine}. 


The range of available methods to bound risk or generalization error have been beautifully reviewed in \cite{audibert2007combining}. Here the authors have grouped the techniques into primarily four categories, $(1)$ ``Supremum Bounds" (like generic chaining, Dudley integral, Rademacher complexity), $(2)$ ```Variance Localized Bounds", $(3)$ ``Data-Dependent Bounds" and $(4)$ ``Algorithm Dependent Complexity". The last category includes PAC-Bayes bounds which have resurfaced as a prominent candidate for a framework to understand risk of neural nets.  

Over the years PAC-Bayesian risk bounds have been formulated in many different forms, \\ \citep{hinton1993keeping,mcallester1999pac, langford2001bounds, mcallester2003simplified}. In the last couple of years, works like \cite{dziugaite2017computing}, \cite{dziugaite2018entropy}, \cite{dziugaite2018data} and \cite{zhou2018non} have shown the power of the PAC-Bayesian form of analysis of risk of neural nets. To the best of our knowledge ``computational" bounds as demonstrated in the above reference are the first examples of non-vacuous/non-trivial upperbounds for risks of neural nets of any kind. 

The above bounds are ``computational" in the sense that they are obtained as outputs of an algorithmic search over hyperparameters on which the posterior distribution used in the bound depends on.  These experiments strongly motivate the current work to search for stronger theoretical basis towards explaining the power of PAC-Bayesian risk bounds in explaining the generalization ability of neural nets. We make progress by identifying certain geometrical properties of the process of training nets which can be leveraged into getting better risk bounds.

\subsection{A summary of our contributions}\label{intro:theory} 
In works like \cite{nagarajan2018deterministic}, \cite{nagarajan2019generalization},  it has been previously understood that risk bound on nets get better if they can appropriately utilize the information about the distance of the trained net from initialization. In this work we take a more careful look at this idea. We can decompose the distance from initialization into two independent quantities (a) a non-compact part, the change in the norm of the vector of weights of the net (i.e the sum of Frobenius norms of the layer matrices for a net without bias weights) and (b) a compact part, the angular deflection of this vector of weights from initialization to the end of training. 
Previous PAC-Bayes bounds have used data dependent priors to track the geometric mean of the spectral norms of the layer matrices (denoted as $\beta$) and have thus tracked the first parameter above. In this work we propose a mechanism of choosing priors in a data-dependent way from a two-indexed finite set, tracking both the quantities specified above. The compact angle is tracked by the index called $\lambda$ in our two-indexed set of priors as specified in Definition \ref{ourgrid}.

Our second key innovation is that we show how in the PAC-Bayesian framework one can leverage more out of the angle parameter by simultaneously training a cluster of nets. 
In our risk bound in Theorem \ref{maincor} {\bf (the main theorem)} we imagine starting from a net $f_\B$ to get to the trained net $f_\A$ - the bold faced letter in the subscript of $f$ will denote the (very high dimensional) vector of weights of the nets.

But alongside training $f_\B$ to obtain $f_\A$, we also obtain a set $\{f_{\A_i}\}_{i=1,\ldots,k_1}$ of trained nets  - of the same architecture as $f_\A$ and obtained using the same data set and using the same algorithm as was used to obtain $f_\A$. This cluster of $k_1$ nets are obtained by doing training starting from  multiple instances of weights initialized at different weight vectors, $\{ \B_{\lambda^*,j} \}_{j=1,\ldots,k_1}$, of the {\it same norm} as $\B$ but within a cone around $\B$ whose half-angle $\lambda^*$ is determined in a data-dependent way. The angle index $\lambda$ of the set of priors, that we introduced previously, covers this choice of the conical half-angle. 

Because of this use of clusters, compared to previous bounds our dependency on the distance from initialization is also more intricate.  Our risk bound as derived in Theorem \ref{maincor} can be seen to be scaling with an effective notion of distance between any one of the $\A_i$ and the {\it initial cluster} of weights around $\B$ the $\{ \B_{\lambda^*,j} \}_{j=1,\ldots,k_1}$. The bound has the flexibility that it will allow us to choose the $\A_i$ which has the smallest value of this effective distance and thus we are able to be more sensitive to the average behaviour. That is, for $h$ being the width of the depth $d$ nets being used, if the $i^{th}$ net of the {\it final cluster} $\{ \A_j\}_{j=1,\ldots,k_1}$ is closest to the {\it initial cluster} $\{ \B_{\lambda^*,j} \}_{j=1,\ldots,k_1}$ then a crude ``order" estimate of the risk bound on the stochastic net centered at $f_\A$, that is given by Theorem \ref{maincor} can be written as,  

\begin{align*}
\bigO\Bigg(\frac{\sqrt{h\log\left(\frac{2dh}{\delta}\right)}}{\sqrt{\text{training set size}}}\times & d\cdot B\Big(\prod_{\ell=1}^{d}\norm{\text{A}_{i,\ell}}_{2}\Big)^{1-\frac{1}{d}}\times\frac{\text{inter-cluster distance between }\text{B}_{\lambda^{*},j}s~\&~\text{A}_{j}s}{\gamma}\Bigg)
\end{align*}

In above, $B$ is a bound on the input vectors at training, the $\ell^{th}-$layer matrix corresponding to $\A_i$ is denoted as $\A_{i,\ell}$, $\gamma$ is the margin value of which the margin-loss is being evaluated and the failure probability for the bound to hold is ${\cal O}(\delta)$. The exact formula for the bound given in Theorem \ref{maincor} {\bf (the main theorem)} makes it explicit that we have effectively built into the theory more data-dependent (and hence tunable) parameters which help us improve over existing bounds in multiple conceptual ways.

We would like to emphasize that our ability to exploit the cluster construction is crucially hinged on us being able to prove novel data-dependent noise resilience theorems for neural nets as given in Theorem \ref{mainNoise}. This theorem is potentially of independent interest and forms the technical core of our theoretical contribution. To the best of our knowledge this is the first such construction of a multi-parameter family of noise distributions on the weights of a neural net with guarantees of stability. In other words, if the net's weights are sampled from any of these noise distributions constructed in the theorem then w.h.p the output on the given data-set is guaranteed to not deviate too much from the given neural function on that architecture.

\paragraph{Summary of the experimental evidence in favour of our bounds} We choose to compare our results against the bounds from \cite{neyshabur2017pac} which we have in turn restated in subsection \ref{review} with more accurate tracking of the various parameters therein. There are two ways to see why they are the appropriate baseline for comparison. (a) {\it Firstly} since our primary goal is to advance mathematical techniques to get better PAC-Bayesian risk bounds on nets we want to compare to other {\it theoretical} bounds in the same framework. The result from \cite{neyshabur2017pac} are known to be the current state-of-the-art PAC-Bayesian risk bounds on nets. (b) {\it Secondly} to the best of our knowledge, for the range of depths of neural architectures that we are experimenting on, the bounds in \cite{neyshabur2017pac} are the state-of-the-art among all risk bounds (PAC-Bayesian or otherwise) on nets as has been alluded to in \cite{nagarajan2019generalization}. 

Further when two different theoretical expressions for risk bounds on nets depend on different sets of parameters of the neural net and its training process, we have to rely on empirical comparison. We recall that under similar situation this was also the adopted method of comparison to baselines for the state-of-the-art compression techniques of risk bounds in \cite{arora2018stronger}. 

In the experiments in Section \ref{sec:pac_exp} we will show multiple instances of nets trained over synthetic data and CIFAR-10 where we supersede the existing state-of-the-art in theoretical PAC-Bayesian bounds of \cite{neyshabur2017pac}. On these two datasets we probe very different regimes of neural net training in terms of the typical values of the angular deflection seen when obtaining the trained weights $\A$ from the initial weights $\B$. In both these situations our bounds are lower than those from, \cite{neyshabur2017pac} as we increase both the depth and the width. {\it Since we plot the bounds in the log scale for comparison we can conclude from the figures that not only do we have lower/tighter upperbounds we indeed also have better ``rates" of dependencies on the architectural parameters.}

In the experiments we also demonstrate different properties of the path travelled by the neural net's weight vector during training and these we report in Section \ref{sec:geom}. For instance we observe that the $2-$norm of the weight vector of the net increases during training by a multiplicative factor which varies very little across all the experiments. The factor is between $1$ and $3$ and is fairly stable for one order of magnitude change in depths. 


\subsection{Reformulation of the PAC-Bayesian risk bound on neural nets from, \cite{neyshabur2017pac}}\label{review}

We start from Theorem 31.1 in \cite{shalev2014understanding} on PAC-Bayesian risk bounds which we re-state below as Theorem~\ref{pacmain}. 

\begin{definition}
Let ${\cal H}$ be a hypothesis class, let $h \in \mathcal{H}$, let ${\cal D}$ be a distribution on an instance space $Z$, let $\ell : {\cal H} \times Z \rightarrow [0,1]$ be a loss function and let $S$ be a finite subset of $Z$. Let $m$ denote the size of $S$, i.e. the training-set size. $z \sim \mathcal{D}$ denotes sampling $z$ from $\mathcal{D}$ and with slight abuse of notation $z \sim S$ denotes sampling $z$ from a uniform distribution over $S$ whenever $S$ is a finite set. Further we define the expected and empirical risks for $h$ as
\[ L(h) := \E_{z \sim {\cal D}} [ \ell(h,z) ]
\text{ and } \hat{L}(h) := \E_{z \sim S} [\ell (h,z) ] \]
\end{definition}

\begin{theorem}{\bf (PAC-Bayesian Bound On Risk)}\label{pacmain}
Consider being given ${\cal D}$, ${\cal H}$ and $\ell$ as defined above. Let ${\cal H}$ also be equipped with the structure of a probability space. Let $P, Q$ be two distributions over ${\cal H}$ called the ``prior'' and the ``posterior'' distribution respectively. Let $S \sim {\cal D}^m(Z)$ and then for every $\delta \in (0,1)$ we have the following guarantee:

\begin{equation*}
\resizebox{0.5\textwidth}{!}{$\mathbb{P}_{S} \left [ \forall Q \text{ } \E_{h \sim Q} [ L(h) ] \leq  \E_{h \sim Q} [ \hat{L}(h) ] {+} \sqrt{\frac{\text{KL}\left (Q \vert \vert P \right ) + \log \frac m \delta }{2(m-1)}}  \right ] \geq 1 {-} \delta$}
\end{equation*}
\end{theorem}

The above theorem shows that given a finite sample $S$ from $Z$, we can choose the posterior distribution $Q$ as a function of $S$ and the above bound on generalization error is still guaranteed to hold w.h.p. The mechanism of choosing this $Q$ in such a data-dependent way is made critical by the trade-offs between keeping the expected empirical risk of $Q$ low and the {\rm KL} divergence between $P$ and $Q$ low.  



The above theorem applies to a wide class of loss functions, but for getting risk bounds specific to neural nets, hence forth we will focus on the following setup of classification loss.

\begin{definition}[\textbf{Margin Risk of a Multiclass Classifier}]\label{loss}
Define $\chi \subseteq \R^n$ to be the input space. Let $k \ge 2$ be the number of classes and set $Z = \chi \times \{1, \ldots, k\}$ for the rest of the chapter. Let $f_\w: \chi \to \R^k$ be a $k-$class classifier parameterized by the weight vector $\w$ and by ``$f(\x)[y]$" we shall mean the $y^{th}$ coordinate of the output of $f$ when evaluated on $\x$. Let $\gamma >0$ be the ``margin" parameter and then the ``$\gamma-$Margin Risk'' of $f$ is
\[L_{\gamma}(f) := \mathbb{P}_{(\x,y)\sim {\cal D}} \Big [ f(\x)[y] \leq \gamma + \max_{i \neq y} f(\x) [i] \Big ].\]
Analogously $\hat{L}_\gamma(f)$ denotes the ``$\gamma-$Empirical Margin Risk'' of $f$ computed on a finite sample $S \subset Z$. We will use $m := \vert S \vert$.  
\end{definition}

Given this we can now present Theorem~\ref{precisepac} below which is a slight variant of Lemma $1$ of \cite{neyshabur2017pac} and for completeness we give its proof in Appendix \ref{proof:precisepac}. 

\begin{theorem}[\textbf{A special case of the PAC-Bayesian bounds for the margin loss}]\label{precisepac}



We follow the notation from Definition \ref{loss}. Let $\{ f_\w \mid \w \in {\rm W} \}$ denote a hypothesis class of $k-$class classifiers, where ${\rm W}$ is a space of parameters. Let $P$ be any distribution (the "data-independent prior") on the space {\rm W}. Then for any $\gamma > 0, \w \in {\rm W}$, and finite sample $S \subset Z$, define the family $D_{\gamma,\w,S}$ of distributions on {\rm W} such that for any $\mu \in D_{\gamma,\w,S}$, we have, 

\[ P_{\w' \sim \mu} \left [ \max_{(\x,\y)\in S} \norm{ f_{\w'}(\x) - f_\w(\x) }_\infty < \frac{\gamma}{4} \right ] \geq \frac{1}{2}  \]

Then for any $\gamma >0$ and $\delta \in [0,1]$, the following holds 

\begin{align}
\nonumber &\mathbb{P}_{S} \Bigg [ \forall \w \text{ and } \mu \in D_{\gamma,\w,S}  ~\exists\text{  distribution }\tilde{\mu} \text{ on } {\rm W} ~\text{ s.t. }\\ 
&\mathbb{E}_{\tilde{\w} \sim \tilde{\mu}} [ L_{0} (f_{\tilde{\w}}) ] \leq \hat{L}_{\frac \gamma 2} (f_{\w}) + \sqrt{\frac{\text{KL}(\mu \vert \vert P) + \log \frac{3m}{\delta}}{m-1}} \Bigg ] \geq 1 - \delta \label{eq:12}
\end{align}


\end{theorem}

The proof of the above in Appendix \ref{proof:precisepac} gives an explicit expression of the distribution $\tilde{\mu}$ in terms of $\mu$ and $\w$. 
As is usual in practice the PAC-Bayesian bound in Theorem~\ref{precisepac} will typically be used on a predictor whose weights have been obtained after training on a data-set $S$. Now we give a restatement of the risk bound on neural nets that was presented in \cite{neyshabur2017pac} and towards that we need the following definition,

\begin{definition}[\textbf{Multiclass Neural-Network classifier}]\label{Bdef}
Define $f_\A: \R^n \to \R^k$ to be a depth-$d$ neural-network with maximum width $h$, whose $\ell^\text{th}$ layer has weight matrix $\A_\ell$.\footnote{This network does not have any bias weights.} The first $d-1$ layers of $f_\A$ use the \textsc{Relu} non-linear activation.\footnote{{In general any 1-Lipschitz activation will do.}} $\A := [\mathrm{vec}(\A_1); \ldots; \mathrm{vec}(\A_d)]$ is the vector of parameters formed by concatenating vectorized layer matrices and each coordinate of $\A$ is a distinct trainable weight in the net.
Let ${\cal N}_{(\A,\sigma^2)}$ denote the isotropic multivariate Gaussian probability distribution with mean $\A$ and variance $\sigma^2 \mathbf{I}$
Define $\beta(\A) = \big ( \prod\nolimits_{\ell=1}\nolimits^d \norm{A_\ell}_2 \big )^{1/d}$. We will omit the argument $\A$ whenever the neural-network under consideration is clear from the context. Clearly $\beta^d$ upper bounds the Lipschitz-constant for $f_\A$.
\end{definition}

Using the definitions above, we now present Theorem~\ref{NBSTheorem} where we give a reformulation of the ``spectrally-normalized margin bound'' originally given in \cite{neyshabur2017pac}. 

\begin{theorem}[{\bf Spectrally-Normalized Margin bound}]\label{NBSTheorem}
 Let $S$ and $m$ be as defined in Definition \ref{loss} and $B$ be a bound on the norm of the input space $\chi$ from Definition \ref{loss}. Let $f_\w$ be a given neural net with parameters as above and let its layer matrices be $\{W_\ell\}_{\ell=1}^d$.  Construct a grid $\mathcal{B}$, called the ``beta-grid", containing $K =  \frac{d}{2} \times \Big(\frac {\sqrt{m-1}} { 2 \exp(3 - 2 / (d-1)) }\Big)^{1/d}$ uniformly spaced points covering the interval $\Big[ \Big(\frac{\gamma}{2B}\Big)^{1/d},  \Big(\frac{\sqrt{m-1}\gamma} { 4 \exp(3 - 2 / (d-1)) B}\Big)^{1/d} \Big]$. 
 If $\tilde{\beta} = \argmin_{x \in \mathcal{B}} | x - \beta(\w)|$ and 
\begin{align} 
\sigma (\tilde{\beta}) := \frac{1} {d\sqrt{2h\log(4 d h)}} \min &\Big \{ \frac{\gamma}{4  e^2 B  \tilde{\beta} ^{d-1} }, \frac{\tilde{\beta}}{e^{\frac 1 {d-1}}} \Big\}
\end{align} 
Then we have the following guarantee for all $\delta \in (0,\frac 1 K)$, 

\begin{align} 
\mathbb{P}_{S \sim {\cal D}^m} \Bigg [ \exists\, \tilde{\mu}_\w \text{ s.t }
\mathbb{E}_{\w + \tilde{\u} \sim \tilde{\mu}_\w} [ L_{0} (f_{\w + \tilde{\u}}) ] \leq \hat{L}_{\frac{\gamma}{2}}(f_\w) + \sqrt{\frac 1 {m-1}}\sqrt{ \sum\nolimits_{\ell=1}\nolimits^d \frac{\norm{W_\ell}_F^2}{2 \sigma(\tilde{\beta})^2} + \log \frac{3m}{\delta}} \Bigg ] \geq 1 - K\delta \label{eq:nbs}
\end{align} 
\end{theorem}

For completeness we have re-derived the above in Appendix \ref{NBSProof}. 

\begin{remark}
 Theorem~\ref{NBSTheorem} above slightly differs from the original statement of Theorem $1$ in \cite{neyshabur2017pac} because of the following adaptations and improvements that we have made, (a) we tracked the various constants more carefully, (b) we have removed some of their assumptions and have chosen to report the bound as being on a stochastic neural risk as is most natural in this context and (c) we used a more refined way to account for the data-dependent priors.
\end{remark}

\section{A noise resilience guarantee for a certain class of neural nets}
\label{sec:pert}

\begin{definition}[\textbf{Mixture Parameters}]\label{mixdef} 
Let $k_1 \ge 2$ denote the number of components in a ``mixture'' distribution. Let $\mathcal{A} = \{ \A_i \in \R^{\dim(\A)} \mid  i=1,..,k_1 \}$ denote a set of neural net weight vectors on the underlying architecture of $f_\A$. Let $\mathcal{P} = \{p_i \mid i=1, \ldots, k_1\}$ be a set of non-zero scalars. The mixture weights satisfy $\sum_i p_i = 1$. Define 
$\beta_i := \beta(\A_i)$ as defined in Definition \ref{Bdef}. Further define $\A_{i,\ell}$ to be the $\ell^\text{th}$ layer of $\A_i$. 
\end{definition}

Using the above setup we can state our main technical result as follows,

\begin{theorem}[{\bf Controlled output perturbation with noisy weights from a mixture of Gaussians}]\label{mainNoise}
Given $S$ and $\chi$ as in Definition \ref{loss}, let $B > 0$ be s.t the input space  $\chi$ is a subset of the ball of radius $B$ around the origin in $\R^n$. Further let $f_{\A}$ and $\beta$ be as in Definition \ref{Bdef} and $ \mathcal{A}, \mathcal{P}$ and $\{ \beta_i\}_{i=1,\ldots,k_1}$ as in Definition \ref{mixdef}, we choose any $\epsilon >0$ s.t the following inequalities hold,
\[ \forall\ i {\in} \{1\ldots k_1\}, \x {\in} S\  \norm{f_{\A_i}(\x) {-} f_{\A}(\x)} \leq \epsilon \norm{f_\A(\x)} \]
Then for every $\gamma > \epsilon \max_{\x \in S} \norm{f_\A(\x)}\text{ and }\delta \in (0,1)$, we have,
\begin{align} 
\nonumber \mathbb{P}_{A' \sim \textrm{MG(posterior)}}
  &{\Big[} \max_{\x \in S} \norm{f_{A'}(\x) {-} f_A(\x)} {>} 2 \gamma  {\Big]}
    \leq \delta
\end{align}
Where $\text{MG(posterior)}(\w) = \sum_i p_i \mathcal{N}_{(\A_i, \sigma^2)}(\w)$ and $\sigma \le \frac{1}{\sqrt{2h \log \left ( \frac{2dhk_1}{\delta} \right )}}  \min_{1 \le i \le k_1} \min \left \{ \frac{ \beta_i }{d} , \frac{\gamma}{k_1 e d B p_i \beta_i ^{d-1}}  \right \}$.
\end{theorem}

The above theorem has been proven in Section \ref{proof:mainNoise}.

\section{Our PAC-Bayesian risk bound on neural nets}

Now we use Theorem \ref{mainNoise} about controlled perturbation of nets to write the following theorem which is adapted to the setting of the experiments to be described in  Section \ref{sec:pac_exp}. 
Towards that we define the following notion of a ``nice" training data which captures the effect that a set of nets evaluates to almost the same output on some given dataset,

\begin{definition}[{\bf ``Nice'' training dataset}]\label{nice}
Given neural weights $\A$ and $\{\A_i\}_{i=1,\ldots,k_1}$  as in Definition \ref{mixdef}, we call a training dataset $S$ as $(\epsilon,\gamma)-$nice w.r.t. them if it satisfies the following conditions: 

\begin{enumerate}
\item $\max_{\x \in S} \norm{f_{\A_i}(\x) - f_{\A}(\x)} \leq \epsilon \norm{f_{\A}(\x)}, \forall 1 {\le} i {\le} k_1$
\item $\gamma > \epsilon \max_{\x \in S} \norm{f_{\A}(\x)}$
\end{enumerate}
\end{definition}

Next we define a two-indexed set of priors which will be critical to our PAC-Bayesian bounds. 

\begin{definition}[Our $2-$indexed set of priors]\label{ourgrid} 
Let $B > 0$ be s.t the input space $\chi$ in Definition \ref{loss} is a subset of the ball of radius $B$ around the origin in $\R^n$. Given $S,m$ as in Definition \ref{loss} and $d,h$ as in Definition \ref{Bdef}, we choose scalars $d_{\min}, \gamma, \delta >0$ s.t.  the following interval $I$ is non-empty,

\begin{align*} 
I :=  &\Bigg [  \Big ( \frac{\gamma}{2B} \Big )^{\frac 1 d}    , \left ( \frac{\gamma \sqrt{2(m-1)}}{(8Be^3d d_{\min})\sqrt{2h\log \Big ( \frac {2dh}{\delta} \Big )}} \right )^{\frac {1}{(d-1)}} \Bigg ] 
\end{align*}

Let $f_\B$ be a neural network. Consider a finite set of indices $\Lambda = \{1,\ldots,314\}$. For each $\lambda \in \Lambda$ we are given $k_1$ distinct neural net weights $\{ \B_{\lambda,j} \}_{j=1}^{k_1}$ within a conical half-angle of $0.01\lambda$ around $\B$. For each $\lambda$, we construct a grid $\mathcal{B}_\lambda$, called the ``beta-grid", containing at most, 

\[ K_1 = \frac{d}{2} \times \frac{\left  ( \frac{\gamma \sqrt{2(m-1)}}{(8Be^3d d_{\min})\sqrt{2h\log \Big ( \frac {2dh}{\delta} \Big )}} \right  )^{\frac {1}{(d-1)}}}{\Big ( \frac{\gamma}{2B} \Big )^{\frac 1 d}  } \] 

points inside the interval $I$ specified above. Now for each $\lambda \in \Lambda$ and $\tilde{\sigma} \in \mathcal{B}_\lambda$ we consider the following mixture of Gaussians $\frac{1}{k_1}\sum_{j=1}^{k_1}  {\cal N}_{(\B_{\lambda,j},\tilde{\sigma}^2 I)}$.  Thus we have a grid of priors of total size $K := 314K_1$.
\end{definition}

{\remark (a) Note that this set of mixture of Gaussian priors above indexed by $\lambda$ and $\tilde{\sigma}$ corresponds to the set of distributions that we call $\{\pi_i\}$ in the general Theorem \ref{thm:b1} (b) The specific choice of the set $\Lambda$ given above is only for concreteness and to keep the setup identical to the experiments in sections \ref{sec:pac_exp} and \ref{sec:geom} and this choice is {\it not} crucial to the the main theorem to be given next.} 

\paragraph{The choice of the parameters $\Lambda, \gamma, \epsilon$ and $d_{\min}$} The parameter $B$ gets fixed by the assumption of boundedness of the data space $\chi$ and we choose the training data size $m$. The set $\Lambda$ above is a convenient choice that we make motivated by experiments as a way to index a grid on the $\pi-$radians of possible deflection that can happen when the net $f_\B$ is trained to some final net (which we have been denoting as $f_\A$ as in subsection \ref{intro:theory}).  Also note that in practice when presented with the neural nets $f_\A$ and $\{f_{\A_i}\}_{i=1}^{k_1}$ (which in turn fixes the value of depth $d$ and width $h$) and the training data set $S$ we would choose the smallest values of $\gamma$ and $\epsilon$ so that the conditions in Definition \ref{nice} are satisfied. Then we choose  $\delta$ (typically $\delta = 0.05$) which determines our confidence parameter $1-\delta$. At this point except $d_{\min}$ all other parameters are fixed that go into determining the interval $I$ above. Now we can just choose $d_{\min}$ low enough so that the interval $I$ is non-empty. Once $d_{\min}$ is chosen the value of $K_1$ and hence the size of the prior set also gets fixed.

Now we use the above definitions and the notations therein to state our main theorem as follows, 

\begin{figure}[htbp]
\begin{center} 
   \includegraphics[trim=10 8 10 35, clip, scale =0.43]{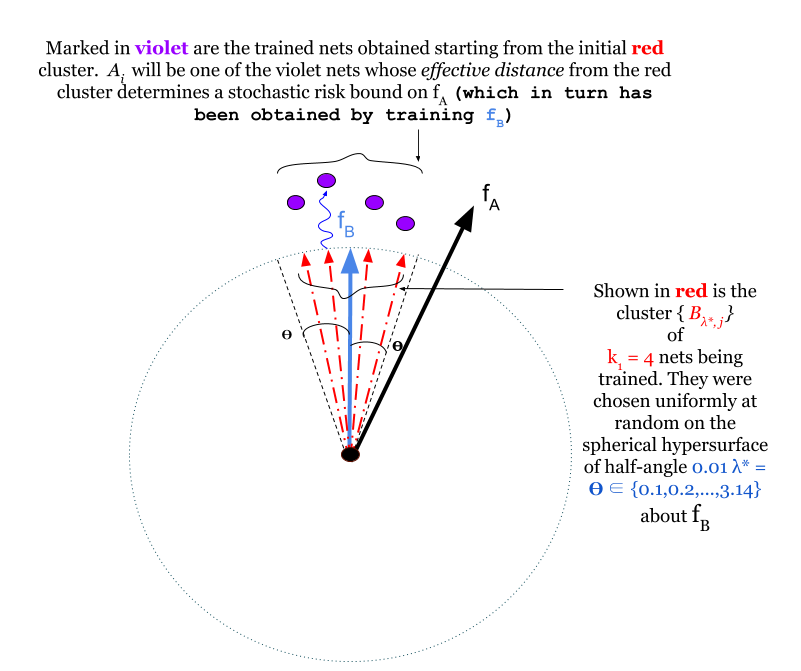}
 \caption{Starting from the weight vector $\B$ we get the trained weight vector $\A$. $\Theta$ is the angle to which the angle of deflection $\measuredangle (\A, \B)$ has been discretized to.} \label{fig:Bound}
 \end{center} 
\end{figure}

\newpage 
\begin{theorem}[{\bf Gaussian-Mixture PAC-Bayesian Bound}]\label{maincor}
As indicated in Figure ~\ref{fig:Bound}, suppose we train using the dataset $S$ to obtain the trained net $f_\A$ from an initial neural net $f_\B$. Let $\alpha = \arccos\frac{\langle \A, \B \rangle}{\norm{\A} \norm{\B}}$. Let $\lambda^* = \argmin_{\lambda \in \Lambda}|0.01\lambda - \alpha|$.  Further, let the neural weight vectors $\{ \A_i\}_{i=1,\ldots,k_1}$ be obtained by training the nets $\{ f_{\B_{\lambda^*,j}} \}_{j=1,\ldots,k_1}$ on $S$. Further for each such $i$ define $d_{i, _*} = \min_{j=1,\ldots,k_1} \norm{\A_i - \B_{\lambda^*,j}}^2$ and $\tilde{\beta}_i = \argmin_{x \in \mathcal{B}_{\lambda^*}} | x - \beta(\A_i)|$.
 
Then it follows that for all $\epsilon >0$ and $\delta \in (0,\frac{1}{K})$, 
\begin{align}
\nonumber &\mathbb{P}_{S} \Bigg [ \exists\, 
\tilde{\mu}_{\A} \text{ s.t. }, \forall i \text{ s.t. } d_{i,_*} \geq d_{\min},  \mathbb{E}_{\A + \tilde{\u} \sim \tilde{\mu}_\A} [ L_{0} (f_{\A + \tilde{\u}}) ] \leq \hat{L}_{\frac \gamma 2} \left (f_{\A} \right ) +\\
&\nonumber \sqrt{\frac 1 {m-1}}\sqrt{{-\log\Big(  \frac{1}{k_1} \sum\nolimits_{j=1}\nolimits^{k_1}    \exp(- {\frac 1 {2\tilde{\sigma}_i^2}}\|\A_i - \B_{\lambda^*,j}\|^2) \Big) + \log \frac{3m}{\delta}}} \\
&\;\;\;\;\Big\vert\;\; S \text{ is }(\epsilon,\gamma)\text{-nice w.r.t }\{ \A, \A_{i=1,\ldots,k_1} \} \Bigg]\geq 1 - K\delta \label{eq:ours}
\end{align}
where $\tilde{\sigma}_i^2 := \frac{1}{2h \log \left ( \frac{2dh}{\delta} \right )}  \left (  \min \left \{  \frac{ \tilde{\beta}_i }{de^{\frac {1}{d-1}}} , \frac{\frac \gamma 8}{e^2dB \tilde{\beta}_i ^{d-1}}  \right \}  \right )^2$ 
\end{theorem}  

The proof of the above Theorem has been given in Section \ref{app:maincor}

\newpage 
{\remark 
We emphasize that the structure of the above theorem is the same as that of the general PAC-Bayes bound as stated in Theorem \ref{pacmain}. The distribution $\tilde{\mu}_\A$ above is a choice of the posterior distribution that is called $Q$ in the general theorem. Hence its w.r.t this $\tilde{\mu}_\A$ that we are bounding with high probability the stochastic risk of the neural net $f_\A$ under the loss function $L_0$

Secondly the distribution $\tilde{\mu}_{\A}$ here is explicitly constructed such that, for the sampled noisy weights $\A +\tilde{\u}$ it is ensured that $\max_{\x \in S} \norm{ f_{\A + \tilde{\u}} (\x) - f_\A (\x)}_\infty < \frac{\gamma}{4}$. Also we emphasize that in the above $\lambda^*$ as defined is s.t $0.01 \times \lambda^*$  is the closest angle in the set  $\{0.01,0.02,\ldots,3.14\}$ to the angle between the initial net's weight vector $\B$ and the trained net's weight vector $\A$
}


We note the following salient points about the above theorem, 

\begin{itemize}
    \item In our experiments the nets $\{ \A_i\}_{i=1,\ldots,k_1}$ will be obtained by training the nets $\{ f_{\B_{\lambda^*,j}} \}_{j=1,\ldots,k_1}$ on the same training data $S$ using the same method by which we obtain $f_\A$ from $f_\B$.
    
    \item We emphasize that the above setup is not tied to any specific method of obtaining the initial and final clusters of nets. For example there are many successful heuristics known for ``compressing" a neural net while approximately preserving the function. One can envisage using the above theorem when such a heuristic is used to get such a cluster $\{B_{\lambda^*,j}\}$ from $\B$ and $\{\A_i\}$ from $\A$. 
    
    
    \item For each of the different choices amongst $\{\A_i\}_{i=1}^{k_1}$ which satisfy the condition $d_{i,*} \geq d_{\min}$ we can get a different upper bound on the risk in the LHS. Hence this theorem gives us the flexibility to choose amongst $\{\A_i\}_{i=1}^{k_1}$ those that give the best bound. 
    
    \item We note that after a training set $S$ has been sampled, we require no niceness condition on the nets which will have to hold over the entire domain of the nets. The upperbound as well as the confidence on the upperbound are all entirely determined by the chosen data-set $S$.
    
    \item Corresponding to the experiments in the next section we will observe in section \ref{sec:geom} that the angular deflection $\alpha$ above is predominantly determined by the data-distribution from which $S$ is being sampled from and at a fixed width it decreases slightly with increasing depth. 
    
    The improvements seen in our approach to PAC-Bayesian risk bounds strongly suggest that the consistent patterns of dilation of the weight matrix norms (also reported in section \ref{sec:geom}) and the angular deflection of the net's weight vector merit further investigation.
\end{itemize}

\section{Experimental comparison of our bounds with \cite{neyshabur2017pac}}\label{sec:pac_exp}

Here we present empirical comparisons between \textsc{Our} PAC-Bayesian risk bound on nets in Theorem \ref{maincor} and the result in Theorem \ref{NBSTheorem} reproduced from \cite{neyshabur2017pac}, which we denote as the \textsc{NBS} bound.


We compute the two bounds for neural nets without bias weights (as needed by these two theorems stated above). We posit that the fair way to do comparison is to only choose the initial net, data-set and the training algorithm (including a stopping criteria) and to compute the different theories on whatever is the net obtained at the end of the training. We test the theories on the following different classification tasks (a) binary classification tasks on parametrically varied kinds of synthetic datasets which have two linearly separable clusters and (b) multi-class classification of the CIFAR$-10$ dataset. In both the cases we study effects of varying the net's width and depth. 

It is to be noted that all the following experiments have been also done over varying sizes of the training data set and the advantage displayed here of \textsc{Our} bound over \textsc{NBS}'s result, is robust to this change.

\newpage 
\subsection{CIFAR-10 Experiments}\label{comp:cifar}

Here the nets we train are of depths $2,3,\ldots,8$ and $16$ and we vary the number of ReLU gates in a hidden layer ($h$) between $100$ and $200$. We train the networks to a test-accuracy of approximately $50\%$ which is close to the best known performance of feed-forward networks on the CIFAR-10 dataset. The neural networks are initialized using the ``Glorot Uniform'' initialization and we use the ADAM weight update rule on the cross-entropy training loss. In each epoch we use mini-batch size $300$ and we set $k_1 = 25$ ($k_1$ as defined in Theorem \ref{maincor}).

\paragraph{Results} We test both the theories, \textsc{Our} bound in equation \eqref{eq:ours} and the \textsc{NBS} bound in equation \eqref{eq:nbs}, at $95\%$ confidence i.e at $K\delta = 0.05$ in the above referenced equations.

Having trained the initial cluster as needed in Theorem~\ref{maincor}, we choose the smallest $\epsilon$ and $\gamma$ that satisfy the ``niceness'' condition in Definition~\ref{nice}. In experiments we see that often (not always) this minimum $\gamma$ needed to satisfy this condition increases with the depth of the net. {\it At any fixed architecture and dataset we evaluate both the theories at the same value of $\gamma$ chosen as said above.} We repeat the experiment with $10$ different random seeds (which changes the data-set, the initial cluster choices and the mini-batch sequence in the training). 

In Figure~\ref{fig:dcif} we see examples of how \textsc{Our} bounds do better than \textsc{NBS}. We plot \textsc{Our} bound for the $i^{\text{th}}$ point in the final cluster (as defined in equation \ref{eq:ours}) that achieves the lowest bound. (We always start from taking a very small value as the choice of $d_{\min}$, as required in Definition \ref{ourgrid}, s.t in experiments the distance between the clusters was always bigger than that.) Note the log-scale in the $y-$axis in this figure and hence the relative advantage of our bound is a significant multiplicative factor which is {\it increasing} with depth. {\it And at large widths our bound seems to essentially flatten out in its depth dependence.}

\begin{figure}[htbp]
\begin{subfigure}[b]{0.5\textwidth}
\includegraphics[trim=0 15 0 38,clip,width=\textwidth]{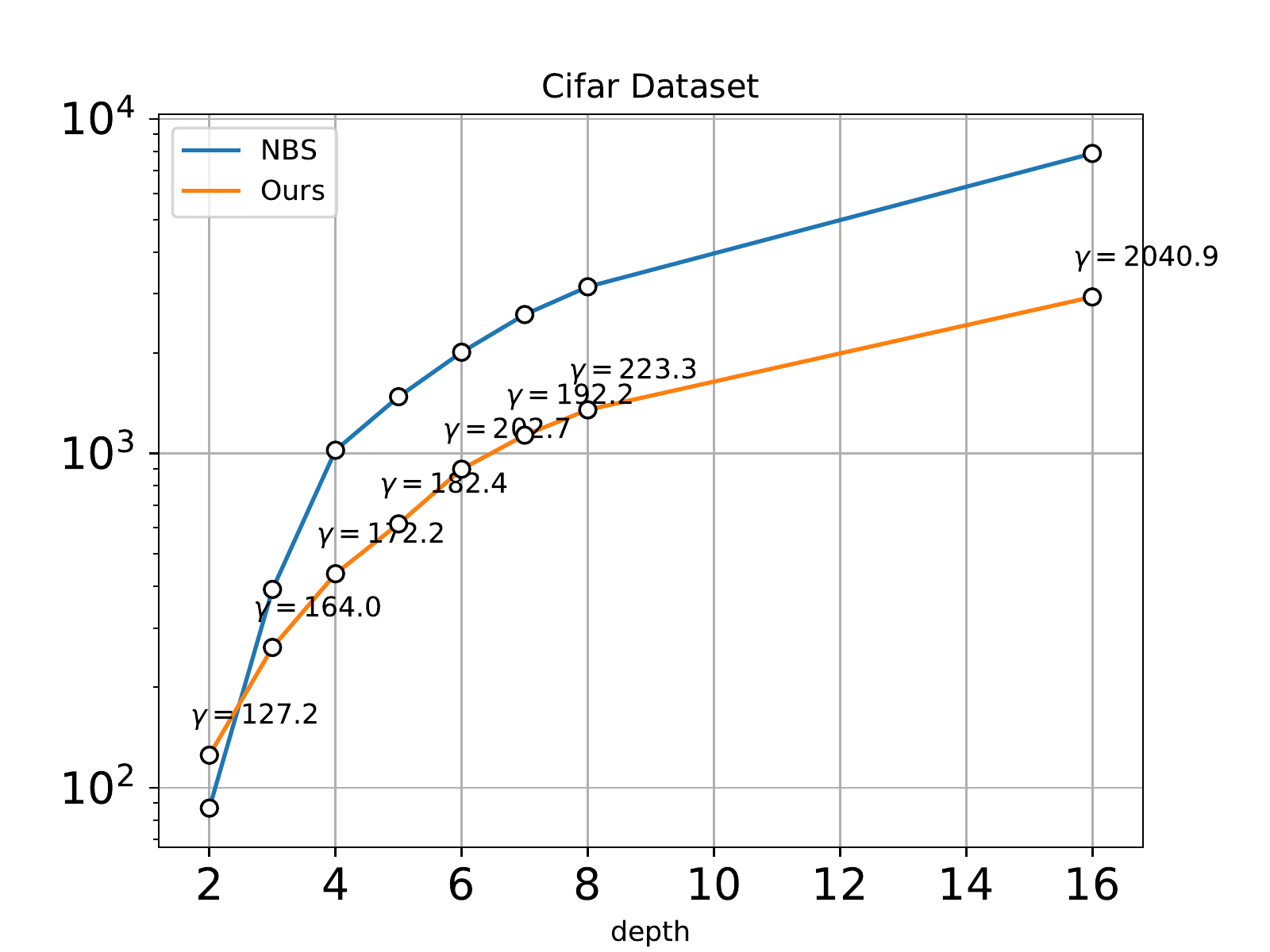}
\end{subfigure}
\begin{subfigure}[b]{0.5\textwidth}
\includegraphics[trim=0 80 0 170,clip,width=\textwidth]{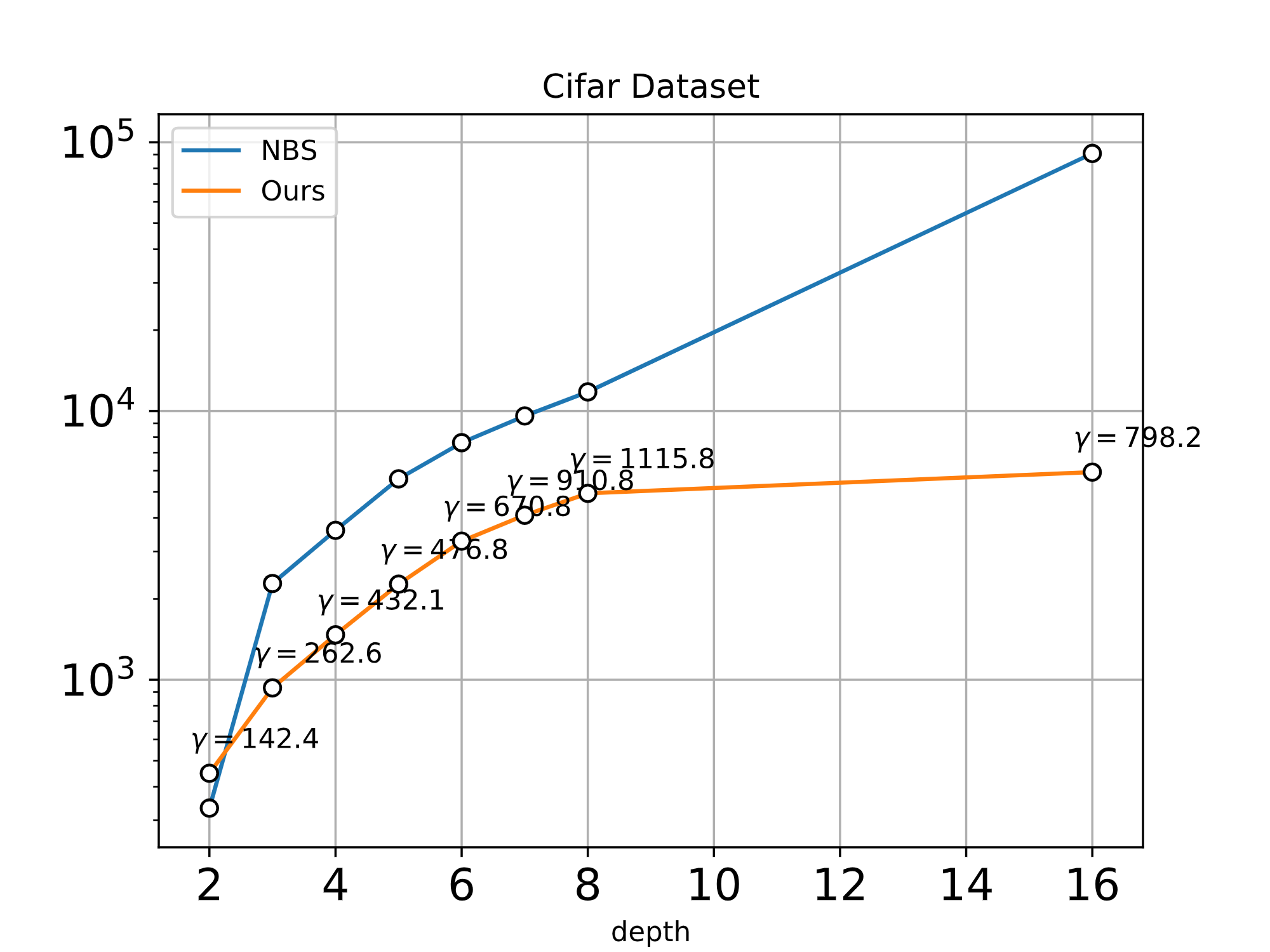}
\end{subfigure}
\caption{In the above figures we plot the risk bounds (in the $y-$axis) predicted by Theorem \ref{maincor} and  Theorem \ref{NBSTheorem} for trained nets at different depths, the $x-$ axis.  We can see the comparative advantage across depths in favour of \textsc{Our} bound over \textsc{NBS} when tested on CIFAR-10 while the width of the net is $100$ for the figure on the left and width is $400$ for the figure on the right.}
\label{fig:dcif}
\end{figure}

\subsection{Synthetic Data Experiments}\label{comp:synth}

In this section we show a comparison between \textsc{Our} and \textsc{NBS} bounds on a synthetic dataset which allows for probing the theories in a very different regime of parameters than CIFAR-10. Here the classification accuracies of the nets are near perfect, the margin parameters and the angular deflection of the net during training are significantly lower. 

\label{sec:exper-synth-data}
\paragraph{Dataset}  We randomly sample $m=1000$ points in $\R^{n}$ from two different isotropic variance $1$ normal distributions centered at $\mathbf{1} = (1,1,\ldots,1)$ and $a\mathbf{1}$ for $n=20$ and $40$ and for $a \in \{2, 4, 6, 8, 10\}$. We reject a sample $\x$ if $\min(||\x - \mathbf{1}||_\infty, ||\x - a \mathbf{1}||_\infty) > 1$. Thus the inter-cluster distance varies with $a$. When $n=20$ then we set $B=50$ and otherwise $B=100$.

\paragraph{Architecture and training} We train fully-connected feed-forward nets of depth $2, 3, \ldots, 8$ on the cross-entropy loss function. Each hidden layer has $800$ ReLU gates. As before we initialize the neural network layers using the ``Glorot Uniform'' initialization and train using ADAM (mini-batch size $100$). Networks with depth $d < 5$ required $5$ epochs and $d \ge 5$ require $8$ epochs for training to $100\%$ train and test accuracy. Our risk bounds are computed using cluster size of $k_1 = 25$ as before.  

\paragraph{Results} The parameters $K$ and $\delta$ are set so that the confidence on the bounds is at $95\%$.  For each network depth we compute our bound $10$ times using $10$ different random seeds. Each trial achieves approximately $100\%$ test accuracy and we plot the bound for the seed which achieves the minimum value of $\gamma$. For each network depth we label the value of $\gamma$ used to compute the bounds. The same value of $\gamma$ is used for computing both \textsc{Our} and \textsc{NBS} bounds. We compute \textsc{Our} bound for the $i^{\text{th}}$ cluster point that achieves the lowest bound. 


In Figure~\ref{fig:gmmcomp} we compare \textsc{Our} bound in Theorem \ref{maincor}
    to the \textsc{NBS} bound in Theorem~\ref{NBSTheorem} at two of the many parameter configurations of the above model where we have tested the theories. Again we find that our bound is consistently lower by a multiplicative factor than the previous PAC-Bayes bound. 

\begin{figure}[htbp]
\begin{subfigure}[b]{0.5\textwidth}
\includegraphics[trim= 0 15 0 0,clip,width=\textwidth]{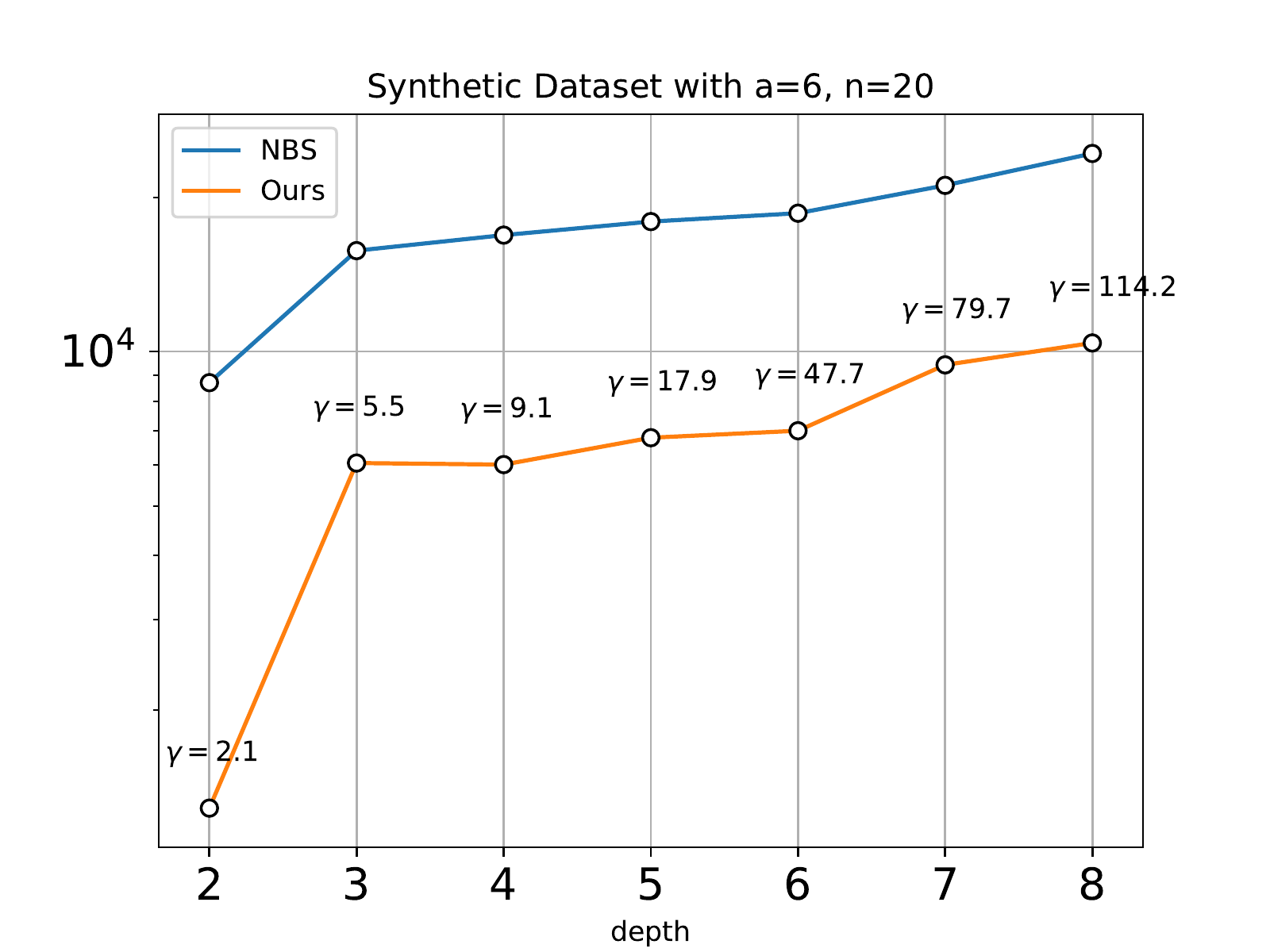}
\end{subfigure}
\begin{subfigure}[b]{0.5\textwidth}
\includegraphics[trim=0 15 0 0,clip,width=\textwidth]{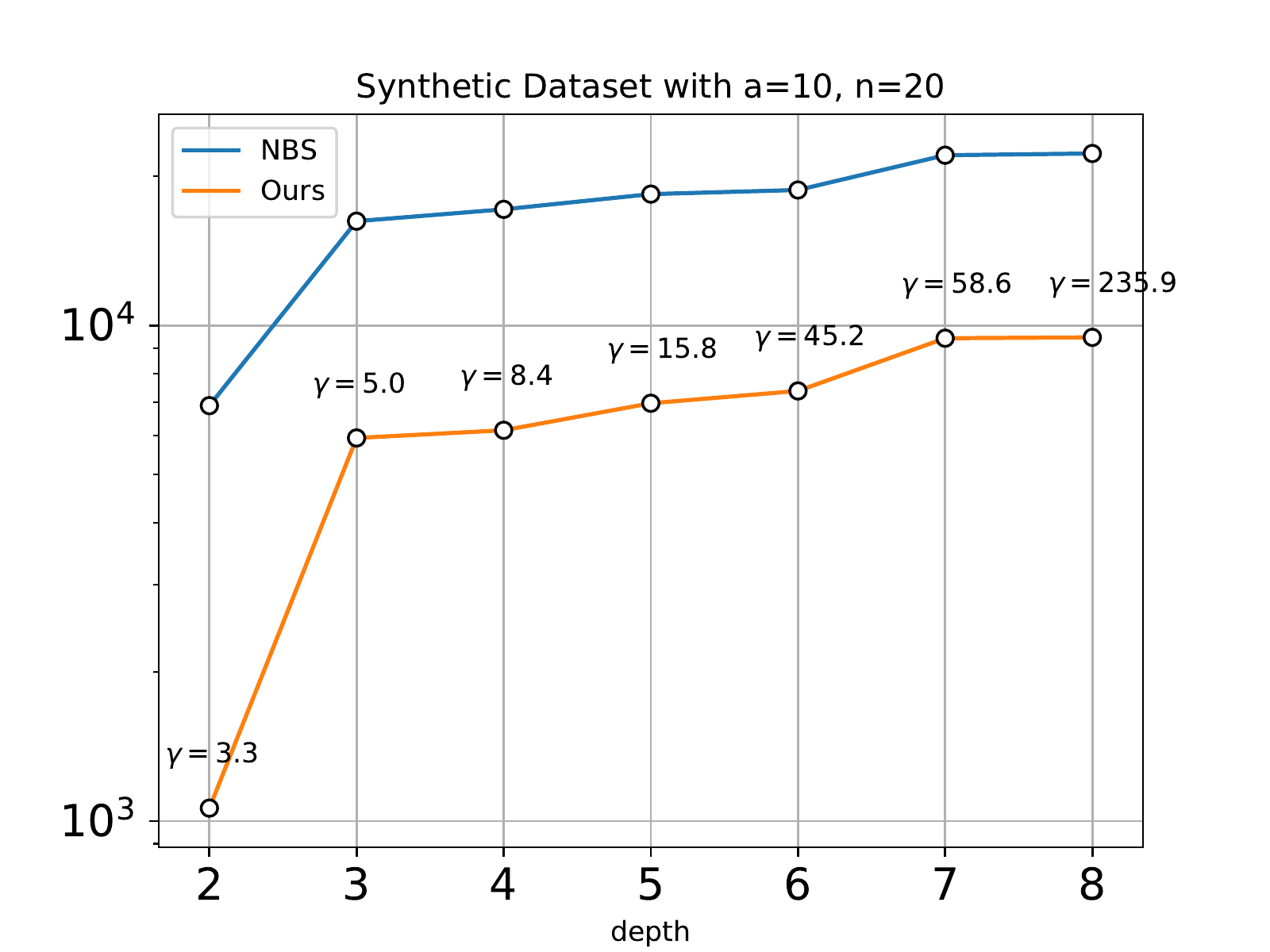}
\end{subfigure}
\caption{In the above figures we plot the risk bounds (in the $y-$axis) predicted by Theorem \ref{maincor} and  Theorem \ref{NBSTheorem} for trained nets at different depths, the $x-$ axis. In particular here we compare the two theories when the synthetic data generation model is sampling in $n=20$ dimensions with the cluster separation parameter $a=6$ in the left figure and $a=10$ in the right figure. 
}\label{fig:gmmcomp}
\end{figure}



\section{Experimental observations about the geometry of neural net's training path in weight space}\label{sec:geom}

Our approach to PAC-Bayesian risk bounds for neural nets motivates us to keep track as to how during training the norm of the neural net's weight vector changes and by how much angle this vector deflects. In here we record our observations about the interesting patterns that these two parameters were observed to have for the two kinds of experiments that were done in the previous section. The structured behaviours as seen here are potentially strong guidelines for directions for future theoretical developments. 

\paragraph{{\bf Data from the experiments on CIFAR-10 in subsection \ref{comp:cifar}}}

\begin{itemize}
\item 
We show in Figure~\ref{fig:dciftheta} the Gaussian kernel density estimation, ~\citep{scott2015multivariate}, of the angular deviation ($\alpha$ in Theorem \ref{maincor}) over the $10$ trials at every depth at a width of $100$. We can see that the angular deflection is fairly stable to architectural changes and is only slightly decreasing with depth. Similar pattern is also observed at higher widths.

\begin{figure}[htbp]
  \centering
  \includegraphics[width=0.5\linewidth]{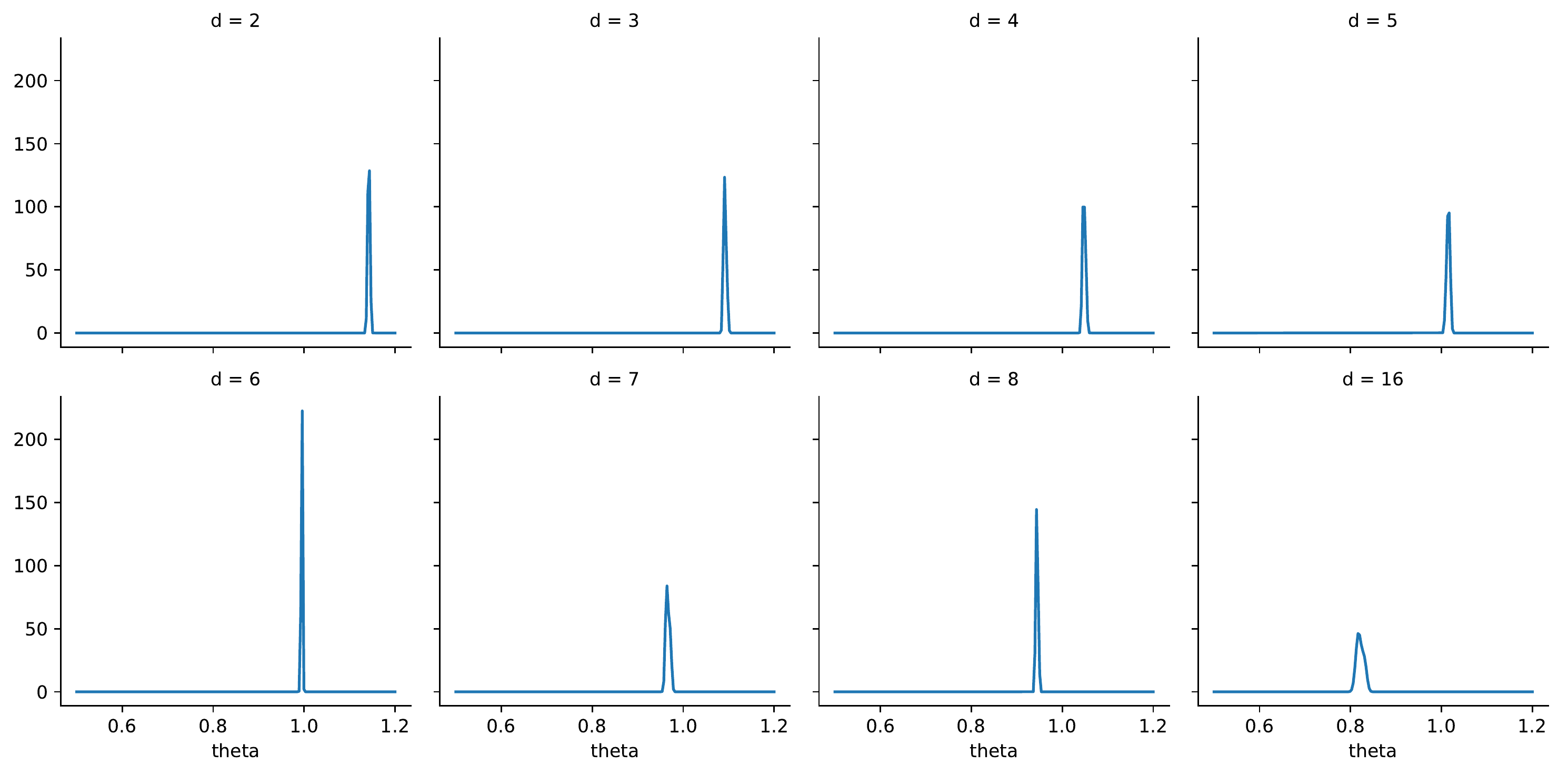}
  \caption{Gaussian kernel density estimate on $10$ trials of the experiment (for every depth $d$ and $100$ width) measuring the angular deviation of the weight vector under training}
  \label{fig:dciftheta}
\end{figure}

\item 
In Figure~\ref{fig:dcifstorel2norm} we show the initial parameter norm vs the final parameter norm for training nets of different depths at width $100$ on the CIFAR dataset. Thus its demonstrated that the multiplicative factor with which the norm increases during training is also fairly stable to architectural changes. 

\begin{figure}[htbp]
  \centering
  \includegraphics[width=0.5\linewidth]{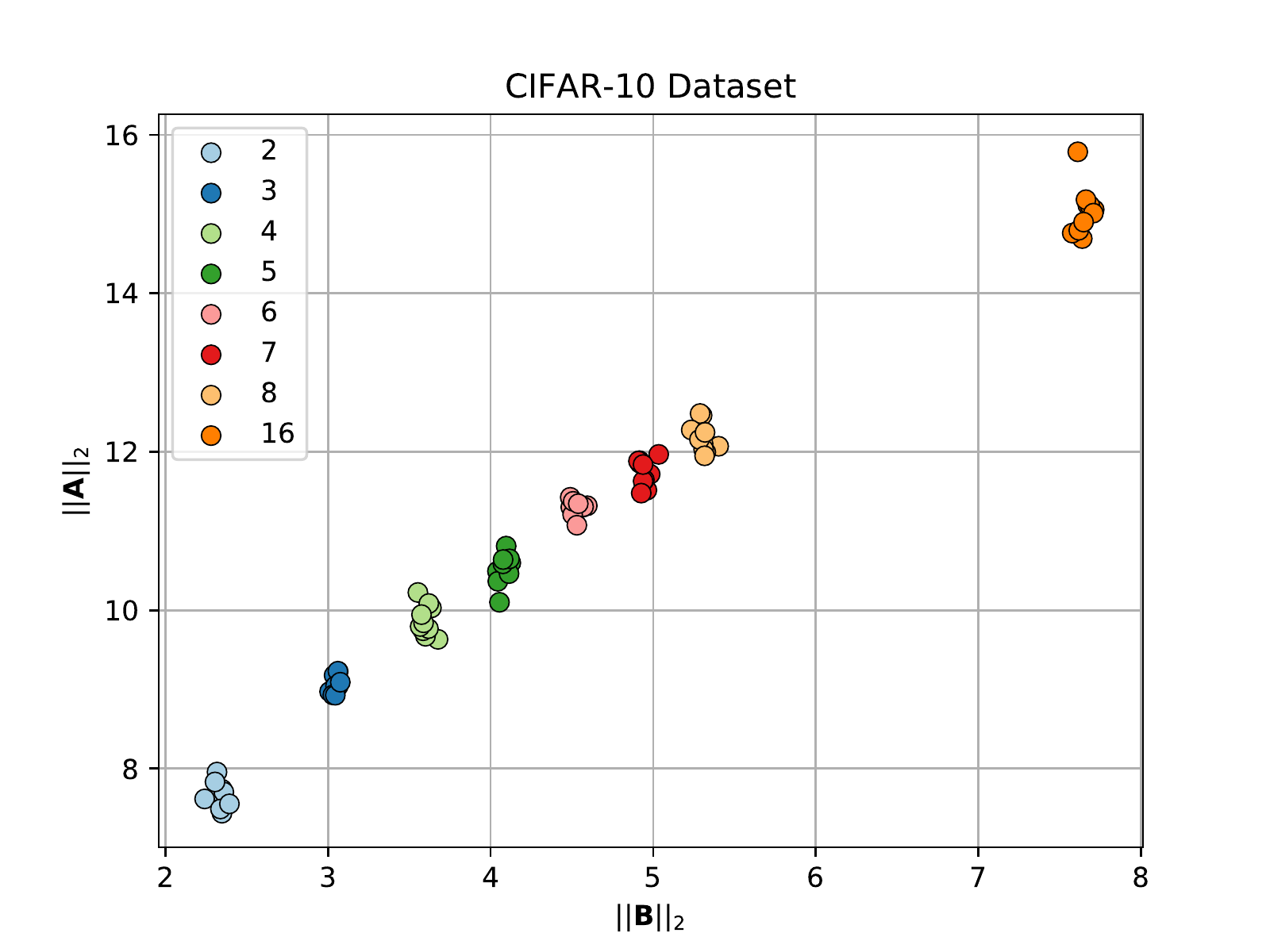}
  \caption{Initial parameter norm $\norm{\B}_2$ vs final parameter norm $\norm{\A}_2$ with increasing depths (at width $100$) on the CIFAR-10 dataset. $10$ trials are displayed for each architecture.}
  \label{fig:dcifstorel2norm}
\end{figure}

\end{itemize}

\paragraph{{\bf Data from the experiments on synthetic data in subsection \ref{comp:synth}}}

\begin{itemize}
    \item In Figure~\ref{fig:dgmm_kdetheta} in Appendix \ref{app:thetasynth} we show the KDE of the angular deviation $\theta$ between the initial and final networks for different depths $d$ and cluster separations $a$. The KDE was obtained from $\theta$ values obtained through $10$ trials. Here we observe that the mean value of the angular deflection due to training is only slightly affected by the architectural choices or the cluster separation parameter. The variance of the distribution of the angle increases with increase in $d$ and $a$ and the mean value tends to slightly decrease with increasing depths.  (Note that the mean angular deflection for CIFAR$-10$ that we saw in the previous experiment was significantly larger than here.)  

   \item Lastly in Figure ~\ref{fig:gmmnorm} we show the variation between the initial and the final norms for the inter-cluster separation of $a=2$. We observe a consistent (and surprising) behaviour that training seems to dilate the sum of Frobenius norms of the net and the dilation factor is close to $1$ at depth $2$ and increases to about $2.5$ for about an order of magnitude of increase in depths. This behaviour is fairly stable across different values of $a$ that we tried and recall that this same phenomenon was also demonstrated on CIFAR-10.  

    \begin{figure}[htbp]
        \centering
        \includegraphics[width=0.5\linewidth]{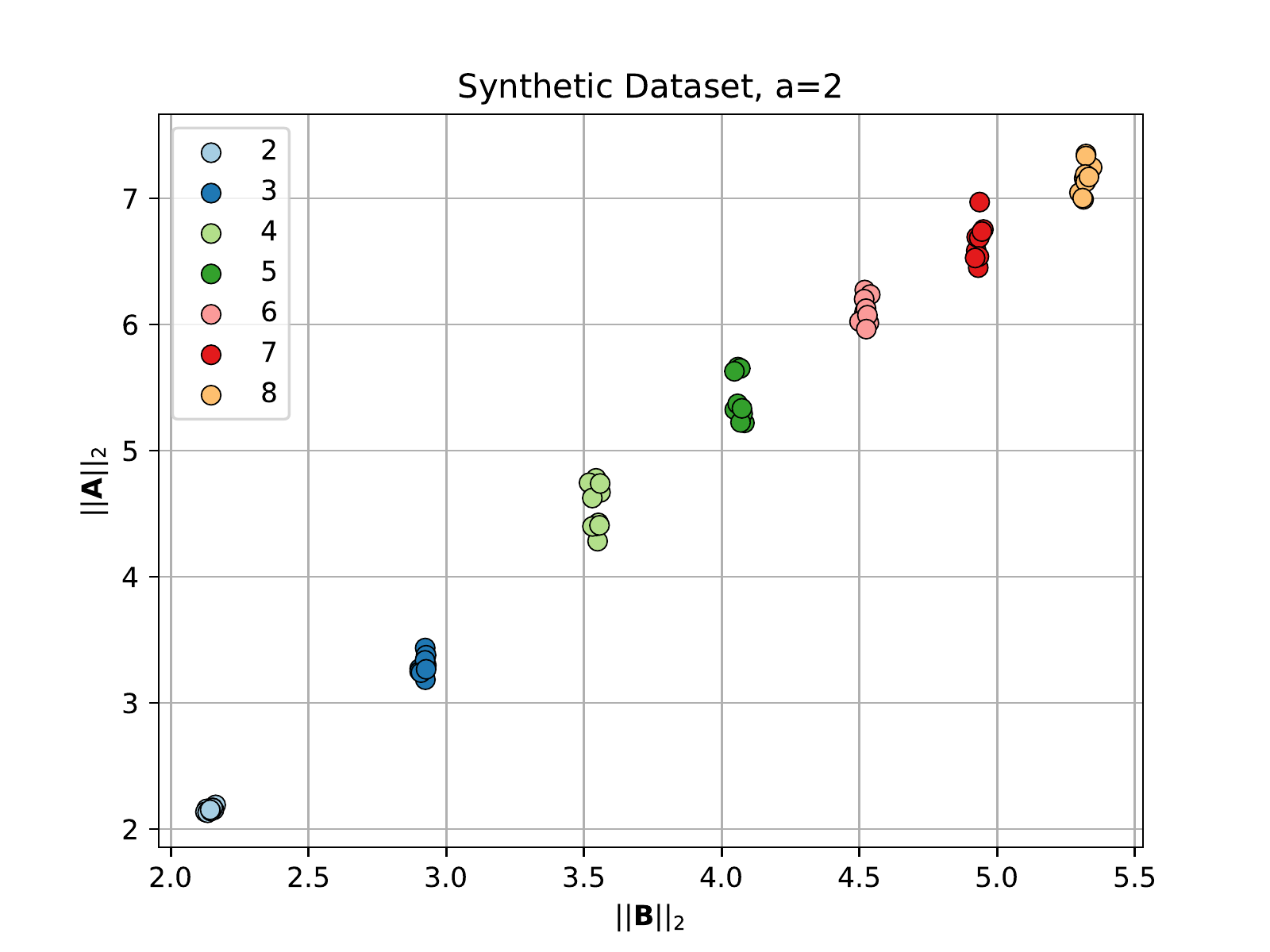}
        \caption{Scatter plot of $||A||_2$ versus $||B||_2$ for neural networks of depths $2,3, \ldots, 8$, for cluster separation $a=2$ in ambient dimension $n=20$. For each depth we run $10$ trials with different random initializations and mini-batch sequence in the SGD.}
        \label{fig:gmmnorm}
    \end{figure}
\end{itemize}

\section{Proof of Theorem \ref{mainNoise}}\label{proof:mainNoise}

Towards proving the main Theorem \ref{mainNoise} we need the following definition and lemma,

\begin{definition}[\textbf{Random Neural Network}]
Let us denote as, $f_{{\cal N}_{\left (\A_{i},\sigma^2 \right )}}$ the random neural net function obtained by sampling its weights from the isotropic Gaussian distribution with p.d.f ${\cal N}_{\left (\A_{i},\sigma^2 \right )}$.
\end{definition}

\begin{lemma}[\textbf{Controlled output perturbation of Gaussian weights}]\label{secondmainlemma} 
Let us be given a set of neural net weight vectors (for a fixed architecture) $\mathcal{A} = \{ \A_i\}_{i=1,..,k_1}$ s.t that $\norm{\A_{i,\ell}} = \beta_i$ for all $i \in \{1,\ldots,k_1\}$ and $\ell \in \{1,\ldots,d\}$. If, 
\begin{align}
\sigma \le \frac{1}{\sqrt{2h \log \frac{2dhk_1}{\delta} }}   \min_{i \in \{1,\ldots,k_1\}} \min \left \{ \frac{\beta_i}{d} , \frac{\gamma}{k_1 edB p_i \beta_i^{d-1}}  \right \}
\end{align} Then,
\begin{equation}\label{eq:secondmainlemma} \mathbb{P} \left(\sum_{i=1}^{k_1} p_i \norm{f_{ {\cal N}_{(\A_i,\sigma^2I)}}-f_{\A_i}}>\gamma \right)< \delta \end{equation}
\end{lemma}

\begin{proof}[Proof of Lemma \ref{secondmainlemma}] 
Let  $\bar{U} = \{ U_{i,l} \mid  i=1,\ldots,k_1, \ell = 1,\ldots,d \}$
 be a set of size 
$d \cdot k_1$ containing $h-$dimensional random matrices, such that each 
matrix $U_{i,\ell} \sim {\cal N}(\mathbf{0},\sigma^2)$

We can define matrices $\{ \B_p \in \R^{h \times h} \mid p = 1,\ldots,h^2 \}$ s.t each $\B_p$ has $\sigma$ in an unique entry of it and all other entries are $0$. Then it follows that as random matrices, $U_{i,\ell} = \sum_{p=1}^{h^2} \gamma_p \B_p$ with $\gamma_p \sim {\cal N}(0,1)$. We note that $\norm{\sum_{p=1}^{h^2} \B_p \B_p^\top} = h \cdot \sigma^2$ since $h$ is the largest eigenvalue of an all ones $h-$dimensional square matrix. Now we invoke Corollary $4.2$ of \cite{tropp2012user} here to get for any $t > 0$,

\begin{align}\label{tropp} 
 \mathbb{P}_{U_{i,\ell}} \left [ \norm{U_{i,\ell}}_2 > t \right ] \leq 2he^{-\frac {t^2}{2h\sigma^2} }  
\end{align}


Using union bound we have, $\mathbb{P}_{\bar{U}} [ \exists (i,\ell) \text{ s.t } \norm{U_{i,\ell}} > t_i] \leq 2dh \sum_{i=1}^{k_1}e^{-\frac{t_i^2}{2h\sigma^2}}$. So we have, 
\begin{align}\label{union} 
1 - 2dh\sum_{i=1}^{k_1}e^{-\frac{t_i^2}{2h\sigma^2}} \leq \mathbb{P}_{\bar{U}} [ \forall (i,\ell) \text{ s.t } \norm{U_{i,\ell}} \leq t_i]
\end{align} 

Let $\A_{i,\ell}$ be the induced matrix in the $\ell^{th}-$layer from the neural weight vector $\A_i$. Let $U_{i,l}$ be the perturbation for $\A_{i,l}$ and suppose,  $\norm{U_{i,\ell}} \leq \frac{1}{d}\norm{\A_{i,\ell}}$.\footnote{Since the width of the net is assumed to be uniformly $h$ it follows that $U_{i,\ell}$ is of the same dimensions as $\A_{i,\ell}$.}
We have by Lemma $2$ of \cite{neyshabur2017pac} and our assumption of uniform spectral norms for the layer matrices,
\begin{align}\label{sripert}
\norm{f_{\A_i + (\textrm{vec}(U_{i,\ell}))_{\ell = 1,2,..,d} } -f_{\A_i}}\le eB \beta_i^{d-1}\sum_{\ell =1}^d \|U_{i,\ell}\|_2
\end{align}
From equations~(\ref{sripert}-\ref{union}) it follows that if $\forall i\, t_i \leq \frac{\beta_i}{d}$ then we have, 
\begin{align*}
1 - 2dh\sum_{i=1}^{k_1}e^{-\frac{t_i^2}{2h\sigma^2}} \leq &\mathbb{P}_{\bar{U}} [ \forall (i,\ell) \text{ s.t } \norm{U_{i,\ell}} \leq t_i]\\
\leq &\mathbb{P}_{\bar{U}} [ \forall i, \norm{f_{\A_i + (\textrm{vec}(U_{i,\ell}))_{\ell = 1,2,..,d} } -f_{\A_i}}\le eB \beta_i^{d-1} \big(\sum_{l=1}^d t_i \big) = eB \beta_i^{d-1} dt_i]\\
\leq &\mathbb{P} \left [ \sum_{i=1}^{k_1} p_i \norm{f_{\A_i + (\textrm{vec}(U_{i,\ell}))_{\ell = 1,2,..,d} } -f_{\A_i}}\le edB \sum_{i=1}^{k_1} p_i \beta_i^{d-1} t_i \right ]
\end{align*}
Let $\gamma > 0$ and chose $t_i$ s.t \begin{equation}\label{eq:mycond}edB \sum_{i=1}^{k_1} p_i \beta_i^{d-1} t_i \leq \gamma\end{equation} hence, we get
\begin{align}\label{eq:myfinal}
1 - 2dh\sum_{i=1}^{k_1}e^{-\frac{t_i^2}{2h\sigma^2}} \leq &\mathbb{P} \left [ \sum_{i=1}^{k_1} p_i \norm{f_{\A_i + (\textrm{vec}(U_{i,\ell}))_{\ell = 1,2,..,d} } -f_{\A_i}}\le \gamma \right ]
\end{align}

A sufficient condition for \eqref{eq:mycond} is 
$t_i \leq \frac{\gamma}{k_1 edB p_i \beta_i^{d-1}}$. Combined with the condition that $t_i \leq \frac{\beta_i}{d}$ it follows that \eqref{eq:myfinal} holds if $\forall i \in \{1,\ldots,k_1\}$, 
\begin{align}\label{tbound} 
t_i \leq \min \left \{ \frac{\beta_i}{d} , \frac{\gamma}{k_1 edB p_i \beta_i^{d-1}}  \right \}
\end{align} 
To get \eqref{eq:secondmainlemma} we choose $\sigma$ s.t $\sum_{i=1}^{k_1}e^{-\frac{t_i^2}{2h\sigma^2}} \leq  \frac{\delta}{2dh}$ This is ensured if $\max_{i=1,\ldots,k_1} e^{-\frac{t_i^2}{2h\sigma^2}} \leq \frac{\delta}{2dhk_1} \iff \frac{\min_{i=1,\ldots,k_1} t_i^2}{2h\sigma^2} \geq -\log \left ( \frac{\delta}{2dhk_1} \right ) \iff$
\begin{align}\label{sigt} 
\sigma^2 \leq \frac{\min_{i=1,\ldots,k_1} t_i^2}{2h \log \left ( \frac{2dhk_1}{\delta} \right )}
\end{align}
We can maximize $\sigma^2$ and obey constraints~(\ref{sigt}-\ref{tbound}) by setting
\begin{align}
\sigma^2 =  \frac{1}{2h \log \left ( \frac{2dhk_1}{\delta} \right )}  \left ( \min_{i \in \{1,\ldots,k_1\}} \min \left \{ \frac{\beta_i}{d} , \frac{\gamma}{k_1 edB p_i \beta_i^{d-1}}  \right \}  \right )^2 
\end{align}
\end{proof}

Using the above now we can demonstrate the proof of Theorem \ref{mainNoise}. 

\begin{proof}[Proof of Theorem \ref{mainNoise}]

Given the assumption we have that, $\forall i\, \in \{1,\ldots,k_1\}$ and $\x \in S$, $\norm{f_{\A_i}(\x) - f_{\A}(\x)} \leq \epsilon \norm{f_{\A}(\x)}$, we have the following inequality for all neural weights $\A'$,

\begin{align}\label{postineq}
\norm{f_{\A'}(\x) - f_\A(\x)} &\leq \norm{f_{\A'}(\x) - f_{\A_i}(\x)} + \norm{f_{\A_i}(\x) - f_{\A}(\x)} \leq \norm{f_{\A'}(\x) - f_{\A_i}(\x)} + \epsilon \norm{f_{\A}(\x)}\\
\implies \norm{f_{\A'}(\x) - f_\A(\x)} &\leq \epsilon \norm{f_{\A}(\x)} + \min_{i \in \{1,\ldots,k_1\}} \norm{f_{\A'}(\x) - f_{\A_i}(\x)}\\
&\leq \epsilon \norm{f_{\A}(\x)} + Z
\end{align} 

where in the last line above we have defined, $Z = \min_{i \in \{1,\ldots,k_1\}} \norm{f_{\A'}(\x) - f_{\A_i}(\x)}$. Now for a choice of $\gamma(\x)$ s.t $\gamma(\x) > \epsilon \norm{f_{\A}(\x)}$ and using inequation \ref{postineq} we have the following inequalities being true for the given distribution $\text{MG(posterior)}$, 

\begin{align}\label{postineq2}
\mathbb{P}_{\A' \sim \text{MG(posterior)}} \left ( \norm{f_{\A'}(\x) - f_\A(\x)} > 2\gamma(\x)  \right ) &\leq \mathbb{P} \left ( \epsilon \norm{f_{\A}(\x)} + Z > 2\gamma(\x)  \right )\\
&\leq \mathbb{P} \left (  Z > 2\gamma(\x) - \epsilon \norm{f_{\A}(\x)}  \right )\\
&\leq \mathbb{P} \left (  Z > \gamma(\x) \right )
\end{align}

$\A'$ being randomly sampled from the distribution MG(posterior) can be imagined to be done in two distinct steps, $(1)$ first we select a center among the set $\{ \A_i \}_{i=1,\ldots,k_1}$ by sampling a random variable $Y$ valued in the set, $\{ 1, \ldots, k_1 \}$ with probabilities $\{ p_i \}_{i=1,\ldots,k_1}$ and then $(2)$ sample the weights from, ${\cal N}_{(\A_Y,\sigma^2 I)}$.

We define a collection of $k_1$ mutually independent random variables $\{ Z_j := \min_{i \in \{1,\ldots,k_1\}} \norm{f_{\P}(\x) - f_{\A_i}(\x)} \}_{j=1,\ldots,k_1}$ with $\P \sim {\cal N}_{(\A_j,\sigma^2 I)}$. Clearly $Z_j = Z \mid Y =j$. 
Thus we have the following relationship among the events, 

\vspace{-3mm} 
\begin{align}
\{ Z > \gamma(\x) \} &= \{ \min_{i \in \{1,\ldots,k_1\}} \norm{f_{\A'}(\x) - f_{\A_i}(\x)} > \gamma(\x) \} = \bigcup_{j=1}^{k_1} \left \{ ( Z > \gamma(\x) ) \cap ( Y = j ) \right \}\\
\implies  \mathbb{P} \left ( Z > \gamma(\x)  \right ) &= \sum_{j=1}^{k_1} \mathbb{P} \left ( Z > \gamma(\x) \mid Y = j \right ) \mathbb{P} \left ( Y = j \right ) =  \sum_{j=1}^{k_1} \mathbb{P} \left ( Z_j  > \gamma(\x)  \right ) \mathbb{P} \left ( Y = j \right )
\end{align}

In the above we invoke the definition that, $\mathbb{P}(Y =j) =p_j$ and that by the definition of $Z_j$ it follows that, $Z_j \leq \norm{f_{\P}(\x) - f_{\A_j}(\x)}$ (where $\P \sim {\cal N}_{(\A_j,\sigma^2 I)}$). Then we get, 

\vspace{-3mm} 
\begin{align}\label{invokeuniform} 
\mathbb{P} \left ( Z > \gamma(\x)  \right ) \leq  \sum_{j=1}^{k_1} p_j \mathbb{P}_{\P \sim {\cal N}_{(\A_j,\sigma^2 I)}} \left ( \norm{f_{\P}(\x) - f_{\A_j}(\x)} > \gamma(\x)  \right )    
\end{align}
\vspace{-3mm} 

Now for the kind of  nets we consider i.e ones with no bias vectors in any of the layers it follows from the definition of $\{ \beta_i \}_{i=1,\ldots,k_1}$ that has been made that the function computed by the net remains invariant if the layer weight $\A_{i,\ell}$ is replaced by $\frac{\beta_i}{\norm{\A_{i,\ell}}} \A_{i,\ell}$. And we see that the spectral norm is identically $\beta_i$ for each layer in this net with modified wights. So we can assume without loss of generality that $\A_{i,\ell} = \beta_i$ for all $i$ and $\ell$. Hence it follows that the RHS in equation \ref{invokeuniform} 
is exactly the quantity for which guarantees have been given in Lemma \ref{secondmainlemma}. And by de-homogenizing the definition of $\beta_i$ as given in Lemma \ref{secondmainlemma}, the appropriate value of $\sigma$ can be realized to be the same as given in the theorem statement, 

\begin{align}
\sigma^2 =  \frac{1}{2h \log \left ( \frac{2dhk_1}{\delta} \right )}  \left ( \min_{i \in \{1,\ldots,k_1\}} \min \left \{  \frac{ \left ( \prod_{\ell=1,\ldots,d} \norm{\A_{i,\ell}} \right )^{\frac 1 d} }{d} , \frac{\gamma (\x) }{k_1 edB p_i \left ( \prod_{\ell=1,\ldots,d} \norm{\A_{i,\ell}} \right )^{1-\frac{1}{d}}}  \right \}  \right )^2 
\end{align}

In the above we replace $\gamma(\x)$ with $\gamma := \epsilon \max_{\x \in S} \norm{f_\A (\x)}$ and going back to equation \ref{postineq2} we can get a concurrent guarantee for all $\x \in S$ as required in the theorem as,  
\[ \mathbb{P}_{\A' \sim \text{MG(posterior)}} \left ( \forall \x \in S, \norm{f_{\A'}(\x) - f_\A(\x)} > 2\gamma   \right ) \leq \delta \]
\end{proof}

\section{Proof of Theorem \ref{maincor} }\label{app:maincor} 

\begin{proof} 
Given that $(\epsilon,\gamma)-$nice w.r.t $\{ \A, \A_{i=1,\ldots,k_1} \}$ (as defined in Definition \ref{nice}), we can invoke Theorem \ref{mainNoise} with $k_1 =1$ between the trained nets $f_\A$ and $f_{\A_i}$. We recall the definition of $\beta_i$ that, $\beta_i^d = \prod_{\ell=1,\ldots,d} \norm{\A_{i,\ell}}$ for $\A_{i,\ell}$ being the $\ell^{th}$ layer matrix corresponding to $\A_i$.  

Let the $\beta-$grid be denoted as $\{ \tilde{\beta}_k\}$ s.t there exists a point  $\tilde{\beta}$ in this grid s.t,

\begin{align}\label{choose:betatilde}
\vert \beta_i - \tilde{\beta} \vert \leq \frac{\beta_i}{d} \implies \frac {\beta_i^{d-1}}{e} \leq \tilde{\beta}^{d-1} \leq e \beta_i^{d-1}
\end{align}

Now we recall that by invoking Theorem \ref{mainNoise} on the net $f_{\A_i}$, the $\sigma$ (in terms of $\beta_i$) that would be obtained from there is a maximal choice that the proof could have given us. Hence a smaller value of $\sigma$ will also give the same guarantees and we go for the following value of the variance defined in terms of the $\tilde{\beta}$ defined above,  

\begin{align}\label{choose:sigmatilde}
\tilde{\sigma}^2 := \frac{1}{2h \log \left ( \frac{2dh}{\delta} \right )}  \left (  \min \left \{  \frac{ \tilde{\beta} }{de^{\frac {1}{d-1}}} , \frac{\frac \gamma 8}{e^2dB \tilde{\beta} ^{d-1}}  \right \}  \right )^2
\end{align} 

In the above we have set the $p_i$ parameter of Theorem \ref{mainNoise} to $1$ and have also rescaled the $\gamma$ parameter to $\frac \gamma 8$ so that we have from Theorem \ref{mainNoise} that,

\begin{align}
&\mathbb{P}_{\A' \sim {\cal N}(\A_i,\tilde{\sigma}^2 I)}
  {\Big[} \max_{\x \in S} \norm{f_{\A'}(\x) {-} f_{\A}(\x)} > 2 \times \frac {\gamma}{8}  {\Big]}
    \leq  \delta \\
\end{align}

Since, $\max_{\x \in S} \norm{f_{\A'}(\x) {-} f_{\A}(\x)} < \frac{\gamma}{4} \implies \max_{\x \in S} \norm{f_{A'}(\x) {-} f_A(\x)}_\infty <  \frac{\gamma}{4}$ we have, 

\[  \mathbb{P}_{\A' \sim {\cal N}(\A_i,\tilde{\sigma}^2 I)}
  \Big [ \max_{\x \in S} \norm{f_{\A'}(\x) {-} f_{\A}(\x)}_\infty < \frac{\gamma}{4}  \Big] \geq 1 - \delta \]

Hence we are in the situation whereby Theorem \ref{precisepac} can be invoked with $\A = \w$ and $\mu_\w = {\cal N}(\A_i,\tilde{\sigma}^2 I)$ to guarantee that there exists a distribution $\tilde{\mu}_\w$ s.t the  following inequality holds with probability at least $1-\delta$  over sampling $m$ sized data-sets  

\begin{align}\label{prefinal}  
\mathbb{E}_{\A + \tilde{\u} \sim \tilde{\mu}_\A} [ L_{0} (f_{\A + \tilde{\u}}) ] \leq \hat{L}_{\frac \gamma 2} \left (f_{\A} \right ) +  \sqrt{\frac{\text{KL}({\cal N}(\A_i,\tilde{\sigma}^2 I) \vert \vert P) + \log \frac{3m}{\delta}}{m-1}}
\end{align} 

Because the grid of $\{ \tilde{\beta}_k\}$ was pre-fixed we can choose the prior distribution $P$ in a data-dependent way from the grid of priors specified in the Definition \ref{ourgrid} with their variance being $\tilde{\sigma}$ as given in Definition \ref{choose:sigmatilde}, determined by a chosen element from this set $\{ \tilde{\beta}_k\}$!  Now we recall the definition of the net weights, $\{ \B_{\lambda^*,j} \}_{j=1,\ldots,k_1}$ in the theorem statement and we choose, 

\[ P := \text{ a distribution s.t its p.d.f is } \frac{1}{k_1} \sum_{j=1}^{k_1} {\cal N}_{(\B_{\lambda^*,j},\tilde{\sigma}^2I)} \]

And this means that for the KL-term above we can invoke Theorem \ref{ourkl} with $f = {\cal N}(\A_i,\tilde{\sigma}^2 I)$ and $\text{GM} = P$ to get, 

\vspace{-4mm} 
\begin{align}\label{KLupperbound1} 
\text{KL}({\cal N}(\A_i,\tilde{\sigma}^2 I) \vert \vert P) \leq -\log\left[  \frac{1}{k_1} \sum_{j=1}^{k_1}    e^{- {\frac 1 {2\tilde{\sigma}^2}}\|\A_i - \B_{\lambda^*,j}\|^2 } \right]
\end{align} 
\vspace{-5mm} 

We recall that our angle grid (which determines the value of $\lambda^*$ above as described in the theorem statement) was of size $314$ and $K_1$ was the size of the $\beta-$grid and thus the size of the full grid of priors is $314K_1 =: K$ 

Hence its clear as to how equation \ref{prefinal} holding true for each of the $K$ possible choices of $P$ decided by the mechanism described above, satisfies the required hypothesis for Theorem \ref{thm:b1} to be invoked. The required theorem now follows by further upperbounding the KL-term in equation \ref{prefinal} as given in equation \ref{KLupperbound1}.

Now we are left with having to specify the required grid $\{ \tilde{\beta}_k\}_{k=1,\ldots,K_1}$ so that we are always guaranteed to find a $\tilde{\beta}$ as given in Definition \ref{choose:betatilde}. Towards that we upperbound equation \ref{KLupperbound1} as follows,  

\begin{align}\label{KLupperbound} 
\text{KL}({\cal N}(\A_i,\tilde{\sigma}^2 I) \vert \vert P) \leq -\log\left[  \frac{1}{k_1} \sum_{j=1}^{k_1}    e^{- {\frac 1 {2\tilde{\sigma}^2}}\|\A_i - \B_{\lambda^*,j}\|^2 } \right] \leq -\log\left[  \frac{1}{k_1} \sum_{j=1}^{k_1}    e^{- {\frac 1 {2\sigma^2}}\|\A_i - \B_{\lambda^*,j}\|^2 } \right]
\end{align}

where we have gotten the second inequality by recalling Definitions \ref{choose:sigmatilde}, \ref{choose:betatilde} and  defining, 

\[ \sigma ^2 := \frac{1}{2h \log \left ( {2dh} / {\delta} \right )}  \left (  \min \left \{  \frac{\beta_i }{de^{\frac {2}{d-1}}} , \frac{\gamma /8}{e^3dB \beta_i ^{d-1}}  \right \}  \right )^2   = \Bigg(\frac{\beta_i \exp{-\frac{2}{(d-1)}}  }{d\sqrt{2h \log \left ( \frac{2dh}{\delta} \right )}}\Bigg)^2   \min \left \{\frac{\gamma^2}{(8 B \beta_i^{d} \exp{(3-\frac{2}{d-1})} )^2}, 1 \right \}\]

Now we observe the following, 

\begin{enumerate}
\item When $\beta_i \le \Big(\frac{\gamma}{2B}\Big)^{1/d}$ this implies $\norm{f_{\A_i}(\x)} \le \frac{\gamma}{2}$ which implies $\hat{L}_\gamma = 1$ by definition. Therefore equation \ref{prefinal} holds trivially. 

\item  We can ask when is it that the upperbound on the KL term given in equation \ref{KLupperbound} in terms of this $\sigma$ such that the resultant upperbound on $\sqrt{\frac{2 \text{KL}({\cal N}(\A_i,\tilde{\sigma}^2 I) \vert \vert P)}{m-1}}$ (when substituted into equation \ref{prefinal}) greater than $1$ i.e the range of $\beta_i$ for which the following inequality holds (and thus making the inequality \ref{prefinal} hold trivially by ensuring that the ensuing upperbound on the RHS of it is greater than $1$), 

\begin{align}\label{suff} 
1 \leq \frac {-\log\left[  \frac{1}{k_1} \sum_{j=1}^{k_1}    e^{- {\frac 1 {2\sigma^2}}\|\A_i - \B_{\lambda^*,j}\|^2 } \right] }{m-1} = \frac{1}{m-1} \log \Big [ \frac{1}{\frac{1}{k_1} \sum_{j=1}^{k_1}    e^{- {\frac 1 {2\sigma^2}}\|\A_i - \B_{\lambda^*,j}\|^2 }} \Big ]
\end{align}

\begin{align*}
\log \Big [ \frac{1}{\frac{1}{k_1} \sum_{j=1}^{k_1}    e^{- {\frac 1 {2\sigma^2}}\|\A_i - \B_{\lambda^*,j}\|^2 }} \Big ] &\geq  \log \Big [ \frac{1}{\frac{1}{k_1} \sum_{j=1}^{k_1}    \max_{i,j} \{ e^{- {\frac 1 {2\sigma^2}}\|\A_i - \B_{\lambda^*,j}\|^2 } \} } \Big ]\\
&= \log \Big [ \frac{1}{\frac{1}{k_1} \sum_{j=1}^{k_1}     e^{- {\frac 1 {2\sigma^2}} \min_{i,j} \{ \|\A_i - \B_{\lambda^*,j}\|^2 \}  } } \Big ]
\end{align*}

Now we invoke the definition of $d_{\min}$ to get,  

\[ \log \Big [ \frac{1}{\frac{1}{k_1} \sum_{j=1}^{k_1}     e^{- {\frac 1 {2\sigma^2}} \min \{ \|\A_i - \B_{\lambda^*,j}\|^2 \}  } } \Big ] \geq \frac{d_{\min}^2}{2\sigma^2}  \] 

Thus substituting back into equation \ref{suff} we see that a sufficient condition for \ref{suff} to be satisfied is, 

\[ \Bigg(\frac{\beta_i \exp{-\frac{2}{(d-1)}}  }{d\sqrt{2h \log \left ( \frac{2dh}{\delta} \right )}}\Bigg)^2   \min \left \{\frac{\gamma^2}{(8 B \beta_i^{d} \exp{(3-\frac{2}{d-1})} )^2}, 1 \right \}  = \sigma^2 \leq  \frac{d_{\min}^2}{2(m-1)} \] 


A further sufficient condition for the above to be satisfied is, 

\[  \Bigg(\frac{\beta_i \exp{-\frac{2}{(d-1)}}  }{d\sqrt{2h \log \left ( \frac{2dh}{\delta} \right )}}\Bigg)^2 \cdot \frac{\gamma^2}{(8 B \beta_i^{d} \exp{(3-\frac{2}{d-1})} )^2} \leq  \frac{d_{\min}^2}{2(m-1)} \]

The above leads to the constraint, $\beta_i \geq \Big ( \frac{\gamma \sqrt{2(m-1)}}{(8Be^3d d_{\min})\sqrt{2h\log \Big ( \frac {2dh}{\delta} \Big )}} \Big )^{\frac {1}{(d-1)}}$
\end{enumerate}

Thus we combine the two points above to see that a relevant interval of $\beta_i$ is, 

\[ \Big [  \Big ( \frac{\gamma}{2B} \Big )^{\frac 1 d}   , \left ( \frac{\gamma \sqrt{2(m-1)}}{(8Be^3d d_{\min})\sqrt{2h\log \Big ( \frac {2dh}{\delta} \Big )}} \right )^{\frac {1}{(d-1)}} \Big ] \] 

We recall that the parameters have been chosen so that the above interval is non-empty. 

We note that if we want a grid on the interval $[a,b]$ s.t for every value $x \in [a,b]$ there is a grid-point $g$ s.t $\vert x - g \vert \leq \frac{x}{d}$ then a grid size of $\frac{bd}{2a}$ suffices.  \footnote{If $g$ is the grid point which is the required approximation to $x$ i.e $ \vert x - g \vert \leq \frac{x}{d} \implies x \in \Big (\frac{d}{d+1}g , \frac{d}{d-1}g  \Big )$ Since $a \leq g \implies \frac{2da}{d^2-1}  \leq \Big ( \frac{d}{d-1} - \frac{d}{d+1} \Big ) \tilde{\beta}$. So $\frac{2da}{(d^2-1)}$ is the smallest grid spacing that might be needed and hence the maximum number number of grid points needed is $\frac{(b-a)(d^2-1)}{2ad} < \frac{(b-a)d}{2a} < \frac{bd}{2a}$  } Hence the a grid of the following size $K_1$ suffices for us to capture with needed accuracy all the possible values of $\beta_i$, 

\[ K_1 = \frac{d}{2} \times \frac{\left  ( \frac{\gamma \sqrt{2(m-1)}}{(8Be^3d d_{\min})\sqrt{2h\log \Big ( \frac {2dh}{\delta} \Big )}} \right  )^{\frac {1}{(d-1)}}}{\Big ( \frac{\gamma}{2B} \Big )^{\frac 1 d}  }   \]

And thus we have specified the ``beta-grid" as mentioned in Definition \ref{ourgrid}.   
\end{proof} 

\section{Conclusion} 
We conclude by reporting two other observations that have come to light from the experiments above. {\it Firstly,}  We have also done experiments (not reported here) where we have used the CIFAR data as a binary classification task and there we observed that the angular deflection under training is significantly lower than whats reported above for CIFAR$-10$. In such situations when this deflection is lower the relative advantage of our bound over \cite{neyshabur2017pac}'s bound is even greater. {\it Secondly,} We have additionally also observed that the maximum angle between $\A$ and any of the $\A_i$s is typically $40-60\%$ larger than the corresponding angular spread of the intial cluster i.e the maximum angle between $\B$ and any of the $\B_i$s. Note that all the final cluster nets are approximately of the same accuracy. Thus nets initialized close by to each other often seem to not end up as close post-training even when trained to the same accuracy on the same data and using the same algorithm. We believe that this dispersion behaviour of nets warrants further investigation. 

Given the demonstrated advantages of our PAC-Bayesian bounds on neural nets we believe that these observations deserve further investigation and being able to theoretically explain them and incorporate them into the PAC-Bayesian framework might contribute towards getting even better bounds.
 \begin{subappendices}
 
\chapter*{\vspace{10pt} Appendix To Chapter \ref{chapPAC}}\label{app:PAC} %

\section{The KL upperbounds}
The normal distributions ${\cal N}_{\A,\sigma^2}$ as stated in definition \ref{Bdef} can be made explicit as,
\vspace{-0.5mm} 
\begin{align}\label{normal}
\log {\cal N}_{\left (\A,\sigma^2 \right )} (\w) = - \frac {1}{2} \left [ \frac{\norm{\w - \A}^2}{\sigma^2} + \dim(\A) \log \left (2\pi\sigma^2)\right ) \right ]
\end{align}

\begin{theorem}\label{ourkl}
Assume being given distributions $f$ and $GM$ on $\R^n$. $f$ is  Gaussian and has mean ${\vec \mu} = \A$ and covariance $\Sigma_f = \sigma_f^2 I_{n}$ for some $\sigma_f >0$. While $\text{GM}$ is a Gaussian mixture with $k_1$ Gaussians $\{ f_{\text{GM},r} \}_{r=1,\ldots,k_1}$ with weights, $\{ a_{\text{GM},i} \geq 0 \}_{i=1,..,k_1}$ and $\sum_{i=1}^{k_1}  a_{\text{GM},i} = 1$ s.t the $k_1$ components have means as $\{ {\vec \mu}_{\text{GM},r} = \B_r \}_{r=1,\ldots,k_1}$ and the covariance matrix of the components as $\Sigma_{\text{GM}} := \sigma_{\text{GM}}^2I_{n}$ for some $\sigma_2 >0$. Then we have the following upperbound, 

\[ \text{KL} (f \Vert GM) \leq  \frac{n}{2}\left(\frac{\sigma_f^2}{\sigma_{\text{GM}}^2}-1\right)+n\log(\frac {\sigma_{\text{GM}}}{\sigma_f})-\log\left[  \sum_{r=1}^{k_1} a_{\text{GM},r}   e^{- {\frac 1 {2\sigma_{\text{GM}}^2}}\|\A - \B_r\|^2 } \right] \] 
\end{theorem}

\begin{proof} 
We recall the following theorem, 
\begin{theorem}[Durrieu-Thiran-Kelly (ICASSP $2012$)]\label{mixkl}
Given the definitions of $f$ and $\text{GM}$ as above but with $\Sigma_f$ and $\Sigma_{GM}$ being generic we have as upperbound for the KL divergence between the two distributions,

\[ \text{KL}(f \Vert \text{GM}) \leq   \left ( H(f) + \log \frac {\exp(L_f(f))}{\sum_{r=1}^{k_2} a_{\text{GM},r} e^{-D_{KL} (f \Vert f_{\text{GM},r})}}\right ),\]
where $H(f) := \mathbb{E} \left [ -\log(f(x)) \right ] =: -L_f(f)$
\qed
\end{theorem}

Now using the definition of $\mu_{\text{GM},r}$ and $\Sigma_{\text{GM}}$ and known expression for the $\text{KL}$ divergence between $2$ Gaussians the upperbound given above simplifies as,

\begin{align*}
\text{KL}(f \Vert \text{GM}) &\le -\log\left[  \sum_{r=1}^{k_1} a_{\text{GM},r} e^{-\text{KL} (f \Vert f_{\text{GM},r})} \right] \\
&=-\log\left[  \sum_{r=1}^{k_1} a_{\text{GM},r}  \sqrt{ \frac{e^n\det(\Sigma_{f})}{\det(\Sigma_{\text{GM}})\det(e^{\Sigma_{\text{GM}}^{-1} \Sigma_{f}})}} e^{- {\frac 1 2}({\vec \mu} - {\vec \mu}_{\text{GM},r})^T\Sigma_{\text{GM}}^{-1}({\vec \mu} - {\vec \mu}_{\text{GM},r}) } \right]
\end{align*}

Now we recall the definitions of $\A, \{\B_r\}_{r=1,\ldots,k_1}, \sigma_f$ and $\sigma_{\text{GM}}$ to further simplify the above to get, 

\begin{align*}
\text{KL}(f \Vert \text{GM}) &\le -\log \sqrt{ \frac{e^n\det(\sigma_f^2I)}{\det(\sigma^2_{\text{GM}} I)\det \left (e^{(\sigma^2_{\text{GM}}I)^{-1} (\sigma_f^2I)}\right )}}
-\log\left[  \sum_{r=1}^{k_2} a_{\text{GM},r}   e^{- {\frac 1 2}(\A - \B_r)^T(\sigma^2_{\text{GM}}I)^{-1}(\A - \B_r)} \right]\\
&=\log \sqrt{ \frac{\det (\sigma^2_{\text{GM}} I)\det \left (e^{(\sigma^2_{\text{GM}}I)^{-1} (\sigma_f^2I)}\right )}{e^n\det(\sigma_f^2I)}}-\log\left[  \sum_{r=1}^{k_1} a_{\text{GM},r}   e^{- {\frac 1 {2\sigma_{\text{GM}}^2}}\|\A - \B_r\|^2 } \right]\\
&\leq \frac{n}{2}\left(\frac{\sigma_f^2}{\sigma_{\text{GM}}^2}-1\right)+n\log(\frac {\sigma_{\text{GM}}}{\sigma_f})-\log\left[  \sum_{r=1}^{k_1} a_{\text{GM},r}   e^{- {\frac 1 {2\sigma_{\text{GM}}^2}}\|\A - \B_r\|^2 } \right]
\end{align*}
And thus the required theorem is proven.  
\end{proof} 

\section{Data-Dependent Priors}
We continue in the same notation as used in Theorem \ref{pacmain} and we prove the following theorem. 
\begin{theorem}\label{thm:b1}
  Suppose we have $K$ prior distributions $\{ \pi_i \}_{i=1,\ldots,K}$ s.t for some $\delta >0$ the following inequality holds for each $\pi_i$ \begin{equation}\label{each}
\mathbb{P}_{S \sim {\cal D}^m(Z)} \Bigg[ \forall Q\, 
\E_{h \sim Q} [ L(h) ] \leq  \E_{h \sim Q} [ \hat{L}(h) ] + \sqrt{\frac{\text{KL} (Q \| \pi_i ) + \log \frac {m}{\delta} }{m-1}} \Bigg] \geq 1- {\delta}.
\end{equation}
Then, 
\begin{equation}\label{eachres}
\mathbb{P}_{S \sim {\cal D}^m(Z)} \Bigg[\forall i \in \{1,\ldots,K \} \, \forall Q\, 
\E_{h \sim Q} [ L(h) ] \leq  \E_{h \sim Q} [ \hat{L}(h) ] + \sqrt{\frac{\text{KL} (Q \| \pi_i ) + \log \frac {m}{\delta} }{m-1}} \Bigg] \geq 1-K\delta.
\end{equation}
\end{theorem}
\begin{proof}
Let $f$ be a  real-valued function on the space of distributions and data sets. Let $E_Q = \{ S \mid f(S,Q) > 0\}$. And note that $\cap_{Q} E_Q = \{S \mid \forall Q f(S,Q) > 0\}$. Therefore $\mathbb{P}_S[S \mid \forall Q f(S,Q) > 0] = \mathbb{P}_S[\cap_{Q} E_Q]$ regardless of the distribution on $S$.

Consider $K$ functions $f_1, \ldots, f_K$ s.t for some $\delta \in [0, 1]$ we have, 

\[ \forall i \in \{1,\ldots, K \} \; \mathbb{P}_S[S \mid \forall Q \; f_i(S,Q) > 0] \ge 1-\delta \] 

Now we define the events $E_{i,Q} = \{ S \mid f_i(S,Q) > 0\}$ for each $i \in \{1,\ldots,K\}$. Thus we can deduce the following:
\begin{equation}\label{A_Label}
  \begin{aligned}
    &\forall i \in \{1,\ldots,K\}\; \mathbb{P}_S[S \mid \forall Q \; f_i(S,Q) > 0] &&\ge 1-\delta \\
    \implies& \forall i \in \{1,\ldots,K\}\; \mathbb{P}_S[\cap_{Q} E_{i,Q}] &&\ge 1-\delta & \text{(as discussed in the beginning of the proof)}\\ 
    \implies& \forall i \in \{1,\ldots, K\}\; \mathbb{P}_S[\cup_{Q} E^c_{i,Q}] &&\le \delta &\\
    \implies& \mathbb{P}_S[\cup_{i=1}^K \Big ( \cup_{Q} E^C_{i,Q} \Big )] &&\le K\delta &\\
    \implies& \mathbb{P}_S[\cap_{i=1}^K \cap_{Q} E_{i,Q}]
    &&\ge 1-K\delta\\
    \implies& \mathbb{P}_S[S \mid \forall Q \; \forall i \in \{1,\ldots, K\}\, f_i(S,Q) > 0] &&\ge 1-K\delta
\end{aligned}
\end{equation}

Now we use $f_i(S,Q) {=} \E_{h \sim Q} [ \hat{L}(h) ] {-} \E_{h \sim Q} [ L(h) ] {+} \sqrt{\frac{\text{KL} (Q \| \pi_i ) + \log \frac {m}{\delta} }{m-1}}$ in equation~\eqref{A_Label} and use condition~\eqref{each} to get result \eqref{eachres}.
\end{proof}

The above theorem encapsulates what is conventionally called as using a ``data-dependent prior". This is because if we can construct a list of $K$ priors and work with $\delta < \frac{1}{K}$ then the above theorem lets us choose in a data (i.e the training data $S$) dependent way not just the posterior distribution $Q$ but also the prior $\pi_i$ from the list and still assures us a high probability upperbound on the difference between the true risk and the empirical risk of the stochastic classifier. 


\section{Proof of Theorem~\ref{precisepac}}\label{proof:precisepac} 

\begin{proof}
Given a predictor weight $\w$, for explicitness we will denote by $\mu_\w$ and $\tilde{\mu}_\w$ what were called $\mu$ and $\tilde{\mu}$ in the theorem statement.  

We will first isolate its set of ``good" perturbations i.e we define the following set $S_\w$ as,
\[ S_\w = \left \{ \w + \u \mid  \max_{\x \in S} \norm{ f_{\w +\u} (\x) - f_\w (\x)} _\infty < \frac{\gamma}{4}   \right \} \]

Corresponding to the predictor weight $\w$, let $\mu_\w'$ be a distribution on the weights s.t the required condition holds i.e,

\[ \mathbb{P}_{\u \sim \mu_\w'} \left [ \max_{\x \in S} \norm{f_{\w + \u}(\x) - f_\w (\x) }_\infty < \frac{\gamma}{4} \right ] \geq \frac{1}{2} \]

Let $\w + \u \sim \mu_\w$ when $\u \sim \mu_\w'$.

Now we define the quantity, $Z(\mu_\w) := \mathbb{P}_{\w + \u \sim \mu_\w} [ \w + \u \in S_\w ]$. From the definition it follows that, $Z(\mu_\w) \geq \frac{1}{2}$.
Now let us define another distribution $\tilde{\mu}_\w$ over the set of predictor's weights s.t the p.d.fs are related as follows (in the following we overload the notation of the distribution to also denote the corresponding p.d.f),
 
\[ \tilde{\mu}_\w (x) = \frac{\mu_\w(x)}{Z(\mu_\w)}\delta_{x \in S_\w }  \]

Thus $\tilde{\mu}_\w$ is supported on $S_\w$. From the above definition it follows that if $\w + \tilde{\u} \sim \tilde{\mu}_\w$ then $\max_{\x \in S} \norm{ f_{\w + \tilde{\u}} (\x) - f_\w (\x)}_\infty < \frac{\gamma}{4}$ which is equivalent to $\max_{i \in \{1,2,..,k\}, \x \in S} \norm { f_{\w + \tilde{\u}}(\x)[i] - f_{\w}(\x)[i] } < \frac{\gamma}{4}$. This in turn implies,

\begin{equation}
    \max_{i,j \in \{1,2,..,k\}, \x \in S} \norm { (f_{\w + \tilde{\u}}(\x)[i] - f_{\w + \tilde{\u}} (\x)[j]) - (f_{\w}(\x)[i] - f_{\w}(\x)[j])} < \frac{\gamma}{2}\label{eq:appb1}
\end{equation}

Which in turn implies the following inequality,
\footnote{
We give the proof in this footnote for convenience. Let $j$ be arbitrary, let $i = \arg\max_{i \ne j}f_{\w + \tilde{\u}}(\x)[i]$.
Let $S_j = \{x \mid f_{\w}(\x)[i] - f_{\w}(\x)[j] \geq -\gamma/2\}$,  and $S'_j = \{x \mid f_{\w+\tilde{\u}}(\x)[i] - f_{\w + \tilde{\u}}(\x)[j] \geq  0\}.$  Clearly $S_j' \subseteq S_j$ 
because  $\displaystyle f_{\w+\tilde{\u}}(\x)[i] - f_{\w + \tilde{\u}}(\x)[j] \geq  0 \implies f_{\w}(\x)[i] - f_{\w}(\x)[j] \geq -\gamma/2$ from \eqref{eq:appb1}. 
~\\ \\
Now recall that we can write, $L_{\gamma/2}(f_\w) = \E_{y}\big[ \E_{\x}[\mathbf{1}(f_\w(\x)[y] \leq \gamma/2 + \max_{i \neq y} f_\w(\x) [i]) \mid y ]$ and $L_{0}(f_{\w + \tilde{\u}}) = \E_{y}\big[ \E_{\x}[\mathbf{1}(f_{\w + \tilde{\u}}(\x)[y] \leq  \max_{i \neq y} f_{\w + \tilde{\u}}(\x) [i]) \mid y ]$ Hence equation \eqref{eq:appb3} follows by noting that the support of the indicator function in $L_{0}(f_{\w + \tilde{\u}})$ is contained in the support of the indicator function in $L_{\gamma/2}(f_\w)$ because $S_j' \subseteq S_j$. 
}


\begin{align}
\hat{L}_{0} (f_{\w + \tilde{\u}}) &\leq \hat{L}_{\gamma/2} (f_\w)\label{eq:appb3}
\end{align}

Similar to $\tilde{\mu}$ we define the following,

\[
\tilde{\mu}_{\w}^c(\x)=\frac{1}{1-Z(\mu_\w)}\mu_\w(\x)\delta_{\x\in S_w^c}
\]

Thus substituting the expressions for $\tilde{\mu}_{\w}$ and $\tilde{\mu}_{\w}^c$ for the given prior $P$ we have,
\begin{equation}
\begin{split}
&\text{KL}(\mu_{\w} \mid P) \\
&=\int \mu_{\w} \log \left (\frac {P}{\mu_{\w}} \right ) = \int_{S_w} \mu_{\w}\log \left (\frac {P}{\mu_{\w}} \right ) + \int_{S_w^c}\mu_{\w}\log \left (\frac {P}{\mu_{\w}} \right )\\
&=\int Z(\mu_\w)\tilde{\mu}_{\w}\log \left (\frac {P}{Z(\mu_\w)\tilde{\mu}_{\w}} \right ) +\int (1-Z(\mu_\w))\tilde{\mu}_{\w}^c\log \left (\frac {P}{(1-Z(\mu_\w))\tilde{\mu}_{\w}^c} \right ) \\
&=Z(\mu_\w) \left [-\log Z(\mu_\w) + \int \tilde{\mu}_{\w} \log  \left (\frac {P}{\tilde{\mu}_{\w}} \right )   \right ]\\ 
&\qquad +(1-Z(\mu_\w))\left [-\log (1-Z(\mu_\w)) + \int \tilde{\mu}_{\w}^c \log \left ( \frac{P}{\tilde{\mu}_\w^c} \right) \right ]\\
&=Z(\mu_\w) \text{KL}(\tilde{\mu}_\w \mid  P)+ (1-Z(\mu_\w))\text{KL}(\tilde{\mu}_\w^c \mid P)\\
&\qquad - Z(\mu_\w)\log Z(\mu_\w) - (1-Z(\mu_\w))\log(1-Z(\mu_\w))\\
\implies \text{KL}(\tilde{\mu}_\w \mid P) &= \frac{1}{Z(\mu_\w)}\bigg \{ \text{KL}(\mu_\w \mid P)-(1-Z(\mu_\w))\text{KL}(\tilde{\mu}_\w^c \mid P)\\
&\qquad \qquad +Z(\mu_\w)\log Z(\mu_\w) +(1-Z(\mu_\w))\log (1-Z(\mu_\w)) \bigg \}
\end{split}
\end{equation}

We recall that for any $Z \in [\frac 1 2,1]$ we have, $|Z\log Z+(1-Z)\log(1-Z)|\le 1$ and $\text{KL}(\tilde{\mu}_\w^c\|P)\ge 0$. Thus we have,

\begin{align}\label{klineq}
\text{KL}(\tilde{\mu}_\w\|P) \le \frac{1}{Z(\mu_\w)}\left (\text{KL}(\mu_\w\|P)+1 \right ) \le 2 \left (\text{KL}(\mu_\w\|P)+1 \right )
\end{align}

We remember that above we had sampled the weights of the perturbed net $f_{\w + \tilde{\u}}$ as $\w + \tilde{\u} \sim \tilde{\mu}_\w$. Now we write the highly likely event guaranteed by PAC-Bayesian bounds in Theorem \ref{pacmain}, for the margin loss $L_{\frac \gamma 2}$ evaluated on the predictor $f_\w$ for some prior $P$ and a (data dependent) choice of posterior $\tilde{\mu}_\w$. Further we invoke equation \eqref{eq:appb3} and \eqref{klineq} on that, to get the following,

\begin{align*}
    \mathbb{E}_{\w + \tilde{\u} \sim \tilde{\mu}_\w} [ L_{0} (f_{\w + \tilde{\u}}) ] &\leq \mathbb{E}_{\w + \tilde{\u} \sim \tilde{\mu}_\w}  [ \hat{L}_{0} (f_{\w + \tilde{\u}})] + \sqrt{\frac{\text{KL}(\tilde{\mu}_\w \vert \vert P)  + \log \frac{m}{\delta}}{2(m-1)}} \\
    &\leq \hat{L}_{\gamma/2}(f_\w) + \sqrt{\frac{\text{KL}(\tilde{\mu}_\w \vert \vert P)  + \log \frac{m}{\delta}}{2(m-1)}} \\
    &\leq \hat{L}_{\gamma/2}(f_\w) + \sqrt{\frac{2\text{KL}(\mu_\w \vert \vert P) + 2  + \log \frac{m}{\delta}}{2(m-1)}} = \hat{L}_{\gamma/2}(f_\w) + \sqrt{\frac{\text{KL}(\mu_\w \vert \vert P) + 1 + \frac{1}{2}\log \frac{m}{\delta}}{m-1}} \\
    &\leq \hat{L}_{\gamma/2}(f_\w) + \sqrt{\frac{\text{KL}(\mu_\w \vert \vert P) + \log \frac{3m}{\delta}}{m-1}}
\end{align*}

where in the last line we have used,  $1 + \frac{1}{2}\log \frac{m}{\delta} < 1 + \log \frac{m}{\delta} < \log 3 + \log \frac{m}{\delta}$

\end{proof}


\newpage 
\section{Proof of Theorem \ref{NBSTheorem}}\label{NBSProof}

\begin{proof}

Firstly we observe the following Theorem~\ref{thmg} which is a slight variation of a lemma in \cite{neyshabur2017pac}. The proof of Theorem~\ref{thmg} follows from exactly the same arguments as was needed to prove Theorem~\ref{precisepac}.

\begin{theorem}\label{thmg}
Let $f_\w : \chi \rightarrow \R^k$ be any predictor with parameters $\w$ and lets use the margin loss as defined in equation \ref{loss}. Let $P$ be any distribution (the ``data-independent prior") on the space of parameters of $f$ and ${\cal D}$ be a distribution on $\chi$. 
Let it be true that for some $\gamma >0$, we know of distributions $\mu_\w'$ and $\mu_\w$ (depending on the weight $\w$ of the given predictor) on the space of parameters of the predictor s.t, 
\begin{align}\label{app-ucond}
\mathbb{P}_{\u \sim \mu_\w'} \left [ \sup_{\x \in \chi} \norm{f_{\w + \u}(\x) - f_\w (\x) }_\infty < \frac{\gamma}{4} \right ] \geq \frac{1}{2} \text{ and } \w + \u \sim \mu_\w
\end{align}
Then for any $\delta \in [0,1]$ the following guarantee holds,
\begin{align}
\nonumber \mathbb{P}_{\chi \sim {\cal D}^m(\chi)} \Bigg [ &\forall \w \text{ and corresponding } \mu_\w \text{ s.t condition \ref{app-ucond} holds,} \exists\, \tilde{\mu}_\w \text{ s.t }\\
&\mathbb{E}_{\w + \tilde{\u} \sim \tilde{\mu}_\w} [ L_{0} (f_{\w + \tilde{\u}}) ] \leq \hat{L}_{\frac \gamma 2} \left (f_{\w} \right ) + \sqrt{\frac{\text{KL}(\mu_\w \vert \vert P) + \log \frac{3m}{\delta}}{m-1}} \Bigg ]\geq 1 - \delta \label{app-eq:12}
\end{align}
\qed
\end{theorem}


The following Theorem~\ref{thmg2} from \cite{neyshabur2017pac} is a bound for neural net functions under controlled perturbations and we state it without proof.

\begin{theorem}[{\bf Neural net perturbation bound of \cite{neyshabur2017pac}}]\label{thmg2}
Let us be given a depth $d$ neural net, $f_\w$, with width $h$ and weight vector $\w$ which is mapping, $B_n(B) \rightarrow \R^k$ where $B_n(B)$ is the radius $B$ ball around the origin in $\R^n$. Now cosider a perturbation on the weights given by, $\u = \textrm{vec}(\{ U_\ell \}_{\ell = 1}^d)$ s.t $\norm{U_\ell}_2 \leq \frac{1}{d} \norm{W_\ell}_2$. Then we have for all $\x \in B_n(B)$,
\begin{equation}
    \norm{f_{\w + \u} (\x) - f_\w (\x) } _2 \leq eB ( \prod_{\ell = 1}^d \norm{W_\ell}_2) \sum_{\ell = 1}^d \frac{\norm{U_\ell}_2}{\norm{W_\ell}_2} \label{eq:myl1}
\end{equation}
\qed
\end{theorem}

Now, for some $\sigma >0$ consider the random variable $\u \sim {\cal N}(0,\sigma^2 I)$ where $\u$ is imagined as the vector of the weights of the neural net $f$ in the above theorem. Let $\{ U_\ell \in \R^{h \times h} \}_{\ell = 1}^d$ be the matrices of the neural net corresponding to $\u$. 

We can define matrices $\{ \B_p \in \R^{h \times h} \mid p = 1,\ldots,h^2 \}$ s.t each $\B_p$ has $\sigma$ in an unique entry of it and all other entries are $0$. Then it follows that as random matrices, $U_{\ell} = \sum_{p=1}^{h^2} \gamma_p \B_p$ with $\gamma_p \sim {\cal N}(0,1)$. We note that $\norm{\sum_{p=1}^{h^2} \B_p \B_p^\top} = h \cdot \sigma^2$ since $h$ is the largest eigenvalue of an all ones $h-$dimensional square matrix. Now we invoke Corollary $4.2$ of \cite{tropp2012user} here to get for any $t > 0$,

\begin{align}\label{tropp2} 
 \mathbb{P}_{U_{\ell}} \left [ \norm{U_{i,\ell}}_2 > t \right ] \leq 2he^{-\frac {t^2}{2h\sigma^2} }  
\end{align}

Using union bound for the $d$ layer matrices of $\u$ we get,
\[ \mathbb{P}_{\{U_\ell \sim {\cal N}(0,\sigma^2 I_{h \times h})\}_{\ell = 1,\ldots,d}} [ \exists i \text{ s.t } \norm{U_\ell} > t] \leq 2dhe^{-\frac{t^2}{2h\sigma^2}}\]

This is equivalent to,

\begin{align}
\mathbb{P}_{\{U_\ell \sim {\cal N}(0,\sigma^2 I_{h \times h})\}_{\ell = 1,\ldots,d}} [\forall i \norm{U_\ell} \leq t ] \geq 1 - 2dhe^{-\frac{t^2}{2h\sigma^2}}
\end{align}

If $t = \sigma \sqrt{2h\log(4dh)}$ then $1 - 2dhe^{-\frac{t^2}{2h\sigma^2}} = \frac{1}{2} $ and so we have,

\begin{align}
\mathbb{P}_{\{U_\ell \sim {\cal N}(0,\sigma^{2} I_{h \times h})\}_{\ell = 1,\ldots,d}} [\forall i \norm{U_\ell} \leq \sigma \sqrt{2h\log(4dh)} ] \geq \frac{1}{2}
\end{align}

Now corresponding to the given predictor weight $\w$, let $\beta^d = \prod_{\ell = 1}^d \norm{W_\ell}$.  For the kind of nets we consider i.e the ones with no bias vectors in any of the layers it follows from the definition of $\beta$ that the function computed by the net remains invariant if the layer matrices $W_{i}$ are replaced by $\frac{\beta}{\norm{W_{i}}} W_{i}$. And we see that the spectral norm is identically $\beta$ for each layer in this net with modified wights. So we can assume without loss of generality that $\forall i \, \norm{W_{i}} = \beta $. By using this uniform norm assumption along with the assumption that 


If \begin{equation}\label{first} \sigma \sqrt{2h\log(4dh)} \leq \frac \beta d \end{equation} then we have,

\begin{align}\label{second}
\nonumber \frac{1}{2} &\leq \mathbb{P}_{\{U_\ell \sim {\cal N}(0,\sigma^{2} I_{h \times h})\}_{\ell = 1,\ldots,d}} [\forall i \norm{U_\ell} \leq \sigma \sqrt{2h\log(4dh)} ]\\ &\leq \mathbb{P} \left [ \norm{f_{\w+\u}(\x)-f_\w(\x)} \le e B \beta^{d-1} \sum_{\ell = 1}^{d}\norm{U_\ell} \right]\\
&\leq \mathbb{P} \left [ \norm{f_{\w+\u}(\x)-f_\w(\x)} \le e B d \beta ^{d-1} \sigma \sqrt{2h\log(4dh)} \right ]
\end{align}

Note that the assumption \eqref{first} is required even in the proof by \cite{neyshabur2017pac} even though it is omitted there.

We will choose the prior -- used in the PAC-Bayes bound -- from a finite set of distributions, $\{\pi_i = \mathcal{N}_{\mathbf{0}, \sigma^2(\tilde{\beta_i})}\}_{i=1}^K$, in a data dependent manner. Given $\beta$ corresponding to the trained net $f_{\w}$, suppose $\exists \tilde{\beta} \in \{\tilde{\beta_i}\}_{i=1}^K$ such that $|\beta - \tilde{\beta}| \le \frac \beta d$. $\vert \beta - \tilde{\beta} \vert \leq \frac{\beta}{d}$ also implies that $\frac {\beta^{d-1}}{e} \leq \tilde{\beta}^{d-1} \leq e \beta^{d-1}$. Furthermore, if $\sigma$ satisfies the inequalities \ref{third-a} then the condition \ref{first} will hold. 
\begin{subequations}\label{third}
\begin{align}
\sigma \sqrt{2h\log(4dh)} \leq \frac {\tilde{\beta}} {de^{\frac {1}{d-1}}}\label{third-a} \\
e^2 B d \tilde{\beta}^{d-1} \sigma \sqrt{2h\log(4dh)} \leq \frac \gamma 4 \label{third-b}
\end{align}
\end{subequations}
And from equations~(\ref{second},\ref{third-b}) we get that 
\[\frac 1 2 \leq \mathbb{P} \left [ \norm{f_{\w+\u}(\x)-f_\w(\x)} \le \frac \gamma 4 \right ].\]
Therefore the condition~\ref{app-ucond} in Theorem~\ref{thmg}  is satisfied. Finally we deduce from \eqref{third} that the largest value of $\sigma$ in terms of $\tilde{\beta}$  is,
\begin{align}\label{eq:beta-tilde}
\sigma (\tilde{\beta}) := \min \left \{ \frac{\gamma}{4  e^2 B d \tilde{\beta} ^{d-1} \sqrt{2h\log(4dh)}}, \frac{\tilde{\beta}}{d e^{\frac 1 {d-1}}\sqrt{2h\log(4dh)}} \right \}    
\end{align}

Note that for a given neural net weight $\w$ (and hence the value $\beta$) the inequality event in \ref{app-eq:12} holds trivially in two conditions:

\begin{enumerate}
\item When $\beta \le \Big(\frac{\gamma}{2B}\Big)^{1/d}$ this implies $\norm{f_\w(\x)} \le \frac{\gamma}{2}$ which implies $\hat{L}_\gamma = 1$ by definition. Therefore \eqref{app-eq:12} holds trivially. 

\item From the local sensitivity analysis done above it follows that we can invoke the above theorem with $P = {\cal N}(0,\sigma(\tilde{\beta})^2I)$ and $\mu_\w =  {\cal N}(\w,\sigma(\tilde{\beta})^2I)$. Which gives us the bound \begin{equation}\label{eq:mytwo}\text{KL}(\mu_{\w} \| P) \le \frac{\norm{\w}^2}{2\sigma(\tilde{\beta})^2} = \frac{\sum_{\ell = 1}^d \norm{W_\ell}_F^2}{2 \sigma(\tilde{\beta})^2}.\end{equation} 
Note that in terms of $\beta$ (which can be directly read-off from the given net $f_\w$ in Theorem \ref{NBSTheorem}) we have $\sigma(\tilde{\beta}) \ge \min \left \{ \frac{\gamma}{4  e^3 B d {\beta} ^{d-1} \sqrt{2h\log(4dh)}}, \frac{{\beta}}{d e^{\frac 2 {d-1}}\sqrt{2h\log(4dh)}} \right \} = \frac{\beta\exp(-2/(d-1))}{d\sqrt{2h\log(4dh)}}\min \left \{ \frac{\gamma}{4  e^{3-{\frac 2 {d-1}}} B {\beta} ^{d}}, 1 \right \}$. 
Therefore  $\text{KL}(\mu_{\w} \| P) \le \frac 1 2 \frac {\sum_{\ell = 1}^d \norm{W_\ell}_F^2}{\beta^2} \frac {2d^2h\log(4dh)}{\exp(-\frac{4}{(d-1)})} \frac 1 {\min \left \{ \frac{\gamma^2}{4^2  e^{6-{\frac 4 {d-1}}} B^2 {\beta} ^{2d}}, 1 \right \}} $. 
This upper bound on KL leads to the following upperbound on the square-root term in equation ~\ref{app-eq:12}, 


\[  \sqrt{ \frac {\sum_{\ell = 1}^d \norm{W_\ell}_F^2}{(m-1)\beta^2} \frac {d^2h\log(4dh)}{\exp(-\frac{4}{(d-1)})} \frac {1} {\min \left \{ \frac{\gamma^2}{4^2  e^{6-{\frac 4 {d-1}}} B^2 {\beta} ^{2d}}, 1 \right \}} + \frac{1}{m-1}\log \frac{3m}{\delta}} \]  

We note that for any $d, h \ge 1$ we have (a)  by A.M-G.M inequality $\frac {\sum_{\ell = 1}^d \norm{W_\ell}_F^2}{\beta^2} \ge d \ge 1$ and (b) $\exp(\frac{4}{(d-1)}) d^2h\log(4dh) > 1$ (which is obvious on taking logarithm of the LHS).  Therefore a sufficient condition for quantity above to be greater than $1$ is that we have, $\min \left \{ 1, \left ( \frac{\gamma}{4\cdot B \cdot \beta^d e^{3 - \frac{2}{d-1}}} \right )^2 \right \} \leq \frac{1}{m-1}$. And a sufficient condition for this to be true is that, $\beta \ge \Big(\frac{\sqrt{m-1}\gamma} { 4 \exp(3 - 2 / (d-1)) B}\Big)^{1/d}$. 

\end{enumerate}

From the above two points it follows that it suffices to prove \eqref{eq:nbs} for, 
\[ \beta \in \Bigg[ \Big(\frac{\gamma}{2B}\Big)^{1/d}, \Big(\frac{\sqrt{m-1}\gamma} { 4 \exp(3 - 2 / (d-1)) B}\Big)^{1/d}  \Bigg] \]. 

 We note that if we want a grid on the interval $[a,b]$ s.t for every value $x \in [a,b]$ there is a grid-point $g$ s.t $\vert x - g \vert \leq \frac{x}{d}$ then a grid size of $\frac{bd}{2a}$ suffices.  \footnote{If $g$ is the grid point which is the required approximation to $x$ i.e $ \vert x - g \vert \leq \frac{x}{d} \implies x \in \Big (\frac{d}{d+1}g , \frac{d}{d-1}g  \Big )$ Since $a \leq g \implies \frac{2da}{d^2-1}  \leq \Big ( \frac{d}{d-1} - \frac{d}{d+1} \Big ) \tilde{\beta}$. So $\frac{2da}{(d^2-1)}$ is the smallest grid spacing that might be needed and hence the maximum number number of grid points needed is $\frac{(b-a)(d^2-1)}{2ad} < \frac{(b-a)d}{2a} < \frac{bd}{2a}$  } Hence a grid of the following size $K$ suffices for us, 
 
\[K = \frac{d}{2} \times \Big(\frac {\sqrt{m-1}} { 2 \exp(3 - 2 / (d-1)) }\Big)^{1/d}\]

Thus the theorem we set out to prove follows by invoking Theorem \ref{thm:b1} with the $K$ computed above and recognizing that the set $\{ \pi_{i} \}$ indexed by $i$ there is our set $\{{\cal N}_{(0,\sigma(\tilde{\beta})^2I)}\}$ indexed by the grid point $\tilde{\beta}$ here, $Q$ there is our $\mu_{\w}$ here and the equation \ref{eq:mytwo}  above is a bound on the term $\text{KL}(Q \Vert \pi_i)$ there. 
\end{proof}

\newpage 
\section{The $\epsilon-\gamma$ lowerbound scatter plots from the experiments}\label{epsgam}  

\begin{figure}[htbp]
  \centering
  \includegraphics[width=\linewidth]{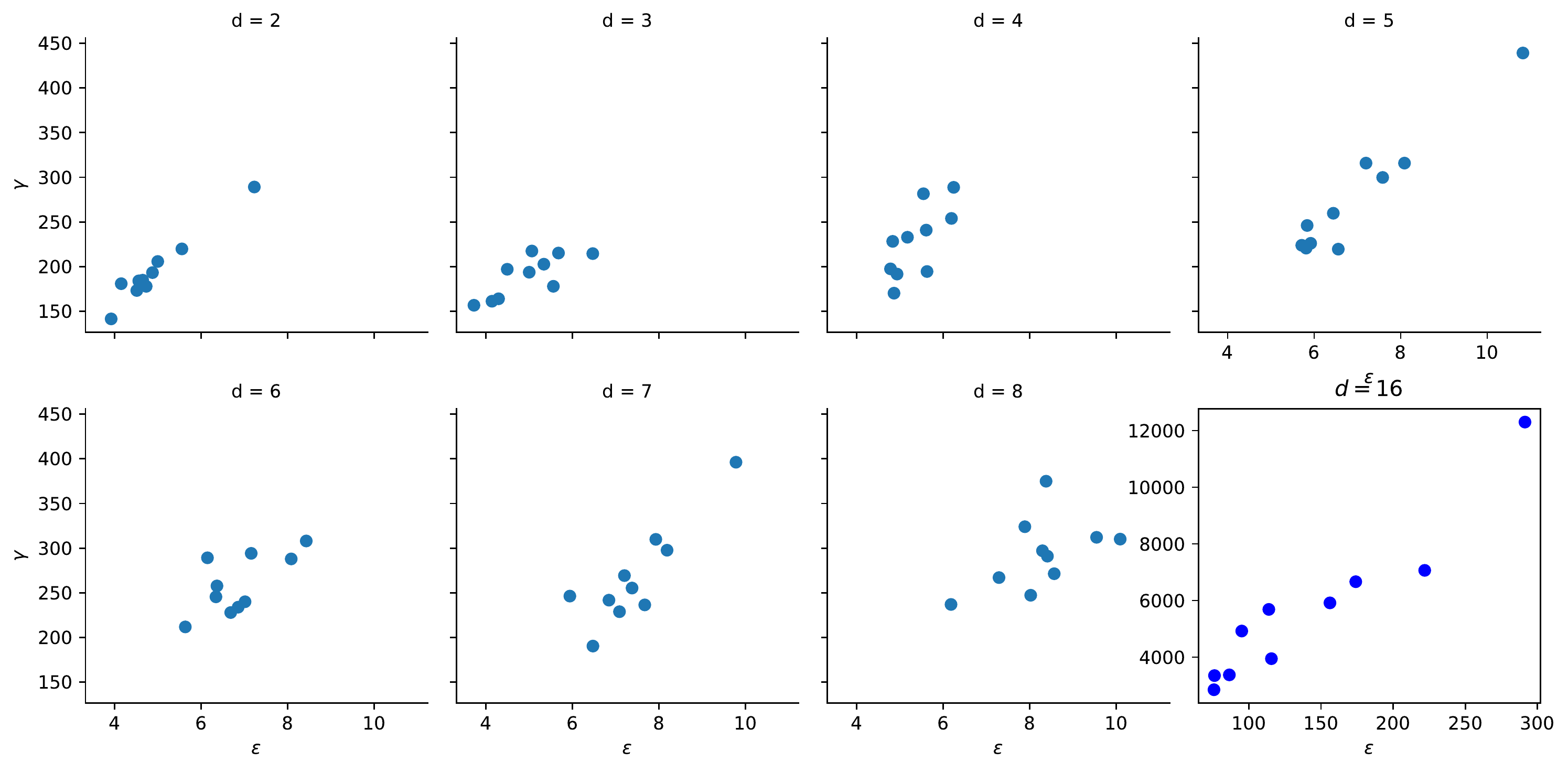}
  \caption{Scatter plots of the lowerbounds on $\epsilon$ and $\gamma$ (as given in definition \ref{nice}) while varying the depth  of the net being trained on the CIFAR-10($10$ trials/seeds for each)}
  \label{fig:dcif-dh-epsgam}
\end{figure}

\begin{figure}[htbp]
  \centering
  \includegraphics[width=0.8\linewidth]{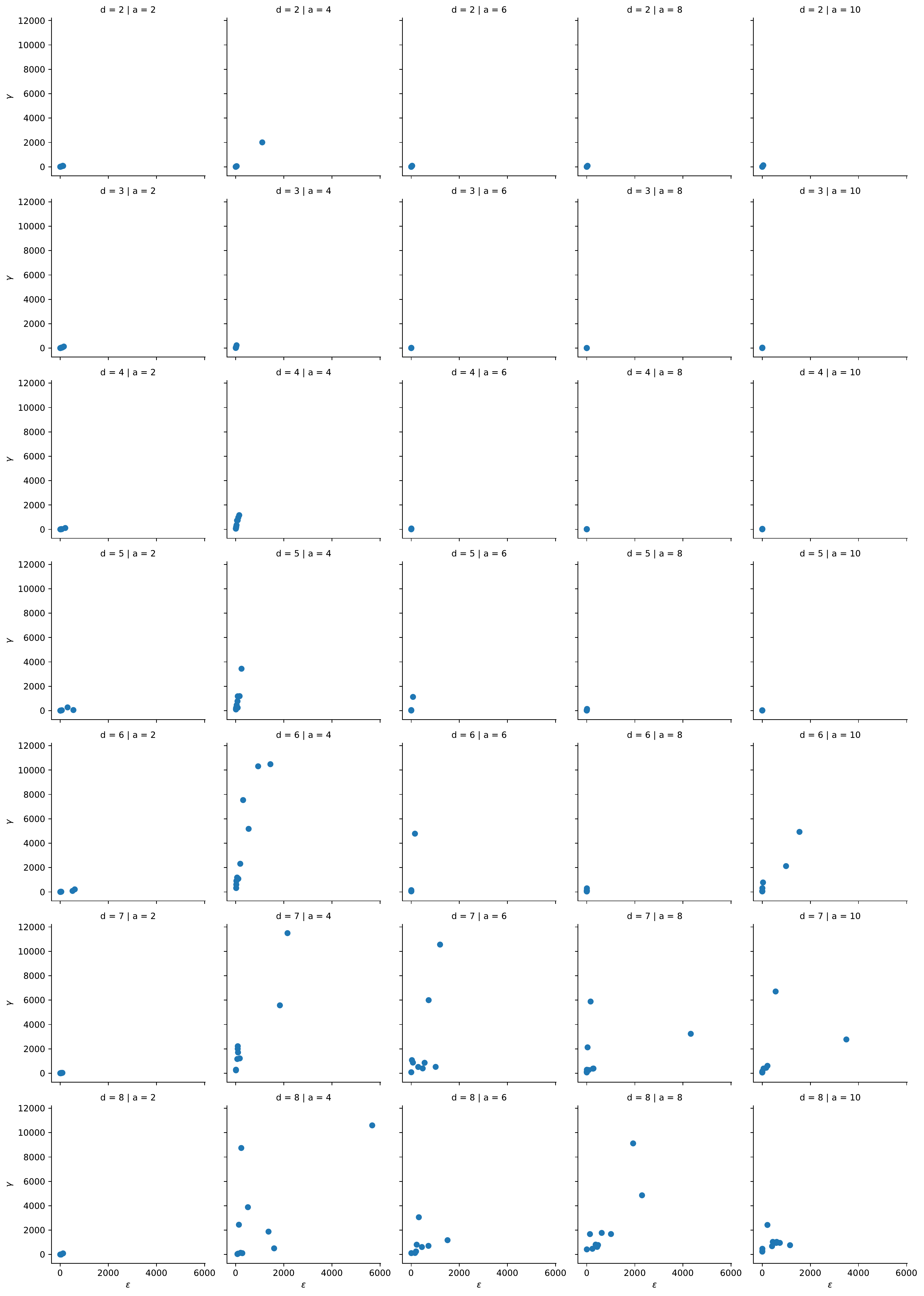}
  \caption{Scatter plots of the lowerbounds on $\epsilon$ and $\gamma$ (as given in definition \ref{nice}) for varying depth $d$ nets trained on the the synthetic dataset for different cluster separation parameter parameter $a$ ($10$ trials/seeds for each)}
  \label{fig:dgmm_depsep_epsgam}
\end{figure}

\newpage 
\section{KDE of the angular deviation during training on the synthetic dataset}\label{app:thetasynth}

\begin{figure}[htbp]
  \centering
  \includegraphics[scale=0.28]{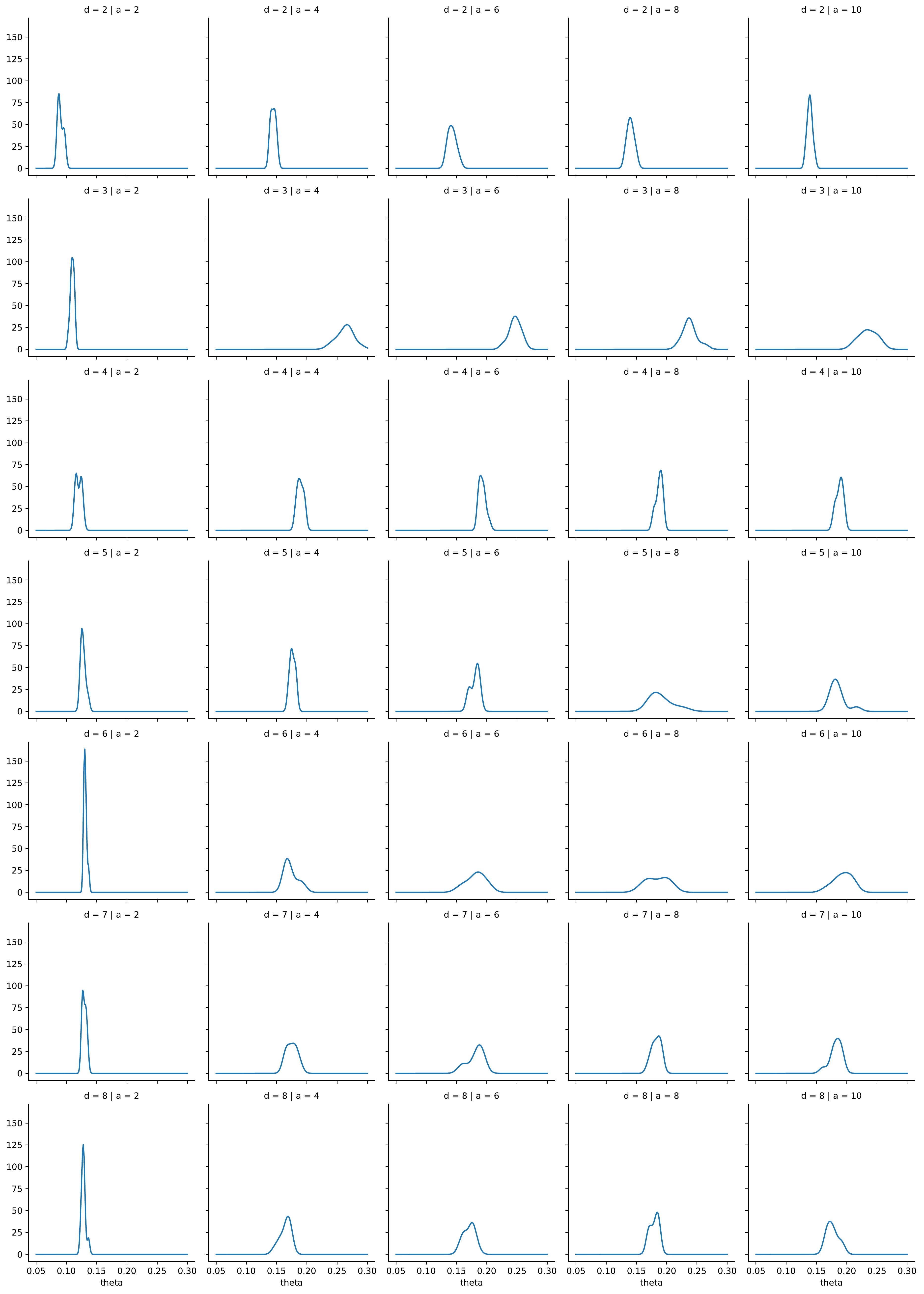}
  \caption{Kernel Density Estimates of the Angular Deviations $\theta$ between the initial and final networks for the differnt net depths $d$ and cluster separation prameters $a$. The $Y$-axis shows the probability density function of the gaussian kernel density estimate.}
  \label{fig:dgmm_kdetheta}
\end{figure}

\end{subappendices}






\printbibliography[heading=bibintoc]

@article{mukherjee2020guarantees,
  title={Guarantees on learning depth-2 neural networks under a data-poisoning attack},
  author={Mukherjee, Anirbit and Muthukumar, Ramchandran},
  journal={arXiv preprint arXiv:2005.01699},
  year={2020}
}

@article{mukherjee2020study,
  title={A Study of Neural Training with Iterative Non-Gradient Methods},
  author={Karmakar, Sayar and Mukherjee, Anirbit},
  journal={http://dx.doi.org/10.2139/ssrn.3767366},
  year={2020}
}

@article{silver2018general,
  title={A general reinforcement learning algorithm that masters chess, shogi, and Go through self-play},
  author={Silver, David and Hubert, Thomas and Schrittwieser, Julian and Antonoglou, Ioannis and Lai, Matthew and Guez, Arthur and Lanctot, Marc and Sifre, Laurent and Kumaran, Dharshan and Graepel, Thore and others},
  journal={Science},
  volume={362},
  number={6419},
  pages={1140--1144},
  year={2018},
  publisher={American Association for the Advancement of Science}
}

@article{fridman2017autonomous,
  title={Mit autonomous vehicle technology study: Large-scale deep learning based analysis of driver behavior and interaction with automation},
  author={Fridman, Lex and Brown, Daniel E and Glazer, Michael and Angell, William and Dodd, Spencer and Jenik, Benedikt and Terwilliger, Jack and Kindelsberger, Julia and Ding, Li and Seaman, Sean and others},
  journal={arXiv preprint arXiv:1711.06976},
  year={2017}
}

@inproceedings{raginsky2017non,
  title={Non-convex learning via Stochastic Gradient Langevin Dynamics: a nonasymptotic analysis},
  author={Raginsky, Maxim and Rakhlin, Alexander and Telgarsky, Matus},
  booktitle={Conference on Learning Theory},
  pages={1674--1703},
  year={2017}
}

@inproceedings{mou2018generalization,
  title={Generalization Bounds of SGLD for Non-convex Learning: Two Theoretical Viewpoints},
  author={Mou, Wenlong and Wang, Liwei and Zhai, Xiyu and Zheng, Kai},
  booktitle={Conference On Learning Theory},
  pages={605--638},
  year={2018}
}

@inproceedings{bartlett2017spectrally,
  title={Spectrally-normalized margin bounds for neural networks},
  author={Bartlett, Peter L and Foster, Dylan J and Telgarsky, Matus J},
  booktitle={Advances in Neural Information Processing Systems},
  pages={6240--6249},
  year={2017}
}

@inproceedings{golowich2018size,
  title={Size-Independent Sample Complexity of Neural Networks},
  author={Golowich, Noah and Rakhlin, Alexander and Shamir, Ohad},
  booktitle={Conference On Learning Theory},
  pages={297--299},
  year={2018}
}

@article{manurangsi2018computational,
  title={The computational complexity of training relu (s)},
  author={Manurangsi, Pasin and Reichman, Daniel},
  journal={arXiv preprint arXiv:1810.04207},
  year={2018}
}

@inproceedings{jin2018local,
  title={On the local minima of the empirical risk},
  author={Jin, Chi and Liu, Lydia T and Ge, Rong and Jordan, Michael I},
  booktitle={Advances in Neural Information Processing Systems},
  pages={4896--4905},
  year={2018}
}

@article{de2018convergence,
  title={Convergence guarantees for RMSProp and ADAM in non-convex optimization and an empirical comparison to Nesterov acceleration},
  author={De, Soham and Mukherjee, Anirbit and Ullah, Enayat},
  journal={ICML 2018 Workshop on Modern Trends in Nonconvex Optimization for Machine Learning (arXiv:1807.06766) https://tinyurl.com/y5hw79vx},
  year={2018}
}

@article{audibert2007combining,
  title={Combining PAC-Bayesian and generic chaining bounds},
  author={Audibert, Jean-Yves and Bousquet, Olivier},
  journal={Journal of Machine Learning Research},
  volume={8},
  number={Apr},
  pages={863--889},
  year={2007}
}

@article{dziugaite2017computing,
  title={Computing nonvacuous generalization bounds for deep (stochastic) neural networks with many more parameters than training data},
  author={Dziugaite, Gintare Karolina and Roy, Daniel M},
  journal={arXiv preprint arXiv:1703.11008},
  year={2017}
}

@inproceedings{dziugaite2018entropy,
  title={Entropy-SGD optimizes the prior of a PAC-Bayes bound: Generalization properties of Entropy-SGD and data-dependent priors},
  author={Dziugaite, Gintare Karolina and Roy, Daniel},
  booktitle={International Conference on Machine Learning},
  pages={1376--1385},
  year={2018}
}

@article{dziugaite2018data,
  title={Data-dependent PAC-Bayes priors via differential privacy},
  author={Dziugaite, Gintare Karolina and Roy, Daniel M},
  journal={arXiv preprint arXiv:1802.09583},
  year={2018}
}

@article{zhou2018non,
  title={Non-vacuous generalization bounds at the imagenet scale: a PAC-bayesian compression approach},
  author={Zhou, Wenda and Veitch, Victor and Austern, Morgane and Adams, Ryan P and Orbanz, Peter},
  year={2018}
}

@article{arora2018stronger,
  title={Stronger generalization bounds for deep nets via a compression approach},
  author={Arora, Sanjeev and Ge, Rong and Neyshabur, Behnam and Zhang, Yi},
  journal={arXiv preprint arXiv:1802.05296},
  year={2018}
}

@article{dey2018approximation,
  title={An Approximation Algorithm for training One-Node ReLU Neural Network},
  author={Dey, Santanu S and Wang, Guanyi and Xie, Yao},
  journal={arXiv preprint arXiv:1810.03592},
  year={2018}
}

@article{boob2018complexity,
  title={Complexity of training relu neural network},
  author={Boob, Digvijay and Dey, Santanu S and Lan, Guanghui},
  journal={arXiv preprint arXiv:1809.10787},
  year={2018}
}

@article{raghu2016expressive,
  title={On the expressive power of deep neural networks},
  author={Raghu, Maithra and Poole, Ben and Kleinberg, Jon and Ganguli, Surya and Sohl-Dickstein, Jascha},
  journal={arXiv preprint arXiv:1606.05336},
  year={2016}
}

@article{serra2017bounding,
  title={Bounding and counting linear regions of deep neural networks},
  author={Serra, Thiago and Tjandraatmadja, Christian and Ramalingam, Srikumar},
  journal={arXiv preprint arXiv:1711.02114},
  year={2017}
}

@article{blum1992training,
  title={Training a 3-node neural network is NP-complete},
  author={Blum, Avrim L. and Rivest, Ronald L.},
  journal={Neural Networks},
  volume={5},
  number={1},
  pages={117--127},
  year={1992},
  publisher={Pergamon}
}

@article{goel2016reliably,
  title={Reliably Learning the ReLU in Polynomial Time},
  author={Goel, Surbhi and Kanade, Varun and Klivans, Adam and Thaler, Justin},
  journal={arXiv preprint arXiv:1611.10258},
  year={2016}
}

@article{dasgupta1995complexity,
  title={On the complexity of training neural networks with continuous activation functions},
  author={DasGupta, Bhaskar and Siegelmann, Hava T. and Sontag, Eduardo},
  journal={IEEE Transactions on Neural Networks},
  volume={6},
  number={6},
  pages={1490--1504},
  year={1995},
  publisher={IEEE}
}

@book{shalev2014understanding,
  title={Understanding machine learning: From theory to algorithms},
  author={Shalev-Shwartz, Shai and Ben-David, Shai},
  year={2014},
  publisher={Cambridge university press}
}

@book{royden,
	Author = {H.L. Royden and P.M. Fitzpatrick},
	Publisher = {Prentice Hall},
	Title = {Real Analysis},
	Year = {2010}}

@article{hinton2006fast,
  title={A fast learning algorithm for deep belief nets},
  author={Hinton, Geoffrey E. and Osindero, Simon and Teh, Yee-Whye},
  journal={Neural computation},
  volume={18},
  number={7},
  pages={1527--1554},
  year={2006}
}

@article{cybenko1989approximation,
  title={Approximation by superpositions of a sigmoidal function},
  author={Cybenko, George},
  journal={Mathematics of control, signals and systems},
  volume={2},
  number={4},
  pages={303--314},
  year={1989},
  publisher={Springer}
}

@article{hornik1991approximation,
  title={Approximation capabilities of multilayer feedforward networks},
  author={Hornik, Kurt},
  journal={Neural networks},
  volume={4},
  number={2},
  pages={251--257},
  year={1991},
  publisher={Elsevier}
}

@book{anthony2009neural,
  title={Neural network learning: Theoretical foundations},
  author={Anthony, Martin and Bartlett, Peter L.},
  year={1999},
  publisher={Cambridge University Press}
}

@inproceedings{krizhevsky2012imagenet,
  title={Imagenet classification with deep convolutional neural networks},
  author={Krizhevsky, Alex and Sutskever, Ilya and Hinton, Geoffrey E.},
  booktitle={Advances in neural information processing systems},
  pages={1097--1105},
  year={2012}
}

@inproceedings{sutskever2014sequence,
  title={Sequence to sequence learning with neural networks},
  author={Sutskever, Ilya and Vinyals, Oriol and Le, Quoc V.},
  booktitle={Advances in neural information processing systems},
  pages={3104--3112},
  year={2014}
}

@inproceedings{le2013building,
  title={Building high-level features using large scale unsupervised learning},
  author={Le, Quoc V.},
  booktitle={2013 IEEE international conference on acoustics, speech and signal processing},
  pages={8595--8598},
  year={2013},
  organization={IEEE}
}

@inproceedings{dahl2013improving,
  title={Improving deep neural networks for LVCSR using rectified linear units and dropout},
  author={Dahl, George E. and Sainath, Tara N. and Hinton, Geoffrey E.},
  booktitle={2013 IEEE International Conference on Acoustics, Speech and Signal Processing},
  pages={8609--8613},
  year={2013},
  organization={IEEE}
}

@article{hinton2012deep,
  title={Deep neural networks for acoustic modeling in speech recognition: The shared views of four research groups},
  author={Hinton, Geoffrey and Deng, Li and Yu, Dong and Dahl, George E. and Mohamed, Abdel-rahman and Jaitly, Navdeep and Senior, Andrew and Vanhoucke, Vincent and Nguyen, Patrick and Sainath, Tara N and others},
  journal={IEEE Signal Processing Magazine},
  volume={29},
  number={6},
  pages={82--97},
  year={2012},
  publisher={IEEE}
}

@article{lecun2015deep,
  title={Deep learning},
  author={LeCun, Yann and Bengio, Yoshua and Hinton, Geoffrey},
  journal={Nature},
  volume={521},
  number={7553},
  pages={436--444},
  year={2015},
  publisher={Nature Publishing Group}
}

@article{srivastava2014dropout,
  title={Dropout: a simple way to prevent neural networks from overfitting.},
  author={Srivastava, Nitish and Hinton, Geoffrey E. and Krizhevsky, Alex and Sutskever, Ilya and Salakhutdinov, Ruslan},
  journal={Journal of Machine Learning Research},
  volume={15},
  number={1},
  pages={1929--1958},
  year={2014}
}

@inproceedings{sermanet2014overfeat,
  title={OverFeat: Integrated Recognition, Localization and Detection using Convolutional Networks},
  author={Sermanet, Pierre and Eigen, David and Zhang, Xiang and Mathieu, Michael and Fergus, Rob and LeCun, Yann},
  booktitle={International Conference on Learning Representations (ICLR 2014)},
  year={2014},
  organization={arXiv preprint arXiv:1312.6229}
}

@inproceedings{salakhutdinov2009deep,
  title={Deep Boltzmann Machines.},
  author={Salakhutdinov, Ruslan and Hinton, Geoffrey E.},
  booktitle={International Conference on Artificial Intelligence and Statistics (AISTATS)},
  volume={1},
  pages={3},
  year={2009}
}

@inproceedings{Mon,
  title={On the number of linear regions of deep neural networks},
  author={Montufar, Guido F. and Pascanu, Razvan and Cho, Kyunghyun and Bengio, Yoshua},
  booktitle={Advances in neural information processing systems},
  pages={2924--2932},
  year={2014}
}

@article{S,
  title={On the expressive power of deep neural networks},
  author={Raghu, Maithra and Poole, Ben and Kleinberg, Jon and Ganguli, Surya and Sohl-Dickstein, Jascha},
  journal={arXiv preprint arXiv:1606.05336},
  year={2016}
}

@article{KW,
  title={Super-linear gate and super-quadratic wire lower bounds for depth-two and depth-three threshold circuits},
  author={Kane, Daniel M. and Williams, Ryan},
  journal={arXiv preprint arXiv:1511.07860},
  year={2015}
}

@article{yarotsky2016error,
  title={Error bounds for approximations with deep ReLU networks},
  author={Yarotsky, Dmitry},
  journal={arXiv preprint arXiv:1610.01145},
  year={2016}
}

@inproceedings{safran2017depth,
  title={Depth-width tradeoffs in approximating natural functions with neural networks},
  author={Safran, Itay and Shamir, Ohad},
  booktitle={International Conference on Machine Learning},
  pages={2979--2987},
  year={2017}
}

@article{liang2016deep,
  title={Why Deep Neural Networks for Function Approximation?},
  author={Liang, Shiyu and Srikant, R},
  year={2016}
}

@article{T1,
  title={Representation Benefits of Deep Feedforward Networks},
  author={Telgarsky, Matus},
  journal={arXiv preprint arXiv:1509.08101},
  year={2015}
}

@inproceedings{T2,
  title={benefits of depth in neural networks},
  author={Telgarsky, Matus},
  booktitle={29th Annual Conference on Learning Theory},
  pages={1517--1539},
  year={2016}
}

@inproceedings{eldan2016power,
  title={The Power of Depth for Feedforward Neural Networks},
  author={Eldan, Ronen and Shamir, Ohad},
  booktitle={29th Annual Conference on Learning Theory},
  pages={907--940},
  year={2016}
}

@article{RH,
  title={Global optimality in tensor factorization, deep learning, and beyond},
  author={Haeffele, Benjamin D. and Vidal, Ren{\'e}},
  journal={arXiv preprint arXiv:1506.07540},
  year={2015}
}

@book{AB,
  title={Computational complexity: a modern approach},
  author={Arora, Sanjeev and Barak, Boaz},
  year={2009},
  publisher={Cambridge University Press}
}

@article{K,
  title={Deep Learning without Poor Local Minima},
  author={Kawaguchi, Kenji},
  journal={arXiv preprint arXiv:1605.07110},
  year={2016}
}

@article{wang2005generalization,
  title={Generalization of hinging hyperplanes},
  author={Wang, Shuning and Sun, Xusheng},
  journal={IEEE Transactions on Information Theory},
  volume={51},
  number={12},
  pages={4425--4431},
  year={2005},
  publisher={IEEE}
}

@article{Shp,
  title={Arithmetic circuits: A survey of recent results and open questions},
  author={Shpilka, Amir and Yehudayoff, Amir},
  journal={Foundations and Trends{\textregistered} in Theoretical Computer Science},
  volume={5},
  number={3--4},
  pages={207--388},
  year={2010},
  publisher={Now Publishers Inc.}
}

@book{Jun,
  title={Boolean function complexity: advances and frontiers},
  author={Jukna, Stasys},
  volume={27},
  year={2012},
  publisher={Springer Science \& Business Media}
}

@article{OS,
  title={Distribution-Specific Hardness of Learning Neural Networks},
  author={Shamir, Ohad},
  journal={arXiv preprint arXiv:1609.01037},
  year={2016}
}

@misc{Sap,
  title={A survey of lower bounds in arithmetic circuit complexity},
  author={Saptharishi, R.},
  year={2014},
  publisher={Manuscript}
}

@misc{B,
  title = {Combinatorics: Circuit Complexity },
  howpublished = {\url{http://www.math.ucsd.edu/~sbuss/CourseWeb/Math262A\_2013F/}},
  year={2013},
  author={Sam Buss}
}

@misc{Z,
  title = {Boolean Circuit Complexity},
  howpublished = {\url{http://www.cs.tau.ac.il/~zwick/scribe-boolean.html}}, 
  year={1994},
  author={Uri Zwick}
}

@misc{All,
  title = {Complexity Theory Lecture Notes},
  howpublished = {\url{https://www.cs.rutgers.edu/~allender/lecture.notes/}}, 
  year={1998},
  author={Eric Allender}
}

@inproceedings{hastad1986almost,
  title={Almost optimal lower bounds for small depth circuits},
  author={Hastad, Johan},
  booktitle={Proceedings of the eighteenth annual ACM symposium on Theory of computing},
  pages={6--20},
  year={1986},
  organization={ACM}
}

@article{razborov1987lower,
  title={Lower bounds on the size of bounded depth circuits over a complete basis with logical addition},
  author={Razborov, Alexander A.},
  journal={Mathematical Notes},
  volume={41},
  number={4},
  pages={333--338},
  year={1987},
  publisher={Springer}
}

@inproceedings{smolensky1987algebraic,
  title={Algebraic methods in the theory of lower bounds for Boolean circuit complexity},
  author={Smolensky, Roman},
  booktitle={Proceedings of the nineteenth annual ACM symposium on Theory of computing},
  pages={77--82},
  year={1987},
  organization={ACM}
}

@article{Pas,
  title={On the number of response regions of deep feed forward networks with piece-wise linear activations},
  author={Pascanu, Razvan and Montufar, Guido and Bengio, Yoshua},
  journal={arXiv preprint arXiv:1312.6098},
  year={2013}
}

@inproceedings{rossman2015average,
  title={An average-case depth hierarchy theorem for Boolean circuits},
  author={Rossman, Benjamin and Servedio, Rocco A. and Tan, Li-Yang},
  booktitle={Foundations of Computer Science (FOCS), 2015 IEEE 56th Annual Symposium on},
  pages={1030--1048},
  year={2015},
  organization={IEEE}
}

@article{goodfellow2013maxout,
  title={Maxout networks},
  author={Goodfellow, Ian J and Warde-Farley, David and Mirza, Mehdi and Courville, Aaron and Bengio, Yoshua},
  journal={arXiv preprint arXiv:1302.4389},
  year={2013}
}

@article{wang2004general,
  title={General constructive representations for continuous piecewise-linear functions},
  author={Wang, Shuning},
  journal={IEEE Transactions on Circuits and Systems I: Regular Papers},
  volume={51},
  number={9},
  pages={1889--1896},
  year={2004},
  publisher={IEEE}
}

@book{matousek2002lectures,
  title={Lectures on discrete geometry},
  author={Matousek, Jiri},
  volume={212},
  year={2002},
  publisher={Springer Science \& Business Media}
}

@book{ziegler1995lectures,
  title={Lectures on polytopes},
  author={Ziegler, G{\"u}nter M.},
  volume={152},
  year={1995},
  publisher={Springer Science \& Business Media}
}

@article{subbotovskaya1961realizations,
  title={Realizations of linear functions by formulas using+},
  author={Subbotovskaya, Bella Abramovna},
  journal={Doklady Akademii Nauk SSSR},
  volume={136},
  number={3},
  pages={553--555},
  year={1961}
}

@inproceedings{yao1985separating,
  title={Separating the polynomial-time hierarchy by oracles},
  author={Yao, Andrew Chi-Chih},
  booktitle={Foundations of Computer Science, 1985., 26th Annual Symposium on},
  pages={1--10},
  year={1985},
  organization={IEEE}
}

@inproceedings{rossman2008constant,
  title={On the constant-depth complexity of k-clique},
  author={Rossman, Benjamin},
  booktitle={Proceedings of the fortieth annual ACM symposium on Theory of computing},
  pages={721--730},
  year={2008},
  organization={ACM}
}

@article{stad1998shrinkage,
  title={The shrinkage exponent of de Morgan formulas is 2},
  author={stad, Johan H{\AA}},
  journal={SIAM Journal on Computing},
  volume={27},
  number={1},
  pages={48--64},
  year={1998},
  publisher={SIAM}
}

@misc{andreev1987one,
  title={ABOUT ONE METHOD OF OBTAINING MORE THAN QUADRATIC EFFECTIVE LOWER BOUNDS OF COMPLEXITY OF PI-SCHEMES},
  author={Andreev, Alexander E},
  journal={VESTNIK MOSKOVSKOGO UNIVERSITETA SERIYA 1 MATEMATIKA MEKHANIKA},
  number={1},
  pages={70--73},
  year={1987},
  publisher={MOSCOW STATE UNIV LENINSKIE GORY, MOSCOW, RUSSIA}
}

@inproceedings{impagliazzo1988decision,
  title={Decision trees and downward closures},
  author={Impagliazzo, Russell and Naor, Moni},
  booktitle={Structure in Complexity Theory Conference, 1988. Proceedings., Third Annual},
  pages={29--38},
  year={1988},
  organization={IEEE}
}

@inproceedings{impagliazzo2012pseudorandomness,
  title={Pseudorandomness from shrinkage},
  author={Impagliazzo, Russell and Meka, Raghu and Zuckerman, David},
  booktitle={Foundations of Computer Science (FOCS), 2012 IEEE 53rd Annual Symposium on},
  pages={111--119},
  year={2012},
  organization={IEEE}
}

@article{paterson1993shrinkage,
  title={Shrinkage of de Morgan formulae under restriction},
  author={Paterson, Michael S and Zwick, Uri},
  journal={Random Structures \& Algorithms},
  volume={4},
  number={2},
  pages={135--150},
  year={1993},
  publisher={Wiley Online Library}
}

@article{silver2017mastering,
  title={Mastering the game of Go without human knowledge},
  author={Silver, David and Schrittwieser, Julian and Simonyan, Karen and Antonoglou, Ioannis and Huang, Aja and Guez, Arthur and Hubert, Thomas and Baker, Lucas and Lai, Matthew and Bolton, Adrian and others},
  journal={Nature},
  volume={550},
  number={7676},
  pages={354--359},
  year={2017},
  publisher={Nature Research}
}

@article{hanin2017universal,
  title={Universal Function Approximation by Deep Neural Nets with Bounded Width and ReLU Activations},
  author={Hanin, Boris},
  journal={arXiv preprint arXiv:1708.02691},
  year={2017}
}

@article{telgarsky2016benefits,
  title={Benefits of depth in neural networks},
  author={Telgarsky, Matus},
  journal={arXiv preprint arXiv:1602.04485},
  year={2016}
}

@article{daniely2017depth,
  title={Depth Separation for Neural Networks},
  author={Daniely, Amit},
  journal={arXiv preprint arXiv:1702.08489},
  year={2017}
}

@article{safran2016depth,
  title={Depth separation in relu networks for approximating smooth non-linear functions},
  author={Safran, Itay and Shamir, Ohad},
  journal={arXiv preprint arXiv:1610.09887},
  year={2016}
}

@inproceedings{razborov1992small,
  title={On small depth threshold circuits},
  author={Razborov, Alexander A},
  booktitle={Scandinavian Workshop on Algorithm Theory},
  pages={42--52},
  year={1992},
  organization={Springer}
}

@inproceedings{krause1994computational,
  title={On the computational power of depth 2 circuits with threshold and modulo gates},
  author={Krause, Matthias and Pudl{\'a}k, Pavel},
  booktitle={Proceedings of the twenty-sixth annual ACM symposium on Theory of computing},
  pages={48--57},
  year={1994},
  organization={ACM}
}

@inproceedings{hajnal1987threshold,
  title={Threshold circuits of bounded depth},
  author={Hajnal, Andr{\'a}s and Maass, Wolfgang and Pudl{\'a}k, Pavel and Szegedy, Mario and Turan, Gyorgy},
  booktitle={Foundations of Computer Science, 1987., 28th Annual Symposium on},
  pages={99--110},
  year={1987},
  organization={IEEE}
}

@article{impagliazzo1997size,
  title={Size--Depth Tradeoffs for Threshold Circuits},
  author={Impagliazzo, Russell and Paturi, Ramamohan and Saks, Michael E},
  journal={SIAM Journal on Computing},
  volume={26},
  number={3},
  pages={693--707},
  year={1997},
  publisher={SIAM}
}

@article{siu1994rational,
  title={Rational approximation techniques for analysis of neural networks},
  author={Siu, Kai-Yeung and Roychowdhury, Vwani P and Kailath, Thomas},
  journal={IEEE Transactions on Information Theory},
  volume={40},
  number={2},
  pages={455--466},
  year={1994},
  publisher={IEEE}
}

@article{williams2018limits,
  title={Limits on representing Boolean functions by linear combinations of simple functions: thresholds, ReLUs, and low-degree polynomials},
  author={Williams, R Ryan},
  journal={arXiv preprint arXiv:1802.09121},
  year={2018}
}

@inproceedings{kane2016super,
  title={Super-linear gate and super-quadratic wire lower bounds for depth-two and depth-three threshold circuits},
  author={Kane, Daniel M and Williams, Ryan},
  booktitle={Proceedings of the forty-eighth annual ACM symposium on Theory of Computing},
  pages={633--643},
  year={2016},
  organization={ACM}
}

@inproceedings{tamaki2016satisfiability,
  title={A Satisfiability Algorithm for Depth Two Circuits with a Sub-Quadratic Number of Symmetric and Threshold Gates.},
  author={Tamaki, Suguru},
  booktitle={Electronic Colloquium on Computational Complexity (ECCC)},
  volume={23},
  number={100},
  pages={4},
  year={2016}
}

@inproceedings{kabanets2017polynomial,
  title={A Polynomial Restriction Lemma with Applications.},
  author={Kabanets, Valentine and Kane, Daniel and Lu, Zhenjian},
  booktitle={Electronic Colloquium on Computational Complexity (ECCC)},
  volume={24},
  pages={26},
  year={2017}
}

@inproceedings{chen2016average,
  title={Average-case lower bounds and satisfiability algorithms for small threshold circuits},
  author={Chen, Ruiwen and Santhanam, Rahul and Srinivasan, Srikanth},
  booktitle={LIPIcs-Leibniz International Proceedings in Informatics},
  volume={50},
  year={2016},
  organization={Schloss Dagstuhl-Leibniz-Zentrum fuer Informatik}
}

@article{maass1997bounds,
  title={Bounds for the computational power and learning complexity of analog neural nets},
  author={Maass, Wolfgang},
  journal={SIAM Journal on Computing},
  volume={26},
  number={3},
  pages={708--732},
  year={1997},
  publisher={SIAM}
}

@article{forster2002linear,
  title={A linear lower bound on the unbounded error probabilistic communication complexity},
  author={Forster, J{\"u}rgen},
  journal={Journal of Computer and System Sciences},
  volume={65},
  number={4},
  pages={612--625},
  year={2002},
  publisher={Elsevier}
}

@inproceedings{forster2001relations,
  title={Relations between communication complexity, linear arrangements, and computational complexity},
  author={Forster, J{\"u}rgen and Krause, Matthias and Lokam, Satyanarayana V and Mubarakzjanov, Rustam and Schmitt, Niels and Simon, Hans Ulrich},
  booktitle={International Conference on Foundations of Software Technology and Theoretical Computer Science},
  pages={171--182},
  year={2001},
  organization={Springer}
}

@article{lokam2009complexity,
  title={Complexity lower bounds using linear algebra},
  author={Lokam, Satyanarayana V and others},
  journal={Foundations and Trends{\textregistered} in Theoretical Computer Science},
  volume={4},
  number={1--2},
  pages={1--155},
  year={2009},
  publisher={Now Publishers, Inc.}
}

@article{chattopadhyay2017weights,
  title={Weights at the Bottom Matter When the Top is Heavy},
  author={Chattopadhyay, Arkadev and Mande, Nikhil S},
  journal={Electronic Colloquium on Computational Complexity, Revision 1 of Report No. 83, https://eccc.weizmann.ac.il/report/2017/083/},
  year={2017}
}

@inproceedings{buhrman2007computation,
  title={On computation and communication with small bias},
  author={Buhrman, Harry and Vereshchagin, Nikolay and de Wolf, Ronald},
  booktitle={Computational Complexity, 2007. CCC'07. Twenty-Second Annual IEEE Conference on},
  pages={24--32},
  year={2007},
  organization={IEEE}
}

@article{sherstov2009separating,
  title={Separating AC\^{}0 from Depth-2 Majority Circuits},
  author={Sherstov, Alexander A},
  journal={SIAM Journal on Computing},
  volume={38},
  number={6},
  pages={2113--2129},
  year={2009},
  publisher={SIAM}
}

@article{sherstov2007powering,
  title={Powering requires threshold depth 3},
  author={Sherstov, Alexander A},
  journal={Information processing letters},
  volume={102},
  number={2-3},
  pages={104--107},
  year={2007},
  publisher={Elsevier}
}

@article{sherstov2011unbounded,
  title={The unbounded-error communication complexity of symmetric functions},
  author={Sherstov, Alexander A},
  journal={Combinatorica},
  volume={31},
  number={5},
  pages={583--614},
  year={2011},
  publisher={Springer}
}

@article{razborov2010sign,
  title={The Sign-Rank of AC \^{}0},
  author={Razborov, Alexander A and Sherstov, Alexander A},
  journal={SIAM Journal on Computing},
  volume={39},
  number={5},
  pages={1833--1855},
  year={2010},
  publisher={SIAM}
}

@inproceedings{bun2016improved,
  title={Improved Bounds on the Sign-Rank of AC\^{} 0},
  author={Bun, Mark and Thaler, Justin},
  booktitle={LIPIcs-Leibniz International Proceedings in Informatics},
  volume={55},
  year={2016},
  organization={Schloss Dagstuhl-Leibniz-Zentrum fuer Informatik}
}

@article{lee2009lower,
  title={Lower bounds in communication complexity},
  author={Lee, Troy and Shraibman, Adi and others},
  journal={Foundations and Trends{\textregistered} in Theoretical Computer Science},
  volume={3},
  number={4},
  pages={263--399},
  year={2009},
  publisher={Now Publishers, Inc.}
}

@book{ledoux2013probability,
  title={Probability in Banach Spaces: isoperimetry and processes},
  author={Ledoux, Michel and Talagrand, Michel},
  year={2013},
  publisher={Springer Science \& Business Media}
}

@inproceedings{kalan2019fitting,
  title={Fitting relus via sgd and quantized sgd},
  author={Kalan, Seyed Mohammadreza Mousavi and Soltanolkotabi, Mahdi and Avestimehr, A Salman},
  booktitle={2019 IEEE International Symposium on Information Theory (ISIT)},
  pages={2469--2473},
  year={2019},
  organization={IEEE}
}

@inproceedings{soltanolkotabi2017learning,
  title={Learning relus via gradient descent},
  author={Soltanolkotabi, Mahdi},
  booktitle={Advances in neural information processing systems},
  pages={2007--2017},
  year={2017}
}

@incollection{neal1996priors,
  title={Priors for infinite networks},
  author={Neal, Radford M},
  booktitle={Bayesian Learning for Neural Networks},
  pages={29--53},
  year={1996},
  publisher={Springer}
}

@inproceedings{chizat2018global,
  title={On the global convergence of gradient descent for over-parameterized models using optimal transport},
  author={Chizat, Lenaic and Bach, Francis},
  booktitle={Advances in neural information processing systems},
  pages={3036--3046},
  year={2018}
}

@inproceedings{jacot2018neural,
  title={Neural tangent kernel: Convergence and generalization in neural networks},
  author={Jacot, Arthur and Gabriel, Franck and Hongler, Cl{\'e}ment},
  booktitle={Advances in neural information processing systems},
  pages={8571--8580},
  year={2018}
}

@article{wu2019global,
  title={Global convergence of adaptive gradient methods for an over-parameterized neural network},
  author={Wu, Xiaoxia and Du, Simon S and Ward, Rachel},
  journal={arXiv preprint arXiv:1902.07111},
  year={2019}
}

@article{dugradient,
  title={Gradient Descent Finds Global Minima of Deep Neural Networks},
  author={Du, Simon S and Lee, Jason D and Li, Haochuan and Wang, Liwei and Zhai, Xiyu},
  year = {2018}
  
}

@inproceedings{su2019learning,
  title={On Learning Over-parameterized Neural Networks: A Functional Approximation Perspective},
  author={Su, Lili and Yang, Pengkun},
  booktitle={Advances in Neural Information Processing Systems},
  pages={2637--2646},
  year={2019}
}

@inproceedings{kawaguchi2019gradient,
  title={Gradient descent finds global minima for generalizable deep neural networks of practical sizes},
  author={Kawaguchi, Kenji and Huang, Jiaoyang},
  booktitle={2019 57th Annual Allerton Conference on Communication, Control, and Computing (Allerton)},
  pages={92--99},
  year={2019},
  organization={IEEE}
}

@article{huang2019dynamics,
  title={Dynamics of deep neural networks and neural tangent hierarchy},
  author={Huang, Jiaoyang and Yau, Horng-Tzer},
  journal={arXiv preprint arXiv:1909.08156},
  year={2019}
}

@inproceedings{allen2019convergenceDNN,
  title={A Convergence Theory for Deep Learning via Over-Parameterization},
  author={Allen-Zhu, Zeyuan and Li, Yuanzhi and Song, Zhao},
  booktitle={International Conference on Machine Learning},
  pages={242--252},
  year={2019}
}

@inproceedings{allen2019learning,
  title={Learning and generalization in overparameterized neural networks, going beyond two layers},
  author={Allen-Zhu, Zeyuan and Li, Yuanzhi and Liang, Yingyu},
  booktitle={Advances in neural information processing systems},
  pages={6155--6166},
  year={2019}
}

@inproceedings{allen2019convergenceRNN,
  title={On the convergence rate of training recurrent neural networks},
  author={Allen-Zhu, Zeyuan and Li, Yuanzhi and Song, Zhao},
  booktitle={Advances in Neural Information Processing Systems},
  pages={6673--6685},
  year={2019}
}

@inproceedings{du2018power,
  title={On the Power of Over-parametrization in Neural Networks with Quadratic Activation},
  author={Du, Simon and Lee, Jason},
  booktitle={International Conference on Machine Learning},
  pages={1329--1338},
  year={2018}
}

@article{zou2018stochastic,
  title={Stochastic gradient descent optimizes over-parameterized deep relu networks},
  author={Zou, Difan and Cao, Yuan and Zhou, Dongruo and Gu, Quanquan},
  journal={arXiv preprint arXiv:1811.08888},
  year={2018}
}

@inproceedings{zou2019improved,
  title={An improved analysis of training over-parameterized deep neural networks},
  author={Zou, Difan and Gu, Quanquan},
  booktitle={Advances in Neural Information Processing Systems},
  pages={2053--2062},
  year={2019}
}

@inproceedings{arora2019exact,
  title={On exact computation with an infinitely wide neural net},
  author={Arora, Sanjeev and Du, Simon S and Hu, Wei and Li, Zhiyuan and Salakhutdinov, Russ R and Wang, Ruosong},
  booktitle={Advances in Neural Information Processing Systems},
  pages={8139--8148},
  year={2019}
}

@article{arora2019harnessing,
  title={Harnessing the Power of Infinitely Wide Deep Nets on Small-data Tasks},
  author={Arora, Sanjeev and Du, Simon S and Li, Zhiyuan and Salakhutdinov, Ruslan and Wang, Ruosong and Yu, Dingli},
  journal={arXiv preprint arXiv:1910.01663},
  year={2019}
}

@article{li2019enhanced,
  title={Enhanced Convolutional Neural Tangent Kernels},
  author={Li, Zhiyuan and Wang, Ruosong and Yu, Dingli and Du, Simon S and Hu, Wei and Salakhutdinov, Ruslan and Arora, Sanjeev},
  journal={arXiv preprint arXiv:1911.00809},
  year={2019}
}

@inproceedings{wei2019regularization,
  title={Regularization matters: Generalization and optimization of neural nets vs their induced kernel},
  author={Wei, Colin and Lee, Jason D and Liu, Qiang and Ma, Tengyu},
  booktitle={Advances in Neural Information Processing Systems},
  pages={9709--9721},
  year={2019}
}

@inproceedings{arora2019fine,
  title={Fine-Grained Analysis of Optimization and Generalization for Overparameterized Two-Layer Neural Networks},
  author={Arora, Sanjeev and Du, Simon and Hu, Wei and Li, Zhiyuan and Wang, Ruosong},
  booktitle={International Conference on Machine Learning},
  pages={322--332},
  year={2019}
}

@inproceedings{allen2019can,
  title={What Can ResNet Learn Efficiently, Going Beyond Kernels?},
  author={Allen-Zhu, Zeyuan and Li, Yuanzhi},
  booktitle={Advances in Neural Information Processing Systems},
  pages={9015--9025},
  year={2019}
}

@article{lee2018deep,
  title={Deep Neural Networks as Gaussian Processes},
  author={Lee, Jaehoon and Bahri, Yasaman and Novak, Roman and Schoenholz, Samuel S and Pennington, Jeffrey and Sohl-Dickstein, Jascha},
  year={2018}
}

@article{rosenblatt1958perceptron,
  title={The perceptron: a probabilistic model for information storage and organization in the brain.},
  author={Rosenblatt, Frank},
  journal={Psychological review},
  volume={65},
  number={6},
  pages={386},
  year={1958},
  publisher={American Psychological Association}
}

@article{pal1992multilayer,
  title={Multilayer perceptron, fuzzy sets, classifiaction},
  author={Pal, Sankar K and Mitra, Sushmita},
  year={1992}
}

@article{freund1999large,
  title={Large margin classification using the perceptron algorithm},
  author={Freund, Yoav and Schapire, Robert E},
  journal={Machine learning},
  volume={37},
  number={3},
  pages={277--296},
  year={1999},
  publisher={Springer}
}

@inproceedings{kakade2011efficient,
  title={Efficient learning of generalized linear and single index models with isotonic regression},
  author={Kakade, Sham M and Kanade, Varun and Shamir, Ohad and Kalai, Adam},
  booktitle={Advances in Neural Information Processing Systems},
  pages={927--935},
  year={2011}
}

@article{goel2018learning,
  title={Learning one convolutional layer with overlapping patches},
  author={Goel, Surbhi and Klivans, Adam and Meka, Raghu},
  journal={arXiv preprint arXiv:1802.02547},
  year={2018}
}

@inproceedings{klivans2017learning,
  title={Learning graphical models using multiplicative weights},
  author={Klivans, Adam and Meka, Raghu},
  booktitle={2017 IEEE 58th Annual Symposium on Foundations of Computer Science (FOCS)},
  pages={343--354},
  year={2017},
  organization={IEEE}
}

@article{goel2017learning,
  title={Learning depth-three neural networks in polynomial time},
  author={Goel, Surbhi and Klivans, Adam},
  journal={arXiv preprint arXiv:1709.06010},
  year={2017}
}

@inproceedings{xu2018global,
  title={Global convergence of langevin dynamics based algorithms for nonconvex optimization},
  author={Xu, Pan and Chen, Jinghui and Zou, Difan and Gu, Quanquan},
  booktitle={Advances in Neural Information Processing Systems},
  pages={3122--3133},
  year={2018}
}

@article{zhang2017hitting,
  title={A Hitting Time Analysis of Stochastic Gradient Langevin Dynamics},
  author={Zhang, Yuchen and Liang, Percy and Charikar, Moses},
  journal={Proceedings of Machine Learning Research vol},
  volume={65},
  pages={1--43},
  year={2017}
}

@article{durmus2019analysis,
  title={Analysis of Langevin Monte Carlo via Convex Optimization.},
  author={Durmus, Alain and Majewski, Szymon},
  journal={Journal of Machine Learning Research},
  volume={20},
  number={73},
  pages={1--46},
  year={2019}
}

@inproceedings{lee2019online,
  title={Online sampling from log-concave distributions},
  author={Lee, Holden and Mangoubi, Oren and Vishnoi, Nisheeth},
  booktitle={Advances in Neural Information Processing Systems},
  pages={1226--1237},
  year={2019}
}

@article{li2019generalization,
  title={On generalization error bounds of noisy gradient methods for non-convex learning},
  author={Li, Jian and Luo, Xuanyuan and Qiao, Mingda},
  journal={arXiv preprint arXiv:1902.00621},
  year={2019}
}

@article{vaswani2018fast,
  title={Fast and faster convergence of SGD for over-parameterized models and an accelerated perceptron},
  author={Vaswani, Sharan and Bach, Francis and Schmidt, Mark},
  journal={arXiv preprint arXiv:1810.07288},
  year={2018}
}

@article{staib2019escaping,
  title={Escaping saddle points with adaptive gradient methods},
  author={Staib, Matthew and Reddi, Sashank J and Kale, Satyen and Kumar, Sanjiv and Sra, Suvrit},
  journal={arXiv preprint arXiv:1901.09149},
  year={2019}
}

@article{li2018convergence,
  title={On the Convergence of Stochastic Gradient Descent with Adaptive Stepsizes},
  author={Li, Xiaoyu and Orabona, Francesco},
  journal={arXiv preprint arXiv:1805.08114},
  year={2018}
}

@inproceedings{ward2019adagrad,
  title={AdaGrad stepsizes: sharp convergence over nonconvex landscapes},
  author={Ward, Rachel and Wu, Xiaoxia and Bottou, Leon},
  booktitle={International Conference on Machine Learning},
  pages={6677--6686},
  year={2019}
}

@article{chen2018convergence,
  title={On the convergence of a class of adam-type algorithms for non-convex optimization},
  author={Chen, Xiangyi and Liu, Sijia and Sun, Ruoyu and Hong, Mingyi},
  journal={arXiv preprint arXiv:1808.02941},
  year={2018}
}

@article{zhou2018convergence,
  title={On the convergence of adaptive gradient methods for nonconvex optimization},
  author={Zhou, Dongruo and Tang, Yiqi and Yang, Ziyan and Cao, Yuan and Gu, Quanquan},
  journal={arXiv preprint arXiv:1808.05671},
  year={2018}
}

@article{zou2018sufficient,
  title={A Sufficient Condition for Convergences of Adam and RMSProp},
  author={Zou, Fangyu and Shen, Li and Jie, Zequn and Zhang, Weizhong and Liu, Wei},
  journal={arXiv preprint arXiv:1811.09358},
  year={2018}
}

@article{chen2018closing,
  title={Closing the generalization gap of adaptive gradient methods in training deep neural networks},
  author={Chen, Jinghui and Gu, Quanquan},
  journal={arXiv preprint arXiv:1806.06763},
  year={2018}
}

@article{loizou2017momentum,
  title={Momentum and stochastic momentum for stochastic gradient, Newton, proximal point and subspace descent methods},
  author={Loizou, Nicolas and Richt{\'a}rik, Peter},
  journal={arXiv preprint arXiv:1712.09677},
  year={2017}
}

@article{denkowski2017stronger,
  title={Stronger baselines for trustable results in neural machine translation},
  author={Denkowski, Michael and Neubig, Graham},
  journal={arXiv preprint arXiv:1706.09733},
  year={2017}
}

@article{gregor2015draw,
  title={DRAW: A recurrent neural network for image generation},
  author={Gregor, Karol and Danihelka, Ivo and Graves, Alex and Rezende, Danilo Jimenez and Wierstra, Daan},
  journal={arXiv preprint arXiv:1502.04623},
  year={2015}
}

@article{radford2015unsupervised,
  title={Unsupervised representation learning with deep convolutional generative adversarial networks},
  author={Radford, Alec and Metz, Luke and Chintala, Soumith},
  journal={arXiv preprint arXiv:1511.06434},
  year={2015}
}

@article{melis2017state,
  title={On the state of the art of evaluation in neural language models},
  author={Melis, G{\'a}bor and Dyer, Chris and Blunsom, Phil},
  journal={arXiv preprint arXiv:1707.05589},
  year={2017}
}

@article{keskar2017improving,
  title={Improving Generalization Performance by Switching from Adam to SGD},
  author={Keskar, Nitish Shirish and Socher, Richard},
  journal={arXiv preprint arXiv:1712.07628},
  year={2017}
}

@article{bahar2017empirical,
  title={Empirical investigation of optimization algorithms in neural machine translation},
  author={Bahar, Parnia and Alkhouli, Tamer and Peter, Jan-Thorsten and Brix, Christopher Jan-Steffen and Ney, Hermann},
  journal={The Prague Bulletin of Mathematical Linguistics},
  volume={108},
  number={1},
  pages={13--25},
  year={2017},
  publisher={De Gruyter Open}
}

@inproceedings{wilson2017marginal,
  title={The marginal value of adaptive gradient methods in machine learning},
  author={Wilson, Ashia C and Roelofs, Rebecca and Stern, Mitchell and Srebro, Nati and Recht, Benjamin},
  booktitle={Advances in Neural Information Processing Systems},
  pages={4151--4161},
  year={2017}
}

@article{gadat2018stochastic,
  title={Stochastic heavy ball},
  author={Gadat, S{\'e}bastien and Panloup, Fabien and Saadane, Sofiane and others},
  journal={Electronic Journal of Statistics},
  volume={12},
  number={1},
  pages={461--529},
  year={2018},
  publisher={The Institute of Mathematical Statistics and the Bernoulli Society}
}

@inproceedings{glorot2010understanding,
  title={Understanding the difficulty of training deep feedforward neural networks},
  author={Glorot, Xavier and Bengio, Yoshua},
  booktitle={Proceedings of the thirteenth international conference on artificial intelligence and statistics},
  pages={249--256},
  year={2010}
}

@article{arpit2015regularized,
  title={Why regularized auto-encoders learn sparse representation?},
  author={Arpit, Devansh and Zhou, Yingbo and Ngo, Hung and Govindaraju, Venu},
  journal={arXiv preprint arXiv:1505.05561},
  year={2015}
}

@article{rangamani2017critical,
  title={Critical Points Of An Autoencoder Can Provably Recover Sparsely Used Overcomplete Dictionaries},
  author={Rangamani, Akshay and Mukherjee, Anirbit and Arora, Ashish and Ganapathy, Tejaswini and Basu, Amitabh and Chin, Sang and Tran, Trac D},
  journal={arXiv preprint arXiv:1708.03735},
  year={2017}
}

@article{jin2017accelerated,
  title={Accelerated Gradient Descent Escapes Saddle Points Faster than Gradient Descent},
  author={Jin, Chi and Netrapalli, Praneeth and Jordan, Michael I},
  journal={arXiv preprint arXiv:1711.10456},
  year={2017}
}

@inproceedings{
kidambi2018on,
title={On the insufficiency of existing momentum schemes for Stochastic Optimization},
author={Rahul Kidambi and Praneeth Netrapalli and Prateek Jain and Sham M. Kakade},
booktitle={International Conference on Learning Representations},
year={2018},
url={https://openreview.net/forum?id=rJTutzbA-},
}

@inproceedings{reddi2018convergence,
  title={On the convergence of adam and beyond},
  author={Reddi, Sashank J and Kale, Satyen and Kumar, Sanjiv},
  booktitle={International Conference on Learning Representations},
  year={2018}
}

@misc{kingma2014adam,
  title={Adam: A Method for Stochastic Optimization. arXiv. org},
  author={Kingma, Diederik P and Ba, Jimmy},
  year={2014},
  publisher={December}
}

@article{tieleman2012lecture,
  title={Lecture 6.5-RMSProp, COURSERA: Neural networks for machine learning},
  author={Tieleman, Tijmen and Hinton, Geoffrey},
  journal={University of Toronto, Technical Report},
  year={2012}
}

@inproceedings{nesterov1983method,
  title={A method of solving a convex programming problem with convergence rate O (1/k2)},
  author={Nesterov, Yurii},
  booktitle={Soviet Mathematics Doklady},
  volume={27},
  number={2},
  pages={372--376},
  year={1983}
}

@article{polyak1987introduction,
  title={Introduction to optimization. Translations series in mathematics and engineering},
  author={Polyak, Boris T},
  journal={Optimization Software},
  year={1987}
}

@article{o2017behavior,
  title={Behavior of accelerated gradient methods near critical points of nonconvex problems},
  author={O'Neill, Michael and Wright, Stephen J},
  journal={arXiv preprint arXiv:1706.07993},
  year={2017}
}

@article{ochs2016local,
  title={Local Convergence of the Heavy-ball Method and iPiano for Non-convex Optimization},
  author={Ochs, Peter},
  journal={arXiv preprint arXiv:1606.09070},
  year={2016}
}

@inproceedings{sutskever2013importance,
  title={On the importance of initialization and momentum in deep learning},
  author={Sutskever, Ilya and Martens, James and Dahl, George and Hinton, Geoffrey},
  booktitle={International conference on machine learning},
  pages={1139--1147},
  year={2013}
}

@article{wiegerinck1994stochastic,
  title={Stochastic dynamics of learning with momentum in neural networks},
  author={Wiegerinck, Wim and Komoda, Andrzej and Heskes, Tom},
  journal={Journal of Physics A: Mathematical and General},
  volume={27},
  number={13},
  pages={4425},
  year={1994},
  publisher={IOP Publishing}
}

@article{yang2016unified,
  title={Unified convergence analysis of stochastic momentum methods for convex and non-convex optimization},
  author={Yang, Tianbao and Lin, Qihang and Li, Zhe},
  journal={arXiv preprint arXiv:1604.03257},
  year={2016}
}

@article{yuan2016influence,
  title={On the influence of momentum acceleration on online learning},
  author={Yuan, Kun and Ying, Bicheng and Sayed, Ali H},
  journal={Journal of Machine Learning Research},
  volume={17},
  number={192},
  pages={1--66},
  year={2016}
}

@article{zavriev1993heavy,
  title={Heavy-ball method in nonconvex optimization problems},
  author={Zavriev, SK and Kostyuk, FV},
  journal={Computational Mathematics and Modeling},
  volume={4},
  number={4},
  pages={336--341},
  year={1993},
  publisher={Springer}
}

@inproceedings{abadi2016tensorflow,
  title={TensorFlow: A System for Large-Scale Machine Learning.},
  author={Abadi, Mart{\'\i}n and Barham, Paul and Chen, Jianmin and Chen, Zhifeng and Davis, Andy and Dean, Jeffrey and Devin, Matthieu and Ghemawat, Sanjay and Irving, Geoffrey and Isard, Michael and others},
  booktitle={OSDI},
  volume={16},
  pages={265--283},
  year={2016}
}

@book{lehoucq1998arpack,
  title={ARPACK users' guide: solution of large-scale eigenvalue problems with implicitly restarted Arnoldi methods},
  author={Lehoucq, Richard B and Sorensen, Danny C and Yang, Chao},
  volume={6},
  year={1998},
  publisher={Siam}
}

@article{duchi2011adaptive,
  title={Adaptive subgradient methods for online learning and stochastic optimization},
  author={Duchi, John and Hazan, Elad and Singer, Yoram},
  journal={Journal of Machine Learning Research},
  volume={12},
  number={Jul},
  pages={2121--2159},
  year={2011}
}

@misc{hvp,
  author = {Jamie Townsend},
  title = {{A new trick for calculating Jacobian vector products}},
  howpublished = "\url{https://j-towns.
github.io/2017/06/12/A-new-trick.html}",
  year = {2008}, 
  note = "[Online; accessed 17-May-2018]"
}

@article{bernstein2018signsgd,
  title={signSGD: compressed optimisation for non-convex problems},
  author={Bernstein, Jeremy and Wang, Yu-Xiang and Azizzadenesheli, Kamyar and Anandkumar, Anima},
  journal={arXiv preprint arXiv:1802.04434},
  year={2018}
}

@inproceedings{baldi2012autoencoders,
  title={Autoencoders, unsupervised learning, and deep architectures},
  author={Baldi, Pierre},
  booktitle={Proceedings of ICML Workshop on Unsupervised and Transfer Learning},
  pages={37--49},
  year={2012}
}

@article{kuchaiev2017training,
  title={Training Deep AutoEncoders for Collaborative Filtering},
  author={Kuchaiev, Oleksii and Ginsburg, Boris},
  journal={arXiv preprint arXiv:1708.01715},
  year={2017}
}

@article{vincent2010stacked,
  title={Stacked denoising autoencoders: Learning useful representations in a deep network with a local denoising criterion},
  author={Vincent, Pascal and Larochelle, Hugo and Lajoie, Isabelle and Bengio, Yoshua and Manzagol, Pierre-Antoine},
  journal={Journal of Machine Learning Research},
  volume={11},
  number={Dec},
  pages={3371--3408},
  year={2010}
}

@inproceedings{martens2015optimizing,
  title={Optimizing neural networks with kronecker-factored approximate curvature},
  author={Martens, James and Grosse, Roger},
  booktitle={International conference on machine learning},
  pages={2408--2417},
  year={2015}
}

@inproceedings{de2017automated,
  title={Automated inference with adaptive batches},
  author={De, Soham and Yadav, Abhay and Jacobs, David and Goldstein, Tom},
  booktitle={Artificial Intelligence and Statistics},
  pages={1504--1513},
  year={2017}
}

@article{babanezhad2015stop,
  title={Stop Wasting My Gradients: Practical SVRG},
  author={Babanezhad, Reza and Ahmed, Mohamed Osama and Virani, Alim and Schmidt, Mark and Konecny, Jakub and Sallinen, Scott},
  journal={arXiv preprint arXiv:1511.01942},
  year={2015}
}

@inproceedings{johnson2013accelerating,
  title={Accelerating stochastic gradient descent using predictive variance reduction},
  author={Johnson, Rie and Zhang, Tong},
  booktitle={Advances in neural information processing systems},
  pages={315--323},
  year={2013}
}

@inproceedings{defazio2014saga,
  title={SAGA: A fast incremental gradient method with support for non-strongly convex composite objectives},
  author={Defazio, Aaron and Bach, Francis and Lacoste-Julien, Simon},
  booktitle={Advances in neural information processing systems},
  pages={1646--1654},
  year={2014}
}

@article{lucas2018aggregated,
  title={Aggregated Momentum: Stability Through Passive Damping},
  author={Lucas, James and Zemel, Richard and Grosse, Roger},
  journal={arXiv preprint arXiv:1804.00325},
  year={2018}
}

@article{simonyan2014very,
  title={Very deep convolutional networks for large-scale image recognition},
  author={Simonyan, Karen and Zisserman, Andrew},
  journal={arXiv preprint arXiv:1409.1556},
  year={2014}
}

@article{ioffe2015batch,
  title={Batch normalization: Accelerating deep network training by reducing internal covariate shift},
  author={Ioffe, Sergey and Szegedy, Christian},
  journal={arXiv preprint arXiv:1502.03167},
  year={2015}
}

@inproceedings{zaheer2018adaptive,
  title={Adaptive Methods for Nonconvex Optimization},
  author={Zaheer, Manzil and Reddi, Sashank and Sachan, Devendra and Kale, Satyen and Kumar, Sanjiv},
  booktitle={Advances in Neural Information Processing Systems},
  year={2018}
}

@inproceedings{nguyen2019dynamics,
  title={On the dynamics of gradient descent for autoencoders},
  author={Nguyen, Thanh V and Wong, Raymond KW and Hegde, Chinmay},
  booktitle={The 22nd International Conference on Artificial Intelligence and Statistics},
  pages={2858--2867},
  year={2019}
}

@article{wu2017towards,
  title={Towards Understanding Generalization of Deep Learning: Perspective of Loss Landscapes},
  author={Wu, Lei and Zhu, Zhanxing and others},
  journal={arXiv preprint arXiv:1706.10239},
  year={2017}
}

@article{tillmann2015computational,
  title={On the computational intractability of exact and approximate dictionary learning},
  author={Tillmann, Andreas M},
  journal={IEEE Signal Processing Letters},
  volume={22},
  number={1},
  pages={45--49},
  year={2015},
  publisher={IEEE}
}

@article{GilbertLectures,
  title={CBMS Conference on Sparse Approximation and Signal Recovery Algorithms, May 22-26, 2017
and 16th New Mexico Analysis Seminar, May 21},
  author={Gilbert, Anna},
  journal={https://www.math.nmsu.edu/~jlakey/cbms2017/ \\cbms{\tt \_}lecture\_notes.html}}

@article{du2017convolutional,
  title={When is a Convolutional Filter Easy To Learn?},
  author={Du, Simon S and Lee, Jason D and Tian, Yuandong},
  journal={arXiv preprint arXiv:1709.06129},
  year={2017}
}

@article{tian2017analytical,
  title={An Analytical Formula of Population Gradient for two-layered ReLU network and its Applications in Convergence and Critical Point Analysis},
  author={Tian, Yuandong},
  journal={arXiv preprint arXiv:1703.00560},
  year={2017}
}

@article{zhang2017electron,
  title={Electron-Proton Dynamics in Deep Learning},
  author={Zhang, Qiuyi and Panigrahy, Rina and Sachdeva, Sushant and Rahimi, Ali},
  journal={arXiv preprint arXiv:1702.00458},
  year={2017}
}

@article{li2017convergence,
  title={Convergence Analysis of Two-layer Neural Networks with ReLU Activation},
  author={Li, Yuanzhi and Yuan, Yang},
  journal={arXiv preprint arXiv:1705.09886},
  year={2017}
}

@article{allen2017natasha2,
  title={Natasha 2:Faster Non-Convex Optimization Than SGD},
  author={Allen-Zhu, Zeyuan},
  journal={arXiv preprint arXiv:1708.08694},
  year={2017}
}

@article{janzamin2015beating,
  title={Beating the perils of non-convexity: Guaranteed training of neural networks using tensor methods},
  author={Janzamin, Majid and Sedghi, Hanie and Anandkumar, Anima},
  journal={arXiv preprint arXiv:1506.08473},
  year={2015}
}

@article{sedghi2014provable,
  title={Provable methods for training neural networks with sparse connectivity},
  author={Sedghi, Hanie and Anandkumar, Anima},
  journal={arXiv preprint arXiv:1412.2693},
  year={2014}
}

@inproceedings{moitra2010settling,
  title={Settling the polynomial learnability of mixtures of gaussians},
  author={Moitra, Ankur and Valiant, Gregory},
  booktitle={Foundations of Computer Science (FOCS), 2010 51st Annual IEEE Symposium on},
  pages={93--102},
  year={2010},
  organization={IEEE}
}

@inproceedings{arpit2016regularized,
  title={Why regularized auto-encoders learn sparse representation?},
  author={Arpit, Devansh and Zhou, Yingbo and Ngo, Hung and Govindaraju, Venu},
  booktitle={International Conference on Machine Learning},
  pages={136--144},
  year={2016}
}

@article{olshausen1996emergence,
  title={Emergence of simple-cell receptive field properties by learning a sparse code for natural images},
  author={Olshausen, Bruno A and Field, David J},
  journal={Nature},
  volume={381},
  number={6583},
  pages={607},
  year={1996},
  publisher={Nature Publishing Group}
}

@article{olshausen1997sparse,
  title={Sparse coding with an overcomplete basis set: A strategy employed by V1?},
  author={Olshausen, Bruno A and Field, David J},
  journal={Vision research},
  volume={37},
  number={23},
  pages={3311--3325},
  year={1997},
  publisher={Elsevier}
}

@article{olshausen2005close,
  title={How close are we to understanding V1?},
  author={Olshausen, Bruno A and Field, David J},
  journal={Neural computation},
  volume={17},
  number={8},
  pages={1665--1699},
  year={2005},
  publisher={MIT Press}
}

@inproceedings{spielman2012exact,
  title={Exact Recovery of Sparsely-Used Dictionaries.},
  author={Spielman, Daniel A and Wang, Huan and Wright, John},
  booktitle={COLT},
  pages={37--1},
  year={2012}
}

@article{blasiok2016improved,
  title={An improved analysis of the ER-SpUD dictionary learning algorithm},
  author={B{\l}asiok, Jaros{\l}aw and Nelson, Jelani},
  journal={arXiv:1602.05719},
  year={2016}
}

@inproceedings{agarwal2014learning,
  title={Learning Sparsely Used Overcomplete Dictionaries.},
  author={Agarwal, Alekh and Anandkumar, Animashree and Jain, Prateek and Netrapalli, Praneeth and Tandon, Rashish},
  booktitle={COLT},
  pages={123--137},
  year={2014},
  journal = {arXiv:1310.7991}
}

@article{anandkumar2014tensor,
  title={Tensor decompositions for learning latent variable models.},
  author={Anandkumar, Animashree and Ge, Rong and Hsu, Daniel J and Kakade, Sham M and Telgarsky, Matus},
  journal={Journal of Machine Learning Research},
  volume={15},
  number={1},
  pages={2773--2832},
  year={2014}
}

@inproceedings{arora2014new,
  title={New Algorithms for Learning Incoherent and Overcomplete Dictionaries.},
  author={Arora, Sanjeev and Ge, Rong and Moitra, Ankur},
  booktitle={COLT},
  pages={779--806},
  year={2014},
  journal = {arXiv:1308.6273}
}

@article{arora2014more,
  title={More algorithms for provable dictionary learning},
  author={Arora, Sanjeev and Bhaskara, Aditya and Ge, Rong and Ma, Tengyu},
  journal={arXiv:1401.0579},
  year={2014}
}

@inproceedings{arora2015simple,
  title={Simple, efficient, and neural algorithms for sparse coding.},
  author={Arora, Sanjeev and Ge, Rong and Ma, Tengyu and Moitra, Ankur},
  booktitle={COLT},
  pages={113--149},
  year={2015},
  journal = {arXiv:1503.00778}
}

@article{ge2017no,
  title={No Spurious Local Minima in Nonconvex Low Rank Problems: A Unified Geometric Analysis},
  author={Ge, Rong and Jin, Chi and Zheng, Yi},
  journal={arXiv preprint arXiv:1704.00708},
  year={2017}
}

@article{mei2016landscape,
  title={The landscape of empirical risk for non-convex losses},
  author={Mei, Song and Bai, Yu and Montanari, Andrea},
  journal={arXiv preprint arXiv:1607.06534},
  year={2016}
}

@inproceedings{bengio2013generalized,
  title={Generalized denoising auto-encoders as generative models},
  author={Bengio, Yoshua and Yao, Li and Alain, Guillaume and Vincent, Pascal},
  booktitle={Advances in Neural Information Processing Systems},
  pages={899--907},
  year={2013}
}

@inproceedings{vincent2008extracting,
  title={Extracting and composing robust features with denoising autoencoders},
  author={Vincent, Pascal and Larochelle, Hugo and Bengio, Yoshua and Manzagol, Pierre-Antoine},
  booktitle={Proceedings of the 25th international conference on Machine learning},
  pages={1096--1103},
  year={2008},
  organization={ACM}
}

@inproceedings{rifai2011contractive,
  title={Contractive auto-encoders: Explicit invariance during feature extraction},
  author={Rifai, Salah and Vincent, Pascal and Muller, Xavier and Glorot, Xavier and Bengio, Yoshua},
  booktitle={Proceedings of the 28th international conference on machine learning (ICML-11)},
  pages={833--840},
  year={2011}
}

@inproceedings{coates2011importance,
  title={The importance of encoding versus training with sparse coding and vector quantization},
  author={Coates, Adam and Ng, Andrew Y},
  booktitle={Proceedings of the 28th International Conference on Machine Learning (ICML-11)},
  pages={921--928},
  year={2011}
}

@inproceedings{coates2011analysis,
  title={An analysis of single-layer networks in unsupervised feature learning},
  author={Coates, Adam and Ng, Andrew and Lee, Honglak},
  booktitle={Proceedings of the fourteenth international conference on artificial intelligence and statistics},
  pages={215--223},
  year={2011}
}

@article{alain2014regularized,
  title={What regularized auto-encoders learn from the data-generating distribution.},
  author={Alain, Guillaume and Bengio, Yoshua},
  journal={Journal of Machine Learning Research},
  volume={15},
  number={1},
  pages={3563--3593},
  year={2014}
}

@article{li2016sparseness,
  title={Sparseness analysis in the pretraining of deep neural networks},
  author={Li, Jun and Zhang, Tong and Luo, Wei and Yang, Jian and Yuan, Xiao-Tong and Zhang, Jian},
  journal={IEEE transactions on neural networks and learning systems},
  year={2016},
  publisher={IEEE}
}

@article{makhzani2013k,
  title={K-sparse autoencoders},
  author={Makhzani, Alireza and Frey, Brendan},
  journal={arXiv preprint arXiv:1312.5663},
  year={2013}
}

@inproceedings{makhzani2015winner,
  title={Winner-take-all autoencoders},
  author={Makhzani, Alireza and Frey, Brendan J},
  booktitle={Advances in Neural Information Processing Systems},
  pages={2791--2799},
  year={2015}
}

@article{ng2011sparse,
  title={Sparse autoencoder},
  author={Ng, Andrew},
  year={2011}
}

@article{bora2017compressed,
  title={Compressed Sensing using Generative Models},
  author={Bora, Ashish and Jalal, Ajil and Price, Eric and Dimakis, Alexandros G},
  journal={arXiv preprint arXiv:1703.03208},
  year={2017}
}

@article{gilbert2017towards,
  title={Towards Understanding the Invertibility of Convolutional Neural Networks},
  author={Gilbert, Anna C and Zhang, Yi and Lee, Kibok and Zhang, Yuting and Lee, Honglak},
  journal={arXiv preprint arXiv:1705.08664},
  year={2017}
}

@article{vardan2016convolutional,
  title={Convolutional Neural Networks Analyzed via Convolutional Sparse Coding},
  author={Vardan, Papyan and Romano, Yaniv and Elad, Michael},
  journal={arXiv preprint arXiv:1607.08194},
  year={2016}
}

@article{tieleman2015rmsprop,
    author = {Tieleman, T. and Hinton, G.},
    title = {{RMSprop Gradient Optimization}},
    url = {http://www.cs.toronto.edu/\~{}tijmen/csc321/slides/lecture\_slides\_lec6.pdf}
}

@book{scott2015multivariate,
  title={Multivariate density estimation: theory, practice, and visualization},
  author={Scott, David W},
  year={2015},
  publisher={John Wiley \& Sons}
}

@article{tropp2012user,
  title={User-friendly tail bounds for sums of random matrices},
  author={Tropp, Joel A},
  journal={Foundations of computational mathematics},
  volume={12},
  number={4},
  pages={389--434},
  year={2012},
  publisher={Springer}
}

@inproceedings{mcallester1999pac,
  title={PAC-Bayesian model averaging},
  author={McAllester, David A},
  booktitle={Proceedings of the twelfth annual conference on Computational learning theory},
  pages={164--170},
  year={1999},
  organization={ACM}
}

@incollection{mcallester2003simplified,
  title={Simplified PAC-Bayesian margin bounds},
  author={McAllester, David},
  booktitle={Learning theory and Kernel machines},
  pages={203--215},
  year={2003},
  publisher={Springer}
}

@techreport{langford2001bounds,
  title={Bounds for averaging classifiers},
  author={Langford, John and Seeger, Matthias},
  year={2001},
  institution={Carnegie Mellon, Department of Computer Science}
}

@inproceedings{
nagarajan2018deterministic,
title={Deterministic {PAC}-Bayesian generalization bounds for deep networks via generalizing noise-resilience},
author={Vaishnavh Nagarajan and Zico Kolter},
booktitle={International Conference on Learning Representations},
year={2019},
url={https://openreview.net/forum?id=Hygn2o0qKX},
}

@article{nagarajan2019generalization,
  title={Generalization in deep networks: The role of distance from initialization},
  author={Nagarajan, Vaishnavh and Kolter, J Zico},
  journal={arXiv preprint arXiv:1901.01672},
  year={2019}
}

@article{bartlett1998sample,
  title={The sample complexity of pattern classification with neural networks: the size of the weights is more important than the size of the network},
  author={Bartlett, Peter L},
  journal={IEEE transactions on Information Theory},
  volume={44},
  number={2},
  pages={525--536},
  year={1998},
  publisher={Institute of Electrical and Electronics Engineers}
}

@inproceedings{hinton1993keeping,
  title={Keeping neural networks simple by minimizing the description length of the weights},
  author={Hinton, Geoffrey and Van Camp, Drew},
  booktitle={in Proc. of the 6th Ann. ACM Conf. on Computational Learning Theory},
  year={1993},
  organization={Citeseer}
}

@inproceedings{harvey2017nearly,
  title={Nearly-tight VC-dimension bounds for piecewise linear neural networks},
  author={Harvey, Nick and Liaw, Christopher and Mehrabian, Abbas},
  booktitle={Conference on Learning Theory},
  pages={1064--1068},
  year={2017}
}

@article{neyshabur2017pac,
  title={A pac-bayesian approach to spectrally-normalized margin bounds for neural networks},
  author={Neyshabur, Behnam and Bhojanapalli, Srinadh and McAllester, David and Srebro, Nathan},
  journal={arXiv preprint arXiv:1707.09564},
  year={2017}
}

\end{document}